\documentclass[a4paper,11pt]{article}

\usepackage{anyfontsize}
\usepackage{natbib}
\usepackage[utf8]{inputenc}
\usepackage[english]{babel}
\usepackage{amssymb,amsmath,amsthm}
\usepackage{bbm}
\usepackage{dsfont}
\usepackage{xcolor}
\usepackage[margin=1in]{geometry}
\usepackage{verbatim}
\usepackage{subcaption}
\usepackage{algorithm}
\usepackage[noend]{algpseudocode}
\usepackage{soul}
\newcommand{\colred}[1]{{\color{red}#1}}

\usepackage{fonttable}
\usepackage{xr}
\usepackage{bm}
\usepackage[colorlinks]{hyperref}
\usepackage[nameinlink,capitalise]{cleveref}

\theoremstyle{definition}
\newtheorem{theorem}{Theorem}[section]

\theoremstyle{definition}
\newtheorem{assume}{Assumption}

\theoremstyle{definition}
\newtheorem{corollary}{Corollary}[section]

\theoremstyle{definition}
\newtheorem{lemma}{Lemma}[section]

\theoremstyle{definition}
\newtheorem{proposition}{Proposition}[section]
\newcommand\independent{\protect\mathpalette{\protect\independenT}{\perp}}
\def\independenT#1#2{\mathrel{\rlap{$#1#2$}\mkern2mu{#1#2}}}
\usepackage{mathtools}

\DeclareMathOperator*{\argmaxA}{arg\,max}
\newcommand{\bzero}{\textbf{0}}

\newcommand{\ba}{\textbf{a}}
\newcommand{\bb}{\textbf{b}}
\newcommand{\bB}{\textbf{B}}
\newcommand{\bc}{\textbf{c}}
\newcommand{\bC}{\textbf{C}}
\newcommand{\bA}{\textbf{A}}

\newcommand{\bD}{\textbf{D}}

\newcommand{\bs}{\textbf{s}}

\newcommand{\bO}{\textbf{O}}
\newcommand{\bM}{\textbf{M}}
\newcommand{\bP}{\textbf{P}}
\newcommand{\bQ}{\textbf{Q}}
\newcommand{\bI}{\textbf{I}}

\newcommand{\bu}{\textbf{u}}
\newcommand{\bv}{\textbf{v}}
\newcommand{\bw}{\textbf{w}}

\newcommand{\bx}{\textbf{x}}
\newcommand{\bX}{\textbf{X}}
\newcommand{\by}{\textbf{y}}
\newcommand{\bom}{\textbf{m}}
\newcommand{\bY}{\textbf{Y}}

\newcommand{\bZ}{\textbf{Z}}
\newcommand{\bL}{\textbf{L}}
\newcommand{\br}{\textbf{r}}

\newcommand{\given}{\,|\,}

\newcommand{\eps}{\epsilon}
\newcommand{\sigs}{\sigma^2}
\newcommand{\taus}{\tau^2}
\newcommand{\iid}{\stackrel{\mathrm{i.i.d.}}{\sim}}

\newcommand{\blue}[1]{{\leavevmode\color{black}{#1}}}

\newcommand{\calB}{{\cal B}}
\newcommand{\calW}{{\cal W}}

\newcommand{\calC}{{\cal C}}
\newcommand{\calN}{{\cal N}}
\newcommand{\calH}{{\cal H}}
\newcommand{\calF}{{\cal F}}

\newcommand{\calK}{{\cal K}}

\newcommand{\calD}{{\cal D}}
\newcommand{\calG}{{\cal G}}

\newcommand{\calP}{{\cal P}}

\newcommand{\calL}{{\cal L}}

\newcommand{\bbeta}{ \mbox{\boldmath $ \beta $} }

\newcommand{\bphi}{ \mbox{\boldmath $\phi$}}

\newcommand{\beps}{ \mbox{\boldmath $\epsilon$}}

\newcommand{\btheta}{ \mbox{\boldmath $ \theta $} }

\newcommand{\bSig}{ \mbox{$\bm{\Sigma}$} }

\newcommand{\bSigma}{ \mbox{\boldmath $\Sigma$} }

\newcommand{\bTheta}{ \mbox{\boldmath $\Theta$} }

\newcommand{\bgamma}{ \mbox{\boldmath $\gamma$} }
\newcommand{\brho}{ \mbox{\boldmath $\rho$} }

\newcommand{\boeta}{\mbox{\boldmath $\eta$} }

\newcommand{\calI}{\mathcal I}

\allowdisplaybreaks
\usepackage[shortlabels]{enumitem}

\usepackage{setspace}
\usepackage{chngcntr}
\usepackage{apptools}
\AtAppendix{\counterwithin{lemma}{section}}

\newcommand{\blind}{1}
\begin{document}
	\def\spacingset#1{\renewcommand{\baselinestretch}%
		{#1}\small\normalsize} \spacingset{1}

\if1\blind
{
	\title{\bf Random forests for spatially dependent data}
 \author{Arkajyoti Saha$^1$\footnote{AS and AD were supported by NSF award DMS-1915803 and the Johns Hopkins Bloomberg American Health Initiative Spark Award.}, Sumanta Basu$^2$\footnote{SB was supported by NSF award DMS-1812128, and NIH awards R01GM135926 and R21NS120227.} , Abhirup Datta$^{1*}$\footnote{Corresponding email: abhidatta@jhu.edu }}
 
 \date{
	$^1$ Department of Biostatistics, Johns Hopkins University\\
	$^2$ Department of Statistics and Data Science, Cornell University\\[2ex]
}
	\maketitle
} \fi

\if0\blind
{
	\bigskip
	\bigskip
	\bigskip
	\title{\bf Random Forests for spatially dependent Data}
	\date{}
		\maketitle
} \fi



\maketitle

\setlength{\abovedisplayskip}{0pt}
\setlength{\belowdisplayskip}{0pt}

\setcounter{page}{1}

\begin{abstract} 
Spatial linear mixed-models, consisting of a linear covariate effect and a Gaussian Process (GP) distributed spatial random effect, are widely used for analyses of geospatial data.  
We consider the setting where the covariate effect is non-linear. Random forests (RF) are popular for estimating non-linear functions but 
applications of RF for spatial data have often ignored 
the spatial correlation. We show that this impacts the performance of RF adversely. 
We propose RF-GLS, a novel and well-principled extension of RF, for estimating non-linear covariate effects  in spatial mixed models where the spatial correlation is modeled using GP. RF-GLS extends RF in the same way generalized least squares (GLS) fundamentally extends ordinary least squares (OLS) to accommodate for dependence in linear models. 
RF becomes a special case of RF-GLS, and is substantially outperformed by RF-GLS for both estimation and prediction across extensive numerical experiments with  spatially correlated data. RF-GLS can be used for functional estimation in other types of dependent data like time series. 
We prove consistency of RF-GLS for $\beta$-mixing dependent error processes that includes the popular spatial 
Mat\'ern GP. As a byproduct, we also establish, to our knowledge, the first consistency result for RF 
under dependence. We establish results of independent importance, including a general consistency result of GLS optimizers of data-driven function classes, and a uniform laws of large numbers 
under $\beta$-mixing dependence with weaker assumptions. These new tools can be potentially useful for asymptotic analysis of other GLS-style estimators in nonparametric regression with dependent data. 
\end{abstract}

Keywords: Spatial, Gaussian Processes, Random forests, generalized least-squares.

\spacingset{1.42} 

\section{Introduction}
Geo-referenced data, exhibiting spatial correlation, are commonly analyzed in a mixed-model framework consisting of a fixed-effect component for the covariates and a spatial random-effect \citep{banerjee2014hierarchical}. If $Y, X$ and $\ell$ respectively denote the response, the $D$-dimensional vector of covariates, and the spatial location, then the spatial linear mixed model can be expressed as $Y = X^\top\bbeta + w(\ell) + \mbox{error}$. The linear regression term $\bx^\top\bbeta$ parsimoniously models the fixed covariate effect $m(\bx) = E(Y \given X=\bx)$, while the spatial random effect $w(\ell)$ \textcolor{black}{is often modeled flexibly using Gaussian Processes (GP) that encode the dependence in the data across locations. The mixed model allows for inference on the covariate effect as the estimate of $\beta$ while adjusting for the spatial dependence via $w$. At the same time, it leverages the conveniences of a GP to seamlessly predict the outcome at a new location via kriging. This flexibility has made the spatial linear mixed model with GP-distributed random effects the flag-bearer in geo-statistics.}

This manuscript focuses on applications where the linearity assumption for the covariate effect $m(\bx)$ is inappropriate. We consider the spatial non-linear mixed effect model $Y = m(X) + w(\ell) + \mbox{error}$. A non-linear $m(\bx)$ 
can be modeled in terms of splines or other basis function expansions, while still modeling the spatial effect using GP. However, 
 smooth and continuous basis functions not ideal for modeling non-smooth and possibly discontinuous (like piecewise constant) covariate effects. Also, 
  basis functions on multi-dimensional covariate domains experience curse of dimensionality even for only $3$ or $4$ covariates as the knots are usually too far apart to adequately represent the covariate space 
 \citep{taylor2013challenging}. 

Random forests (RF)  \citep{breiman2001random} estimates non-linear regression functions using an ensemble of non-smooth, data-adaptive family of basis functions, which have more expressive powers than traditional fixed basis methods \citep{lin2006random, scornet2016random}. 
RF has become one of the most popular method for flexible non-linear function estimation, with wide applications in different scientific and engineering fields. While RF has also been used for a number of geo-spatial applications \citep{ahijevych2016probabilistic,di2019ensemble,lim2019mapping,viscarra2014mapping,fayad2016regional},  
most of the aforementioned work 
do not make procedural considerations for the RF algorithm to address the spatial correlation in the data. This is fundamentally at odds with the tenets of spatial modeling where the spatial correlation is explicitly accommodated using GP-distributed random effects. 

Very little attention has been paid to how ignoring this spatial correlation affects function estimation using RF. 
The criterion (loss function) used to recursively split the nodes of decision trees of RF 
is essentially an ordinary least squares (OLS) loss. 
It is well known that OLS is sub-optimal for dependent data as it
ignores data correlation. 
Also, RF creates and consolidates the trees using ``bagging" (bootstrap aggregation) \citep{breiman1996bagging}. As spatial data are correlated, 
this resampling violates the assumption of independent and identically distributed (i.i.d.) data units fundamental to bootstrapping. We provide empirical evidence that these limitations manifest in inferior estimation/prediction performance of RF under spatial dependence. 
\blue{There have been two ways of accounting for spatial dependence in RF. The first approach performs spatial adjustment through kriging \emph{after} fitting an RF without accounting for the spatial dependence. For prediction at new locations, it uses the spatial mixed model framework to combine this vanilla RF estimate of the covariate effect $m(X)$ with the kriging from the GP part (using the residuals from the RF fit) \citep{fayad2016regional,fox2020comparing}. This  is referred to as Random Forest residual kriging (RF-RK).
The other approaches attempt to explicitly use spatial information in RF.}
The spatial RF (RFsp) of \cite{hengl2018random} adds all pairwise spatial-distances as additional covariates. \cite{georganos2019geographical} proposes geographically local estimation of RF. These approaches abandon the spatial mixed model framework, \blue{only focusing on prediction, modeling the response as a joint function $g$ of the covariates and the locations, i.e., $E(Y) = g(X,\ell)$, and are thus not able to isolate or estimate the covariate effect $m(\bx)$.}  
Also, to our knowledge, there is no  asymptotic theory justifying these approaches.

In this article, we bridge the gap between spatial mixed effect modeling using GP and regression function estimation using RF by proposing a well-principled rendition of RF to explicitly incorporate the spatial correlation structure implied by GP. {\em This enables using RF for estimating a non-linear covariate effect in the spatial mixed effect model while explicitly accounting for spatial correlation via the GP-distributed random effects.} 
Our approach is motivated by the fact that in a GP-based spatial linear mixed model, marginal likelihood maximization for $\beta$ is equivalent to a generalized least squares (GLS) optimization. Under data dependence, GLS is more efficient than OLS for linear regression 
leading to better finite sample performance. 

GLS has been previously used for parameter estimation in non-linear regression under data dependence \citep{kimeldorf1971some,diggle1989spline}.
Since regression with polynomials, splines or other \emph{known} basis functions are essentially linear in the parameters, the extension of these approaches to a GLS-style quadratic form global optimization is immediate and the solution is available in closed form. However, RF uses \emph{basis functions obtained by a data-adaptive, greedy algorithm}, so a GLS formulation for estimating the regression part using RF in a spatial GP regression is not immediate and has not been explored before.

 
An observation, central to our GLS extension of regression trees and forests is that to split a node of decision tree in RF, 
the local (intra-node) loss 
can be equivalently represented as a global (in all nodes) OLS linear regression problem with a binary design matrix. The subsequent node representative assignment is simply the OLS fit given the chosen split. 
For dependent data, we can replace this OLS step with a GLS optimization problem at every node split and grow the tree accordingly. The node representatives also become the GLS fits instead of node means. The global GLS loss-based splitting and fits thus use information from all current nodes as is desirable under data-dependence and are more efficient than the OLS analogs (Figure \ref{fig:gls}). 

Our GLS-style RF also naturally accommodates data resampling or sub-sampling used for creating a forest of trees. The key observation for this is that in a linear regression between response $\bY$ and covariates $\bX$, the GLS loss with covariance matrix $\bSigma_0$ 
coincides with an OLS loss with $\tilde{\bY} = \bSigma_0^{-1/2} \bY$, $\tilde{\bX} = \bSigma_0^{-1/2} \bX$. 
Thus GLS can be thought of as OLS with the decorrelated responses $\tilde \bY$. 
Analogously, we introduce resampling after the decorrelation in our algorithm, 
essentially resampling the uncorrelated contrasts $\tilde \bY$ instead of the correlated outcomes $\bY$. 

We refer to our method as RF-GLS and present a computationally efficient algorithm for implementing it. RF-GLS reduces exactly to RF when the working correlation matrix used in the GLS-loss is the identity matrix. For spatial mixed model regression, RF-GLS estimates the non-linear covariate effect while using the GP to model the spatial random effect. Hence, traditional spatial tasks like kriging (spatial predictions) can be performed easily. For large spatial data, RF-GLS also avoids onerous big GP computations by harmonizing with Nearest Neighbor Gaussian Processes \citep{nngp} to yield an algorithm of linear time-complexity. 

\blue{We show that RF-GLS has distinct advantages over competing methods for both estimation and prediction. For estimation, RF-GLS, accounting explicitly for the spatial correlation via GLS loss, provides improved estimates of the covariate effect $m(X)$ in a wide range of simulation scenarios over RF which completely ignores the spatial information. Other methods like RFsp do not even produce a separate estimate of the covariate effect and cannot be considered for the task of estimation. 
For prediction, RF-GLS outperforms both RF-RK and RFsp. RF-GLS conveniently uses the spatial mixed model framework where the structured spatial dependence is parsimoniously encoded via GP. 
Methods like RFsp abandon this mixed-model framework and uses a large number of additional distance-based covariates in RF. A downside to this unnecessary escalation of the problem to high-dimensional settings is that the several additional distance-based covariates  far outnumber the true set of covariates thereby biasing the covariate selection at each node-splitting towards the spatial covariates. This 
leads to poor prediction performance of RFsp when the spatial noise is small relative to the covariate signal. 
On the other hand, when the spatial noise is large, RF-RK, estimating the covariate effect $m(X)$ without accounting for the spatial dependence, performs substantially poorly while RFsp (dominated by spatial covariates) performs comparably to RF-GLS. RF-GLS performs best at all ranges of the covariate-signal-to-spatial-noise-ratio (SNR), where each of RF-RK and RFsp suffers at opposite ends of this SNR spectrum (See Figure \ref{Fig:friedman_MSE_ratio}).} 


\subsection{Theoretical contributions}
Another major contribution of this article is a thorough theoretical analysis of RF-GLS under dependent error processes like GP. For i.i.d. data, \cite{scornet2015consistency} proved the consistency of Breiman's RF, which allows the node splitting and node representation to be based on the same data. The other  strand of RF theory consider ``honest'' trees \citep{mentch2016quantifying,wager2018estimation} which use disjoint data subsets for  splitting and representation. To our knowledge, study of RF under dependent processes has not been conducted in either paradigms. 
RF-GLS, like Breiman's RF, uses the same data for partitioning and node representation.  Hence we adopt the framework of \cite{scornet2015consistency} to study  consistency. 


Our main result (Theorem \ref{th:main}) proves RF-GLS is $\mathbb L_2$ consistent if the error process is absolutely-regular ($\beta$-) mixing \citep{bradley2005basic}. 
As corollary, we prove RF-GLS is consistent for non-linear mean estimation in the spatial mixed model regression under the ubiquitous Mat\'ern family of covariance functions.  
As a byproduct of the theory, we also establish consistency of the vanilla CART \citep{breiman1984classification} and of RF under $\beta$-mixing error processes. To our knowledge, this is the first result on consistency of these procedures under dependence.  

\blue{The theoretical analysis, besides being novel in terms of establishing consistency for forest estimators under dependence, required addressing several new challenges that do not arise for the analysis of classic RF for iid settings. These complications were introduced by the global GLS-style quadratic loss and the dependent errors.
    We summarize the new contributions here:
    
 \noindent   \begin{enumerate}[(i), wide=0pt]
        \item \textbf{\textit{Limits of GLS regression-trees:}} 
        Node splitting and node representation in regression trees of classic RF uses local (intra-node) sample means and variances which trivially asymptote to their population analogs. For RF-GLS, to adjust for data correlation, we use a global GLS (quadratic form) loss for node-splitting and the set of node representatives is the global GLS estimate. 
        Both steps now depend on data from all nodes (weighted by their spatial dependence) and their asymptotic limits are non-trivial. 
        We meticulously track these data-weights from all nodes when splitting a given node to show in Lemma \ref{lem:beta} and Theorem \ref{lemma:DART-theoretical} that the RF-GLS split-criterion and node representatives converge to the same desirable local population limits as in RF while being empirically more efficient under dependence (Figure \ref{fig:gls}). 
        These findings may seem similar to the well-known fact that GLS and OLS estimates both converge to the same limits with the GLS being more efficient under dependence. 
        However, those classic results assume a true linear model 
        whereas for us the true function is non-linear $\mathbb{E}(y)=m(\bX)$, and is being estimated by a mis-specified tree estimate. 
        To our knowledge, limits of the global GLS tree  estimates 
        as population-level local node means for that tree is a new contribution.
        
        \item \textbf{\textit{Equicontinuity of split-criterion:}} A center-piece in the proof of consistency of RF in \cite{scornet2015consistency} was stochastic equicontinuity of the CART split criterion, such that if two set of splits 
        are close, their corresponding empirical split-criterion values are close, irrespective of the location of the splits. 
        For RF-GLS, the split criterion (GLS loss with plugged in GLS estimate) is
        is a complicated matrix-based function of the design matrix representing a set of splits. Hence, the scalar techniques of \cite{scornet2015consistency} do not apply here. Our equicontinuity proof is quite involved and 
        is built on viewing GLS predictions as oblique projections on the design matrices corresponding to the splits and subsequently using results on perturbations of projection operators. 
        More details about the technique used can be found in Section \ref{sec:equi}.
        
        \item \textbf{\textit{New uniform laws of large number (ULLNs):}} Controlling the quadratic form estimation error for RF-GLS under dependence by an ULLN is key to establishing consistency. The standard technique of first proving an ULLN that replaces the dependent errors by iid ones fails here as the iid error sequence does not preserve the distribution of the cross-product terms of the quadratic form. We use a novel strategy of creating separate bivariate iid error sequences for each of the off-diagonal bands of cross-product terms in the quadratic form. We then establish a new ULLN (Proposition \ref{lemma:cross_prod_indep}) for these cross-product terms having the same concentration bound (in terms of random $\mathbb L_p$ norm entropy numbers) as the squared (diagonal) terms. 
   		Next, to generalize the ULLNs from iid to dependent settings, existing results require Lipschitz continuity or an uniform bound on the estimators none of which are satisfied by regression-trees. We established a general ULLN (Proposition \ref{lemma:dependent_data}) for $\beta$-mixing processes that uses a weaker $(2+\delta)^{th}$ moment assumption 
        that is satisfied by regression-trees. Neither of the two ULLNs were needed in the asymptotic study of RF for iid data as there was no quadratic form loss or cross-product terms, and the error process was independent thereby not requiring any dependent ULLNs.  
        \item \textbf{\textit{General tools for machine learning for dependent data:}} Using these ULLNs described above we established a general consistency result (Theorem \ref{th:gyorfi}) for GLS estimates for a broad class of functions under dependent error. This result generalizes Theorem 10.2 of \cite{gyorfi2006distribution} (which was for iid data and OLS estimates) to $\beta$-mixing dependent processes and GLS losses. We believe this general result would be of widespread interest. Just like the iid result of \cite{gyorfi2006distribution} was used to establish consistency of a wide range of OLS estimators (including piecewise polynomials, univariate and multivariate splines, data-driven partitioning based estimators like random forests and trees), our Theorem \ref{th:gyorfi} can potentially be used to establish consistency of the analogs of each of these methods that
        explicitly accounts for dependence (spatial/serial correlation) in the data by switching to a  GLS procedure. 
        
        \item \textbf{\textit{Consistency for Mat\'ern Gaussian Processes:}} 
        We prove consistency of RF-GLS when the true dependent error process comes from the Mat\'ern Gaussian family (the staple choice for geospatial analysis due to interpretability of the parameters in terms of spatial surface smoothness). The proof combines several results spread over the fields of time-series, 
        stochastic differential equations, 
         information theory, 
          and spatial statistics. 
        To our knowledge, this is the first consistency result for an RF-type algorithm under Mat\'ern spatial correlation. 
 
    \end{enumerate}}

While RF-GLS is motivated by the spatial GP-based mixed model framework, the method and the theory is developed much more broadly and can be used for regression function estimation in many other dependent data settings. The method only relies on knowledge of the residual covariance matrix and thus can work with any dependent error process with a valid second moment. The theory is applicable for the large class of $\beta$-mixing processes. We briefly discuss how RF-GLS can be used for function estimation under serially correlated errors (time series). In particular, we show that for autoregressive errors, one of the mainstays of time-series analysis, RF-GLS  would yield a consistent estimate of the regression function. 

The rest of the paper is organized as follows. In Section \ref{sec:method} we present our methodology. 
The theory of consistency is presented in Section \ref{sec:theory}, including theory for common spatial and time-series processes in 
Section \ref{sec:examples}. 
Results from a variety of simulation experiments validating our method is presented in Section \ref{sec:sim}. 
While the formal proofs of all the results 
are provided in the Supplementary Materials, an  outline of the proof of the main consistency result is presented in Section \ref{sec:outline} that highlights the new technical tools developed in the process.  We discuss future extension of the presented work in Section \ref{sec:disc}. 

\subsection*{Notations}

 $|S|$ denotes the cardinality of any set $S$.  $\mathds{I}(\cdot)$ is the indicator function. The null set is $\{ \}$. For any matrix $\bM$, $\mathbf{{M}}^{+}$ denotes its generalized Moore-Penrose inverse, and $\|\bM\|_p$, for $1 \leq p \leq \infty$, denotes its matrix $\mathbb L_p$ norm. 
For any $n \times n$ symmetric matrix $\bM$, $\lambda_{\min}=\lambda_{1} \leq \lambda_{2} \leq \ldots \leq \lambda_{n}=\lambda_{\max}$ denotes its eigenvalues. A sequence of numbers $\{a_n\}_{n \ge 1}$ is  $O(b_n)$ (or $o(b_n)$) when the sequence $|a_n/b_n|$ is bounded above (or goes to $0$) as $n \rightarrow \infty$. A sequence of random variables is called $O_b(1)$ if it is uniformly bounded almost surely, $O_p(1)$ if it is bounded in probability, and $o_p(1)$ if it goes to 0 in probability. $X \sim Y$ implies $X$ follows the same distribution as $Y$. $\mathbb R$, $\mathbb Z$, and $\mathbb N$ 
denote the set of real numbers, integers, and natural numbers 
respectively. For $M \in \mathbb{R}^+, \: T_{M}$ is the truncation operator, i.e.
$T_{M}(u) = \max(-M,\min(u,M))$.

\section{Method}\label{sec:method}
\subsection{Review of spatial Gaussian Process regression}\label{sec:gpintro}
The standard geospatial data unit is the triplet $(Y_i, X_i, \ell_i)$ where $Y_i$ is the univariate response, $X_i$ is a $D$-dimensional covariate (feature) vector, and $\ell_i$ is the spatial location (often, geographical co-ordinates). A spatial linear mixed-model for such data is specified as $Y_i = X_i^\top\bbeta + w(\ell_i) + \eps^*_i$, where $X_i^\top\bbeta$ is the linear mean (covariate effect), $w(\ell_i)$ models the spatial structure in the response not accounted for by the covariates, and $\eps^*_i \iid N(0,\taus)$ is the random noise. The spatial effect $w(\ell_i)$ is often posit to be a smooth surface across space and is  modeled as a centered Gaussian Process (GP) $w(\cdot) \sim GP(0,C(\cdot,\cdot \given \theta))$ with covariance function $C$. This essentially means that for any finite collection of locations $\ell_1,\ldots, \ell_n$, we have $\bw=(w(\ell_1),\ldots,w(\ell_n))^\top \sim N(\bzero,\bC)$ where $\bC$ is the $n \times n$ matrix with entries $Cov(w(\ell_i),w(\ell_j))=C(\ell_i,\ell_j \given \btheta)$. 

We can marginalize over the spatial random effects $w(\ell_i)$ and write $\bY=(Y_1, \ldots, Y_n)^\top \sim N(\bX\bbeta, \bSigma_0)$ where $\bSigma_0 = \bC + \taus \bI$. 
For known $\bSigma_0$, maximizing this marginalized likelihood of $\bY$ is equivalent to minimizing the quadratic form
$
\min_{\bbeta} \frac{1}{n} (\bY - \bX \bbeta)^\top \bSigma_0^{-1} (\bY - \bX \bbeta). 
$ This expression can be recognized as the generalized least square (GLS) loss which  replaces the squared error loss $(1/n)\sum_{i=1}^n (Y_i - X_i^\top \bbeta)^2$ of OLS when the data are correlated. 

We focus on estimating a general mean function in the spatial non-linear mixed model

\begin{equation}\label{eq:spnlmm}
    Y_i = m(X_i) + w(\ell_i) + \eps^*_i;\, w(\cdot) \sim GP(0,C(\cdot,\cdot \given \theta))
\end{equation}
When $m(\bx)$ is modeled in terms of a fixed basis expansion, the marginalized model becomes $\bY \sim N(B(\bX)\bgamma, \bSigma_0)$ where $B(\bX)$ are the bases and $\bgamma$ are the coefficients. Hence, the model remains linear in the unknown coefficients $\bgamma$ which can be once again estimated using GLS. 

As discussed in the introduction, the scope of fixed basis function expansions are limited due to their inability to model regression functions with discontinuities, and curse of dimensionality in multi-dimensional covariate domains. Random forests (RF) are particularly suitable for such general regression function estimation due to their use of regression trees, and natural accommodation of higher covariate-dimensionality using random selection. However, unlike fixed basis expansion models which are linear in the parameters, RF estimation is a non-linear greedy algorithm for which a GLS extension is not straightforward. In the next sections, we will develop a GLS-style RF algorithm for estimating $m(X)$ under dependence. 

\subsection{Revisiting the RF algorithm}\label{sec:revisit}
We first review the original RF algorithm. Given data $(Y_i,X_i) \in \mathbb R \times \mathbb R^D$, $i=1,\ldots,n$, the RF estimate of the mean function $m(\bx)=\mathbb E(Y\given X=\bx)$ is the average of $n_{tree}$ regression tree estimates of $m$. 
In a regression tree, data are split recursively into nodes of a tree starting from a root node.   
To split a node, a set of $M_{try} (\ll D)$ features are chosen randomly and the best split point is determined with respect to each feature $d$ by searching over all the ``gaps" in that feature of the data as cutoff candidate $c:=c(d)$. Here ``best'' is determined as the feature-cutoff combination $(d,c)$ maximizing the CART (classification and regression trees) split criterion \citep{breiman1984classification}: 
\begin{equation}\label{eq:cart_local}
v_{n}^{CART}((d,c)) = \frac{1}{n_P}\left[\sum_{i=1}^{n_P}(Y_i^P - \bar{Y}^P)^2 - \sum_{i_r=1}^{n_R}(Y_{i}^R - \bar{Y}^R)^2 - \sum_{i_l=1}^{n_L}(Y_{i}^L - \bar{Y}^L)^2\right]
\end{equation}
where, $Y_i^P, Y_i^R$ and $Y_i^L$ denote the responses in parent node, right child and left child respectively. $\bar{Y}^P, \bar{Y}^R$ and $\bar{Y}^L$ denote the respective node means, $n_P = n_R + n_L$, $n_R$ and $n_L$ are the respective node cardinalities. The feature and cut-off value combination that minimizes the CART-split criterion $(\ref{eq:cart_local})$ is chosen to create the child nodes. As each split is only based on one feature, all nodes are hyper-rectangles. Each newly created node is assigned a node representative -- the mean of the responses of the node members. 

The nodes are iteratively partitioned this way till a pre-specified stopping rule is met -- we arrive at leaf nodes having single element (fully grown tree) or the number of data points in each leaf node is less than a pre-specified number or the total number of nodes reaches a pre-specified threshold. The regression tree estimate for an input feature $\bx \in \mathbb R^D$ is given by the representative value of the leaf node containing $\bx$. RF is the average of an ensemble of regression trees with each tree only using a resample or subsample of the data. 

\subsection{Dependency adjusted  node-splitting}\label{sec:dart}

The RF algorithm of Section \ref{sec:revisit} does not utilize any information on the locations $\ell_i$ or the ensuing spatial correlation. The fundamental unit of the algorithm is splitting of a given node, and it is inherently local  in nature (in the covariate domain). The split criterion (\ref{eq:cart_local}) and subsequent assignment of the node representatives are both based only on members within the parent node. For i.i.d. data, this local approach is reasonable as members of one node are independent of the others. However, geo-referenced data units data can be distant in the covariate-domain while being close in the spatial domain, and therefore strongly correlated. 


The spatial non-linear mixed model from Section \ref{sec:gpintro} can be written more generally as $Y_i = m(X_i) + \eps_i$ where $\eps_i = w(\ell_i) + \eps^*_i$ is not an i.i.d process but a stochastic (Gaussian) Process capturing the spatial dependence. Hence, $Cov(Y_i,Y_j) = Cov(\eps_i,\eps_j) \neq 0$ for most or all $(i,j)$ pairs implying that members of the other nodes can be highly correlated with those of a node-to-be-split and operating locally (intra-node) leaves out this information. We now propose a new global split-criterion using all the correlations in the data, as is desirable and  in line with the common practice in spatial analysis. 

To explore generalizations of RF for dependent settings, we first recall an equivalent global representation of the CART-split criterion (\ref{eq:cart_local}). Consider a regression tree grown up to the set of leaf nodes $\{\mathcal{C}_1, \mathcal{C}_2, \cdots, \mathcal{C}_K\}$ forming a partition of the feature space. The node representatives $\widehat\bbeta^{(0)}=(\hat\beta_1,\ldots,\hat\beta_{K})^\top$ are the corresponding means.  
To split the next node, say $\calC_{K}$ without loss of generality, the CART-split criterion (\ref{eq:cart_local})  
determines the best feature-cutoff combination $(d^*,c^*)$ and creates child nodes $\calC^{(L)}_{K}=\calC_{K} \cap \{\bx \in \mathbb R^D | x_{d^*} < c^*\}$ and $\calC^{(R)}_{K}=\calC_{K} \cap \{\bx \in \mathbb R^D | x_{d^*} \geq c^*\}$ and the node representatives $\hat\beta_K^{(L)}$ and $\hat\beta_K^{(R)}$ (simply the means of the left and right child nodes respectively). We can write the optimal split-direction $(d^*,c^*)$ and node representatives $\hat\beta_K^{(L)},\hat\beta_K^{(R)}$ as the maximizer 
\begin{equation}\label{eq:cart_local2}
\begin{array}{c}
(d^*,c^*,\hat\beta_K^{(L)},\hat\beta_K^{(R)}) = \underset{{d,c,(\beta^{(L)},\beta^{(R)}) \in {\mathbb{R}^2}}}{\arg \max} \frac{1}{|\calC_K|}\Bigg[\sum_{X_i \in \calC_K}(y_i - \hat\beta_K)^2 -\\
 \Big(\sum_{X_i \in \calC_K, X_{id} < c}\left(y_{i} - \beta^{(L)}\right)^2 + \sum_{X_i \in \calC_K, X_{id} \geq c}\left(y_{i} - \beta^{(R)}\right)^2\Big)\Bigg]
\end{array}
\end{equation}

After the split, the new set of nodes is $\{\calC_1, \mathcal{C}_2, \cdots, \mathcal{C}_{K-1}, \mathcal{C}_K^{(L)},\mathcal{C}_K^{(R)}\}$ and the set of representatives is updated to $\widehat\bbeta=(\hat\beta_1,\ldots,\hat\beta_{K-1},\hat\beta_K^{(L)},\hat\beta_K^{(R)})^\top$. Let $\bZ_0 =(\mathds{I}(X_{i}) \in \calC_j)$ be the $n \times K$ membership matrix before the split, and $\bZ$ denote the $n \times (K+1)$ membership matrix for a split of $\calC_K$ using $(d,c)$, i.e. $\bZ_{ij} = \bZ^{(0)}_{ij}$ for $j < k$, $\bZ_{iK}=\mathds{I}(\{X_i \in \calC_K\} \cap \{X_{id} < c\})$ and $\bZ_{i(K+1)}=\mathds{I}(\{X_i \in \calC_K\} \cap \{X_{id} \geq c\})$. We note but suppress the dependence of $\bZ$ on $(d,c)$. Then the optimization in (\ref{eq:cart_local2}) can be rewritten in the following way. 
\begin{align}\label{eq:cart_global}
(d^*,c^*,\hat\bbeta) = \underset{d,c,\bbeta \in \mathbb{R}^{K+1}}{\arg \max} \frac{1}{n} \left( \|\bY - \bZ^{(0)}\widehat\bbeta^{(0)}\|_2^2 - \|\bY - \bZ\bbeta\|_2^2 \right).
\end{align}
This connection of OLS regression with RF is well-known \citep[see, e.g.,][]{friedman2008predictive} where RF rules are used in downstream Lasso fit). To see why (\ref{eq:cart_local2}) and (\ref{eq:cart_global}) are equivalent, note that, for a given $(d,c)$ the $\hat{\bbeta}$ optimizing (\ref{eq:cart_global}) is simply $\widehat\bbeta_{OLS}=(\bZ^\top\bZ)^{-1}\bZ^\top\bY$, the minimizer of the OLS loss $\|\bY - \bZ\bbeta\|_2^2$. Since the membership matrices have orthogonal columns consisting of only $1$s and $0$s, the components of $\widehat \bbeta_{OLS}$ are simply the node means. Hence, the last two components of $\bbeta$ are $\hat\beta_K^{(L)}$ and $\beta_K^{(R)}$. Since the first $(K-1)$ columns of $\bZ$ are same as that of $\bZ^{(0)}$,  this implies that the first $(K-1)$ components of $\hat\bbeta$ are the same as that of $\hat\bbeta^{(0)}$, i.e. the means of the first $(K-1)$ nodes. So although $\hat{\bbeta}$ in  (\ref{eq:cart_global}) is a global optimizer of a linear regression re-estimating all the node representatives, in practice, only the representatives of the child nodes of $\calC_K$ are updated and this is equivalent to the local optimization (\ref{eq:cart_local2}). 

The global formulation of node-splitting offers an avenue for GLS-style generalization of the split criterion for the spatial GP regression. Let $\bSigma_0=Cov(\bY)=Cov(\beps)$ denote the covariance matrix of the marginalized response, and $\bQ=\bSigma_0^{-1}$. Then \`a la GLS we can simply replace the squared error loss $\|\bY - \bZ\bbeta\|_2^2$ with a quadratic  loss $(\bY - \bZ\bbeta)^\top\bQ(\bY - \bZ\bbeta)$, and   propose a GLS-style split criterion
\begin{equation}
\label{DART_def}
\begin{aligned}
v_{n,\bQ}^{DART}((d,c))  =& 
\frac{1}{n} \Bigg[\left(\mathbf{Y} - \mathbf{Z}^{(0)}\bm{\hat{\beta}}_{GLS}(\mathbf Z^{(0)}) \right)^\top \bQ\left(\mathbf{Y} - \mathbf{Z}^{(0)}\bm{\hat{\beta}}_{GLS}(\mathbf Z^{(0)}) \right)\\ 
-&\left(\mathbf{Y} - \mathbf{Z}\bm{\hat{\beta}}_{GLS}(\mathbf Z) \right)^\top \bQ\left(\mathbf{Y} - \mathbf{Z}\bm{\hat{\beta}}_{GLS}(\mathbf Z) \right) \Bigg].
\end{aligned}
\end{equation}
We refer to (\ref{DART_def}) as the Dependency Adjusted Regression Tree (DART) split criterion.
For a split $(d^*,c^*)$ maximizing (\ref{DART_def}), the new set of node representatives are given by the GLS estimate 
\begin{equation}
\label{DART_update}
\hat{\bm{\beta}}_{GLS}(\bZ) = \hat{\bm{\beta}} = \left({\mathbf{Z}}^\top\bQ\mathbf{Z}\right)^{-1}\left({\mathbf{Z}}^\top\bQ\mathbf{Y}\right).
\end{equation}
Both the loss used to choose the optimal split  and the node representatives now depend on data from all nodes, weighted by the precision (inverse covariance) matrix $\bQ$, akin to the spatial linear mixed-model. 
Even though we preserve the greedy recursive partitioning strategy of RF, for each node split our loss function and node representation are now global in the sense that they consider all the data points, not just the ones inside the node to be split.

To demonstrate why the DART loss function and the GLS node representatives are respectively more appropriate for dependent data than the CART loss and node means used in the original RF, we now present a theoretical result and an empirical example. To split a given parent node $\calB$ into left and right child nodes $\calB^L$ and $\calB^R$ respectively based on the split direction $(d,c)$, it is almost immediate that the CART split criterion (\ref{eq:cart_local}), being the difference between the sample variance of the parent node with the those of the children nodes, is an estimator of its asymptotic limit 
\begin{equation}\label{eq:cart_lim}
    Vol(\mathcal B) \Big[\mathbb{V}(Y | \mathbf{X} \in \mathcal B) - \mathbb{P}(\mathbf{X} \in \mathcal B^R | \mathbf{X} \in \mathcal B)\mathbb{V}(Y | \mathbf{X} \in \mathcal B^R)\\
	- \mathbb{P}(\mathbf{X} \in \mathcal B^L |  \mathbf{X} \in \mathcal B)\mathbb{V}(Y | \mathbf{X} \in \mathcal B^L) \Big],
\end{equation}
i.e, the population difference between the variance in the parent node and the pooled variance in the potential children nodes. Indeed, if we were privileged to infinite amount of data, we should use (\ref{eq:cart_lim}) to split a node as it minimizes the total intra-node variances in the children. Hence, for i.i.d. data it is reasonable to use the CART criterion (\ref{eq:cart_local}) which is the finite sample analogue of (\ref{eq:cart_lim}). Under dependence, however, it is unclear if (\ref{eq:cart_local}) is an efficient estimator of (\ref{eq:cart_lim}). 

Our construction of the DART split criterion  is guided by the same principles of GLS used in linear regression for dependent data. The resulting loss function (\ref{DART_def}) depends on data from not only the node $\calB$ to be split, but on all data, and it is not immediately clear what its population limit is. Optimization of (\ref{DART_def}) for the optimal split $(d^*,c^*)$ relies on the GLS estimate (\ref{DART_update}) which is also a function of all data. Hence, we first provide a result about its asymptotic limit. We use assumptions on the dependence structure of the error process (Assumption \ref{as:mix}) and regularity of the precision matrix (Assumption \ref{as:chol}) discussed in more details later in Section \ref{sec:theory} on consistency of our method. The proofs of these results can be found in Section \ref{sec:pf-approx}. 
\blue{\begin{assume}[Mixing condition]\label{as:mix} $Y_i=m(X_i) + \eps_i$ where the error process $\{\eps_i\}$ is a stationary, {\em absolutely regular} ( $\beta$-mixing) process \citep{bradley2005basic} with finite $(2+\delta)^{th}$ moment for some $\delta > 0$.
\end{assume}\begin{assume}[Regularity of the working precision matrix]\label{as:chol}
	The working precision matrix $\bQ=\bSigma^{-1}$ admits a regular and sparse lower-triangular Cholesky factor $\bSig^{-\frac 12}$ such that 
	\begin{equation*}
	\bSig^{-\frac 12} = \left(\begin{array}{ccccc}
	\bL_{q \times q} & 0 & 0 & \cdots & \cdots\\
	\multicolumn{2}{c}{\brho^\top_{1 \times (q+1)}} & 0 & \cdots & \cdots\\
	0 & \multicolumn{2}{c}{\brho^\top_{1 \times (q+1)}} & 0 & \cdots\\
	\vdots & \multicolumn{3}{c}{\ddots} & \vdots \\
	\cdots & 0 & 0 & \multicolumn{2}{c}{\brho^\top_{1 \times (q+1)}}
	\end{array} \right), 
	\end{equation*}
	where $\bm{\rho}  = (\rho_q, \rho_{q-1},\cdots, \rho_{0} )^\top \in \mathbb{R}^{q+1}$ for some fixed $q \in \mathbb N$, and $\bL$ is a fixed lower-triangular $q \times q$ matrix . 
\end{assume}}
\begin{lemma}[Limit of GLS estimate for fixed partition regression-tree]\label{lem:beta} Under Assumptions \ref{as:mix} and  \ref{as:chol}, for a regression-tree based on a fixed (data-independent) partition $\calC_1, \ldots, \calC_K$, the GLS estimate $\widehat \bbeta = (\hat \beta_1, \ldots, \hat \beta_K)^\top$ from (\ref{DART_update}) has the following limit.
	$$
	\bm{\hat\beta}_{{l}} \overset{a.s.}{\to} \mathbb{E}(Y | X \in \mathcal{C}_{{l}}) \text{ as $n \to \infty$} \mbox{ for } l=1,\ldots,K.
	$$
\end{lemma}

Lemma \ref{lem:beta} shows that the GLS-estimate for the node representatives, although using data from all the nodes, asymptote to within node population means -- the same limit as that of the sample node means used in RF. This is a new result of independent importance showing that the GLS estimate for a fixed-partition regression-tree, is a consistent estimator of the node-specific conditional means. This in turn leads to the following result about the DART loss. 

 \begin{theorem}[Theoretical DART-split criterion]
	\label{lemma:DART-theoretical}
	Under Assumptions \ref{as:mix} and \ref{as:chol}, for a tree built with a fixed (data-independent) set of splits $\calC_1, \ldots, \calC_K$, as $n \to \infty$ the empirical DART-split criterion (\ref{DART_def}) to split a node $\calB=\calC_l$ into left and right child nodes $\calB^L$ and $\calB^R$ respectively, converges almost surely to the following for some constant $\alpha=\alpha(\bQ)$:
	\begin{equation}
	\begin{array}{cc}
	v_{\bQ}^*((d,c))=& \alpha Vol(\mathcal B) \Big[\mathbb{V}(Y | \mathbf{X} \in \mathcal B) - \mathbb{P}(\mathbf{X} \in \mathcal B^R | \mathbf{X} \in \mathcal B)\mathbb{V}(Y | \mathbf{X} \in \mathcal B^R)\\
	&- \mathbb{P}(\mathbf{X} \in \mathcal B^L |  \mathbf{X} \in \mathcal B)\mathbb{V}(Y | \mathbf{X} \in \mathcal B^L) \Big].
	\end{array}                    
	\end{equation}
\end{theorem}

The aforementioned result shows that remarkably the limit of the 
the DART-split criterion converge to the same respective limit (\ref{eq:cart_lim}) of the CART-split criterion up to a constant $\alpha$. This result is intriguing as although the empirical DART split criterion depends on the entire set of splits 
and data from all the nodes, its asymptotic limit is simply the population variance difference (\ref{eq:cart_lim}) between the parent node and the children nodes and does not depend on the other nodes. As discussed before, the quantity (\ref{eq:cart_lim}) would be the ideal one to use for splitting a node if one had knowledge of the population distribution, and its reassuring that the DART criterion asymptotes to it.  

These asymptotic results are in line with the conventional GLS wisdom. GLS estimators are known to have the same asymptotic limit as the OLS estimators but are more efficient under dependence. To demonstrate how these asymptotic results translate to finite sample performance, we conduct a simple experiment using data generated from $Y_i = m(X_i) + \eps_i$, where $\eps_i$ is an GP with exponential covariance function on the regularly spaced one-dimensional lattice, and $m(x)$ is a function of a single covariate supported on $[0,1]$ such  that $m(x)=1$ for $x \leq 0.5$, and $m(x) =1.5$ for $x > 0.5$. If we use a two-node decision tree to estimate $m$, it is obvious that a good loss criterion should be maximized near the cutoff value of $0.5$ where there is a discontinuity in the true regression function. In Figure \ref{fig:gls}(a), we plot the average CART and DART  split criterion as a function of the choice of cutoff, and the point-wise confidence bands. The DART split criterion was scaled by $\alpha$ to have the same asymptotic limit as the CART split criterion. For reference we also plot the true regression function on the secondary $y$-axis. We see that the average curves for both the CART and DART criterion are quite identical, both peaking near the true cutoff of $0.5$. However, the CART criterion using the OLS loss has large uncertainty as reflected by the much wider confidence bands, than the DART criterion. This in turn affects the estimates of the cutoff and the node representatives as reflected in Figure \ref{fig:gls}(b). While both losses lead to similar mean estimates of the cutoff, and node representatives, the variability is substantially higher for the OLS loss. This large variability is especially evident for the choice of the cutoff where the CART loss can choose cutoff far away from $0.5$ with high probability whereas the DART loss chooses a cutoff near $0.5$ almost always. We will show in Section \ref{sec:sim} over a wide range of simulation studies how this leads to poor finite sample estimation and prediction performance of RF for dependent data. 

\begin{figure}[t!]
    \centering
    \begin{subfigure}[t]{0.5\textwidth}
    \label{Fig:cost_known_sinpi}
        \centering
        \includegraphics[height=2.5in]{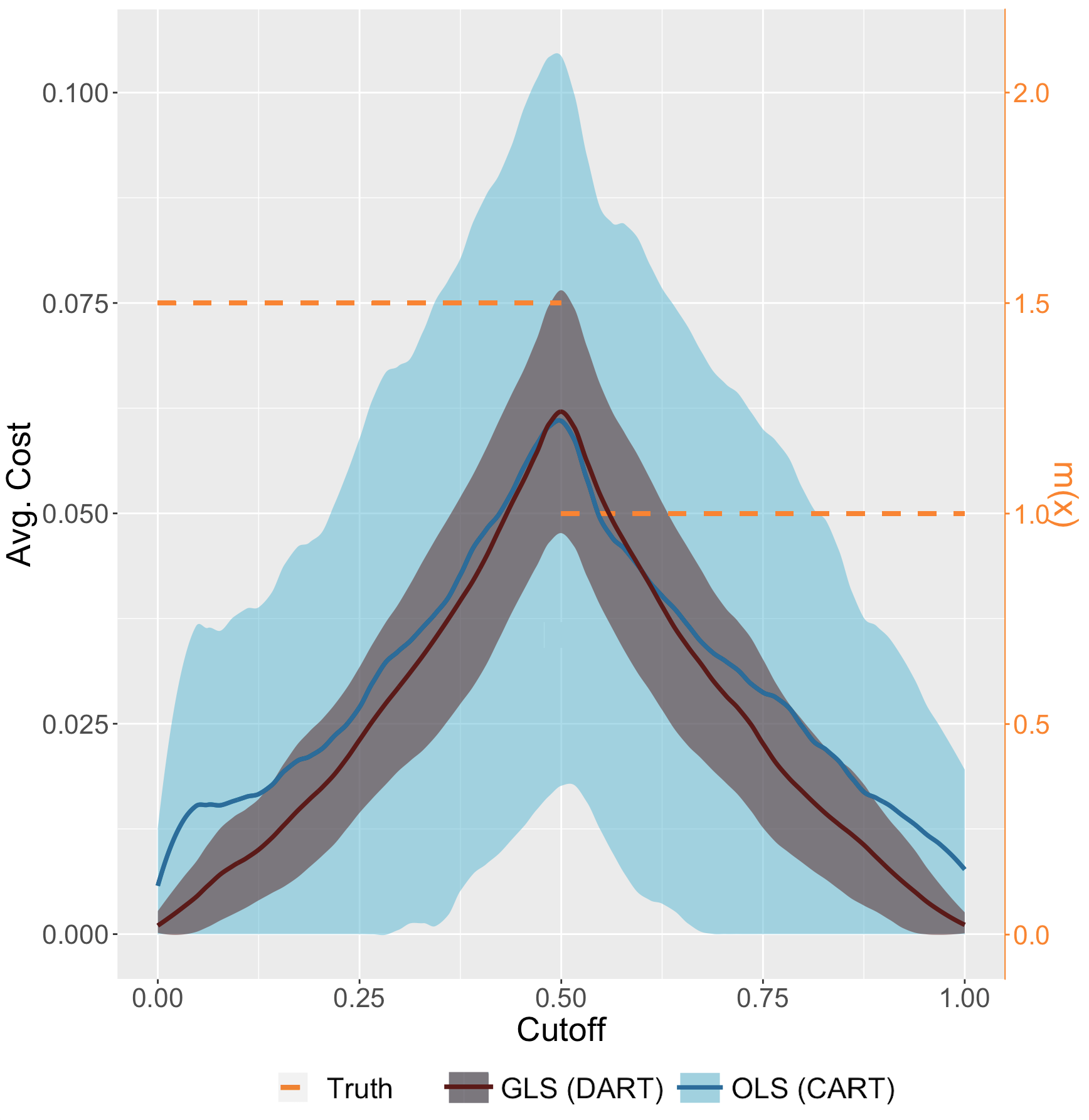}
        \caption{CART and DART criterion}
    \end{subfigure}%
    \hskip -1cm 
    \begin{subfigure}[t]{0.55\textwidth}
    \label{Fig:}
        \centering
        \includegraphics[height=2.5in]{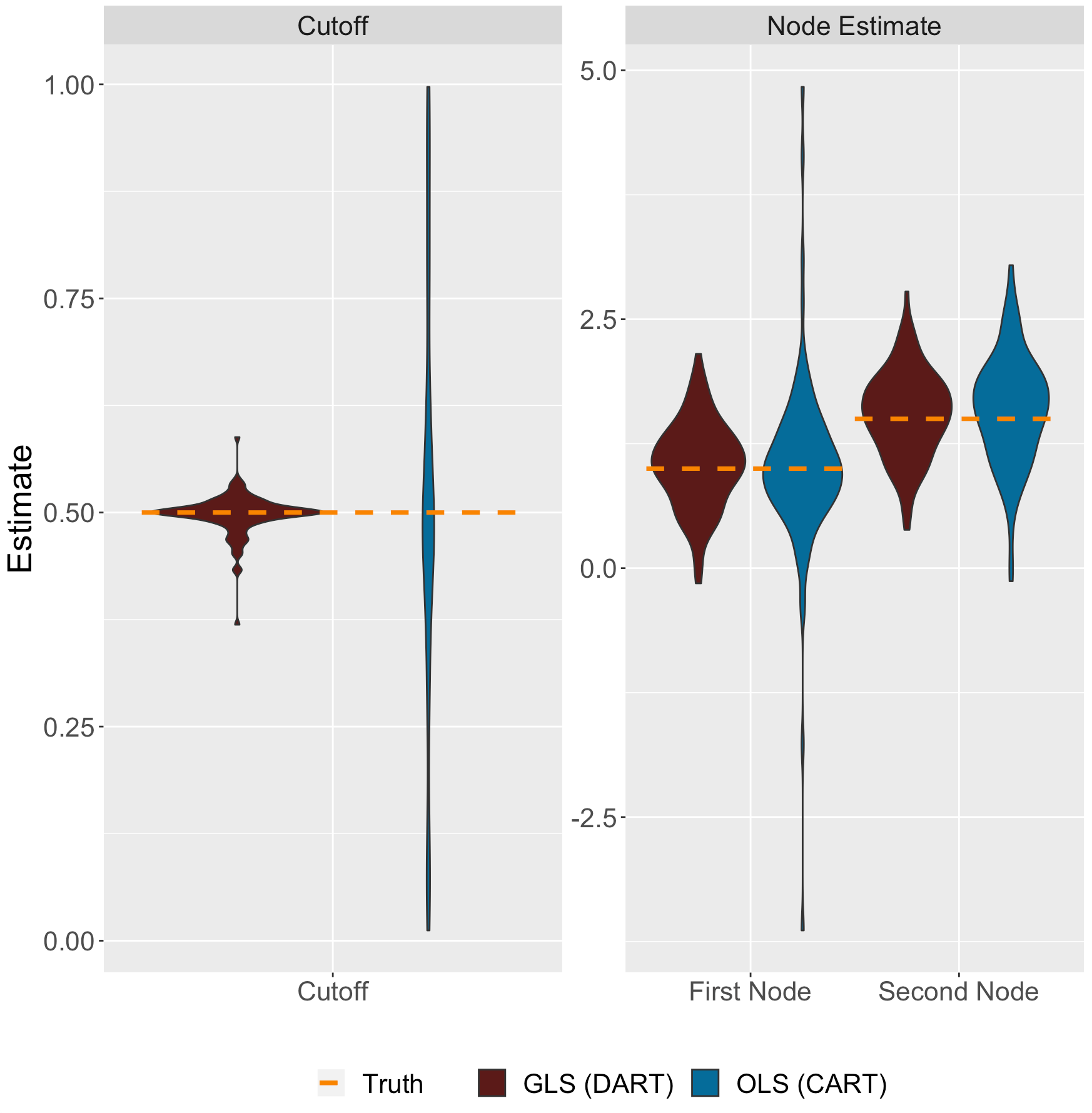}
        \caption{Estimates of split-cutoff and node representatives.}
    \end{subfigure}
    \caption{Comparison between OLS loss (CART criterion) and our GLS loss (DART  criterion) for node-splitting in a regression tree under Gaussian Process correlated errors. Left Figure (a) plots the average CART (blue) and DART (brown) criterion and point-wise $95\%$ confidence bands over $100$ replicate datasets. The true regression function $m(x)$ is plotted in orange in the secondary axis. Right figure (b) plots the densities (violin plots) of the estimates of the cutoff used for node-splitting and the node representatives from the two loss functions.}\label{fig:gls}
\end{figure}


\subsection{GLS-style regression tree}\label{sec:glstree}
The DART split criterion (\ref{DART_def}) and node representation (\ref{DART_update}) constitute the fundamental node-splitting operation of a GLS-style regression tree algorithm that incorporates the spatial information via the correlation matrix. We now introduce additional notation to formally detail the algorithm of this GLS-style regression tree. 
Creation of a forest from the trees will be discussed in Section \ref{sec:rfgls}. 

For ease of presentation, so far we have talked about the case of splitting the last node ($K^{th}$) at a given level of the tree. 
More generally, we can denote the 
complete set of nodes in level $k - 1$   by $\mathfrak{C}^{(k-1)} = \{\mathcal{C}_1^{(k-1)}, \mathcal{C}_2^{(k-1)}, \cdots, \mathcal{C}_{g^{(k-1)}}^{(k-1)}\}$ which is a partition of the feature space.  
To split the $l_1^{th}$ node $\mathcal{C}_{l_1}^{(k-1)}$, we consider the following membership matrices:  $\mathbf{Z}^{(0)}$, which corresponds to the nodes in parent level, with the column for the node-to-be-split pushed to the last column, and $\mathbf{Z}$ which corresponds to the membership of the potential child nodes 
\begin{equation}\label{eq:child}
    \mathcal{C}_{l_1^{(1)}}^{(k)} = \mathcal{C}_{l_1}^{(k-1)} \cap \{\mathbf{x} \in \mathbb R^D | x_d < c\},\, 
    \mathcal{C}_{l_1^{(2)}}^{(k)} = \mathcal{C}_{l_1}^{(k-1)} \cap \{\mathbf{x} \in \mathbb R^D | x_d \geq c\}
\end{equation} 
based on a split $(d,c)$. Note that $\bZ$ above is a function of the previous set of nodes $\mathfrak C^{(k-1)}$, the node to be split $l_1$ and the split $(d,c)$ all of which is kept implicit. 
We denote 
the GLS-style split criterion (\ref{DART_def}) as $v_{n,\bQ}^{DART}(\mathfrak{C}^{(k-1)},l_1,(d,c)) :=v_{n,\bQ}^{DART}(d,c)$ for the node $\mathcal{C}_{l_1}^{(k-1)}$ at 
$(d,c)$. 

Equipped with the notation, 
the GLS-style regression tree algorithm is presented below: 

\begin{algorithm}[!h]
	\caption{GLS-style random regression tree}\label{algo:glstree}
	\hspace*{\algorithmicindent} \textbf{Input:} Data $\mathcal{D}_n=(Y_1,X_1,\ldots,Y_n,X_n)$, working precision matrix $\bQ$, stopping rules $t_n$ (maximum number of nodes) and $t_c$ (minimum number of members per node), number of features considered for each split $M_{try}$, randomness $R_\Theta$ (some probability distribution to choose a subsample of size $M_{try}$ from $\{1.\ldots,D\}$), output point $\bx_0$. \\
	\hspace*{\algorithmicindent} \textbf{Output:} Estimate $m_n(\bx_0 ; \Theta)$ of the mean function $m$ at $\mathbf{x}_0$.
	\begin{algorithmic}[1]
		\Procedure{}{}
	
		\State \text{Initialize} $k \gets 1; \mathfrak{C}^{(1)} \gets \{\mathbb R^D\},$   
		$g^{(1)} \gets 1$; 
		\While{$g^{(k)} < t_n$ and $|\mathcal{C}_{l}^{(k)}| > t_c$ for at least one $l \in 1,2,\ldots,g^{(k)}$}
		\State Update $k \gets k+1$;
		\State \text{Initialize} $\mathfrak{C}^{(k)} \gets \{ \};\,$  
		$g^{(k)} \gets 0; $
		\For{$l_1 \in 1:g^{(k-1)}$}
		\If {$|\mathcal{C}_{l_1}^{(k-1)}| \leq t_c$ or $g^{(k)} 
							\geq t_n$}
		\State $\mathfrak{C}^{(k)} \gets \mathfrak{C}^{(k)} \cup \mathcal{C}_{l_1}^{(k-1)};$  
		$g^{(k)} \gets g^{(k)} + 1;$ 
		\Else
		\State  
		$R \gets \iid R_\Theta$
		\For{$d \in R$, $c \in \mbox{gaps}(\{X_{id} |1 \leq i \leq n-1\})\footnotemark$
		}
		\State $v_{n,\bQ}^{DART}((d,c)) \gets $ Equation (\ref{DART_def})
		\EndFor		
		\State $(d^*,c^*) \gets \arg \min _{(d,c)} v_{n,\bQ}^{DART}((d,c))$
	\State \text{$\mathcal{C}_{l_1^{(1)}}^{(k)}$ and $\mathcal{C}_{l_1^{(2)}}^{(k)} \gets$ Equation (\ref{eq:child}) with $(d^*,c^*)$} 
		\State $\mathfrak{C}^{(k)} \gets \mathfrak{C}^{(k)} \cup \mathcal{C}_{l_1^{(1)}}^{(k)} \cup \mathcal{C}_{l_1^{(2)}}^{(k)};$
		\State $g^{(k)} \gets g^{(k)} + 2;$
		\EndIf
		\EndFor
		\EndWhile
		\State Representatives $\widehat\bbeta =(\hat\beta_1,\ldots,\hat\beta_{g^{(k)}})^\top \gets $ Equation (\ref{DART_update}) with $\bZ_{n \times g^{(k)}}=\left(\mathds{I}(X_i \in \calC^{(k)}_l)\right)$;\label{line:finupdate}
		\State Output $m_n(\mathbf{x}_0; \Theta) = \sum_{l = 1}^{g^{(k)}} \hat\beta_l \mathds{I}(\bx_0 \in \calC^{(k)}_l)$;
		\EndProcedure
	\end{algorithmic}
\end{algorithm}
\footnotetext{For any set of real numbers $S=\{r_1 <  \ldots < r_s\}$, gaps$(A)=\{ (r_i + r_{i+1})/2 : i=1,\ldots,s-1\}$}

Note that, if the true covariance matrix $\bSigma_0$ is unknown, as in practice, $\bQ$ will be $\bSigma^{-1}$ where $\bSigma$ is some working covariance matrix (estimate and/or computational approximation of $\bSigma_0$).  
We discuss estimation of $\bSigma_0$ and choice of $\bQ$ in Section \ref{sec:prac}. The algorithm works with any choice of the working precision matrix $\bQ$. For example, with $\bQ=\bI$, the algorithm is identical to the usual regression-tree used in RF. 

\subsection{RF-GLS}\label{sec:rfgls}
We now focus on growing a random forest from $n_{tree}$ number of GLS-style trees. In RF, each tree is built using a resample or subsample of the data $\calD_{n,t}=(\bY_t,\bX_t)$ where $\bY_t$ is the resampled (or subsampled) vector of the responses and $\bX_t$ is the design matrix using the corresponding rows. For our GLS-style approach, naive emulation of this is not recommended. To elucidate, if one resorts to resampling, under dependent settings, this would mean resampling correlated data units $(Y_i, X_i)$ thereby violating the principle of bootstrap. Also, it is unclear what the working covariance matrix would be between 
the resampled points. 
If subsampling is used, this is avoided, as one can use the submatrix $\bSigma_t$ corresponding to the subsample. However, each tree will use a different subsample and hence a different $\bSigma_t$. Inverting covariance matrices of dimension $O(n)$ require $O(n^3)$ operations, and this approach would require inverting $n_{tree}$ such matrices, thereby substantially increasing the computation. 

Interestingly, the GLS-loss itself offers a synergistic solution to resampling of dependent data. To motivate our approach, we once again revisit RF, and note that the CART-split criterion (\ref{eq:cart_global}) for a re(sub)sample can be expressed using the squared error loss $\|\bP_t\bY - \bP_t\bZ\bbeta\|_2^2$, where $\bP_t$ is the selection matrix for the resample. Now GLS loss with $\bY$ and $\bZ\bbeta$ 
coincides with an OLS loss with $\tilde{\bY} = \bSigma^{-1/2} \bY$, $\tilde{\bZ} = \bSigma^{-1/2} \bZ$. Hence, the immediate extension for the resample (subsample) in our setup would be using the loss 
\begin{equation}\label{eq:resample}
\begin{array}{cl}
\|\bP_t \tilde\bY - \bP_t\tilde \bZ\bbeta\|_2^2 &= \|\bP_t\bSigma^{-\frac 12} \bY - \bP_t\bSigma^{-\frac 12}\bZ\bbeta\|_2^2\\
&= (\bY - \bZ\bbeta)^\top\bSigma^{-\top/2}\bP_t^\top\bP_t\bSigma^{-1/2}(\bY - \bZ\bbeta).
\end{array}
\end{equation}
Thus to use a GLS-loss in RF, we essentially resample the  contrasts $\tilde \bY$ instead of the outcomes $\bY$. This principle of contrast resampling has been used in parametric bootstrapping of spatial data \citep{pardo2012varboot,saha2018brisc}. 
In our algorithm the resampling amounts to simply   
replacing the $\bQ$ in the DART-split criterion (\ref{DART_def}) and node representative calculation (\ref{DART_update}) with 
$\bQ_t=\bSigma^{-\top/2}\bP_t^\top\bP_t\bSigma^{-1/2}$. \blue{Computationally, our approach has the advantage of only requiring a one-time evaluation of the Cholesky factor $\bSigma^{-1/2}$ that is used in all trees.} As in RF, subsequent to growing $n_{tree}$ trees corresponding to each different resample we take the average to get the forest estimate. 
This completes the specification of a novel GLS-style random forest for dependent data. We refer to the algorithm as {\em RF-GLS} and summarize it in Algorithm \ref{algo:rfgls}. It is clear when $\bQ=\bI$, the node-split criterion, the node representatives, and the resampling step all become identical to RF. Hence, RF is simply a sub-case of RF-GLS with an identity working correlation matrix. 

\begin{algorithm}[!ht]
	\caption{RF-GLS}\label{algo:rfgls}
	\hspace*{\algorithmicindent} \textbf{Input:} Data $\mathcal{D}_n=(Y_1,X_1,\ldots,Y_n,X_n)$, working correlation matrix $\bSigma$, number of trees $n_{tree}$, stopping rules for the trees: $t_n$ (maximum number of nodes) and $t_c$ (minimum number of members per node), number of features considered for each split $M_{try}$, randomness $R^{(1)}_\Theta$ (some probability distribution to choose a resample of size $n$ from $\{1.\ldots,n\}$), randomness $R^{(2)}_\Theta$ (some probability distribution to choose a subsample of size $M_{try}$ from $\{1.\ldots,D\}$), output point $\bx_0$. \\
	\hspace*{\algorithmicindent} \textbf{Output:} Estimate $m_n(\bx_0 ; \Theta)$ of the mean function $m$ at $\mathbf{x}_0$.
	\begin{algorithmic}[1]
		\Procedure{}{}
	    \State Calculate $\bSigma^{-1/2}$;
	    \State Initialize $t=1$;
	    \For{$t = 1:n_{tree}$}
	    \State Generate $R_t \gets \iid R^{(1)}_\Theta$
	    \State $\bP_t \gets (I(i = R_t[j]))$
	    \State $\bQ_t=\bSigma^{-\top/2}\bP_t^\top\bP_t\bSigma^{-1/2}$
		\State $m_n(\bx_0;\Theta_t) \gets$ Algorithm \ref{algo:glstree} with $\calD_n,\bQ_t,t_n,t_c,M_{try},R^{(2)}_\Theta,\bx_0$;
		\EndFor
		\State Output $\hat m_n(\mathbf{x}_0) = \frac 1{n_{tree}}\sum_{t = 1}^{n_{tree}} m_n(\bx_0;\Theta_t)$
		\EndProcedure
	\end{algorithmic}
\end{algorithm}

\subsection{Kriging using RF-GLS with Gaussian Processes}\label{sec:krig}
Subsequent to estimating the regression function $m(\bx)$ using RF-GLS in the spatial non-linear mixed model (\ref{eq:spnlmm}), we can seamlessly perform traditional spatial tasks like predictions (kriging) and recovery of the latent spatial random surface $w(\ell)$. This is because RF-GLS only estimates the mean part, and the covariance is still being modeled using a GP. 
This facilitates easy formulation of the predictive or latent distribution conditional on the data as standard conditional normal distributions. 
If $\calL$ denotes the training data locations,  $\bY=(Y_1,\ldots,Y_n)^\top$, and $\bom=(\widehat m(X_1),\ldots,\hat m(X_n))^\top$ where $\widehat m$ is the estimate of $m$ from RF-GLS, then prediction at a new location $\ell_{new}$ with covariate $\bx_{new}$ will simply be given by the kriging estimate
\begin{equation}\label{eq:rfpred}
\widehat y_{new}(\bx_{new},\ell_{new}) = \widehat m(\bx_{new}) + \bv^\top\bSig^{-1}(\bY-\widehat \bom)
\end{equation}
where  $\bv^\top=Cov(\eps(\ell_{new}),\eps(\calL)), \bSigma=Cov(\eps(\calL),\eps(\calL))$. The prediction equation (\ref{eq:rfpred}) 
possess the advantage of non-parametrically estimating the mean function of the covariates while retaining the spatial structure encoded in the GP covariance function which 
adheres to the philosophy of {\em first law of geography}, i.e., proximal things are more correlated than distant ones. The prediction framework is completely agnostic to the choice of the covariance function and can work with non-stationary or multi-resolutional covariance functions if deemed appropriate. One can also obtain estimates of the  latent surface using the conditional distributions $w(\ell) \given \bY, \bX$ akin to the spatial linear model. 

\blue{Prediction equations of the form (\ref{eq:rfpred}) has been used in applications of RF to spatial settings and has been termed as random forest {\em residual kriging} (RF-RK) \citep{viscarra2014mapping,fayad2016regional}. The estimate $\widehat m$ in these applications come from a naive application of RF without accounting for the dependence. Our empirical studies in Section \ref{sec:sim} will demonstrate how residual kriging using RF-GLS improves over use of RF in dependent settings. }

RF-GLS also has several advantages over the spatial RF (RFsp) of \cite{hengl2018random}, which does not use GP but include pairwise distances between a location and all other locations, as additional covariates. For $n$ locations, this adds $n-1$ covariates.
RF-GLS avoids such unnecessary escalation of the problem to  high-dimensional settings. Direct use of the GP and the mixed-model framework helps model the spatial structure parsimoniously via the covariance function parameters. Our simulation studies in Section \ref{sec:sim} demonstrates the improved prediction performance of RF-GLS over RFsp. Most importantly, RF-GLS separates out the contribution of the covariates and the spatial component, thereby allowing estimation of the regression function $m$. 
Estimate of $m$ is not available from the spatial RF of \cite{hengl2018random} which only offers a prediction given  the covariates and the location.

\subsection{Practical considerations}\label{sec:prac} The development of the RF-GLS method has been presented using a working covariance matrix $\bSigma$. In practice, the true correlation $\bSigma_0$ will not be known and needs to be estimated assuming a parametric form $\bSigma$ based on the covariance function used. In Section \ref{sec:theory} we present the result that both RF and RF-GLS are consistent under dependent errors (akin to both OLS and GLS being consistent for linear models). Hence, for data analysis, we recommend running a first pass of RF on the data to get a preliminary estimate of $m$, use it to obtain the residuals from which the parameters of $\bSigma$ can be estimated using a maximum likelihood approach. This once again parallels the practice in linear models where the {\em oracle GLS} assuming knowledge of the true covariance matrix is replaced by a {\em feasible GLS} where the covariance matrix is estimated based on residuals from an initial OLS estimate. 

The second consideration concerns computational scalability of the approach. RF-GLS requires computing the Cholesky factor $\bSig^{-1/2}$. It is clear from Algorithm \ref{algo:rfgls}, that this is only a one-time cost, unlike the possible alternate approach discussed in Section \ref{sec:rfgls} where subsampling is conducted before decorrelation, which would lead to computing a different Cholesky factor for each tree. 
However, spatial covariance and precision matrices arising from GP are dense and for large $n$, evaluating $\bSig^{-1}$ even once still incurs the computational cost of $O(n^3)$ and storage cost of $O(n^2)$ both of which are taxing on typical personal computing resources. 

Over the last decade, the inventory of approximation techniques attacking the computational weakness of GP has grown and become increasingly sophisticated \citep[see][for a review]{heaton2019case}. Nearest Neighbor Gaussian Processes (NNGP)  \citep{nngp,finley2019efficient,datta2016nearest,dnngp} has emerged as one of the leading candidates. Centered on the principle that a few nearby locations are enough to capture the spatial dependence at a given location \citep{vecchia1988estimation}, NNGP replaces the dense graph among spatial locations with a nearest neighbor graphical model. This was shown to directly yield a sparse Cholesky factor $\widetilde \bSig^{-1/2}$ that offers an excellent approximation to the original dense $\bSig^{-1/2}$ \citep{nngp}. Software packages implementing NNGP are also publicly available \citep{spnngppaper,brisc}. 
Hence, for very large data, we recommend using this NNGP  sparse  
Cholesky factor $\widetilde\bSig^{-1/2}$ instead of $\bSig^{-1/2}$. 
This will reduce both the computation and storage cost from cubic to linear in sample size. We will show in Proposition \ref{th:mat}, that using NNGP for RF-GLS ensures a consistent estimate of $m$ even when the true data generation is from a full Mat\'ern GP. 

\section{Consistency}\label{sec:theory}
In this Section we present the main theoretical result on consistency of RF-GLS for a very general class of dependent error processes. 
The outline of the proof highlighting the new theoretical challenges addressed are presented in Section \ref{sec:outline} along with some general results of independent importance. The formal proofs are provided in the Supplementary materials. 

\subsection{Assumptions}
We first discuss assumptions \ref{as:mix} and \ref{as:chol} and make additional assumptions required for the proof of consistency. 
In Assumption \ref{as:mix}, we focus on absolutely regular or $\beta$-mixing processes, since this class of stochastic processes is rich enough to accommodate many commonly used dependent error processes like ARMA \citep{mokkadem1988mixing}, GARCH \citep{carrasco2002mixing}, certain Markov processes \citep{doukhan2012mixing} and  Gaussian processes with Mat\'ern covariance family. At the same time, uniform law of large numbers (ULLN) from independent processes can be extended to this dependent process under moderate restriction on the class of  functions under consideration. No additional assumption is required on the decay rate of the $\beta$-mixing coefficients (which are often hard to check).

\blue{Assumption \ref{as:chol} requires the Cholesky factor of the precision matrix to be sparse and regular. Such structured Cholesky factors routinely appear in time series analysis for AR$(p)$ process. For spatial data, exponential covariance family on a $1$-dimensional grid satisfies this. Other covariances like the Mat\'ern family (except the exponential covariance) do not generally satisfy this assumption. However, NNGP covariance matrices satisfy this 
and are now commonly used as an excellent approximation to the full GP covariance matrices \citep{nngp}. Since this assumption is on the working covariance matrix and not on the true covariance of the process, we can always use an approximate working covariance matrix like ones arising from NNGP to satisfy this.} 
We discuss these examples in Section \ref{sec:examples}.
Under Assumption \ref{as:chol}, for any two vectors $\bx$ and $\by$, defining $x_i = y_i = 0$ for $i \leq 0$, we have
\begin{align}\label{eq:qf}
\bx^\top\bQ\by 
&= \alpha \sum_i x_iy_i + \sum_{j \neq j' = 0}^q \rho_j\rho_{j'} \sum_{i} x_{i-j}y_{i-j'}+\sum_{i \in \tilde{\mathcal{A}}_1}\sum_{i' \in \tilde{\mathcal{A}}_2} \tilde\gamma_{i,i'}x_iy_{i'},
\end{align}
where $\alpha=\|\brho\|_2^2$, $\tilde{\mathcal{A}}_1, \tilde{\mathcal{A}}_2 \subset \{1,2,\cdots,n \}$ with $|\tilde{\mathcal{A}}_1|, |\tilde{\mathcal{A}}_2| \leq 2q$, $\tilde \gamma_{i,i'}$'s are fixed (independent of $n$) functions of $\bL$ and $\brho$.
The expression of the quadratic form in (\ref{eq:qf}) makes it evident that $\lambda_{\max}(\bQ)$ is bounded as $n \rightarrow \infty$. As the third term is a sum of fixed (at most $4q^2$) number of terms, it is $O(1)$ as long as $\bx$ and $\by$ are bounded. 

\begin{assume}[Diagonal dominance of the working precision matrix]\label{as:diag} 
	$\bQ$ is diagonally dominant satisfying $\bQ_{ii} - \sum_{j \neq i} |\bQ_{ij}| > \xi$ for all $i$, for some constant $\xi >0$. 
\end{assume}
Diagonal dominance implies $\lambda_{\min}(\bQ)$ is bounded away from zero as $n \rightarrow \infty$ which is needed to ensure stability of the GLS estimate. We will discuss in Section \ref{sec:examples} how working correlation matrices from popular time series and spatial processes with regular design  satisfy this Assumption. Note that under Assumption \ref{as:chol}, checking that the first $(q+1)$ rows of $\bQ$ are diagonally dominant is enough to verify Assumption \ref{as:diag}.

\begin{assume}[Tail behavior of the error distribution]\label{as:tail} 
\hfill
\begin{enumerate}[(a)]
	\item $\exists \{\zeta_n \}_{n \ge 1}$ such that  
	$$
	\begin{aligned}
	&\zeta_n \to \infty,\,
	\frac{t_n (\log n)\zeta_n^4}{n} \to 0,\; \mbox{and }\\ 
	&\mathbb{E}\left[\left(\max_{ i } \eps_i^2 \right) \mathds{I} \left(\max_{ i} \eps_i^2 > \zeta_n^2\right) \right] \to 0 \text{ as } n \to \infty.
	\end{aligned}
	$$
	\item $\exists$ constant $C_{\pi} > 0$ and $n_0 \in \mathbb{N}^{*}$ such that with probability $1 - \pi$, $\forall n > n_0$,
	$$
	\max_{i} |\eps_i| \leqslant C_{\pi} \sqrt{\log n}.
	$$
	\item Let  $\mathcal{I}_n \subseteq \{1,2,\cdots,n \}$  
	with $|\mathcal{I}_n| = a_n$ and $a_n \to \infty$ as $n \to \infty$. Then 
	$\frac 1{a_n} | \sum_{i \in \calI_n} \eps_i | > \delta$ with probability at most $C \exp(- ca_n)$ and $\frac 1{n} | \sum_{i} \eps_i^2 | > \sigs_0$ with probability at most $C 
	\exp(- cn)$ for any $\delta > 0$, and some constants $c,C,\sigs_0 > 0$.
\end{enumerate}
\end{assume}
 
\blue{We show for Gaussian errors, $\zeta_n$ needs to be $O(\log n)^2$ which makes the scaling condition in Assumption \ref{as:tail}(a) as $t_n (\log n)^9 /n \to 0$. This is the same scaling used in \cite{scornet2015consistency} for Gaussian errors and using the entire sample.} In general, the choice of $\zeta_n$ will be dependent on the error distribution. 
Assumption \ref{as:tail}(a), (b) and (c) will all be satisfied by sub-Gaussian errors. 

\begin{assume}[Additive model]\label{as:add}
The true mean function $m(\bx)$ is additive on the coordinates $x_d$ of $\bx$, i.e.,  
$m(\bx) = \sum_{d = 1}^D m_d({x}_d)$, where each component $m_d$ is continuous.  
\end{assume}

As demonstrated in \cite{scornet2015consistency}, additive models provide a rich enough environment to address the asymptotic properties of nonparametric methods like RF sans the additional complexities in controlling asymptotic variation of $m$ in leaf nodes. Since RF is invariant to monotone transformations of covariates \citep{friedman2001elements,friedman2006recent}, without loss of generality, the covariates can be distribution function transformed to be Unif$[0,1]$ distributed. Hence we assume that the components (functions) $m_d$ are supported on $[0,1]$, implying $m$ is uniformly bounded by some constant $M_0$.

\subsection{Main result}
For the $t^{th}$ tree, the predicted value from our method at a new point $\mathbf{x}_0$ in covariate space is denoted by $m_n (\mathbf{x}_0; \Theta_t, \mathbf{\Sigma}, \mathcal{D}_n)$ where $\calD_n=\{(\bx_i,y_i) \given i=1,\ldots,n\}$ denote the data.  \blue{Note that the $i^{th}$ data unit corresponds to location $\ell_i$ and that the covariance matrix $\bSigma$ is based on the covariance function evaluated at pairs of locations $\ell_i$ and $\ell_j$. But unless otherwise needed we suppress the locations $\ell_i$ and simply use the sub-script $i$.} 

$\Theta_t$ indicates the randomness associated with each tree. In practice, $\Theta_t$ will include both the re-sampling of data-points used in each tree as well as the choice of random splitting variable for iterative splitting in the tree. For tractability, \cite{scornet2015consistency} considered sub-sampling instead of re-sampling for the theoretical study. In our theoretical study, for analytical tractability of the GLS weights, we consider trees that use the entire set of samples and the randomness $\Theta_t$ in each tree is only used to choose the candidate set of features for each split.
The randomness for each tree are i.i.d., i.e.,  $\Theta_t \overset{i.i.d}{\sim} \Theta; \:\: \Theta \independent \mathcal{D}_n, \forall t \in \{ 1,\cdots, n_{tree}\}$. 
The finite RF-GLS  estimate $\widehat{m}_{n,n_{tree} }(\mathbf{x}_0; \Theta_1,\cdots,\Theta_{n_{tree}}, \mathbf{\Sigma}, \mathcal{D}_n)$ that will be used in practice is given by the sample average of the individual tree estimates. Conceptually, $n_{tree}$ can be arbitrarily large, hence following \cite{scornet2015consistency}, we focus on ``infinite" RF-GLS estimate given by $\bar{m}_n  (\mathbf{x}_0; \mathbf{\Sigma}, \mathcal{D}_n) = \mathbb{E}_{\Theta}m_n (\mathbf{x}_0; \Theta, \mathbf{\Sigma},  \mathcal{D}_n)$ where the expectation w.r.t $\Theta$ is conditional on $\mathcal{D}_n$. For notational convenience, we hide the dependence of $m_n, \widehat{m}_{n,n_{tree}}, \bar{m}_n$ on $\mathbf{\Sigma}$ and $\mathcal{D}_n$ throughout the rest of this article. Our main result on $\mathbb{L}_2$-consistency is stated next, the proof is deferred to Section  \ref{sec:gyorfi}.
\begin{theorem}\label{th:main}
	Under Assumptions 1-5 and if for some $\delta>0$,  
	$\lim_{n \to \infty }  \mathbb{E} \frac{1}{{n}} \sum_i |m_n(X_i)|^{2+\delta} < \infty$, then 
	RF-GLS is $\mathbb L_2$-consistent, i.e., 
	 $
	\lim_{n \to \infty} \mathbb{E} \int \left(\bar{m}_n(X) - m(X) \right)^2 \, dX = 0,
	$.
\end{theorem}

The uniformly-bounded $(2+\delta)^{th}$ moment assumption in Theorem \ref{th:main} is needed to \blue{generalize uniform laws of large number to bound the GLS estimation error from the i.i.d. setting to the dependent setting. We discuss this in details in Section \ref{sec:estimation}.} 
The following corollaries discuss three specific cases where this assumption is met. 
Their proofs are deferred to Section \ref{sec:pf-l2}. 

\begin{corollary}\label{cor:rfgls} Under Assumptions 1-5, RF-GLS is $\mathbb L_2$ consistent if either:
	\begin{enumerate}[(a)]
		\setlength\itemsep{0em}
		\item Case 1: The errors are bounded. 
		\item Case 2: The working precision matrix $\bQ$ satisfies $\min_i \bQ_{ii} > \sqrt 2 \max_i \sum_{j \neq i} |\bQ_{ij}|$. 
	\end{enumerate} 
\end{corollary}

For bounded errors (part (a)), the $(2+\delta)^{th}$ moment-bound of Theorem \ref{th:main} is immediately satisfied, and hence consistency can be established without further assumptions. 
For unbounded errors, a stronger form of diagonal dominance condition is needed in Corollary \ref{cor:rfgls} Part (b). This is used to control the $(2+\delta)^{th}$ moment of the data weights arising from the gram-matrix $(\bZ^\top\bQ\bZ)^{-1}$ which in turn ensures the moment-bound. We discuss examples and specific parameter choices ensuring this in Section \ref{sec:examples}. Also note that, the assumption of diagonal dominance is not on the true correlation matrix of the error process and hence is not a restriction on the data-generation mechanism, but rather on the working correlation matrix which is chosen by the user. One can always use parameters in the working correlation matrix that satisfies this  (although enforcing this is not needed in practice).  

RF is RF-GLS with $\bQ=\bI$. Hence the assumption of Corollary \ref{cor:rfgls} part (b) is trivially satisfied. 
This proves consistency of RF under $\beta$-mixing dependence. 

\begin{corollary}\label{cor:rf} Under Assumptions 1, 4 and 5, RF \citep{breiman2001random} is $\mathbb L_2$-consistent. 
\end{corollary}

To our knowledge, Corollary \ref{cor:rf} is the first result on consistency of RF under a dependent ($\beta$-mixing) error process. Since RF is simply RF-GLS with the working correlation matrix $\bSig=\bI$, Assumptions 2 and 3 are automatically satisfied, and hence we only need the Assumptions of $\beta$-mixing process, tail bounds and additive model. The consistency result is analogous to the ordinary least squares estimate being consistent even for correlated errors. Besides its own importance, Corollary \ref{cor:rf} also heuristically justifies the first step used in practical implementation of RF-GLS. The parameters in the working correlation matrix is unknown, and as highlighted in Section \ref{sec:prac}, we use the RF to get a preliminary estimate of $m$, estimate the spatial parameters using the residuals, and use these estimated parameters in the working correlation matrix for RF-GLS. This is again, analogous to feasible GLS which estimates the working correlation matrix using residuals based on OLS. Corollary \ref{cor:rf} guarantees that the initial estimator used to obtain the residuals is consistent. 

\subsection{Examples}\label{sec:examples} In this Section, we give examples of two popular dependent error processes under which a consistent estimate of $m$ can be obtained using RF-GLS. 

\subsubsection{Spatial Mat\'ern Gaussian processes}\label{sec:gp} 
Our main example focuses on the spatial non-linear mixed model using Gaussian Processes as described in Section \ref{sec:krig}. 
While many candidates exist for the covariance function of GP, the class of Mat\'ern covariances enjoy hegemonic popularity in the spatial literature owing to its remarkable property of characterizing the smoothness of the spatial surface $\eps(\ell)$ \citep{stein2012interpolation}. The stationary (isotropic) Mat\'ern covariance function is specified by 
\begin{equation}\label{eq:matern}
C(\ell_i,\ell_j \given \bphi) = C(\|\ell_i - \ell_j\|_2) = \sigs \frac{2^{1-\nu} \left(\sqrt 2 \phi\|\ell_i - \ell_j\|_2\right)^\nu}{\Gamma(\nu)} \calK_\nu \left(\sqrt 2 \phi\|\ell_i - \ell_j\|_2\right), 
\end{equation}
where $\bphi=(\sigs,\phi,\nu)^\top$ and $\calK_\nu$ is the modified Bessel function of second kind. 

We consider a Mat\'ern process sampled on one-dimensional regular lattice. This regular design is considered both for tractability of the Mat\'ern GP likelihood but also for ensuring stationarity of the process in the sense required in Theorem \ref{th:main} as for irregular spaced data $Cov(\eps_1,\eps_2) \neq Cov(\eps_2,\eps_3)$ whenever $\|\ell_1-\ell_2\|_2 \neq  \|\ell_2 - \ell_3\|_2$. Such assumptions on the dimensionality and/or regularity of design has been widely used for theoretical studies of spatial processes \citep{du2009fixed,stein2002screening}. By keeping the gap in the lattice fixed, we are also essentially using increasing-domain asymptotics, as parameters are generally not identifiable in fixed domain asymptotics for Mat\'ern GPs \citep{zhang2004inconsistent}. 


The error process arising from the marginalization of (\ref{eq:spnlmm}) is the sum of a Mat\'ern process and a nugget (random error) process. We consider half-integer $\nu \in {1.5,2.5,\ldots}$. This class of processes are popularly studied and used owing to their convenient state-space representation \citep{hartikainen2010kalman} which in turn leads to efficient computation of these Mat\'ern GP likelihoods. The state-space representation of half-integer Mat\'ern GP is equivalent to that of a stable $AR(q_0)$ process on the continuous one-dimensional domain with $q_0=\nu + 1/2$. 
However, unlike an $AR(q_0)$ time series, the Mat\'ern GP when sampled on the discrete integer lattice is no longer an $AR(q_0)$ process. Consequently, unlike covariance matrices from AR processes, covariance matrices $\bSig$ generated from Mat\'ern GP (expect for exponential GP), do not satisfy the sparsity and regularity of the working correlation matrix of Assumption \ref{as:chol}. 

Instead, we consider the working correlation $\bSigma$ to come from a Nearest Neighbor Gaussian Process (NNGP) \citep{nngp} based on the Mat\'ern covariance. As discussed in Section \ref{sec:prac}, NNGP covariance matrices are one of the most successful surrogates for full GP covariances for large spatial data, reducing likelihood computations from $O(n^3)$ to $O(n)$. What is important for the theoretical study is that an NNGP is constructed  by sequentially specifying the conditional distributions as $\eps_i \given \eps_{1:i-1} \sim \eps_i \given \eps_{N_q(i)}$ where $N_q(i) \subset \{1,\ldots,i-1\}$ is the set of $q$-nearest neighbors of $\ell_i$ among $\ell_1, \ldots, \ell_{i-1}$. When the locations are the integer grid, $N_q(i)$ becomes $\{i-1,\ldots,i-q\}$, and the NNGP construction is akin to an $AR(q)$ process. Consequently, the 
Cholesky factor $\bSig^{-1/2}$ from NNGP on an integer lattice satisfies Assumption \ref{as:chol} with $\brho = (1,-\bc^\top\bC^{-1})^\top/\sqrt{1-\bc^\top\bC^{-1}\bc}$ and $\bL$ such that $\bL^\top\bL = \bC^{-1}$ where $\bC=Cov(\eps_{1:q})$, $\bc=Cov(\eps_{1:q},\eps_{q+1})$ \citep{finley2019efficient}. This ensures the following consistency result of RF-GLS fitted with NNGP for data generated using Mat\'ern GP. The proof is in Section \ref{sec:pf-examples}.
\begin{proposition}\label{th:mat}
	Consider a spatial process $y(\ell_i)=m(X_i) + \eps(\ell_i)$ from (\ref{eq:spnlmm}) where $m$ is an additive model as specified in Assumption \ref{as:add}, $\eps(\ell_i)=w(\ell_i)+\eps^*(\ell_i)$ where $\eps^*(\ell)$ denote i.i.d. $N(0,\tau_0^2)$ noise, and $w(\ell)$ be a Mat\'ern GP, sampled on the integer lattice, with parameters $\bphi_0=(\sigs_0,\phi_0,\nu_0)^\top$, $\nu_0$ being a half-integer. Let $\bSigma$ denote a covariance matrix from a Nearest Neighbor Gaussian Process (NNGP) derived from a Mat\'ern covariance with parameters $\bphi=(\sigs,\phi,\nu)^\top$ and $\taus$. Then there exists some $K >0$ such if $\phi > K$, then RF-GLS using $\bSigma$ yields an $\mathbb L_2$ consistent estimate of $m$.
\end{proposition}

One observation is central to the proof. The half-integer Mat\'ern GP, which is an $AR(q_0)$ process in the continuous domain, when sampled on a discrete lattice becomes an ARMA process (\citep{ihara1993information} Theorem 2.7.1). 
This will establish absolutely regular mixing of these Mat\'ern processes using the result of \cite{mokkadem1988mixing} on ARMA processes, and subsequently consistency of RF-GLS by Theorem \ref{th:main}. 

\subsubsection{Autoregressive time series}\label{sec:ar}
Our main focus in this manuscript is estimation of nonlinear regression function in the spatial mixed model (\ref{eq:spnlmm}). However, the scope of our RF-GLS algorithm is much broader. It can be used for functional estimation in the general nonlinear regression model $Y_i=m(X_i) + \eps_i$ where $\eps_i$ is a dependent stochastic process with valid second moment. The method only relies on knowledge of an estimate of the residual covariance matrix $\bSig=Cov(\beps)$. The general consistency result (Theorem \ref{th:main}) is also not specific to the spatial GP setting, and only relies on general assumptions on the nature of the dependence, tail bounds of the error, and structure of the working correlation matrix. 
In this Section, we demonstrate that RF-GLS can be used for consistent function estimation for  time series data, i.e., where $\eps_i$ models the serial (temporal) correlation. In particular, we discuss consistency of RF-GLS for Autoregressive (AR) error processes, one of the mainstays of time series studies. 
An $AR(q)$ model can be written as:
\begin{align}\label{eq:ar}
\eps_i = a_1 \eps_{i-1} + a_2 \eps_{i-2} + \ldots + a_q \eps_{i-q} + \eta_i
\end{align}
where $\eta_i$ is a realization of a white noise process at time $i$. AR processes are $\beta$-mixing \citep{mokkadem1988mixing}, they also produce a banded Cholesky factor of the precision matrix as required in Assumption \ref{as:chol}. Hence, we have the following assertion. Its proof is in Section \ref{sec:pf-examples}.

\begin{proposition}\label{th:ar} Consider a time series $Y_i=m(X_i) + \eps_i$ where $i$ is the time, $m$ satisfies Assumption \ref{as:add}, $\eps_i$ denote a sub-Gaussian stable $AR(q_0)$ process. Let $\bSigma$ denote a working correlation matrix from a stationary $AR(q)$ process. 
Then RF-GLS using $\bSigma$ produces an $\mathbb L_2$ consistent estimate of $m$ if
	\begin{enumerate}
		\item $q=1$ and the working autocorrelation parameter $\rho$ used in $\bSigma$ satisfies $|\rho| < 1$ (for bounded errors) or $|\rho| < 1/(2 \sqrt 2)$ (for unbounded errors). 
		\item $q > 1$ and the AR(q) working precision matrix $\bQ=\bSigma^{-1}$ satisfies Assumption \ref{as:diag} (for bounded error) or $\min_i \bQ_{ii} > \sqrt 2 \max_i \sum_{j \neq i} |\bQ_{ij}|$ (for unbounded error). 
	\end{enumerate}
\end{proposition}

We separate the results for $q=1$ and $q \geq 2$ since unlike $AR(1)$, for general $AR(q)$ it is challenging to derive closed form expressions of the constraints on the parameter space needed to satisfy Assumption \ref{as:diag} or the stronger diagonal dominance condition in Proposition \ref{th:ar} part 2 (for unbounded errors). However, verifying these conditions for a given AR precision matrix $\bQ$ is straightforward due to stationarity and banded structure. 
One only needs to check the first $q+1$ rows of $\bQ$ irrespective of the sample size $n$. 
The necessary condition for 
row $q+1$ is 
\begin{equation}\label{eq:diagdiff}
\|\brho\|^2 > 2 \kappa  \sum_{j=1}^q |\sum_{j'=j}^q \rho_{j'} \rho_{j'-j}| 
\end{equation}
where $\kappa$ equals $1$ for bounded errors and equals $\sqrt 2$ for unbounded errors. Additional checks are only needed for the first $q$ rows of $\bQ$. 

In practical implementation of AR processes, the order of the autoregression is often chosen based on analysis of the auto-correlation function of the residuals and may not equal the true order of the autoregression. Proposition \ref{th:ar} accommodates this scenario by not restricting the working autoregressive covariance to be of the same order or have the same coefficients as the ones generating the data. 

\section{Illustrations}\label{sec:sim}
We conduct simulation experiments to demonstrate the advantages of the RF-GLS over competing methods for both estimation and prediction in finite samples. We simulate data from the spatial non-linear mixed model of (\ref{eq:spnlmm}). We consider the following choices for the true mean function $m_1(x) = 10\sin(\pi x)$ and $m_2(\mathbf x) = (10\sin(\pi x_1 x_2) + 20(x_3 - 0.5)^2 + 10x_4 + 5x_5)/6$  \citep[Friedman function,][]{friedman1991multivariate}. $w(\ell)$ is an exponential GP on a two-dimensional spatial domain with $27$ different combinations of covariance parameters spatial variance $\sigma^2$, spatial decay $\phi$ and error variance $\taus$ as a $\%$ of the spatial variance. For each setting, we perform simulations for $100$ replicate datasets. To evaluate prediction performance, we keep $10\%$ hold-out data. Details of the parameter choices and hold-out design are in Section \ref{sec:simdetails}.

For evaluating estimation performance, we consider  
RF (which does not use spatial information), RF-GLS (Oracle) using true covariance parameter values, and RF-GLS which 
    obtains an estimate of the covariance parameters from RF residuals using the BRISC package  \citep{brisc} and uses them in RF-GLS. 
The estimation performance of the methods are evaluated based on a discrete Mean Integrated Square Errors (MISE) for $m$, evaluated over uniformly generated data points from the covariate space that provides a good coverage of the entire space. \blue{MISE for the estimated function $\hat{m}$ is given as:

$$
\text{MISE} = \int (m(\bx) - \hat{m}(\bx))^2 d\bx \approx \frac{1}{n_0} \sum_{i = 1}^{n_0} (m(\bx_i) - \hat{m}(\bx_i))^2
$$
where $\bx_1, \cdots, \bx_{n_0}$ are a dense set of points in the covariate domain (See Section \ref{sec:simdetails} for details).}

\begin{figure}[t!]
    \centering
    \begin{subfigure}[t]{0.5\textwidth}
    \label{Fig:known_sinpi}
        \centering
        \includegraphics[height=3.2in]{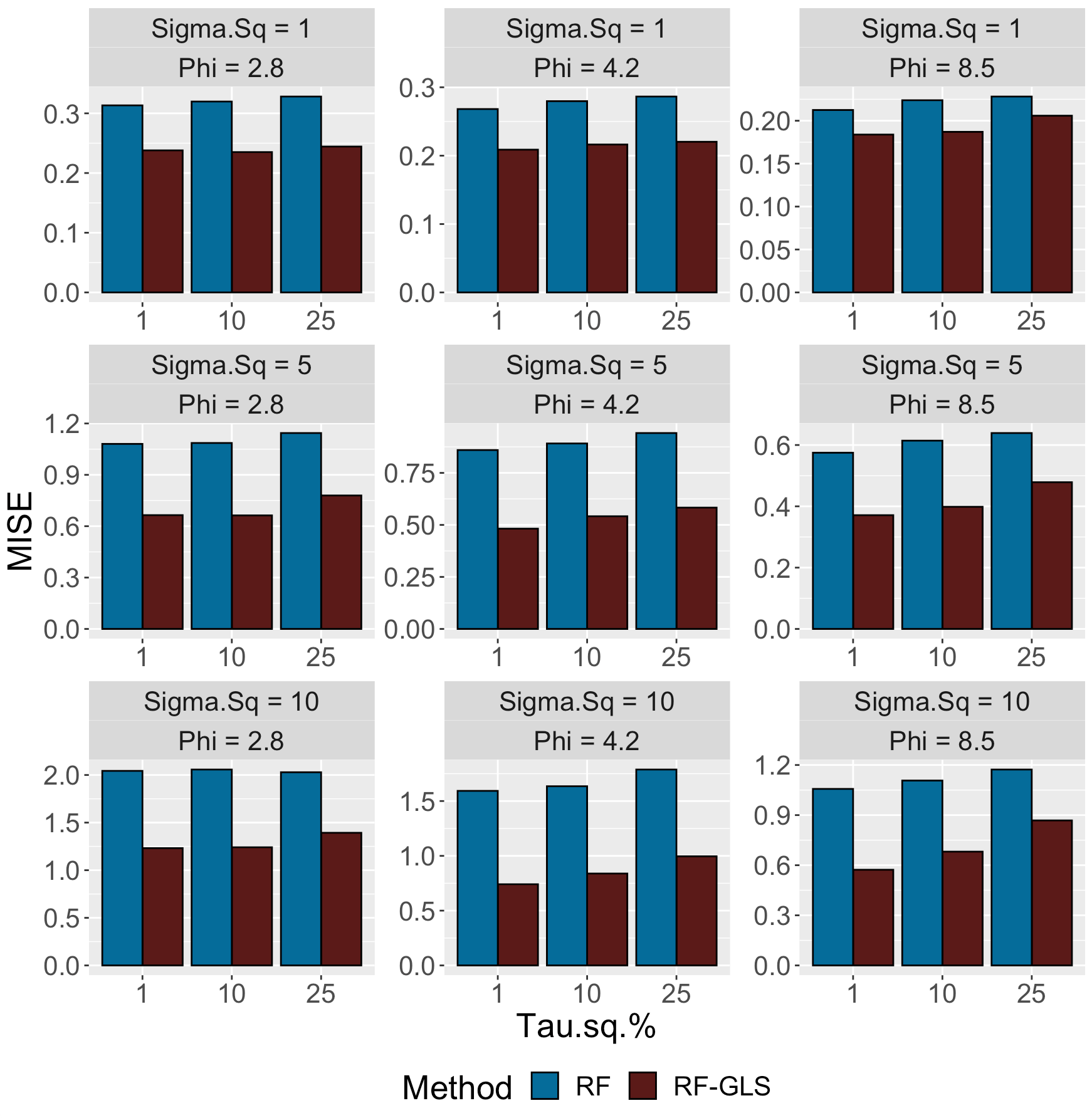}
        \caption{Estimation performance}
    \end{subfigure}%
    ~ 
    \begin{subfigure}[t]{0.5\textwidth}
    \label{Fig:known_friedby6}
        \centering
        \includegraphics[height=3.2in]{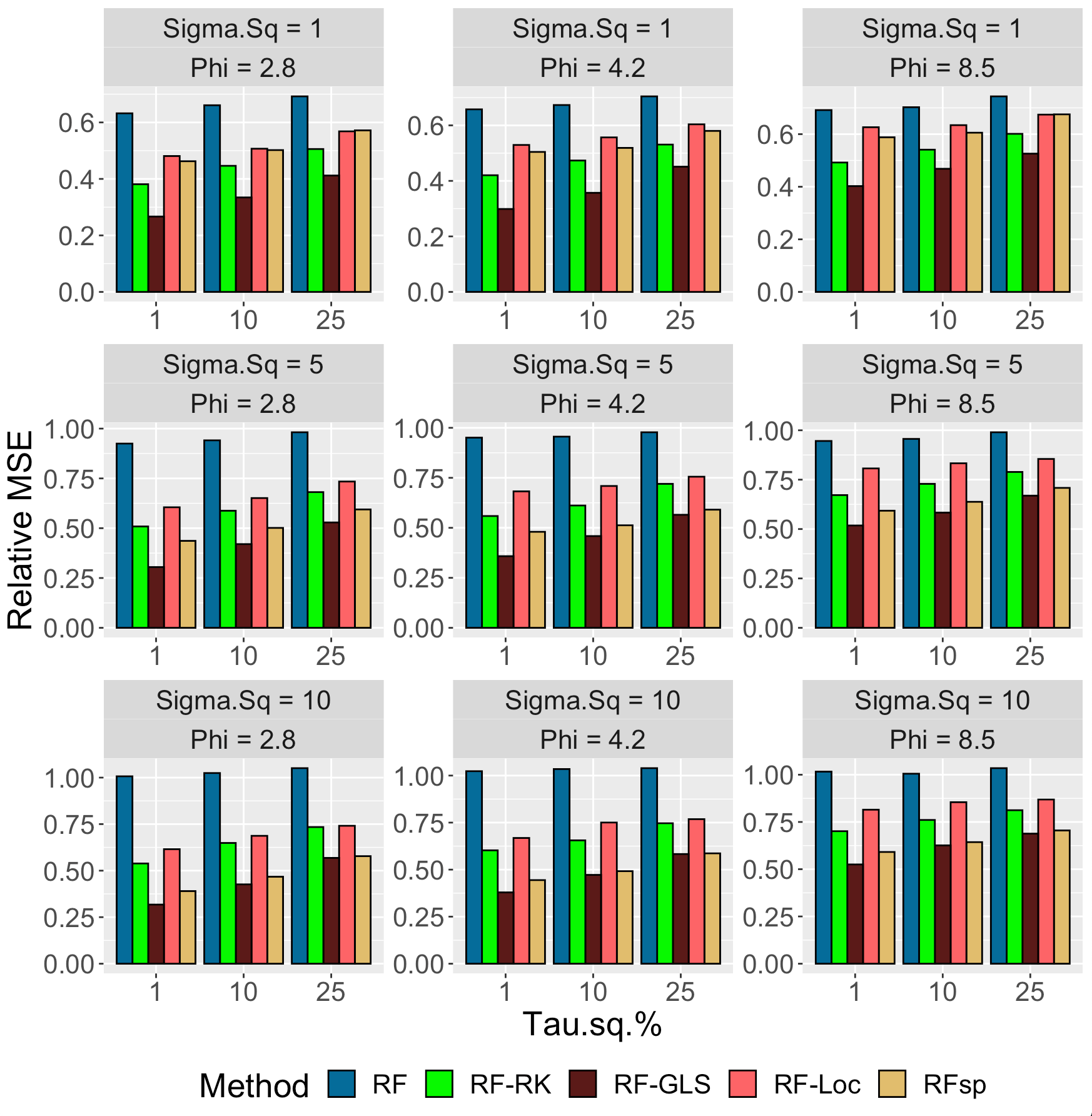}
        \caption{Prediction performance}
    \end{subfigure}
    \caption{Comparison between competing methods on (a) estimation and (b) spatial prediction when the mean function is $m = m_2$.}\label{Fig:friedman_performance}
\end{figure}

We observe that the RF-GLS performs at par with RF-GLS (Oracle) which assumes knowledge of the true spatial parameters (Supplementary Materials, \ref{sec:Oracle_comparison}). As in reality, we won't have knowledge of the true parameters, for the rest of the article we will be comparing the competing methods only with RF-GLS, where model parameters are unknown and estimated from RF residuals. We present the estimation and prediction results for $m = m_2$. The results for $m = m_1$ are similar, and are provided in the Supplementary Materials (\ref{sec:sin_comparison}).

The median MISE over 100 simulations for all the 27 setups are shown in Figure \ref{Fig:friedman_performance} (a). RF-GLS outperforms RF across all the scenarios demonstrating that exploiting the spatial information substantially improves estimation performance over the vanilla RF. Additionally, we also notice that as the strength of the spatial variation increases ($\sigs$ goes from $1$ to $10$), the ratio of MISE of RF and RF-GLS increases indicating a greater gain in terms of MISE.

\blue{We consider five candidates for evaluating prediction performance: RF, RF-RK \citep{fox2020comparing}, RF-GLS, RF-Loc (the 2-D spatial locations of the data are used as additional covariates in RF  to account for the spatial structure), and RFsp \citep{hengl2018random}. 
Among these, RF is the only one that does not use spatial information and only accounts for the mean. For RFsp, we used unordered pairwise Euclidean distances as additional covariates in RF to account for the spatial structure in the data.
The prediction performance are evaluated based on the Relative Mean Squared Error (MSE) for the test data. MSE measures the average squared difference between the estimated values and the actual value of the response. Relative MSE helps compare MSE across different simulation setups, while standardizing by the variance of the response.

$$
\text{MSE} = \frac{1}{n_{test}}\sum_{i = 1}^{n_{test}}(y_i - \hat{y}_i)^2; \text{Relative MSE} = \frac{\text{MSE}}{\frac{1}{n_{test}}\sum_{i = 1}^{n_{test}}(y_i - \bar{y}_{test})^2}
$$

Figure \ref{Fig:friedman_performance} (b) shows the median Relative MSE over 100 simulations for all the 27 setups. Expectedly, all methods performed better than RF, as unlike the other methods, RF doesn't use any spatial information. RF-GLS performs best or at par with the best across all settings.  

RF-Loc and RFsp belong to the class of approaches that add extra spatial covariates to RF. RF-Loc always performed worse than RFsp and RF-GLS. However, when the spatial variance $(\sigma^2)$ is high, RFsp performs comparably to that of RF-GLS, both of which outperform the others. 
For lower $\sigs$, RF-GLS significantly outperforms all the other methods. In this case, RFsp performs worse than RF-RK which is now the second best method.  

We conducted an in-depth study in Section \ref{sec:extra} to understand the performance of RFsp. We summarize the findings here. \cite{mentch2020getting} showed that for iid errors, adding additional noise covariates can improve prediction performance of RF for low (covariate-)signal-to-(random-)noise ratio (SNR). 
For high SNR, the trend was reversed and using extra noise covariates (when uncorrelated with the true covariates) did not help. For our setting of spatially correlated errors, akin to how noise covariates can be added to RF to explain random variation, RFsp adds distance-based covariates to explain spatial variation. The appropriate quantity to understand the performance of RFsp would be the 
(covariate-)signal-to-(spatial-)noise ratio $\mbox{SNR} = Var(m(\bX))/Var(w(\ell)) = Var(m(\bX))/\sigs$. This SNR is low when $\sigs$ is large and vice versa (Figure \ref{Fig:SNR_comparison}). 

RFsp adds $n$ pairwise distance-based covariates to RF to explain the spatial variation. 
If $D$ denotes the number of true covariates, then RFsp uses a total of $n+D$ covariates in RF and the probability to include a true covariate in the $M_{try}$ set of candidates for splitting a node becomes vanishingly small when $n\gg D$. When $\sigs$ is small (high SNR), the true covariates predominantly dictates the variation in the outcome but is rarely selected in RFsp (Figure \ref{Fig:friedman_variable}) resulting in its poor performance. In fact for high SNR, RFsp performs even worse than RF-RK which can estimate $m$ reasonably well in this setting despite ignoring the spatial dependence as the covariate signal dominates. For low SNR ($\sigs=10$), as the spatial contribution increases, RFsp (being naturally equipped to capture the spatial correlation in the data) performs comparably to RF-GLS (Figure \ref{Fig:friedman_performance}). For low SNR, RF-RK, ignoring the dominant spatial variation when estimating $m$,  performs worse than both RF-GLS and RFsp.

\begin{figure}[]
	\centering
	\includegraphics[height=2.7in]{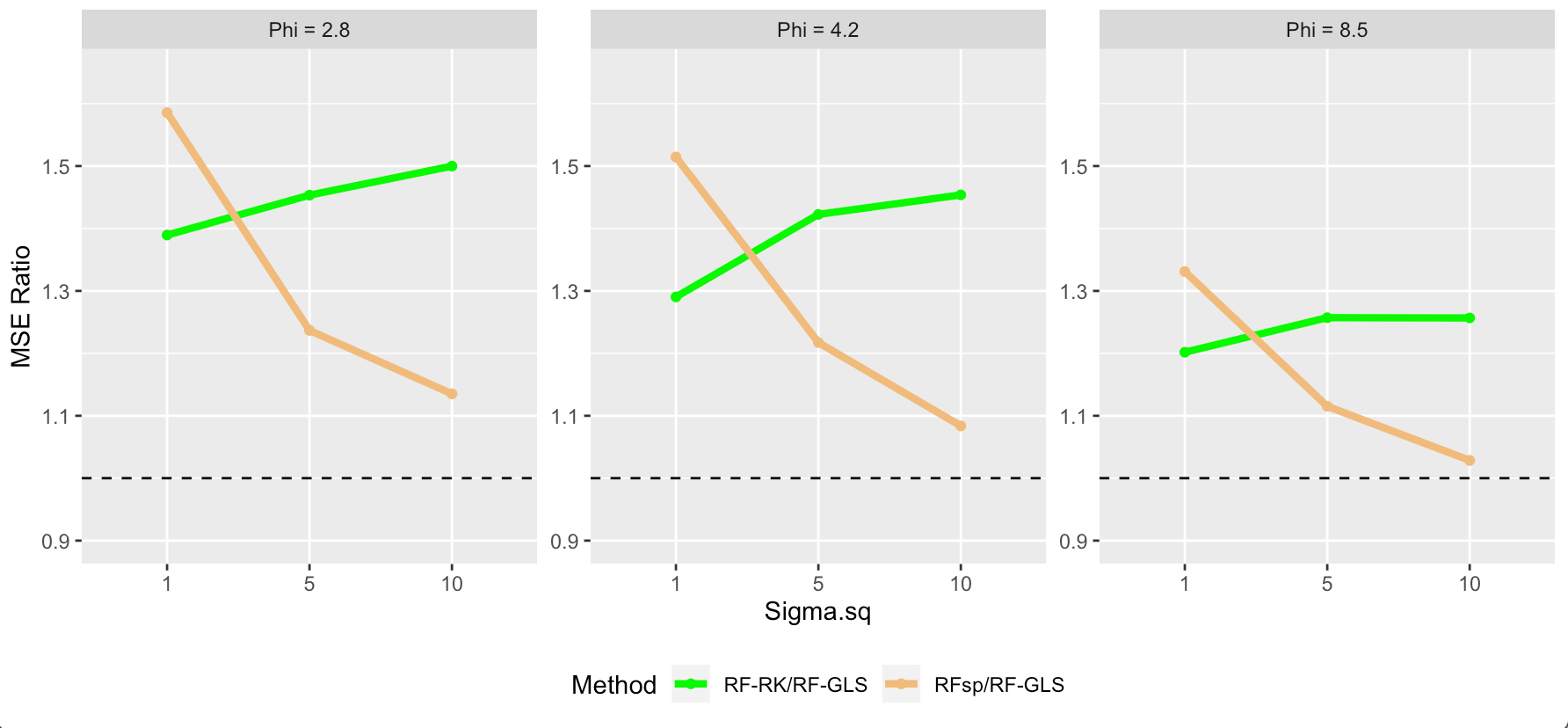}
	\caption{MSE ratio for RF-RK \& RF-GLS ({\tiny$\frac{\mbox{MSE(RF-RK)}}{\mbox{MSE(RF-GLS)}}$}) and RFsp \& RF-GLS ({\tiny$\frac{\mbox{MSE(RFsp)}}{\mbox{MSE(RF-GLS)}}$}) for $\tau^2 = 10\% \sigma^2$ when the mean function is $m = m_2$. }\label{Fig:friedman_MSE_ratio}
\end{figure}

RF-RK and RFsp outperforms each other in two  ends of the SNR spectrum. RF-GLS combines the strengths of both, using RF to estimate a non-linear covariate effect while parsimoniously modeling the spatial effect using a GP specified by only $2$ or $3$ parameters thereby avoiding introduction of a large number of spatial covariates. Consequently, RF-GLS produce comparable or better results to both RF-RK and RFsp across the SNR spectrum (Figure \ref{Fig:friedman_MSE_ratio}). 

Section \ref{sec:suppsim} of the Supplementary Materials file contains a number of additional simulation studies. These include results for larger sample size (Section \ref{sec:large}), mean function with more covariates (Section \ref{sec:highdim}), misspecified Gaussian Process smoothness (Section \ref{sec:matern_simul}), and misspecified entire spatial effect (not generated from a Gaussian Process, Section \ref{sec:smooth_simul}).  
For each study, and all choices of data generation parameters, 
RF-GLS performed as the best or comparably with the best method across all scenarios for both estimation and prediction.
}

\section{Proof of Consistency}\label{sec:outline}

 For studying RF-GLS with dependent data, we adopt the framework of consistency analysis of nonparametric regression (introduced in \cite{nobel1996histogram}, generalized in  \cite{gyorfi2006distribution}). In \cite{scornet2015consistency}, the authors also adopted this framework to prove consistency of RF for i.i.d. errors. After presenting an informal outline of the consistency argument in \cite{gyorfi2006distribution}, we provide a road map of how we extend different pieces of this argument for RF-GLS with dependent data, and highlight additional technical challenges that we resolved in this work. All formal proofs are provided in the Supplementary materials. 

 \cite{gyorfi2006distribution} consider a least squares estimator of the form 
$$m_n(, \Theta) \in arg \min_{f \in \mathcal{F}_n} \frac{1}{n} \sum_i \left[f(X_i) - Y_i \right]^2$$
where $\mathcal{F}_n$ is a carefully chosen, data-dependent function class which is large enough to control {\em approximation error} (i.e., how well the true function is estimated by the function class), and small enough to control {\em estimation error} (i.e., how far the estimate $m_n$ is from the representative of $\calF_n$ closest to $m$) 
in the sense that a Uniform Law of Large Number (ULLN) holds on this class. $m_n$ is consistent if both errors vanish asymptotically. 

In order to use powerful exponential inequalities which hold for classes of bounded functions, the proof of $\mathbb L_2$-consistency in \cite{gyorfi2006distribution} uses a standard truncation argument in probability theory with a diverging sequence of truncation thresholds $\{ \zeta_n \}$. They first show that if \emph{uniformly over the class $T_{\zeta_n} \mathcal{F}_n$ of truncated functions} in $\mathcal{F}_n$, 
\begin{enumerate}
    \item \emph{\underline{Approximation Error}} of $m$ by a data-driven class $\calF_n$ containing $m_n$ is small, i.e.,  
    $$\mathbb{E}\left[\inf_{f \in T_{\zeta_n} \mathcal{F}_n} \mathbb{E}_X \left[f(X) - m(X) \right]^2 \right] \to 0;
    $$
    \item \emph{\underline{Estimation Error}} is  small so that a ULLN holds over $\mathcal{F}_n$ for squared error loss, i.e. \\ $\mathbb{E} \left[\sup_{f \in T_{\zeta_n} \mathcal{F}_n} \left| \frac{1}{n} \sum_i (f(X_i) - Y_i))^2 - \mathbb{E} [f(X) - Y]^2 \right| \right] \rightarrow 0$; 
\end{enumerate}
then the truncated estimator $T_{\zeta_n}m_n$ is $\mathbb L_2$-consistent for $f$. Then they extended the consistency guarantee from truncated to the original estimators by showing that the {\em truncation error} vanishes under suitable tail decay assumptions on the error distribution.

In the analysis of RF-GLS, there are two main challenges. The error process $\epsilon_i$ is no longer an i.i.d process but a stochastic process capturing the dependence. Also, we work with a quadratic loss $(\bY-f(\bX))^\top \mathbf Q (\bY - f(\bX))$ instead of $\|\bY - f(\bX)\|^2$ where, $f (\mathbf X) -\mathbf Y = \left(f (X_1) -Y_1, f (X_2) -Y_2, \cdots, f (X_n) -Y_n \right)^\top$. 
Addressing these require non-trivial generalizations of each of the above pieces. 

\subsection{Consistency of quadratic loss optimizers in data-driven function classes under dependent errors}\label{sec:gyorfi} 
Our main technical statement (Theorem \ref{th:gyorfi}) is a generalization of   \cite[Theorem $10.2$]{gyorfi2006distribution} to the setting of dependent error processes and quadratic loss functions. 
Let $\mathcal{D}_n = \big\{(X_1, Y_1), \cdots, (X_n, Y_n)\big\}$ be the data where $Y_i=m(X_i) + \eps_i$. With randomness parameter $\Theta$, let $\mathcal{F}_n = \mathcal{F}_n(\mathcal{D}_n, \Theta)$ be a data-dependent class of functions. We will consider an optimal estimator $m_n \in \calF_n$ with respect to quadratic loss:
\begin{equation}\label{eqn:m_n}
		m_n \in 
		 \arg \min_{f \in \mathcal{F}_n} \frac{1}{{n}}( f (\mathbf X) -\mathbf Y)^\top \bQ( f (\mathbf X) -\mathbf Y).
\end{equation}		
This simply states that $m_n$ is the GLS estimate with respect to the working precision matrix $\bQ$ in the class $\calF_n$. This is analogous to the OLS assumption used in \cite{gyorfi2006distribution}. When $\calF_n$ is the class of piecewise constant functions on the partitions generated by a regression tree, $m_n$ is our GLS-style regression-tree estimate. Hence studying estimators of the form (\ref{eqn:m_n}) more generally suffices to prove consistency of GLS-style regression tree and henceforth of RF-GLS.

We now state a general technical result establishing $\mathbb{L}_2$-consistency of such GLS estimators $m_n$ under $\beta$-mixing (absolutely regular) error processes. This is a sufficiently large class of processes that includes spatial Mat\'ern GP and autoregressive time series as discussed in Sections \ref{sec:gp} and \ref{sec:ar}. The result is applicable beyond RF to more general nonparametric GLS estimators from dependent data using other suitable of function classes. 


\begin{theorem}\label{th:gyorfi}
	Let $\{\eps_i\}$ be a stationary $\beta$-mixing process satisfying Assumption \ref{as:mix}, and the matrix $\bQ$ satisfies Assumptions \ref{as:chol} and \ref{as:diag}. Let $m_n(.,\Theta) : \mathbb{R}^D \to \mathbb{R}$ denote a quadratic-loss optimizer (with respect to $\bQ$) of the form \eqref{eqn:m_n} in a data-dependent function class $\calF_n$. If $m_n$ and $\calF_n$ satisfies the following conditions:\\
	\noindent \textbf{(C.1)} (Truncation error) $\exists$ $\{\zeta_n\}$ such that $\lim_{n \to \infty} \zeta_n = \infty$ and $\zeta^2_n/n \to 0$, such that we have, 
$$\lim_{n \to \infty}\mathbb{E}\max_i \left[m_n(X_i) - T_{\zeta_n}m_n(X_i) \right]^2 = 0$$
	\noindent \textbf{(C.2)}  (Approximation error) 
$\lim_{n \to \infty} \mathbb E_{\Theta} \left[\inf_{f \in T_{\zeta_n}\mathcal{F}_n}  \mathbb{E}_{X} |f(X) -m(X) |^2 \right] =0$\\
\noindent \textbf{(C.3)} (Estimation error)
		Let $\dot X_i$, $\dot \eps_i$ and $\dot Y_{i} = m(\dot X_{i})+ \dot\eps_{i}$ be such that $\dot \calD_n=\{(\dot X_i, \dot Y_i) | i=1,\ldots,n\}$ be identically distributed as $\calD_n$ but independent of $\calD_n$. Define $f(\dot \bX)$ and $\dot \bY$ similar to $f(\bX)$ and $\bY$. Then, 
		we have for all arbitrary $L > 0$
		$$
			\lim_{n \to \infty} \mathbb{E} \left[ \sup_{f \in T_{\zeta_n}\mathcal{F}_n}  |\frac{1}{{n}}(f(\bX)-\bY)^\top\bQ(f(\bX)-\bY)  - \mathbb{E}\frac{1}{{n}}(f(\dot\bX)-\dot\bY)^\top\bQ(f(\dot\bX)-\dot\bY)|\right]=0.
			$$
Then we have 
$$
		\begin{aligned}
		\lim_{n \to \infty}\mathbb{E}\left[\mathbb{E}_{X} ( m_n( X, \Theta) - m( X))^2)\right]&=0, \mbox{ and }\\
		\lim_{n \to \infty} \mathbb{E}_{X} ( \bar{m}_n( X) - m( X))^2 &= 0; 
		\end{aligned}
		$$
		where $\bar{m}_n( X) =\mathbb{E}_{\Theta} m_n( X, \Theta)$ and $X$ is a new sample independent of the data.
\end{theorem}

The proof 
is deferred to Section \ref{sec:pf-l2}. Theorem \ref{th:gyorfi} is a result of independent importance as it is a general statement on  $\mathbb L_2$ consistency of a wide class of GLS estimates under $\beta$-mixing dependent errors. Besides data-driven-partitioning-based estimates like RF or RF-GLS, it can be used to study properties of histograms, kernel-density estimates, local polynomials, etc., under $\beta$-mixing error processes. 
The second part of the theorem states that the consistency also holds for an ensemble estimator $\bar m_n$ that averages many such estimates $m_n$ each specified with random parameters $\Theta$. This will be used to show consistency of the RF-GLS forest estimate subsequent to showing consistency of each RF-GLS tree estimate.

\subsection{Approximation Error}\label{sec:approx}
The condition \textbf{C.2} of asymptotically vanishing approximation error ensures that as sample size increases, the growing class of approximating functions (e.g. piece-wise constant functions in the case of regression trees) is rich enough to approximate the target function. In earlier works on consistency of RF, approximation error was controlled under a stringent  assumption of vanishing diameter of leaf nodes. \cite{scornet2015consistency} replaced this by a condition that the variance of $Y$ (or equivalently, variation of $m$) within a leaf node of a regression-tree vanishes asymptotically, and verified this condition for RF. 
There are two steps to show this. 

\noindent (i) Establish a {\em theoretical or population-level split-criterion} --- an asymptotic limit of the empirical split-criterion used in practice, such that variation of $m$ in the leaf-nodes of a hypothetical regression tree generated using the theoretical criterion is small. 

\noindent (ii) 
Establish stochastic equicontinuity of the empirical split-criterion, such that if two set of qualifying splits $\bZ^{(1)}$ and $\bZ^{(2)}$ are close, their corresponding empirical split-criterion values are close, irrespective of the location of the splits. 

For our RF-GLS trees, 
the partitioning is driven by the DART criterion (\ref{DART_def}) and is different from the CART (\ref{eq:cart_local}) criterion used in the RF trees. Since the GLS loss and estimator involves the matrix $\bQ$, they are not available in simple scalar expressions unlike the OLS loss (sum of squares) and estimator (mean response within a node). So we address a number of technical challenges for  steps (i) and (ii) that do not appear in the analysis of RF. 

For (i), Lemma \ref{lem:beta} and Theorem \ref{lemma:DART-theoretical},  establishes that the DART split-criterion remarkably has the same limit of as that for the CART criterion. Hence, variation of $m$ in trees generated by this theoretical criterion is controlled in the same way as for RF.

For (ii), we require an entirely new and involved proof of equicontinuity for the DART-split criterion of RF-GLS as the previous arguments of \cite{scornet2015consistency} do not immediately generalize for RF-GLS loss function (\ref{DART_def}). We discuss the new contributions in Section \ref{sec:equi}.

\subsubsection{Equicontinuity of the split criterion}\label{sec:equi}
Equicontinuity of the CART-split criterion $\frac{1}{n_P}[\sum_{i=1}^{n_P}(Y_i^P - \bar{Y}^P)^2 - \sum_{i_r=1}^{n_R}(Y_{i}^R - \bar{Y}^R)^2 - \sum_{i_l=1}^{n_L}(Y_{i}^L - \bar{Y}^L)^2]$ was the center-piece of the theory in \cite{scornet2015consistency}, requiring involved but elegant arguments on the geometry of splits. Since the CART criterion only concerns the parent node to be split and its potential two child nodes,  the equicontinuity essentially boiled down to showing closeness of the respective means and variances of these 3 nodes for the two sets of splits. These 3 scalar mean and variance differences are functions of the difference in volumes of the respective nodes   
 which goes to zero uniformly as the splits come closer.



For RF-GLS, to update each node, 
the entire set of node representatives get updated via the GLS-estimate (\ref{DART_update}) which is analytically intractable due to the matrix inversion. Also, the DART-split criterion (\ref{DART_def}) 
\begin{small}$$\frac{1}{n} \Bigg[\left(\mathbf{Y} - \mathbf{Z}^{(0)}\bm{\hat{\beta}}_{GLS}(\mathbf Z^{(0)}) \right)^\top \bQ\left(\mathbf{Y} - \mathbf{Z}^{(0)}\bm{\hat{\beta}}_{GLS}(\mathbf Z^{(0)}) \right)
- \frac{1}{n} \left(\mathbf{Y} - \mathbf{Z}\bm{\hat{\beta}}_{GLS}(\mathbf Z) \right)^\top \bQ\left(\mathbf{Y} - \mathbf{Z}\bm{\hat{\beta}}_{GLS}(\mathbf Z) \right) \Bigg]
$$
\end{small}
is a quadratic form of the plugged-in GLS-estimate and thus a function of the representatives of all nodes and not just the 3 nodes as in RF. 
Our equicontinuity proof is built on viewing GLS predictions as oblique projections on the design matrices corresponding to the splits. 

We consider two scenarios: $\textbf{R1}$  -- where for at least one set of split, both potential child nodes have substantial volumes, and $\textbf{R2}$  -- where for each set of split, one potential child node have ignorable volume. 
Under $\textbf{R1}$, all nodes for at least one set of split have substantial representation ensuring that the gram matrix has norm bounded away from zero and equicontinuity can be established using perturbation bounds on orthogonal projection operators \citep{chen2016perturbation}. Under $\textbf{R2}$, this will not be the case as for both set of splits one child node will have small volume. Instead, we first show that difference of the DART-split criterion at the previous level of (parent) splits is small using the same perturbation argument as the parent node volumes are bounded away from zero. Subsequently we argue that the creation of the children nodes do not change the split criterion substantially as one of the child nodes is essentially empty. 
This new matrix-based proof for equicontinuity also circumvents the need to invoke mathematical induction as required in \cite{scornet2015consistency}. 

\begin{proposition}[Equicontinuity of empirical DART-split criterion]
	\label{lemma:equicontinuity}
	 Under Assumptions \ref{as:chol}, \ref{as:diag}, and \ref{as:tail}(b) and \ref{as:tail}(c), the DART split criterion (\ref{DART_def}) is stochastically equicontinuous with respect to the set of splits. 
\end{proposition} 

The proof is involved, and is deferred to Section \ref{sec:pf-approx} where the Proposition is more technically phrased using additional notation on splits. Subsequent to proving equicontinuity, we can prove the following result of approximation error

\begin{proposition}\label{lem:approx} Let $\calF_n = \calF_n(\Theta)$ is the set of all functions $f:[0,1]^D \to \mathbb{R}$, piecewise constant on each
cell of the partition obtained by an RF-GLS tree. Then, under Assumptions \ref{as:mix} - \ref{as:add}, the class $T_{\zeta_n}\calF_n$ satisfies the approximation error condition \textbf{(C.2)}.
\end{proposition}




\subsection{Estimation Error}\label{sec:estimation}
Theorem \ref{th:gyorfi} shows that for GLS estimators like (\ref{eqn:m_n}), one needs to control the quadratic form estimation error (ULLN \textbf{C.3}). 
A common technique for proving ULLNs under dependence is \\
\noindent (i) to prove an analogous ULLN for i.i.d. error processes,  and then \\
\noindent (ii) use mixing conditions to generalize the result for the dependent process of interest. 

Both steps require addressing new challenges for RF-GLS which we discuss in the next two subsections. 

\subsubsection{Cross-product function classes}\label{sec:crossprod}
For step (i), it is difficult to directly state an i.i.d. analogue of the ULLN \textbf{C.3} as it is not possible to have an error process $\{\eps_i^*\}$ which is simultaneously i.i.d., satisfies $\eps^*_i \sim \eps_i$ and  $E(\by^\top\bQ\by)=E(\by^{*\top}\bQ\by^*)$. To see this, simply note that as $Cov(\eps^*_i,\eps^*_{i-j})=0 \neq Cov(\eps_i,\eps_{i-j})$, with $y_i^*=m(X_i)+\eps_i^*$ we will have $E(q_{i,i-j} y_i y_{i-j}) \neq E(q_{i,i-j} y^*_i y^*_{i-j})$ where $\bQ=(q_{ii'})$. 
Instead, we create separate i.i.d. analogues of ULLN for each term in the expansion of the quadratic form, i.e., both the 
squared error and the cross-product terms.
These ULLN are stated as condition \textbf{C.3.iid} in the Supplementary materials. 
Condition \textbf{C.3.iid}(a) for the squared (diagonal) terms is the standard squared error ULLN using the i.i.d error processes, 
 and has been proved in \cite{gyorfi2006distribution} Theorem 10.2 generally, and in \cite{scornet2015consistency} in particular for RF.

For the cross-product (off-diagonal) terms, the ULLN is stated in Condition \textbf{C.3.iid}(b) of the Supplement, and is to our knowledge a novel strategy. \blue{The analysis of RF under iid settings in \cite{scornet2015consistency} only uses a squared error loss which does not involve any cross-product terms and hence did not need to prove any ULLN for cross-product terms. We construct separate ULLN for the cross-product terms $\sum_i q_{i,i-j}\eps_i\eps_{i-j}$ for each lag $j=1,\ldots,q$.} As mentioned earlier, use of the univariate copy $\{\eps^*_i\}$ will not allow to generalize to the corresponding term in \textbf{C.3}. This is because $\eps^*_i \perp \eps^*_{i-j}$ but $\eps_i$ and $\eps_{i-j}$ are correlated. 
Instead, \blue{for each lag $j=1,\ldots,q$, we create bivariate i.i.d. sequences $(\tilde \eps_i,\ddot \eps_{i-j})$ 
that are identically distributed as the joint (bivariate) distribution of the pairs $(\eps_i,\eps_{i-j})$, but are independent over $i$. 
We thus exploit the banded nature of the working precision matrix $\bQ$ to prove i.i.d ULLNs for cross-product terms at each of the $q$ lags. 
Combining these with the ULLN for the squared terms   
 gives us a ULLN for the entire quadratic form. }

Formulation and establishing this cross-product ULLN \blue{using lag-specific bivariate i.i.d copies} is a new contribution and is of independent importance for establishing vanishing limits of any estimation error that involves interaction terms. 
We prove this ULLN in Proposition \ref{lemma:cross_prod_indep} of the Supplement by showing that  
cross-product function classes has the same concentration rate as that for squared-error function classes with respect to the 
random $\mathbb L_p$ norm entropy number.  
	
\subsubsection{ULLN for $\beta$-mixing processes}\label{sec:ulln_dep}
For step (ii), we need to go from Conditions \textbf{(C.3.iid)}(a) and \textbf{(C.3.iid)}(b) for i.i.d processes to their analogs for dependent error processes, which would then immediately establish 
 \textbf{C.3}. 
 It has been shown that the mixing condition of the stochastic process determines the assumptions required on the class $\mathcal{F}_n$ 
 \citep{dehling2002empirical}. If we look at the ``hierarchy" of dependence structures, strong-mixing or $\alpha$-mixing \citep{bradley2005basic}, is one of the broadest family of dependent processes accommodating dependent structures ``furthest" from independence. However, existing ULLN results for $\alpha$-mixing processes require the class of functions in $\mathcal{F}_n$ to be Lipschitz continuous  \citep{dehling2002empirical}. As regression-tree estimates are inherently discontinuous due to nature of discrete partitioning, this will not be satisfied here. 
 
 Hence we focus on absolutely regular or $\beta$-mixing process. This class of mixing processes is rich enough to include a number of commonly used spatial or time-series structures as discussed in the examples of Section \ref{sec:examples}. Our main challenge here is that no existing ULLN for $\beta$-mixing processes apply to the class of functions on partitions of the RF-GLS trees. 
ULLN for Glivenko-Cantelli classes under $\beta$-mixing was established in \cite{nobel1993note}. Similar results have been established for a class of $\tilde \phi$-mixing processes in \cite{peligrad2001note}. Both results, do not need any convergence rate on the mixing coefficients, but require the class $\calF_n$ to have an envelope (dominator) $F$ (free of $n$). For data-driven partitioning based estimates like RF or RF-GLS trees such uniform envelopes are not available (as the envelop is the truncation threshold $\zeta_n \to \infty$). Instead we propose a ULLN for dependent processes that uses a weaker assumption of a $n$-varying envelop with a moment-bound. 
The proof is deferred to 
Section \ref{sec:pf-estimn}.
\begin{proposition}
	\label{lemma:dependent_data}
	(A general ULLN for $\beta$-mixing processes) Let $\{U_i\}$ be an $\mathbb{R}^d$-valued stationary $\beta$-mixing process. Let $\mathcal G_n (\{U_i\}_{i = 0}^{n-1})$ be a class of functions $\mathbb{R}^d \to \mathbb{R}$ with envelope $G_n \geq \sup_{g \in \calG_n} |g|$, such that $G_n$ is ``uniformly mean integrable", i.e.
	\begin{equation}\label{eq:ui}
	\lim_{C\to\infty}\lim_{n \to \infty }\mathbb{E}\frac{1}{{n}} \sum_i |G_n(U_i)|\mathds{I}(|G_n(U_i)|> C ) = 0.
	\end{equation}
	Let $\{U_i^*\}$ be such that $U_i^*$ is identically distributed as $U_i, \: \: \forall \: i$ and $U_i^* \independent U_j^*; \: \: \forall i \neq j$. 
	Then, $\sup_{g \in \mathcal{G}_n}\Big|\frac{1}{{n}}\sum_i (g(U_i^*) - \mathbb{E}g(U_i^*))\Big| \overset{\mathbb{L}_1}{\to} 0
	\implies 
	\sup_{g \in \mathcal{G}_n}\Big|\frac{1}{{n}}\sum_i (g(U_i) - \mathbb{E}g(U_i))\Big| \overset{\mathbb{L}_1}{\to} 0.
	$
\end{proposition}

Proposition \ref{lemma:dependent_data} ensures that ULLN for i.i.d. errors is enough to generalize to $\beta$-mixing error processes as long as the function classes are contained within a sequence of {\em `` mean uniform integrable"} envelopes in the sense of (\ref{eq:ui}). Next, we show that the ULLN holds for RF-GLS trees under a $(2+\delta)^{th}$ moment assumption that is sufficient for mean uniform integrability. 

\begin{proposition}[Estimation error for RF-GLS]\label{sec:prop-rf-ulln}
Let $\calF_n = \calF_n(\Theta)$ is the set of all functions $f:[0,1]^D \to \mathbb{R}$, piecewise constant on each
cell of the partition obtained by an RF-GLS tree. If any subset $\tilde F_n \subseteq \calF_n$ satisfies the following condition:\\
\noindent \textbf{(C.4)} (Moment bound:) $\exists$ an envelope $F_n \geqslant \sup_{f \in \tilde \calF_n} |f|$, such that 
$\lim_{n \to \infty }  \mathbb{E} \frac{1}{{n}} \sum_i |F_n(X_i)|^{2+\delta} < \infty$ for some $\delta > 0$,\\ then $T_{\zeta_n}\tilde \calF_n$ satisfies the ULLN \textbf{(C3)} for $\beta$-mixing error processes, and under Assumption \ref{as:chol}.
\end{proposition}
The proof is in Section \ref{sec:pf-estimn}. The $(2+\delta)^{th}$ moment condition \textbf{C.4} is easier to verify for RF or RF-GLS as discussed in Corollaries \ref{cor:rfgls} and \ref{cor:rf}.  

\subsection{Proof of Theorem \ref{th:main}}
Equipped with Theorem \ref{th:gyorfi}, we can prove Theorem \ref{th:main} by showing that RF-GLS meets the conditions \textbf{C.1} - \textbf{C.3}. 
Proposition \ref{lem:trunc} in the Supplement shows that the truncation error condition \textbf{C.1} is met for any $\zeta_n$ satisfying the scalings of Assumption \ref{as:tail}(a) required for proving the approximation and estimation error conditions. 

We have already shown conditions \textbf{C.2} (Proposition \ref{lem:approx}) and \textbf{C.3} (Proposition \ref{sec:prop-rf-ulln}) for two separate choices of function classes.
The last step of the proof is choosing the function class that satisfies both conditions. 
Let $\calF_n=\calF_n(\Theta)$ be the set of all functions $f : [0,1]^D \to \mathbb{R}$ piece-wise constant on each cell of the partition $\mathcal{P}_n(\Theta)$ created by an RF-GLS tree with data $\calD_n$ and randomization $\Theta$. We have already shown in in Proposition \ref{lem:approx} that $\calF_n$ satisfies \textbf{C.2}. 
To apply Proposition \ref{sec:prop-rf-ulln}, $\calF_n$ also needs to satisfy the moment condition of \textbf{C.4}. 
As 
$T_{\zeta_n} \calF_n$ is only bounded by $\zeta_n$ which goes to $\infty$, clearly $\calF_n$ will not satisfy \textbf{C.4} . 

We carefully carve out a subclass $\tilde \calF_n \subseteq \calF_n$ which is still wide enough to satisfy the approximation error condition \textbf{(C.2)}, while satisfying the additional restriction \textbf{(C.4)}.
For a given partition $\mathcal{P}_n(\Theta)$, we define $\tilde\calF_n$ as follows: 

\begin{equation}\label{eq:class}
\tilde\calF_n = \tilde\calF_n(\Theta)= \{m_n\} \cup \bigg\{\cup_{\mathbf{x}_{\mathcal{B}} \in \mathcal{B} \in  \mathcal{P}_n(\Theta)} \sum_{\mathcal B \in \mathcal{P}_n(\Theta)}m(\mathbf{x}_{\mathcal B})\mathds{I}(\bx \in \mathcal B)\bigg\}  \subseteq \mathcal{F}_n(\Theta).
\end{equation}
    
Since by construction of RF-GLS, $m_n$ is the optimizer over a much larger set $\mathcal{F}_n(\Theta)$, trivially $m_n$ is also the optimizer in $\tilde\calF_n$. The first step of the proof of Proposition \ref{lem:approx} makes it evident why Condition \textbf{C.2} will also hold for this smaller class $\tilde \calF_n$. 
To apply Proposition \ref{sec:prop-rf-ulln} and show Condition \textbf{C.3}, the final piece is to show that the condition \textbf{(C.4)} is satisfied by $\calF_n$. Since apart from $m_n$, $\calF_n$ consists of functions that are bounded by $M_0$, we can have the envelope to be $F_n=|m_n| + M_0$. Hence, for Condition \textbf{(C.4)} to hold, it is enough to show $\lim_{n \to \infty} \frac 1n \sum_i \mathbb E |m_n(X_i)|^{2+\delta} < \infty$ which is an assumption of the Theorem. \qed

\section{Discussion}\label{sec:disc}
We considered non-linear regression function estimation in the spatial mixed model and developed a random forest method (RF-GLS) for estimating the non-linear covariate effect, while still modeling the spatial effect using Gaussian Processes, as is conventional. Retaining the GP framework facilitates parsimonious encoding of structured spatial dependence, and conducting all traditional spatial tasks like kriging (prediction at a new location), and recovery of the latent spatial random effect surface. We show in Section \ref{sec:sim} how these advantages of RF-GLS manifest into superior estimation performance over naive RF that does not use any spatial information, and superior predictive performance over many existing spatial random forest methods. 

Our method RF-GLS, more generally, can be used for functional estimation under many types of dependence. 
RF-GLS uses the same fundamental principle that generalizes OLS to GLS. 
We show how adapting the concept of GLS in random forests synergistically mitigates all the issues encountered in naive application of RF to dependent settings. While simple in principle, RF-GLS algorithm differs inherently from RF, by optimizing globally (across all nodes) for each split, 
to account for dependence across all data points. We show RF is a special case of RF-GLS with an identity working correlation matrix, and is substantially outperformed by RF-GLS for both estimation and prediction under dependence.  

We present a thorough theoretical study establishing consistency of RF-GLS under a general assumption of $\beta$-mixing dependence. In particular, we establish consistency of function estimation by RF-GLS for the spatial non-linear mixed-model using the ubiquitous Mat\'ern GP. We also establish consistency of RF-GLS for functional estimation under auto-regressive time-series errors. 
Finally, as a byproduct of the theory, 
we also establish consistency of RF for dependent settings, which to our knowledge, is the first such result. 

The theoretical results required involved proofs to address the challenges of accommodating dependent error processes and use of a quadratic form loss for node-splitting. In the process, we  developed a number of tools of independent importance. The general result (Theorem \ref{th:gyorfi}) on $\mathbb L_2$ consistency of data-driven partitioning-based GLS estimates under $\beta$-mixing dependence extends the analogous result in \cite{gyorfi2006distribution} Theorem 10.2 which was for i.i.d. settings and OLS estimates. This will be useful for studying other function estimators like local polynomials under dependence. Proposition \ref{prop:dep-square-crossloss} establishes random-norm entropy number concentration bounds for general cross-product function classes, which can find its use in establishing ULLN for any class of functions containing interaction terms. Finally, Proposition \ref{lemma:dependent_data} proposes a general ULLN for $\beta$-mixing processes requiring less restrictive assumptions than existing results.  


For future work, extension to multivariate outcomes is an important direction. The spatial community is increasing shifting towards multivariate analysis, as GIS systems are empowered to collect data on many variables. Multivariate extension of RF-GLS can leverage the rich literature of multivariate cross-covariance functions for GP \citep[see][for a review]{genton2015cross}. 
When working with a very large number of variables $p$, a potential computational roadblock for RF-GLS will be evaluation of the $np \times np$ Cholesky factor. 
Strategies like graphical models may need  to be adopted to enforce sparsity and effectuate computational speedup. 

\bibliographystyle{asa}
\bibliography{ref_revision}

\pagebreak

\section*{Supplementary Materials for ``GLS-style Random Forests for Dependent Data"}
\begin{center}Arkajyoti Saha, Sumanta Basu, Abhirup Datta
\end{center}
\renewcommand\theequation{S\arabic{equation}}
\renewcommand\thesection{S\arabic{section}}
\setcounter{equation}{0}
\setcounter{section}{0}

\section{Additional Simulation Results}\label{sec:suppsim}
In this section, we provide additional details of the simulation studies presented in the main manuscript, as well as present some new illustrations for  performance comparison between RF-GLS and other state-of-the-art methods for estimation and prediction under varied scenarios. 

\subsection{Simulation details}\label{sec:simdetails}
For the experiments in Section \ref{sec:sim}, we simulated data from the spatial non-linear mixed model of (\ref{eq:spnlmm}) where $w(\ell)$ is an exponential GP on a two-dimensional spatial domain $[0,1]^2$. We varied all the covariance parameters:  spatial signal strength $\sigma^2 \in \{1, 5, 10\}$   
and spatial correlation strength $\phi \in \{3/(0.25\sqrt{2}), 3/(0.5\sqrt{2}), 3/(0.75\sqrt{2}) \}$. The three choices of $\phi$ correspond to the spatial correlation decaying to approximately $0.05$ at respectively $1/4$, $1/2$ and $3/4$ of the maximum inter-site distance. The covariates are simulated independently and uniformly from $[0,1]^D$. 
The i.i.d. error process $\eps^*(\ell) \sim N(0,\tau^2)$ where $\tau^2 = \mu \sigs$, $\mu \in \{0.01,0.1,0.25\}$. We perform the simulations for 100 times, for each of the $27$ combinations of the parameter triplet $(\sigma^2, \phi, \tau^2)$. We consider the following choices for the true mean function $m$:

\begin{enumerate}
	\item $m_1(x) = 10\sin(\pi x)$
	\item $m_2(\mathbf x) = (10\sin(\pi x_1 x_2) + 20(x_3 - 0.5)^2 + 10x_4 + 5x_5)/6$  \citep[Friedman function, ][]{friedman1991multivariate}.
\end{enumerate}

For fitting the forest estimators, we fix number of trees $n_{tree} = 100$; minimum cardinality of leaf nodes $t_c = 5$ and
number of features to be considered at each split $M_{try} = \min \{1, \left[\frac{D}{3}\right] \}$ where $D$ is the dimension of the covariate space. 
Along with function estimation, we also focus on prediction performance at new set of locations, which is often the primary goal in spatial analysis. We simulate a total of $ n = 250$ locations and use a $90\% - 10\%$ test-train split. 
The splitting is done by first dividing the spatial domain i.e. $[0,1]^2$ into $10 \times 10$ equal square boxes with 0.1 unit length sides. We randomly choose one box from each row and column (i.e. total $10\%$ of the boxes) and keep the data with spatial locations within those boxes as test/holdout data. This strategy ensures that entire regions are left hold out for testing instead of just random locations which may be proximal to other locations of data points in the training set. 

We use the training data for evaluating the estimation performance of the methods under consideration. This fitted model is also used for prediction performance at test locations.

In order to evaluate the MISE, for $m = m_1$, we generate $n_0=1000$ equally spaced points (at $1/1000$ distance) on $[0,1]$; for $m = m_2$, we randomly generate $n_0=1000$ points with Latin hypercube sampling in $[0,1]^5$. 

\subsection{Performance comparison between RF-GLS and RF-GLS  (Oracle)}\label{sec:Oracle_comparison}
In Figure~\ref{Fig:oracle_comparison} we show that in the practical implementation of RF-GLS, where the covariance parameters are unknown and estimated from classical RF residual performs comparably to RF-GLS (Oracle), where the original parameters are known and used in model fitting. 

\begin{figure}[H]
	\centering
	\begin{subfigure}[t]{0.5\textwidth}
		
		\centering
		\includegraphics[height=3.2in]{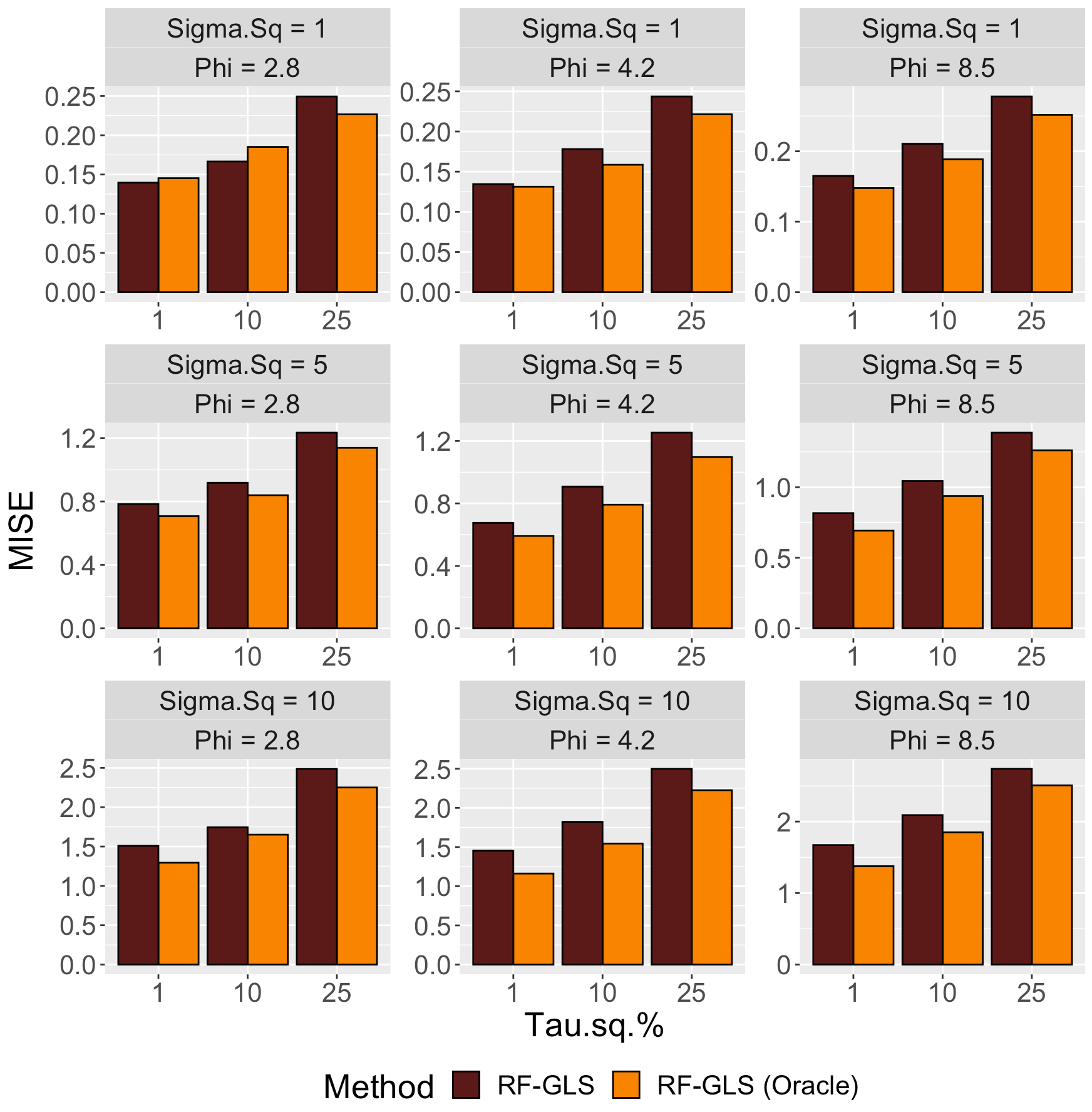}
		\caption{$m= m_1$}
	\end{subfigure}%
	~ 
	\begin{subfigure}[t]{0.5\textwidth}
		
		\centering
		\includegraphics[height=3.2in]{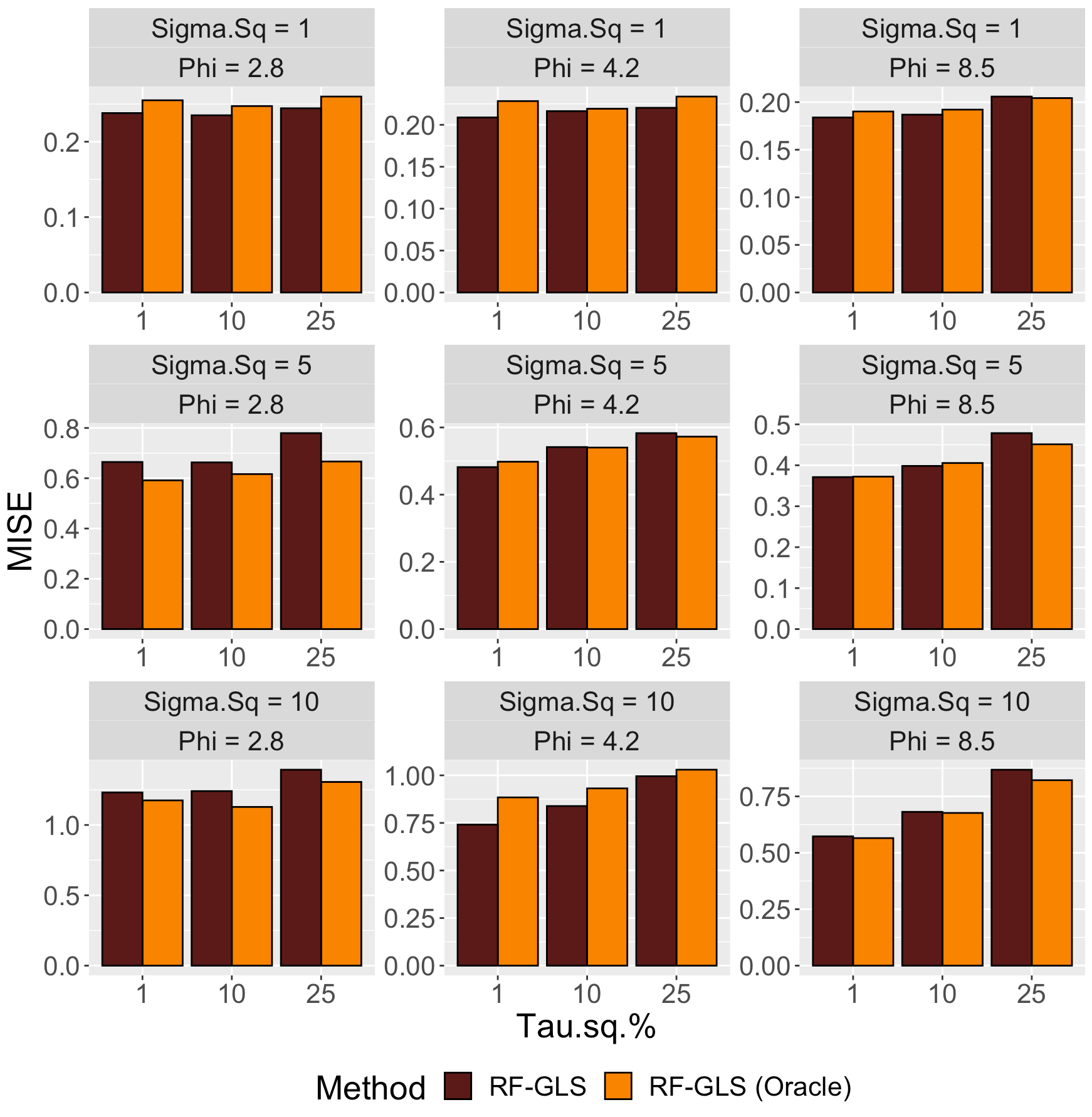}
		\caption{$m= m_2$}
	\end{subfigure}
	\caption{Comparison between RF-GLS and RF-GLS (Oracle) on estimation when mean function is (a) $m = m_1$ and (b) $m = m_2$.}\label{Fig:oracle_comparison}
\end{figure}

\subsection{Performance comparison for $m = m_1$}\label{sec:sin_comparison}
Here, we show the estimation and prediction performance comparison for $m = m_1$. In this scenario, RF-GLS performs better than the competing methods in all the scenarios under consideration. 

\begin{figure}[H]
	\centering
	\begin{subfigure}[t]{0.5\textwidth}
		
		\centering
		\includegraphics[height=3.2in]{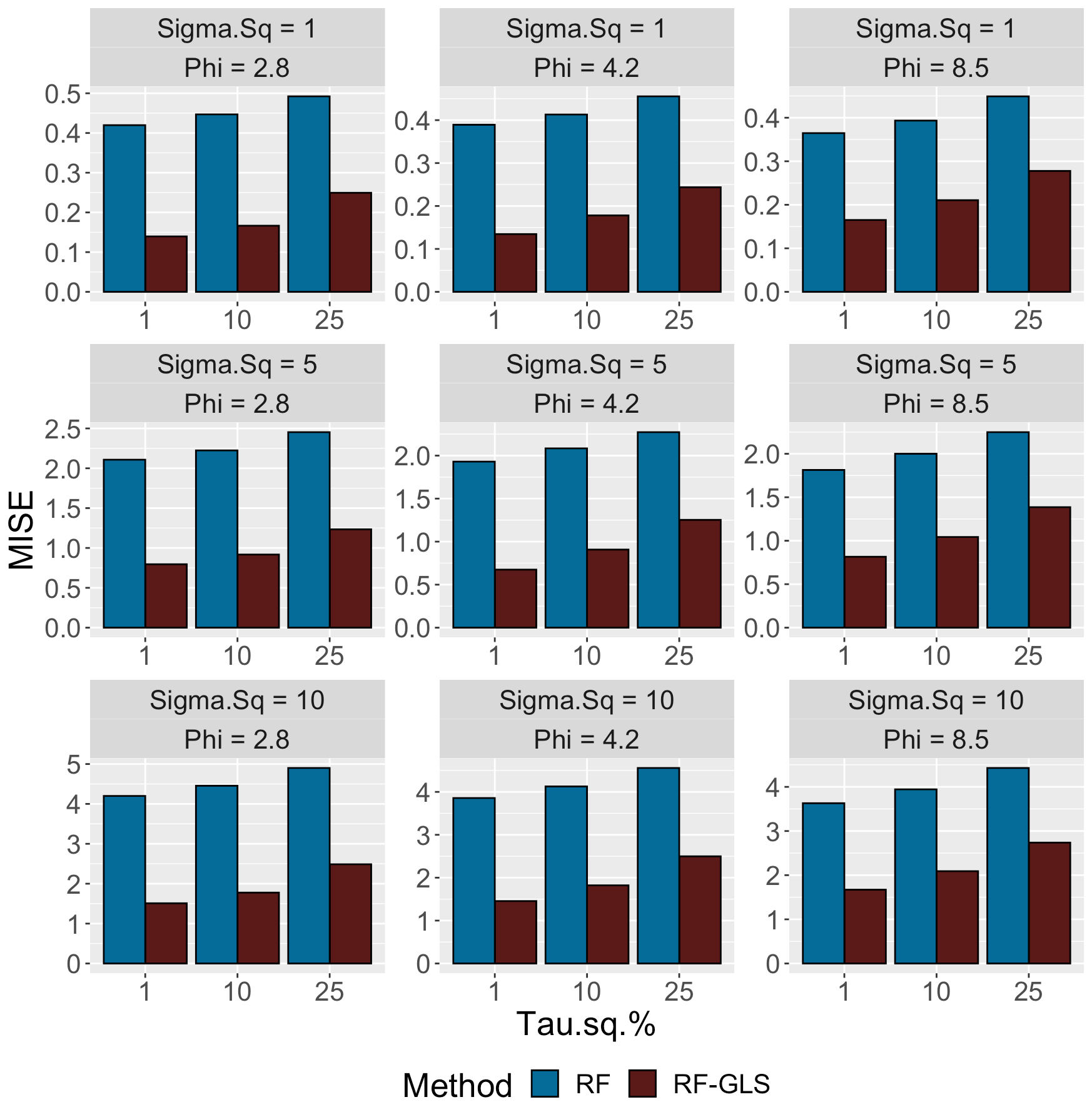}
		\caption{Estimation performance}
	\end{subfigure}%
	~ 
	\begin{subfigure}[t]{0.5\textwidth}
		
		\centering
		\includegraphics[height=3.2in]{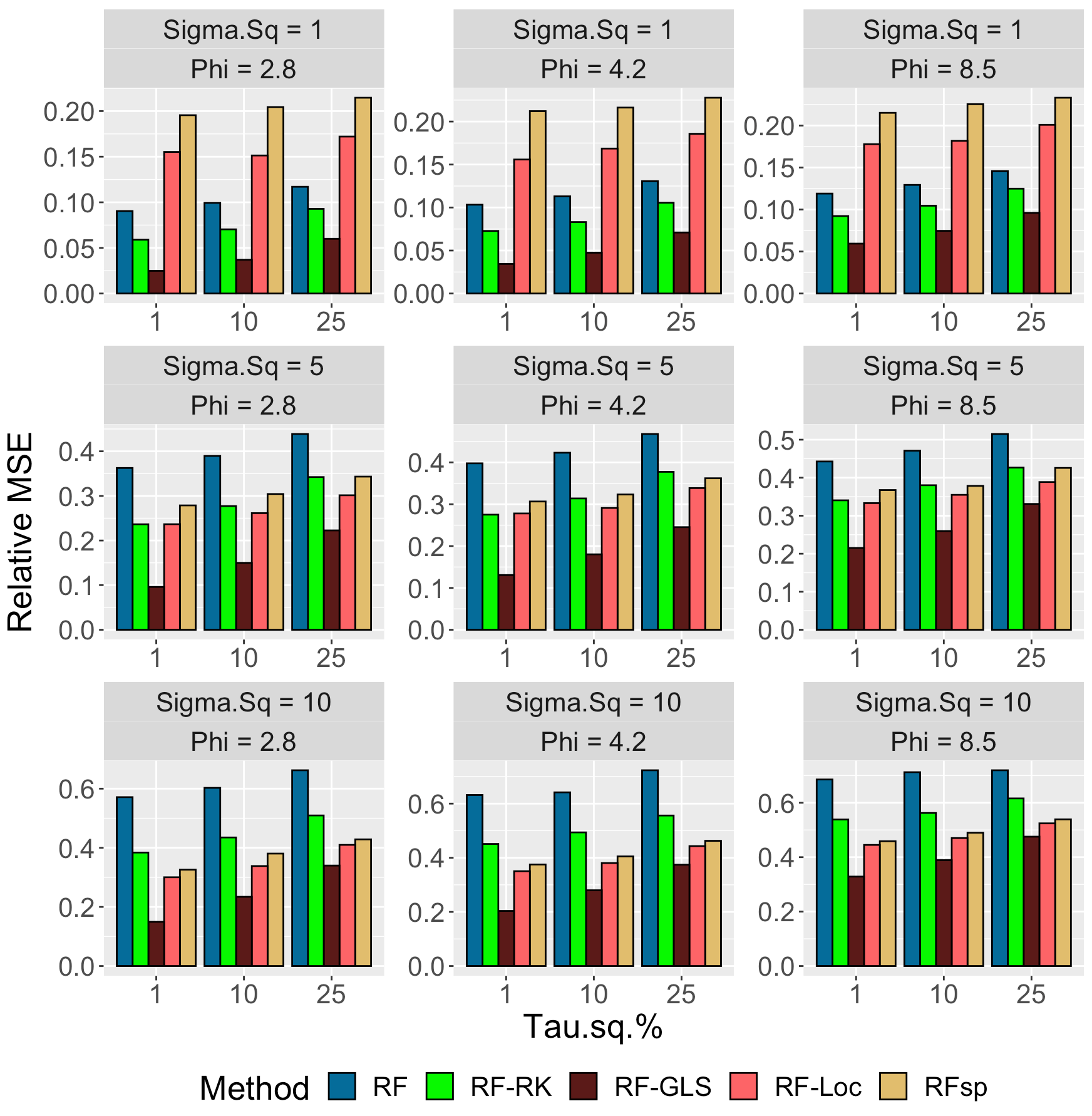}
		\caption{Prediction performance}
	\end{subfigure}
	\caption{Comparison between competing methods on (a) estimation and (b) spatial prediction when the mean function is $m = m_1$.}\label{Fig:sin_comparison}
\end{figure}

\blue{	\subsection{Detailed comparison of RF-GLS to methods adding extra spatial covariates in RF}\label{sec:extra}
	The methods RF-Loc and RFsp belong to the class of alternatives that add extra spatial covariates to RF. \cite{mentch2020getting} showed that for iid errors, adding additional noise covariates can improve prediction performance of RF. For our setting of spatially dependent data, akin to how noise covariates can be added to RF to explain random variation, RFsp adds distance-based covariates to explain spatial variation. 
	Figure \ref{Fig:friedman_performance}(b) of the manuscript  plots the prediction performance of all the 5 candidates for $m = m_2$ (i.e., the Friedman function). We see that when the spatial variation is large ($\sigs=10$), RF-GLS performs comparably to or a little better than RFsp. In this setting, RF-GLS and RFsp are the two best methods. However, when the spatial variance is low ($\sigs=1$), RF-GLS substantially outperforms RFsp. Interestingly, now RF-RK is the second best method also outperforming RFsp.  
	
	To explain these trends, we note that in the iid setting of \cite{mentch2020getting}, the performance of adding random noise covariates to RF was tied to the signal-to-noise ratio SNR which is essentially the ratio of the variation due to the true covariates to the random variation. The gain from adding extra noise covariates was only for low SNR. For high SNR, the trend was reversed and using additional noise covariates (when uncorrelated with the true covariates) did not help.
	The appropriate spatial analogue of this to understand the performance of RFsp would be the ratio of the variation due to the true covariates to the spatial variation, i.e., the covariate-signal-to-spatial-noise ratio $\mbox{SNR} = Var(m(\bX))/Var(w(\ell))$.  
	Figure \ref{Fig:SNR_comparison} plots this SNR for the Friedman function for different choices of $\sigs$ and $\phi$. We see that this SNR decreases with increasing spatial variance $\sigs$. For $\sigs=10$, the SNR is low ($<0.1$) while it is comparatively high ($>0.7$) for $\sigs=1$ for all values of $\phi$. 
	
	RFsp adds $n$ pairwise distance-based covariates to RF to explain the spatial variation. 
	If $D$ denotes the number of true covariates, then RFsp uses a total of $n+D$ covariates in RF and the probability of including a true covariate in the set of $M_{try}$ possible  directions at each split of the regression tree is thus $1 - {n \choose M_{try}}/{n+D \choose M_{try}}$. For large $n$ (and fixed $D$), this inclusion probability goes to $0$ implying that the true covariates are unlikely to be considered for split candidates. Figure \ref{Fig:friedman_variable}, plots the average fraction of decrease in the out-of-bag MSE for RFsp for the true covariates. We see that even when SNR is high ($\sigs=1$), this fraction is only around $25\%$ indicating the low relative importance of the true covariates in RFsp even when they are truly important. For low SNR ($\sigs=10$), this fraction is unsurprisingly even lower but this is less of a concern as in this case the covariates truly explain a small fraction of the variation. 
	
	\begin{figure}[H]
		\centering
		\begin{subfigure}[t]{0.35\textwidth}
			\centering
			\includegraphics[height=2.2in]{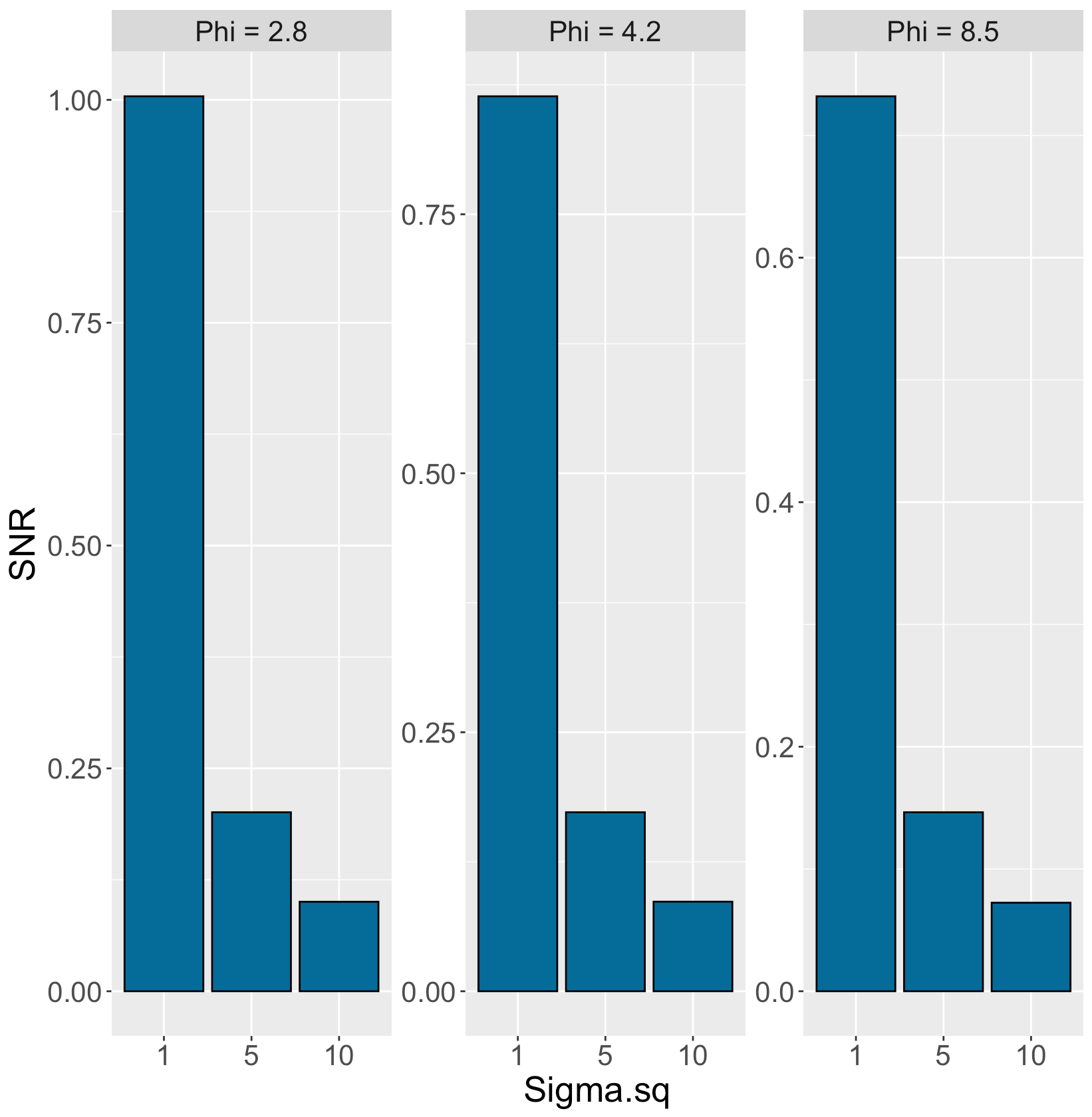}
			\caption{SNR across different $\sigma^2$ and $\phi$ values when $m = m_2$. }\label{Fig:SNR_comparison}
		\end{subfigure}
		\begin{subfigure}[t]{0.5\textwidth}
			\centering
			\includegraphics[height=2.2in]{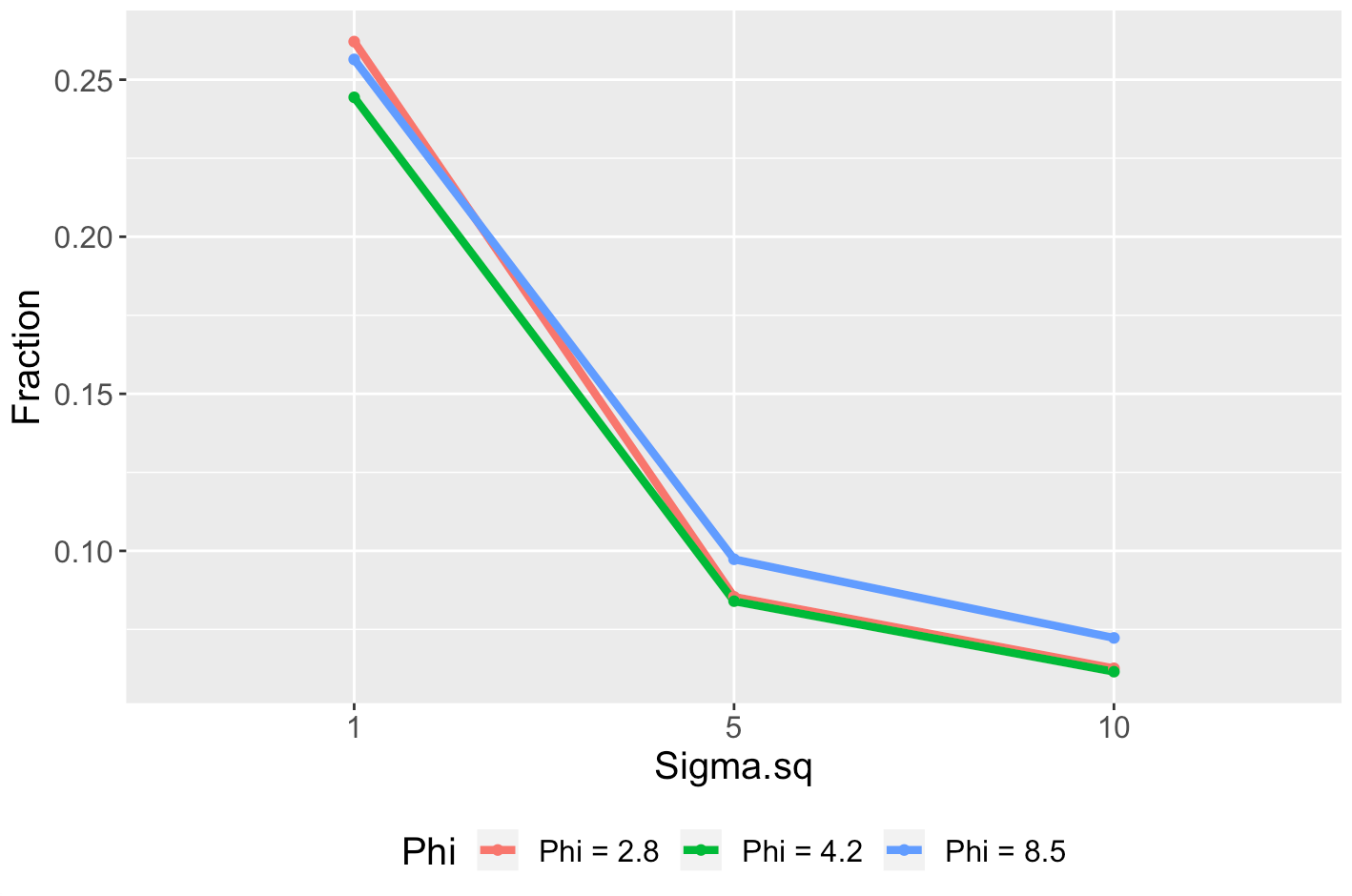}
			\caption{Fraction of decrease in the out-of-bag MSE that can be attributed to $\bX$ for $\tau^2 = 10\% \sigma^2$, when mean $m = m_2$.}\label{Fig:friedman_variable}
		\end{subfigure}
		\caption{Role of SNR in variable importance for RFsp.}
	\end{figure}
	
	Hence for high SNR ($\sigs=1$) where the true covariates explain a large part of the variation in the response, use of $n$ additional spatial covariates is detrimental to prediction performance for RFsp as it dilutes the importance of the true covariates (Figure \ref{Fig:friedman_variable}). In fact for high SNR, RFsp performs even worse than RF-RK which can estimate $m$ reasonably well in this setting despite ignoring the spatial dependence as the covariate signal dominates. For low SNR ($\sigs=10$), as the spatial contribution increases, RFsp (being naturally equipped to capture the spatial correlation in the data) performs comparably to RF-GLS (Figure \ref{Fig:friedman_performance}). RF-RK, ignoring the dominant spatial signal when estimating $m$,  performs worse than both RF-GLS and RFsp. 
	
	\noindent\textbf{RF-GLS: Best of both worlds.}
	This investigation of the predictive performances of the methods, highlights the systematic advantage of RF-GLS which combines together the strengths of RF for modeling the covariate effect and GP for parsimoniously modeling the spatial effect without need for adding many distance-based covariates. The prediction performances are summed up in Figure \ref{Fig:friedman_MSE_ratio}, where we plot the MSEs of RFsp and RF-RK relative to that of RF-GLS as a function of the spatial variance $\sigs$. For high values of $\sigma^2$ (low SNR), performance of RFsp and RF-GLS are comparable, but RF-GLS significantly outperforms RF-RK for higher values of $\sigma^2$. Whereas for lower $\sigma^2$ (high SNR), RF-GLS performs better or comparably to that of RF-RK, but significantly outperforms RFsp. The trends are similar for all the values of $\phi$. 
	RF-RK and RFsp outperforms each other in two opposite ends of the spatial signal strength spectrum. RF-GLS brings together the best of both worlds and produce comparable or better results to both of them across the spectrum, outperforming RFsp and RF-RK when the spatial signal strength is low and high respectively.

}

\blue{	\subsection{Larger sample size}\label{sec:large}
	In this section, we replicate the simulation setup described in Section \ref{sec:sim}, with $n = 1000$. As the exact RF-GLS algorithm can become computationally intensive with increment in $n$, we use Nearest Neigbhor Gaussian Processes (NNGP) for RF-GLS (as discussed in Section \ref{sec:prac}). For sake of computational convenience and stability of the results, we use minimum cardinality of leaf nodes $t_c = 20$ and number of nearest neighbors $= 20$.
	
	\begin{figure}[H]
		\centering
		\begin{subfigure}[t]{0.5\textwidth}
			
			\centering
			\includegraphics[height=3.2in]{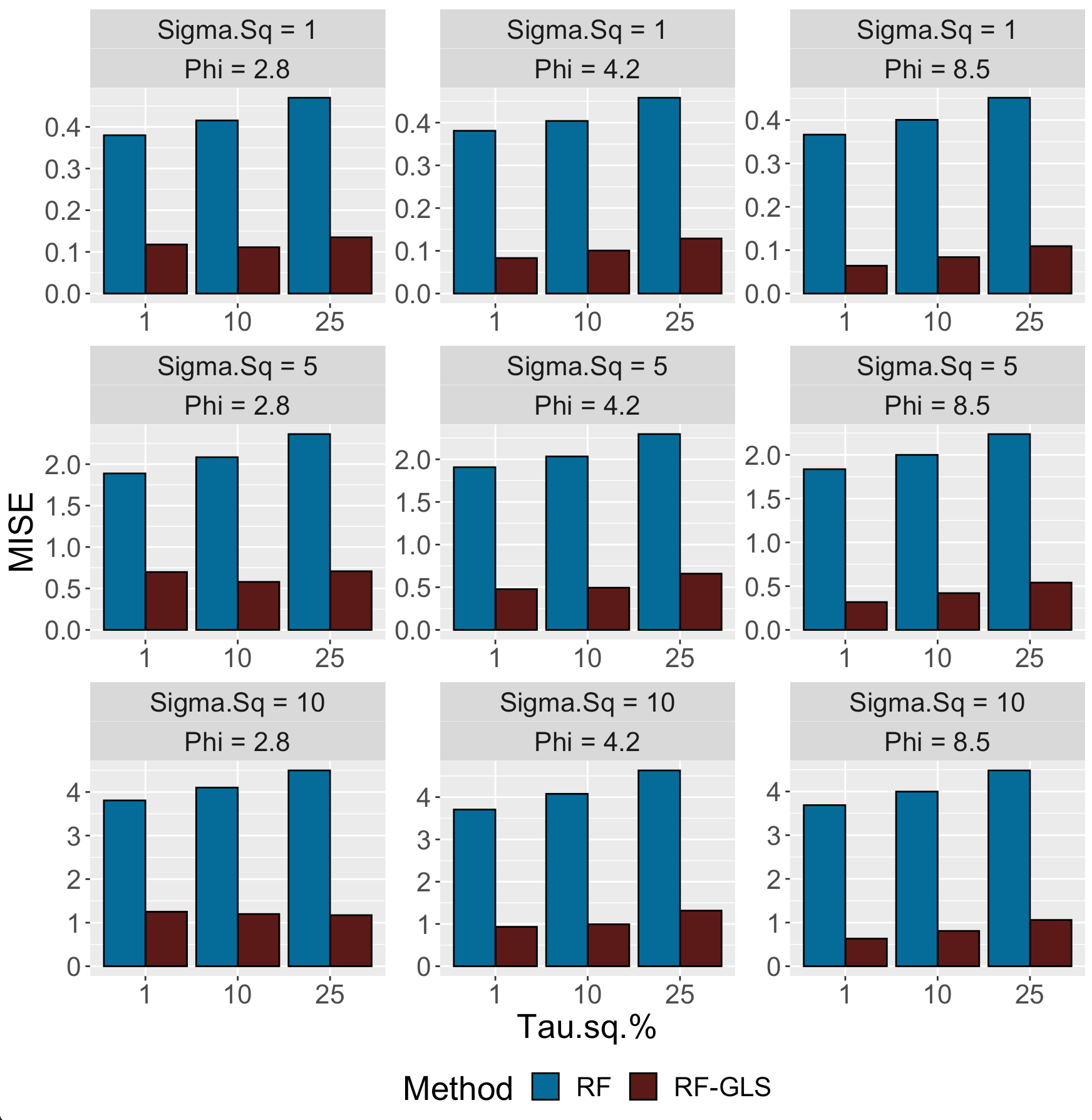}
			\caption{Estimation performance}
		\end{subfigure}%
		~ 
		\begin{subfigure}[t]{0.5\textwidth}
			
			\centering
			\includegraphics[height=3.2in]{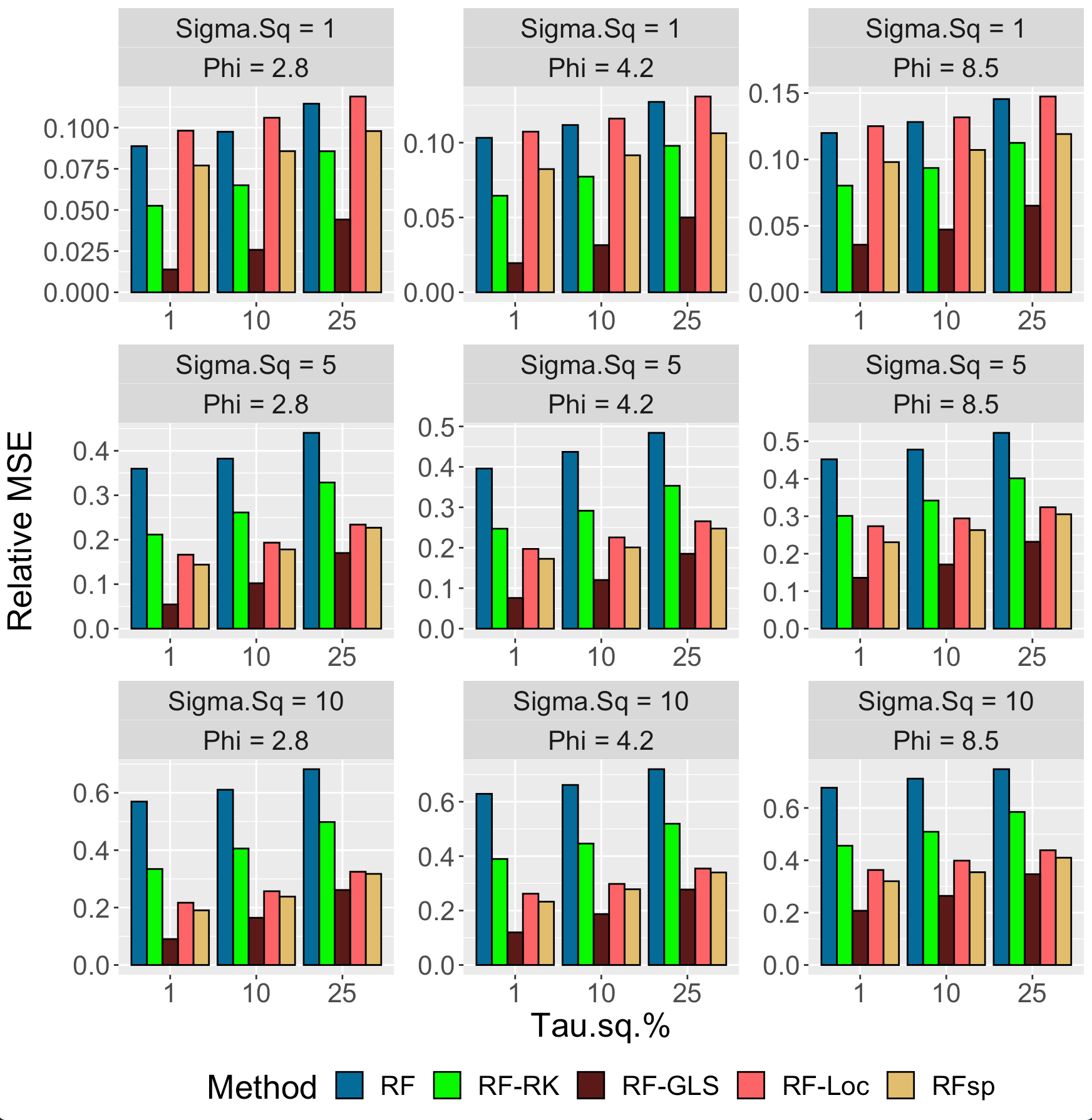}
			\caption{Prediction performance}
		\end{subfigure}
		\caption{Comparison between competing methods on (a) estimation and (b) spatial prediction for $n  = 1000$ when the mean function is $m = m_1$.}\label{Fig:sin_1000}
	\end{figure}

	\begin{figure}[H]
		\centering
		\begin{subfigure}[t]{0.5\textwidth}
			
			\centering
			\includegraphics[height=3.2in]{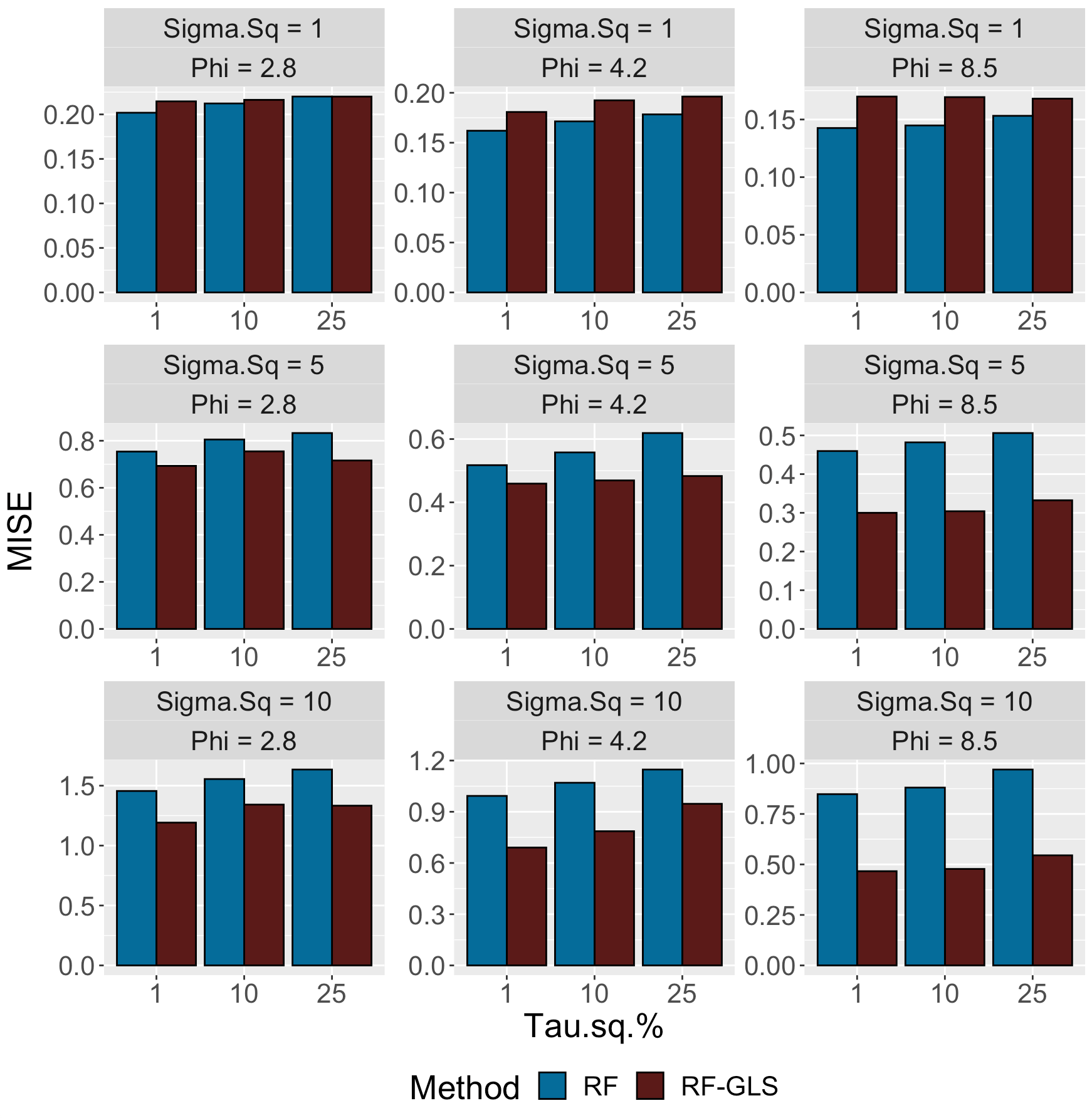}
			\caption{Estimation performance}
		\end{subfigure}%
		~ 
		\begin{subfigure}[t]{0.5\textwidth}
			
			\centering
			\includegraphics[height=3.2in]{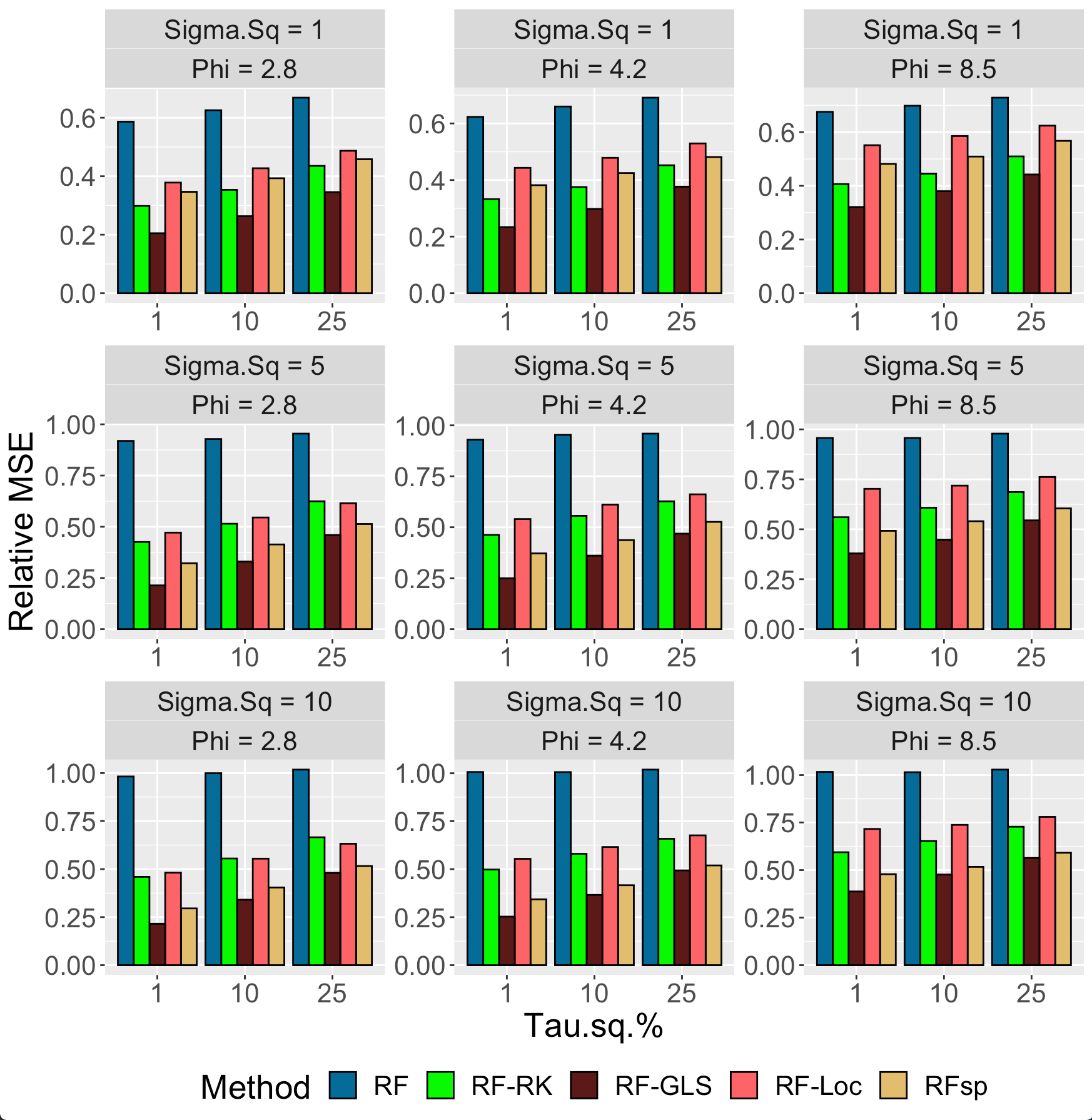}
			\caption{Prediction performance}
		\end{subfigure}
		\caption{Comparison between competing methods on (a) estimation and (b) spatial prediction for $n  = 1000$ when the mean function is $m = m_2$.}\label{Fig:friedman_1000}
	\end{figure}

	RF-GLS completely outperforms RF in estimation for $m = m_1$. For prediction, akin to the trends observed in Figure \ref{Fig:sin_1000}, RF-GLS outperforms its competition when the spatial contribution in the response process is on the lower side. As the spatial contribution in the response process increases, the prediction performance of RFsp becomes comparable with that of RF-GLS. For $m = m_2$ (Figure \ref{Fig:friedman_1000}) the estimation performance of RF and RF-GLS are comparable when the spatial contribution in response process is low and there is not much correlation in the data. As the spatial contribution in response increases, RF-GLS outperforms RF. The prediction results for $m = m_2$ show a similar trend to that of $m = m_1$. Overall, for both estimation and prediction, 
	we observe trends similar to the analysis for $n = 250$ in the main manuscript. 
	
	\subsection{Higher dimensional function}\label{sec:highdim}
	
	We also replicated the simulation setup described for $m = m_2$, with a higher dimensional function $m = m_3$ with $p = 15$ covariates, which is given as follows:
	\begin{small}$$
		\begin{aligned}
		m_3(\bx) &= \bigg(10 \sin(\pi x_1 x_2) + 20(x_3 - 0.5)^2 + 10x_4 + 5x_5 + \frac{3}{(x_6+1)(x_7+1)} + 4 \exp(x_8^2) + 30 x_9^2 * x_{10}\\ 
		&+ 5 (\exp(x_{11}^2) * sin(\pi x_{12}) + \exp(x_{12}^2) * sin(\pi x_{11})) + 10 x_{13}^2 * \cos(\pi x_{14})+ 20 x_{15}^4\bigg)/6
		\end{aligned}
		$$
	\end{small}
	
	In Figure \ref{Fig:MV_comparison}, we observe that in $m = m_3$, the estimation performance of RF-GLS is comparable to that of RF when the spatial signal strength is low and the correlation is weak. As the spatial signal and correlation strength increases, RF-GLS out performs RF in estimation. In prediction, RF-GLS performs favorably to the competing methods in lower spatial signal strength. The  performance of RFsp gradually becomes comparable to that of RF-GLS as the spatial signal increases and the correlation strength decreases.  
	
	\begin{figure}[h!]
		\centering
		\begin{subfigure}[t]{0.5\textwidth}
			
			\centering
			\includegraphics[height=3.2in]{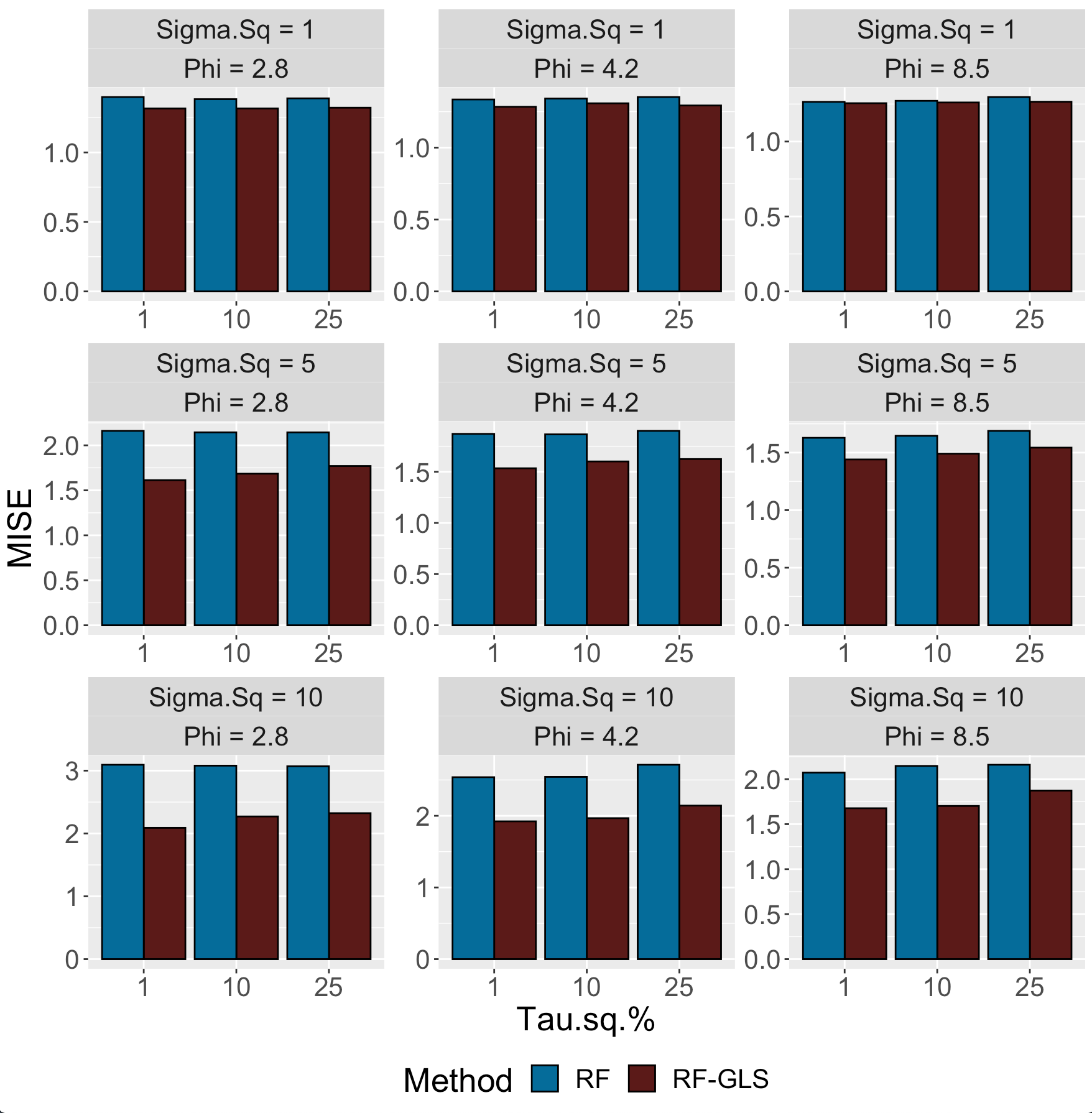}
			\caption{Estimation performance}
		\end{subfigure}%
		~ 
		\begin{subfigure}[t]{0.5\textwidth}
			
			\centering
			\includegraphics[height=3.2in]{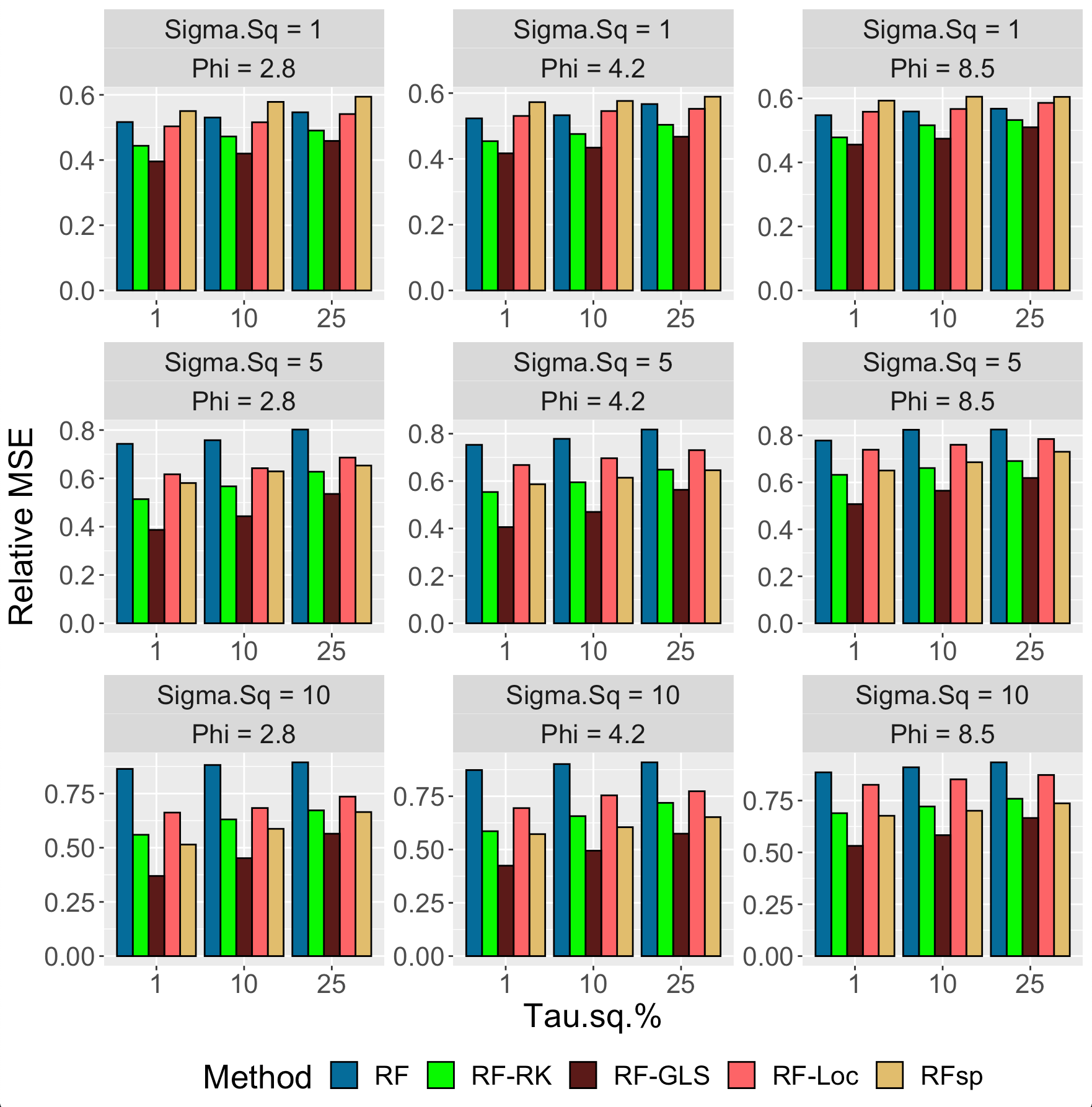}
			\caption{Prediction performance}
		\end{subfigure}
		\caption{Comparison between competing methods on (a) estimation and (b) spatial prediction when the mean function is $m = m_3$.}\label{Fig:MV_comparison}
	\end{figure}
	
	
	\subsection{Model misspecification: Misspecified Mat\'ern covariance}\label{sec:matern_simul}
	In the earlier simulation setups, even though we estimated the correlation parameters, we assumed that the true data was generated using a Gaussian Process and that the correlation structure (form of the covariance function) was correctly specified. 
	We conducted additional studies to test the robustness of RF-GLS under model misspecification. 
	
	We simulate the spatial process $w(\ell)$ from a Gaussian Process with a smoother Mat\'ern covariance with $\nu=3/2$ in (\ref{eq:matern}) and fit RF-GLS and RF-RK assuming a less smooth exponential covariance structure (Mat\'ern with $\nu=1/2$). 
	
	We replicate the simulation described in Section \ref{sec:sim} with 
	all the 27 combinations of $\sigma^2, \phi$ and $\tau^2$ values as in Section \ref{sec:simdetails}. 
	In the implementation, for both RF-GLS and RF-RK, the covariance was modeled through exponential function. The results for $m=m_1, m_2$, and $m_3$ are provided respectively in Figures \ref{Fig:sin_matern}, \ref{Fig:friedman_matern}, and \ref{Fig:MV_matern}. Under this misspecified model, RF-GLS performs comparably to RF for estimation, when the spatial signal strength ($\sigma^2$) is low, but outperforms significantly as spatial signal strength increases. For prediction, performance of RF-GLS and RFsp are comparable for high signal strength, but RF-GLS outperforms in lower signal strength. 
	
	\begin{figure}[H]
		\centering
		\begin{subfigure}[t]{0.5\textwidth}
			
			\centering
			\includegraphics[height=3.2in]{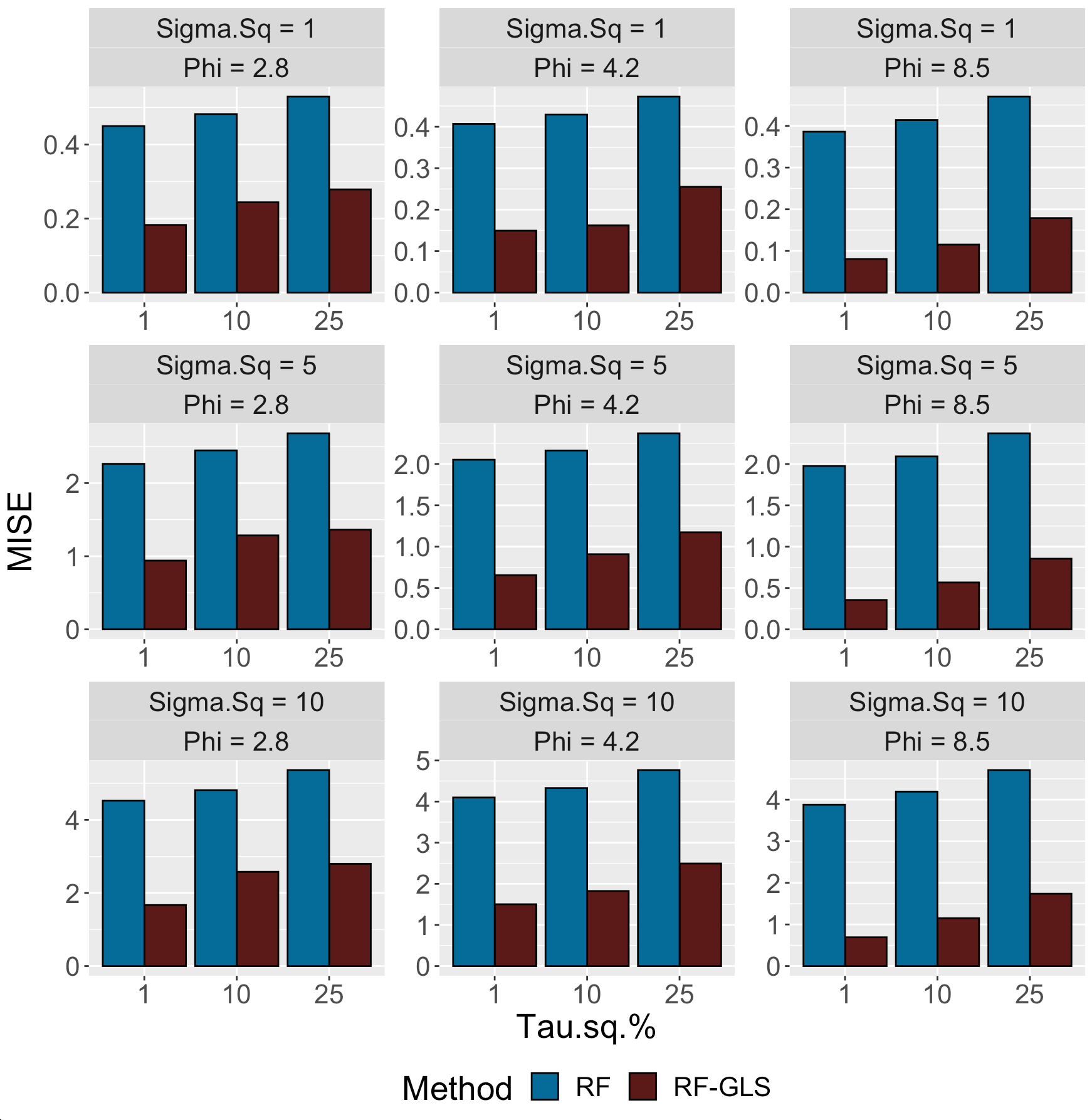}
			\caption{Estimation performance}
		\end{subfigure}%
		~ 
		\begin{subfigure}[t]{0.5\textwidth}
			
			\centering
			\includegraphics[height=3.2in]{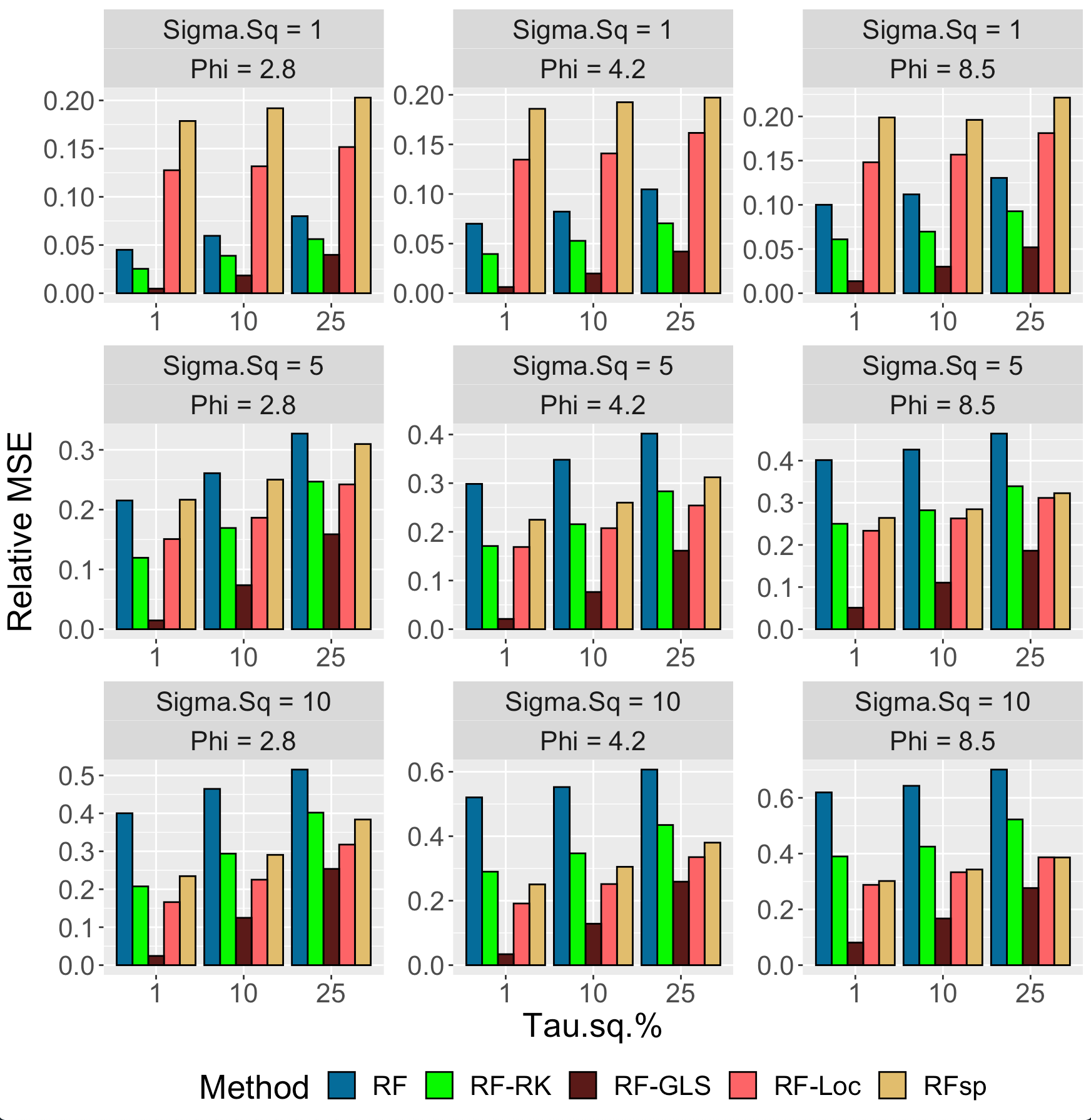}
			\caption{Prediction performance}
		\end{subfigure}
		\caption{Comparison between competing methods on (a) estimation and (b) spatial prediction when the spatial surface $w(\ell)$ is generated from Mat\'ern covariance and the mean function is $m = m_1$.}\label{Fig:sin_matern}
	\end{figure}
	
	\begin{figure}[H]
		\centering
		\begin{subfigure}[t]{0.5\textwidth}
			
			\centering
			\includegraphics[height=3.2in]{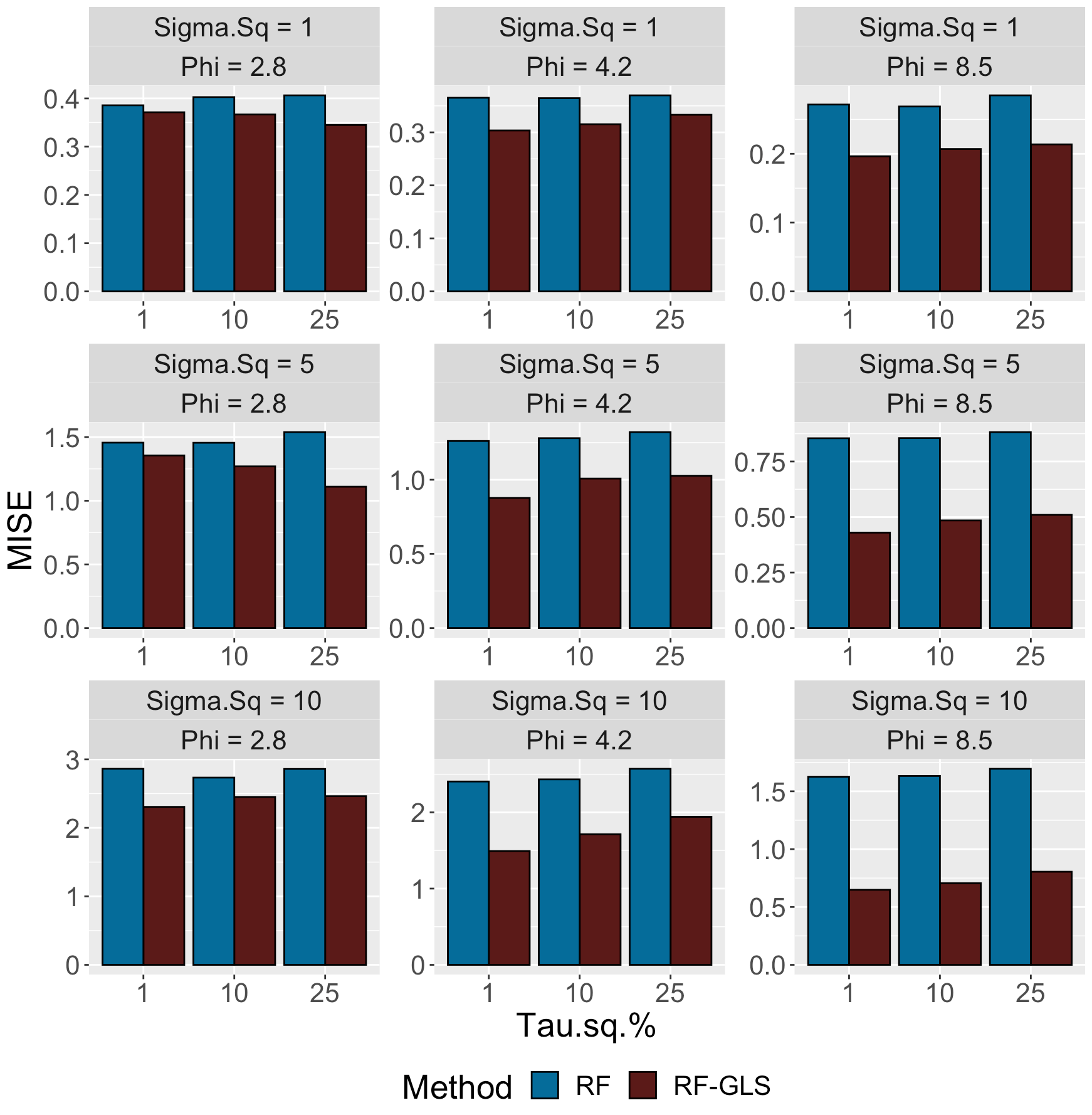}
			\caption{Estimation performance}
		\end{subfigure}%
		~ 
		\begin{subfigure}[t]{0.5\textwidth}
			
			\centering
			\includegraphics[height=3.2in]{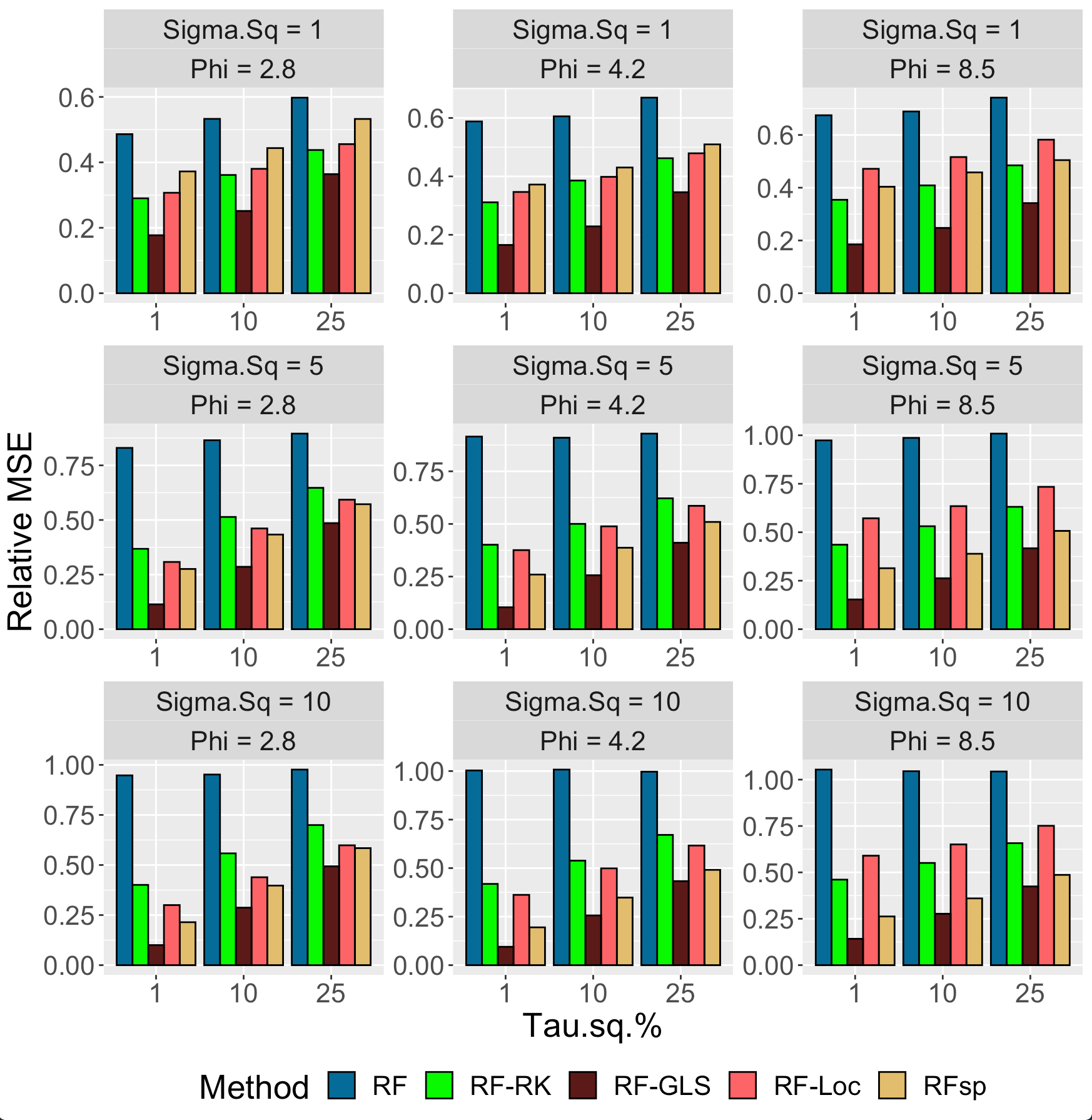}
			\caption{Prediction performance}
		\end{subfigure}
		\caption{Comparison between competing methods on (a) estimation and (b) spatial prediction when the spatial surface $w(\ell)$ is generated from Mat\'ern covariance and the mean function is $m = m_2$.}\label{Fig:friedman_matern}
	\end{figure}
	
	\begin{figure}[H]
		\centering
		\begin{subfigure}[t]{0.5\textwidth}
			
			\centering
			\includegraphics[height=3.2in]{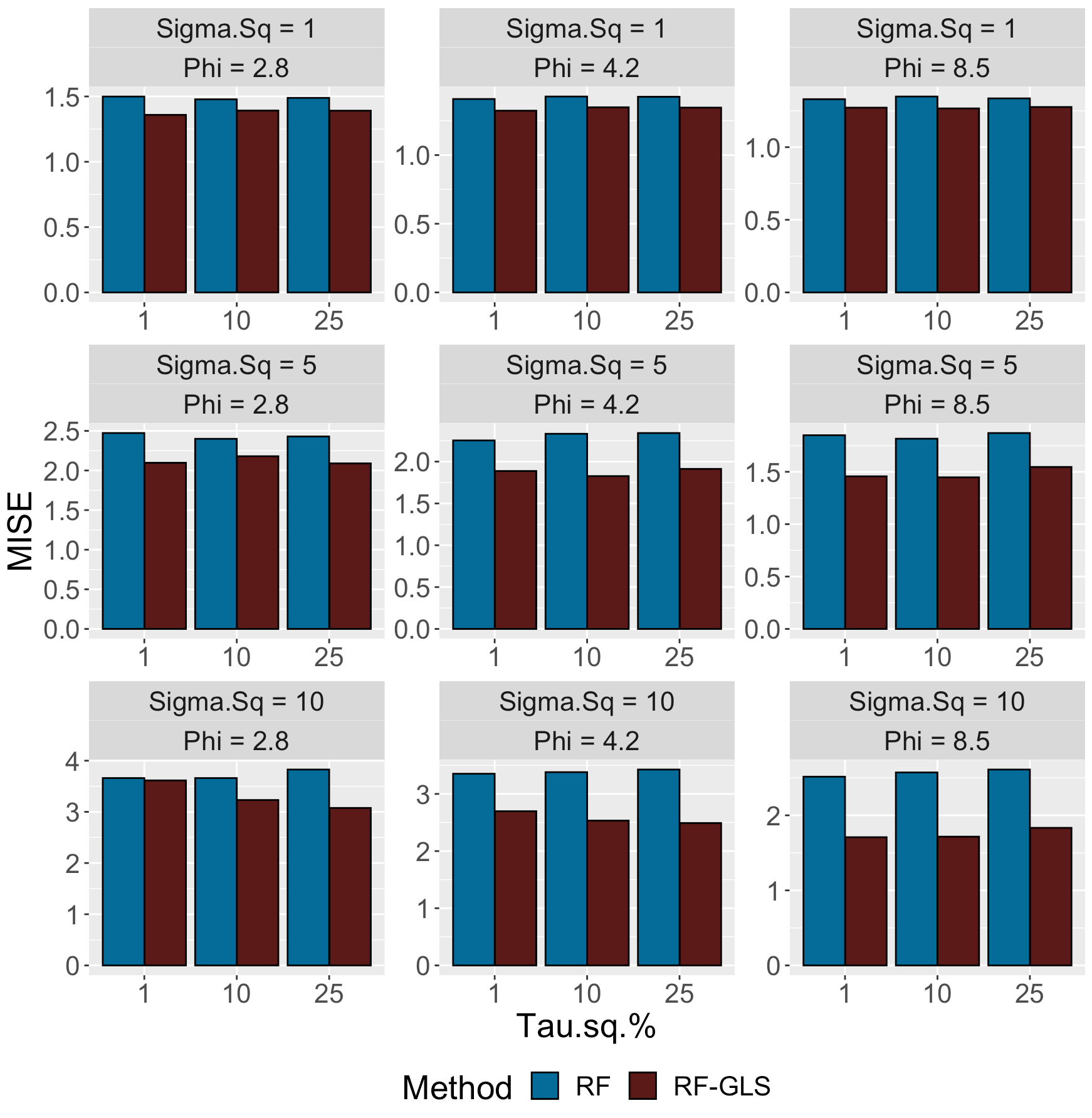}
			\caption{Estimation performance}
		\end{subfigure}%
		~ 
		\begin{subfigure}[t]{0.5\textwidth}
			
			\centering
			\includegraphics[height=3.2in]{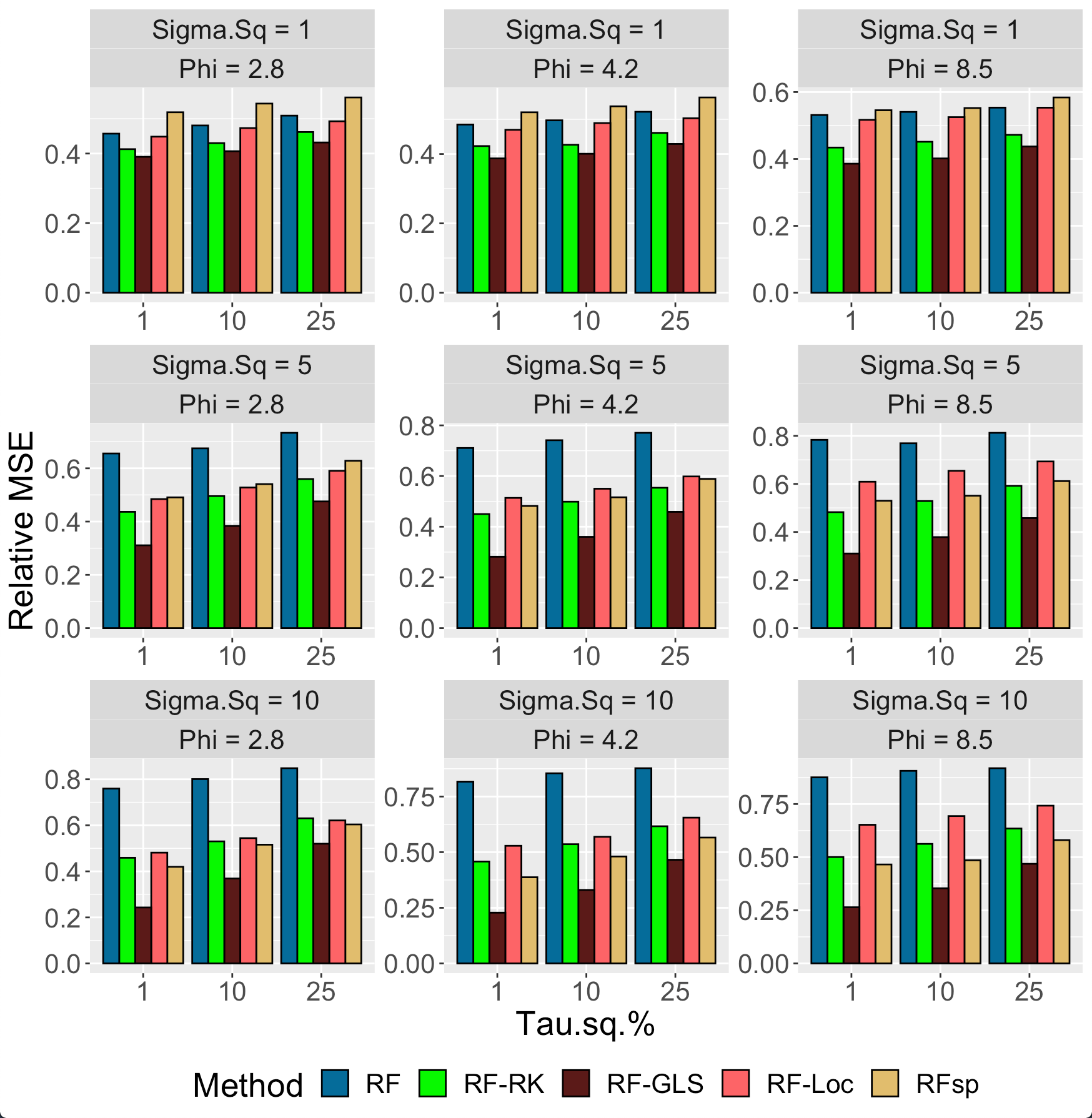}
			\caption{Prediction performance}
		\end{subfigure}
		\caption{Comparison between competing methods on (a) estimation and (b) spatial prediction when the spatial surface $w(\ell)$ is generated from Mat\'ern covariance and the mean function is $m = m_3$.}\label{Fig:MV_matern}
	\end{figure}
	
	
	
	\subsection{Model misspecification: misspecified spatial effect}\label{sec:smooth_simul}
	Next, we focus on the scenario, where the entire spatial effect $w(\ell)$ is misspecified. In particular, we consider the scenario where the spatial effect $w(\ell)$ is not even generated from a Gaussian Process but is a fixed smooth spatial surface. 
	We consider 3 choices of such fixed surfaces each generated as (centered) densities of a bivariate mixture normal distribution with two components (with means $\bm{\mu}_1$ and $\bm{\mu}_2$ representing the locations of the two-peaks in the surface) and with  
	with covariance matrices $\bSigma_1$ and $\bSigma_2$ representing the slopes and dispersion around the peaks. 
	
	In particular, we consider 3 setups with varying locations and variances corresponding to the two Gaussian components. 
	\begin{enumerate}
		\item \textbf{Setup 1:} $\mu_1 = (0.25, 0.5); \:\mu_2 = (0.75, 0.5), \bSigma_1^2 = \bSigma_2^2 = 0.0025\bI $;
		\item \textbf{Setup 2:} $\mu_1 = (0.25, 0.5); \:\mu_2 = (0.75, 0.5), \bSigma_1^2 = 0.01\bI;\:  \bSigma_2^2= 0.0025\bI $;
		\item \textbf{Setup 3:} $\mu_1 = (0.25, 0.25); \mu_2 = (0.6, 0.9), \bSigma_1^2 = \bSigma_2^2 = 0.0025\bI $;
	\end{enumerate}
	We centered each surface and scaled them to have the sample variance to be $\sigma^2$ for choices $\sigs=1,5$ and $10$. During implementation,  for RF-RK and RF-GLS, the covariance was modeled through an exponential Gaussian process, whose parameters were estimated from the RF residuals.
	
	\begin{figure}[h!]
		\centering
		\begin{subfigure}[t]{0.5\textwidth}
			\centering
			\includegraphics[width=2.5in]{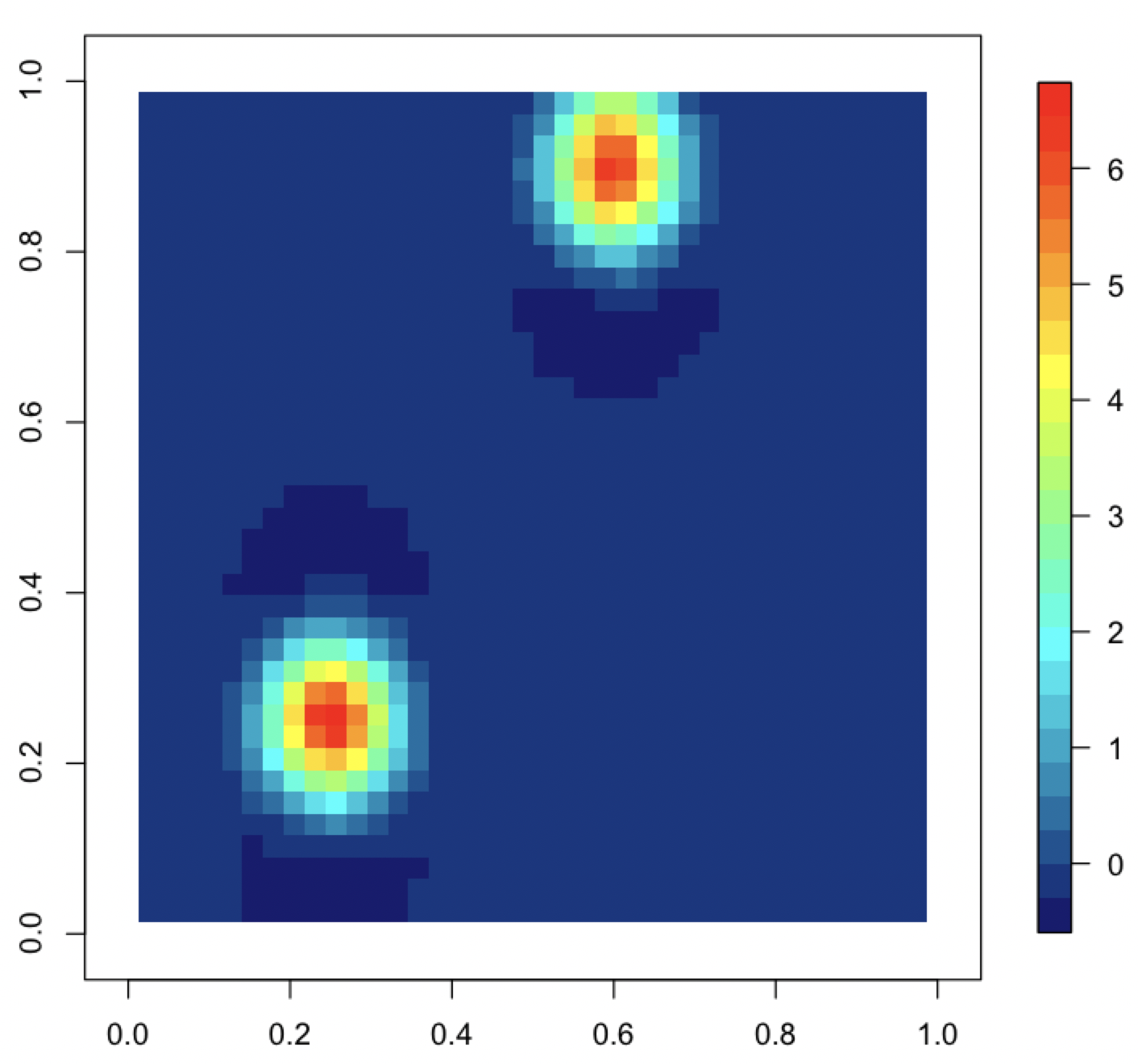}
			\caption{Setup 3}
		\end{subfigure}\\
		\begin{subfigure}[t]{0.4\textwidth}
			\centering
			\includegraphics[width=2.in]{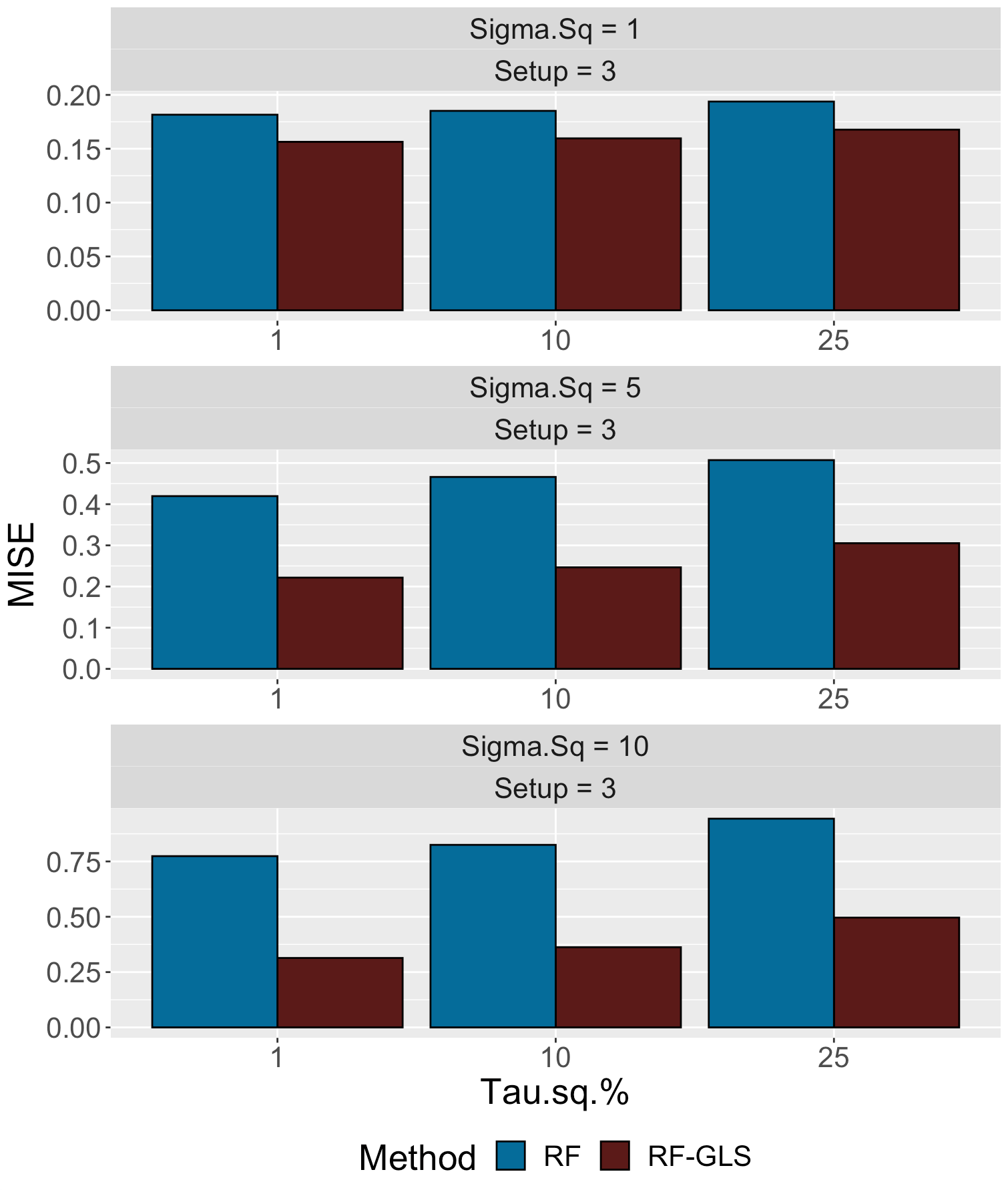}
			\caption{Estimation for Setup 3}
		\end{subfigure}
		\begin{subfigure}[t]{0.4\textwidth}
			\centering
			\includegraphics[width=2.in]{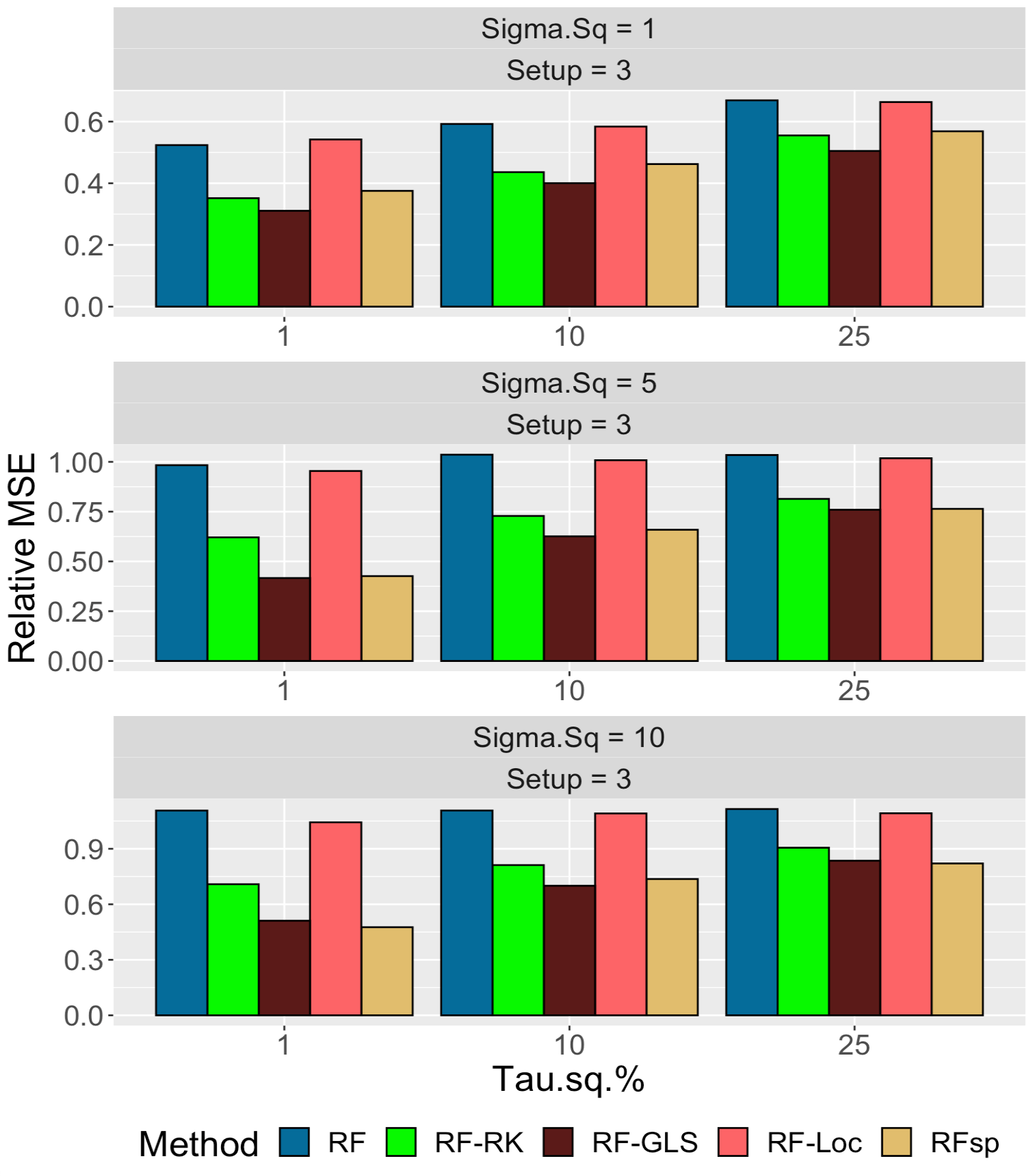}
			\caption{Prediction for Setup 3}
		\end{subfigure}
		
		\caption{Comparison between competing methods on (b) estimation and (c) spatial prediction when the spatial surface $w(\ell)$ is simulated from the smooth function given by (a) and the mean function is $m = m_2$.}\label{Fig:smooth_short_performance}
	\end{figure}
	We present the results for one of the surfaces (Setup 3) first. 
	Figure \ref{Fig:smooth_short_performance} (a) plots the fixed surface $w(\ell)$. During implementation, for RF-RK and RF-GLS we used a working covariance matrix modeled as an exponential Gaussian process. 
	Figures \ref{Fig:smooth_short_performance} (b) and (c) respectively plots the estimation and prediction performance. RF-GLS performs comparably or better to RF in estimation for all values of $\sigma^2$ -- the empirical variance of $w(\ell)$. For prediction, once again RF-GLS and RFsp perform comparably when $\sigma^2$ is high, but RF-GLS performs better in low signal strength. These trends were consistent with those explained in Section \ref{sec:extra}. For the other choices of the surfaces (Figure \ref{Fig:smooth_setup}), the results were similar (Figures \ref{Fig:sin_smooth},\ref{Fig:friedman_smooth}, and \ref{Fig:MV_smooth}). The gains of using RF-GLS being more substantial for certain choices of the covariate effect $m$ and spatial surface $w$. 
	
	
	\begin{figure}[H]
		\centering
		\begin{subfigure}[t]{0.5\textwidth}
			\centering
			\includegraphics[width=3.4in]{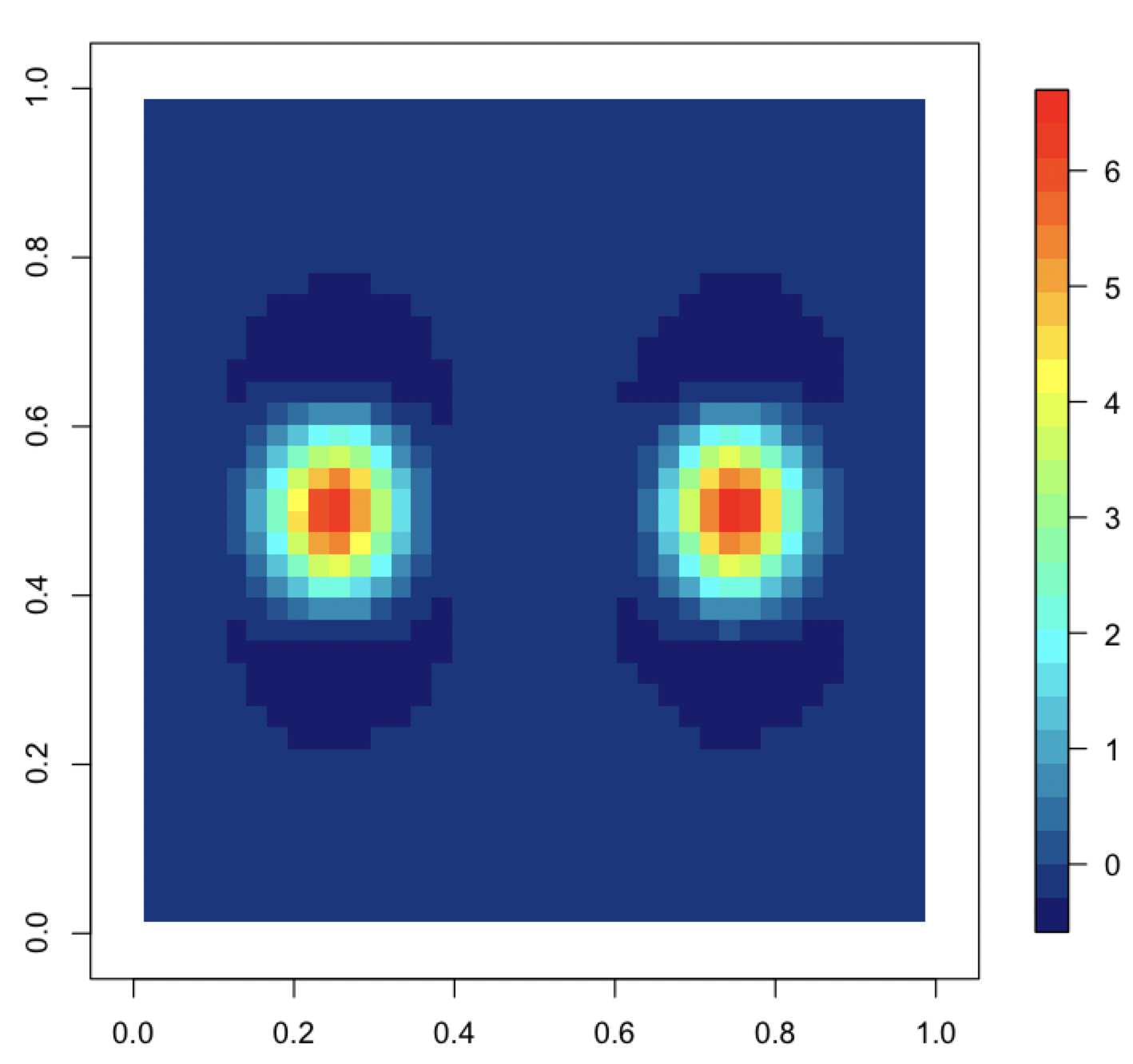}
			\caption{Setup 1}
		\end{subfigure}%
		~ 
		\begin{subfigure}[t]{0.5\textwidth}
			\centering
			\includegraphics[width=3.4in]{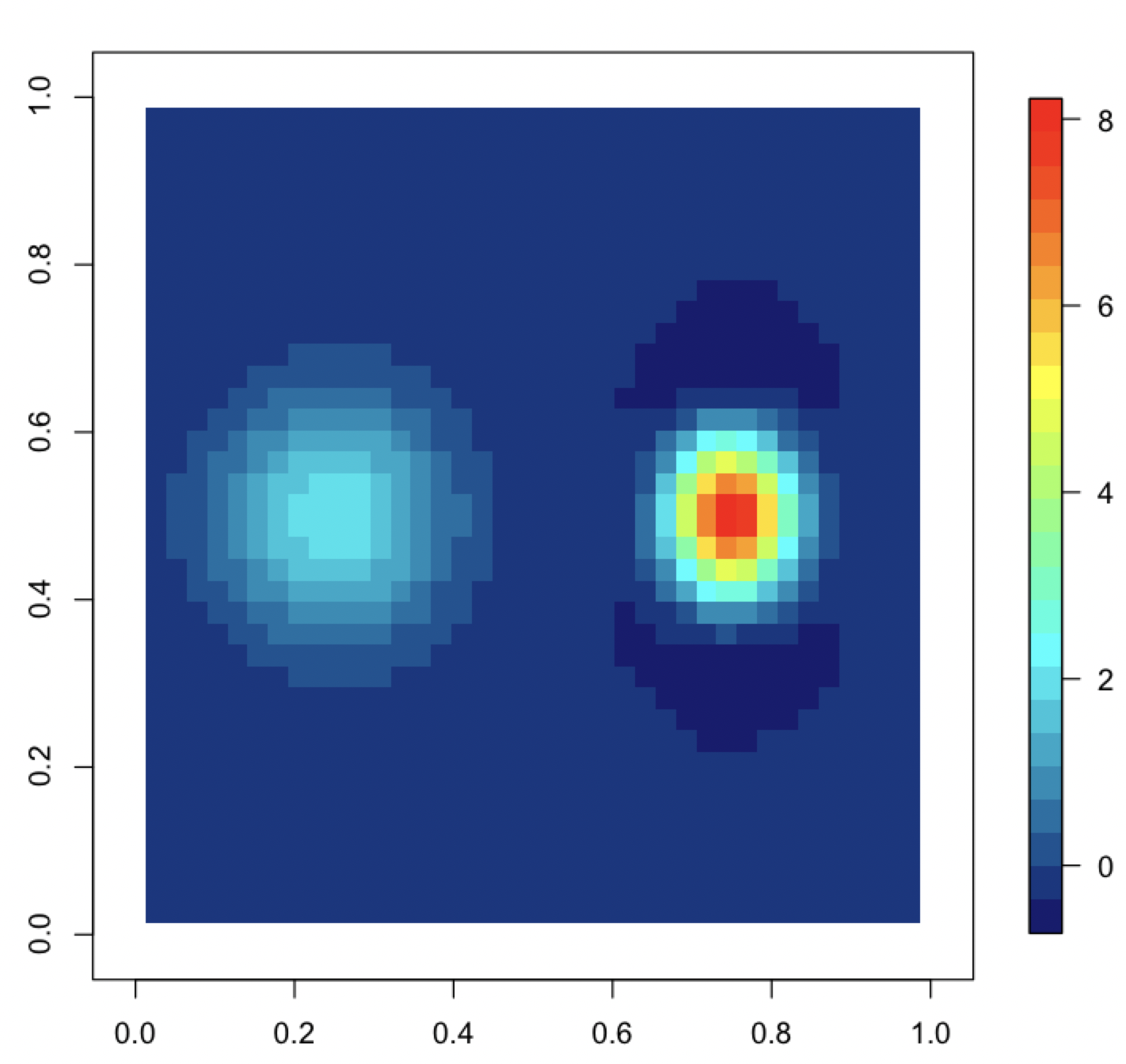}
			\caption{Setup 2}
		\end{subfigure}%
		
		\begin{subfigure}[t]{0.5\textwidth}
			\centering
			\includegraphics[width=3.4in]{Revision_Figure_V2/Surface_3.png}
			\caption{Setup 3}
		\end{subfigure}%
		
		\caption{ (a) Setup 1, (b) Setup 2 and (c) Setup 3 corresponding to the smooth functions with $(\sigma^2 = 1)$. }\label{Fig:smooth_setup}
	\end{figure}
	
	\begin{figure}[H]
		\centering
		\begin{subfigure}[t]{0.5\textwidth}
			
			\centering
			\includegraphics[height=3.2in]{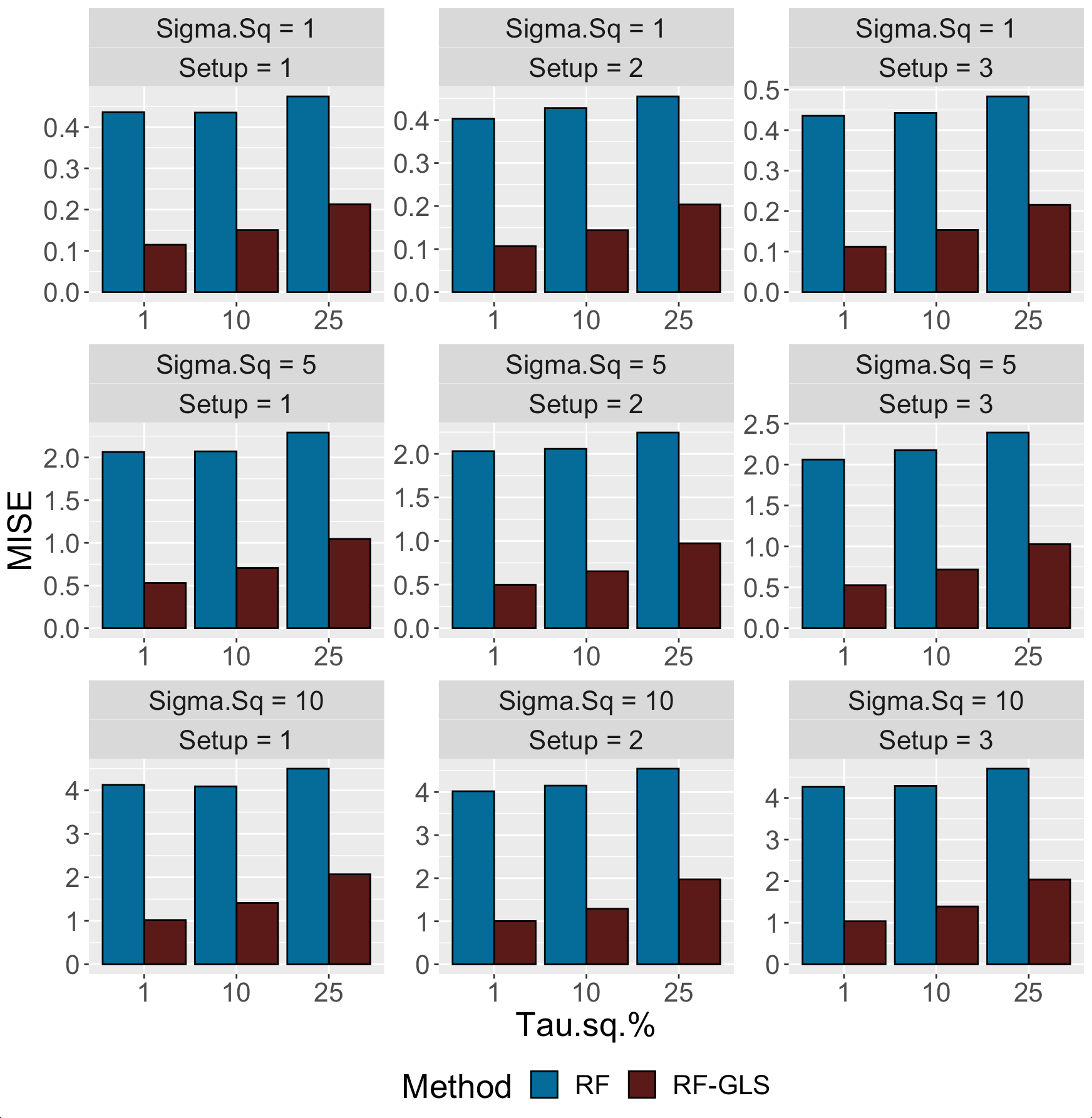}
			\caption{Estimation performance}
		\end{subfigure}%
		~ 
		\begin{subfigure}[t]{0.5\textwidth}
			
			\centering
			\includegraphics[height=3.2in]{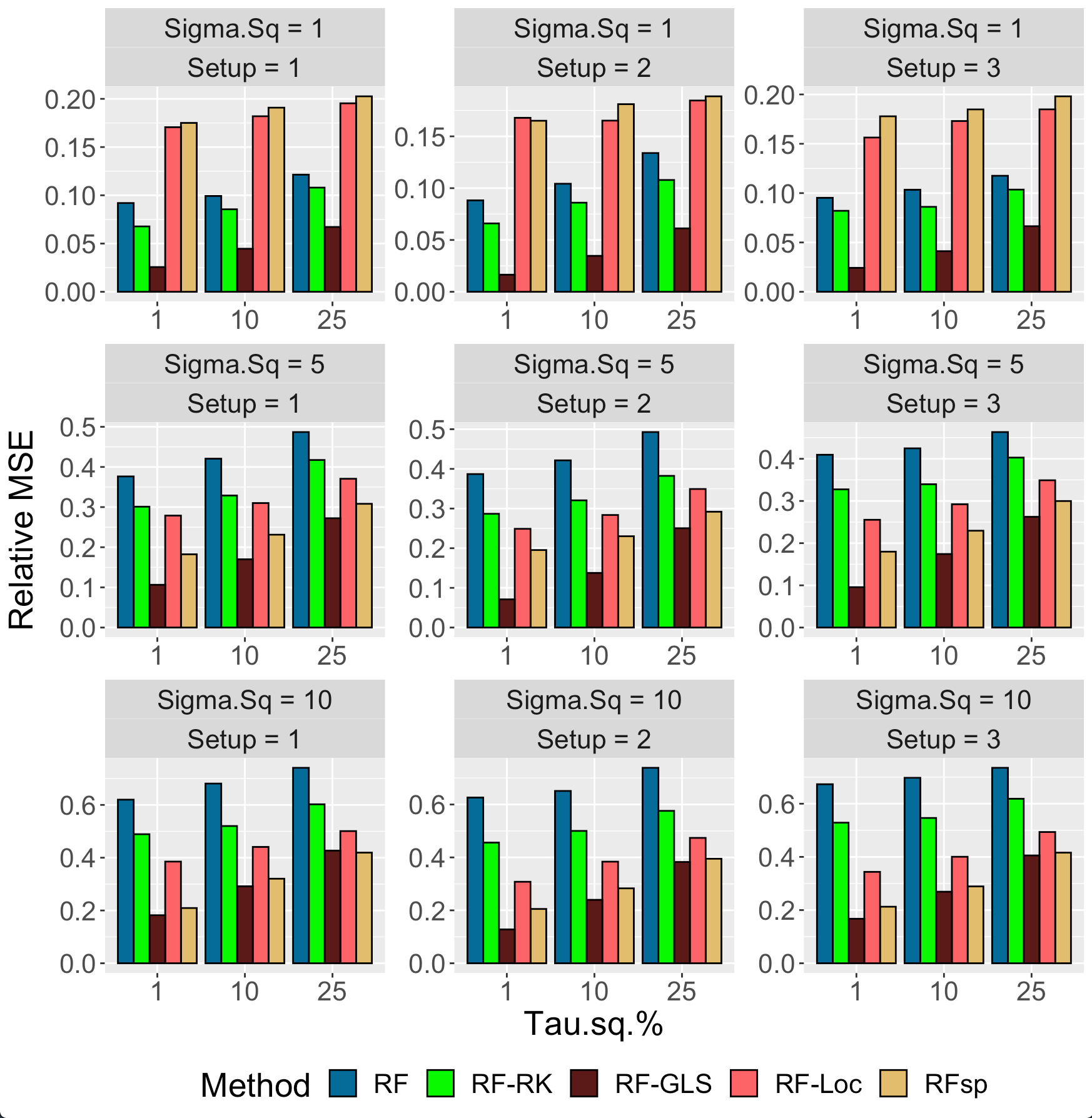}
			\caption{Prediction performance}
		\end{subfigure}
		\caption{Comparison between competing methods on (a) estimation and (b) spatial prediction when the spatial surface $w(\ell)$ is generated as a fixed smooth surface and the mean function is $m = m_1$.}\label{Fig:sin_smooth}
	\end{figure}
	
	\begin{figure}[H]
		\centering
		\begin{subfigure}[t]{0.5\textwidth}
			
			\centering
			\includegraphics[height=3.2in]{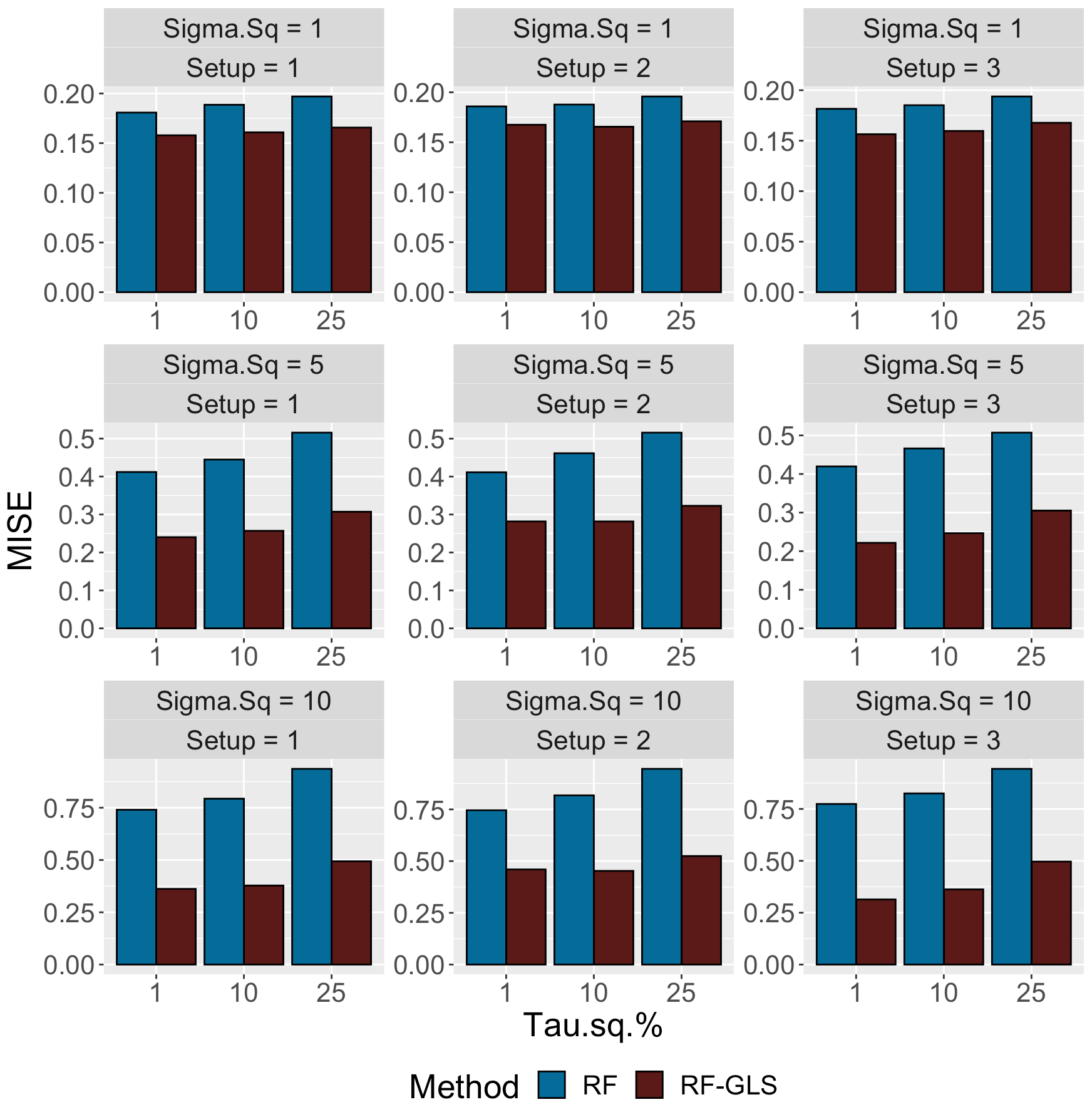}
			\caption{Estimation performance}
		\end{subfigure}%
		~ 
		\begin{subfigure}[t]{0.5\textwidth}
			
			\centering
			\includegraphics[height=3.2in]{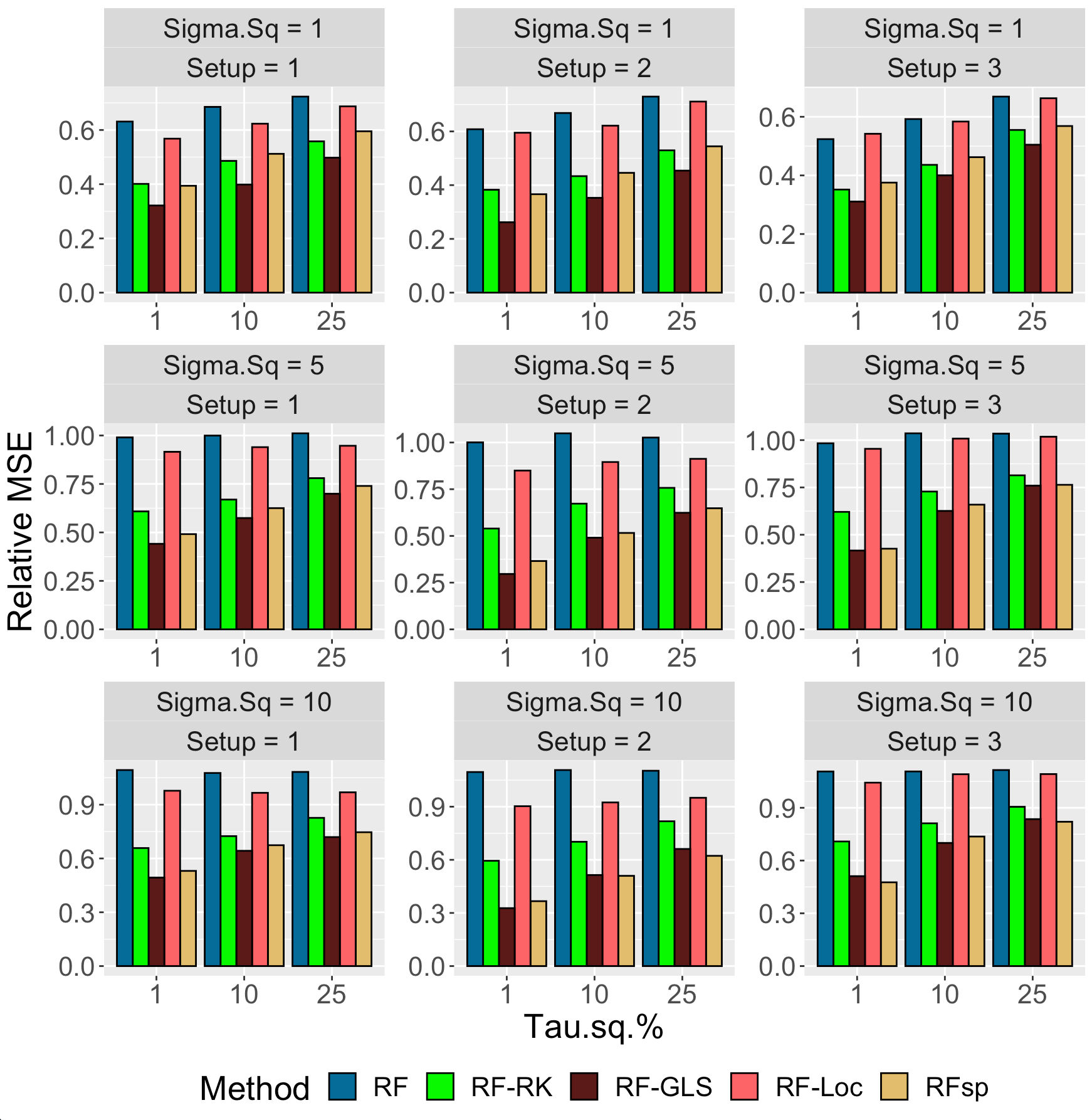}
			\caption{Prediction performance}
		\end{subfigure}
		\caption{Comparison between competing methods on (a) estimation and (b) spatial prediction when the spatial surface $w(\ell)$ is generated as a fixed smooth surface and the mean function is $m = m_2$.}\label{Fig:friedman_smooth}
	\end{figure}
	
	\begin{figure}[H]
		\centering
		\begin{subfigure}[t]{0.5\textwidth}
			
			\centering
			\includegraphics[height=3.2in]{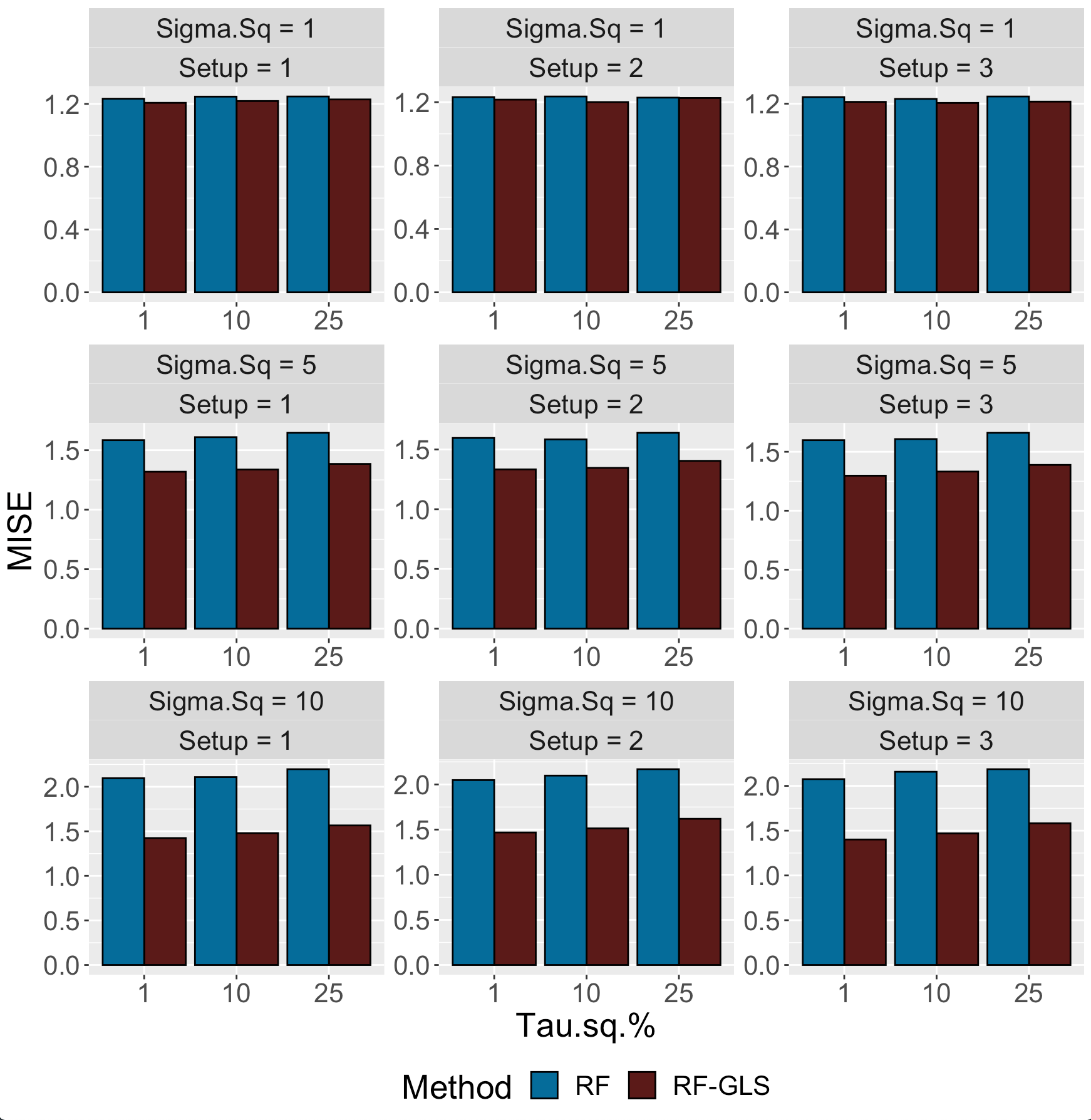}
			\caption{Estimation performance}
		\end{subfigure}%
		~ 
		\begin{subfigure}[t]{0.5\textwidth}
			
			\centering
			\includegraphics[height=3.2in]{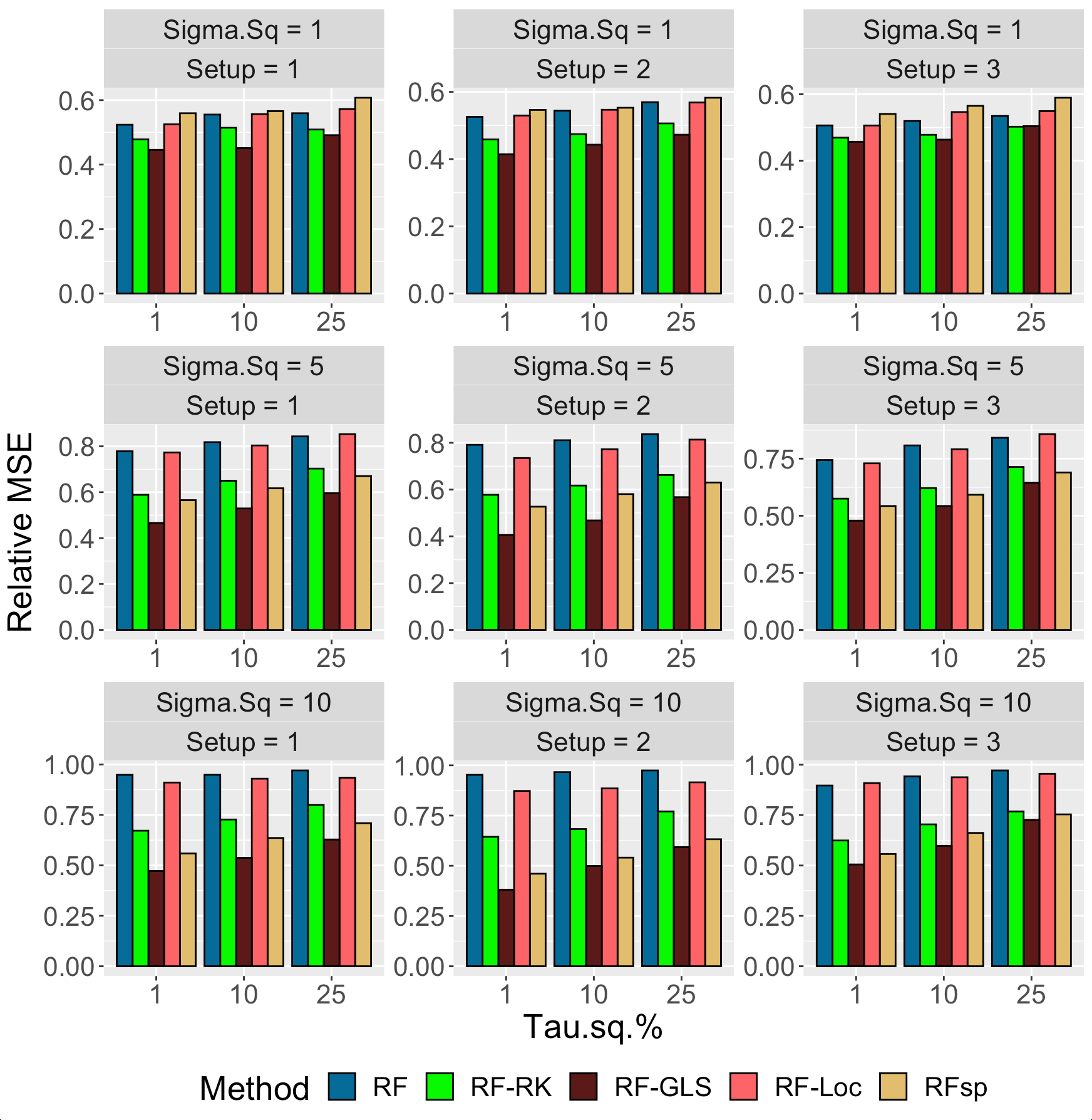}
			\caption{Prediction performance}
		\end{subfigure}
		\caption{Comparison between competing methods on (a) estimation and (b) spatial prediction when the spatial surface $w(\ell)$ is generated as a fixed smooth surface and the mean function is $m = m_3$.}\label{Fig:MV_smooth}
	\end{figure}
}

\newpage
\section{Proofs of main results}\label{sec:proofs}
\subsection{Approximation Error}\label{sec:pf-approx}
\begin{proof}[Proof of Lemma \ref{lem:beta}]
	For any $k$, denote the participating leaf nodes in the partition under consideration as $\mathcal{C}_1, \mathcal{C}_2, \cdots, \mathcal{C}_{g^{(k)}+1}$.
	Using Assumption \ref{as:chol}, and letting $\alpha=\|\brho\|^2$, we have from (\ref{eq:qf})
	\begin{equation}\label{eq:gram}
	\frac{1}{n}\left(\mathbf{Z}^\top\bQ \mathbf{Z} \right)_{l,l'} 
	= \frac{1}{n} \mathbf{Z}_{.,l}^\top\bQ\mathbf{Z}_{.,l'} = \frac{1}{n}\left[\alpha\sum_{{i}}\mathbf{Z}_{i,l}\mathbf{Z}_{i,l'} + \sum_{j \neq j' = 0}^q \rho_j\rho_{j'} \sum_i\mathbf{Z}_{i-j,l} \mathbf{Z}_{i-j',l'} + O_b(1) \right]
	\end{equation}
	Here the last term is $O_b(1)$ since $\mathbf{Z}_{i,l}$'s are independent and identically distributed for $i=1,\ldots,n$, and this term is a sum of $O(q^2)$ terms of the form $\mathbf{Z}_{i,l}\mathbf{Z}_{i',l'}$. Using strong law of large numbers, 
	\begin{equation*}
	\label{eq:AS_M1}
	\lim_{n \to \infty}\frac{1}{n} \sum_i\mathbf{Z}_{i,l}\mathbf{Z}_{i,l'} \overset{a.s.}{=} \mathbb{E}\left(\mathbf{Z}_{1,l}\mathbf{Z}_{1,l'} \right) = \mathbb{P}(X \in \mathcal{C}_l)\mathds{I}(l=l') = Vol(\mathcal{C}_l)\mathds{I}(l=l').
	\end{equation*}
	The last equality follows from the fact that $X$ is uniformly distributed over $[0,1]^D (\supseteq \mathcal{C}_l)$.
	
	Next note that the interaction terms   
	$t_i=\mathbf{Z}_{i-j,l} \mathbf{Z}_{i-j',l}$ are identically distributed but are not independent. However, as $0 \leq j,j' \leq q$, $t_i$ is independent of $t_{i'}$ for $i - i' > q$. Hence, $\{t_i\}$ is an {\em m-dependent process}, with maximum lag $q$ \citep{hoeffding1948central}. Following the hierarchy of mixing conditions \citep[p. 112,][]{bradley2005basic}, $m$-dependence $\Rightarrow \psi$-mixing $\Rightarrow \psi^*$-mixing $\Rightarrow$ information regularity $\Rightarrow$ absolute regularity.  
	Hence, $\{t_i\}$ is a stationary $\beta$-mixing process. Using Theorem 1 in \cite{nobel1993note} with the class $\mathcal{F}$ equal to the identity function, and the fact that  $|\mathbf{Z}_{i-j,l} \mathbf{Z}_{i-j',l'}| \leqslant 1$, we have the following for $j \neq j'$:
	\begin{equation*}
	\label{eq:AS_M2}
	\lim_{n \to \infty}\frac{1}{n} \sum_i\mathbf{Z}_{i-j,l} \mathbf{Z}_{i-j',l'} \overset{a.s.}{=} \mathbb{E}(\mathbf{Z}_{i-j,l} \mathbf{Z}_{i-j',l'}) \overset{ind.}{=} \mathbb{E}(\mathbf{Z}_{i-j,l})\mathbb{E}(\mathbf{Z}_{i-j',l'}) = Vol(\mathcal{C}_l)Vol(\mathcal{C}_l').
	\end{equation*}

	Letting $n \rightarrow \infty$ in both sides of  (\ref{eq:gram}), we have
	\begin{equation}\label{eq:gramlim}
	\lim_{n \to \infty}\frac{1}{n}\left(\mathbf{Z}^\top\bQ \mathbf{Z}\right)_{l,l'} \overset{a.s.}{=} \alpha Vol(\mathcal{C}_{{l}})\mathds{I}(l=l') + \sum_{j \neq j' = 0}^q \rho_j\rho_{j'} Vol(\mathcal{C}_{{l}})Vol(\mathcal{C}_{{l}'}).
	\end{equation}
	Since $\frac{1}{n}\mathbf{Z}^\top\bQ \mathbf{Z}$ is finite dimensional (i.e. $(g^{(k)}+1) \times (g^{(k)}+1)$, does not depend on $n$), we also have
	$$
	\lim_{n \to \infty}\frac{1}{n}\mathbf{Z}^\top\bQ \mathbf{Z} \overset{a.s.}{=} 
	\alpha \times diag(\mathbf{c}) + \left(\sum_{j \neq j' = 0}^q \rho_j\rho_{j'} \right)\mathbf{c}\mathbf{c}^\top
	$$
	where $\mathbf{c} = (c_1, c_2, \cdots, c_{g^{(k)}+1}), c_l = Vol(\mathcal{C}_{{l}})$ and $ \mathbf{1}^\top\mathbf{c} = 1. $

	Defining $\mathbf{c}^{-1} := (Vol(\mathcal{C}_{{1}})^{-1}, Vol(\mathcal{C}_{{2}})^{-1}, \cdots, Vol(\mathcal{C}_{{g^{(k)}+1}})^{-1})$ and using Sherman–Morrison–Woodbury identity  we have, 
	$$
	\begin{aligned}
	\left(\lim_{n \to \infty}\frac{1}{n}\mathbf{Z}^\top\bQ \mathbf{Z}\right)^{-1} &= \alpha ^{-1}diag(\mathbf{c}^{-1}) - \frac{\alpha ^{-2} diag(\mathbf{c}^{-1})\left(\sum_{j \neq j' = 0}^q \rho_j\rho_{j'}\right)\mathbf{c}\mathbf{c}^\top diag(\mathbf{c}^{-1}) }{1+\alpha ^{-1}\left(\sum_{j \neq j' = 0}^q \rho_j\rho_{j'}\right)\mathbf{c}^\top diag(\mathbf{c}^{-1}) \mathbf{c}}\\
	&=  \alpha ^{-1} diag(\mathbf{c}^{-1}) - \frac{\alpha ^{-1} \left(\sum_{j \neq j' = 0}^q \rho_j\rho_{j'}\right)\mathbf{1}\mathbf{1}^\top }{\alpha  + \left(\sum_{j \neq j' = 0}^q \rho_j\rho_{j'}\right)\mathbf{1}^\top \mathbf{c}}\\
	&= \alpha ^{-1} \left[diag(\mathbf{c}^{-1}) -\frac{\left(\sum_{j \neq j' = 0}^q \rho_j\rho_{j'}\right)}{\alpha + \sum_{j \neq j' = 0}^q \rho_j\rho_{j'}}\mathbf{1}\mathbf{1}^\top \right].
	\end{aligned}
	$$
	Next we focus on $\frac{1}{n}\mathbf{Z}^\top\bQ\mathbf{Y}$. Proceeding similarly, we have
	\begin{equation}\label{eq:cor}
	\begin{aligned}
	\frac{1}{n}(\mathbf{Z}^\top\bQ\mathbf{Y})_l &= \frac{1}{n}\left[\alpha\sum_i\mathbf{Z}_{i,l}Y_{i} + \sum_{j \neq j' = 0}^q \rho_j\rho_{j'} \sum_i\mathbf{Z}_{i-j,l} Y_{i-j'} + O_b(1) \right]\\
	&= \alpha \frac{\sum_i\mathbf{Z}_{i,l}(m(X_i) + \eps_i)}{n} + \sum_{j \neq j' = 0}^q \rho_j\rho_{j'} \frac{\sum_i\mathbf{Z}_{i-j,l} (m(X_{i-j'}) + \eps_{i-j'})}{n}+ O_b(1/n)\\
	&= \alpha \frac{\sum_i\mathbf{Z}_{i,l}m(X_i)}{n} + \sum_{j \neq j' = 0}^q \rho_j\rho_{j'} \frac{\sum_i\mathbf{Z}_{i-j,l} m(X_{i-j'} )}{n}\\
	&\quad + \alpha \frac{\sum_i\mathbf{Z}_{i,l}\eps_i}{n} + \sum_{j \neq j' = 0}^q \rho_j\rho_{j'} \frac{\sum_i\mathbf{Z}_{i-j,l} \eps_{i-j'} }{n}+ O_b(1/n)
	\end{aligned}
	\end{equation}
	Under Assumption \ref{as:mix},  Lemma \ref{lem:strlaw} with $B_i = \mathbf{Z}_{i,l}$ implies
	\begin{equation*}
	\label{eq:AS_P2_1}
	\lim_{n \to \infty}\frac{1}{n}\sum_i\mathbf{Z}_{i,l}\eps_i \overset{a.s.}{=} 0 \mbox{ and } 
	\lim_{n \to \infty}\frac{\sum_i\mathbf{Z}_{i-j,l} \eps_{i-j'} }{n}\overset{a.s.}{=} 0.
	\end{equation*}
	Since $\mathbf{Z}_{i,l}m(X_i)$'s are i.i.d copies of each other, using strong law of large numbers, we have
	$$
	\begin{aligned}
	&\lim_{n \to \infty}\frac{1}{n}\sum_i\mathbf{Z}_{i,l}m(X_i)\\ &\overset{a.s.}{=} \mathbb{E}\left(\mathbf{Z}_{1,l}m(X_1) \right)\\ 
	&= \mathbb{E}\left(\mathbf{Z}_{1,l}m(X_1) \right) + \mathbb{E}\left(\mathbf{Z}_{1,l}\eps_1 \right);\: \left[X \independent \eps, \mathbb{E}\left(\mathbf{Z}_{1,l}\eps_1 \right) = \mathbb{E}\left(\mathbf{Z}_{1,l}\right)\mathbb{E}\left(\eps_1 \right)=0 \right]\\
	&= \mathbb{E}\left(\mathbf{Z}_{1,l}(m(X_1) + \eps_1)) \right)\\
	&= \mathbb{E}\left(\mathbf{Z}_{1,l}Y_1 \right)\\
	&= \mathbb{E}\left(Y | X \in \mathcal{C}_l\right)Vol(\mathcal{C}_l).
	\end{aligned}
	$$
	Similarly, for $j \neq j'$, using {\em m-dependence} 
	we can show
	\begin{equation*}
	\label{eq:AS_P2_5}
	\lim_{n \to \infty}\frac{1}{n}\sum_i\mathbf{Z}_{i-j,l} m(X_{i-j'} ) \overset{a.s.}{=} \mathbb{E}\left(Y\right)Vol(\mathcal{C}_l).
	\end{equation*}
	Applying all the limits to (\ref{eq:cor}), we have 
	\begin{equation}\label{eq:r}
	\begin{aligned}
	\hat{r}_l := \lim_{n \to \infty} \frac{1}{n}(\mathbf{Z}^\top\bQ\mathbf{Y})_l &\overset{a.s.}{=} \alpha \mathbb{E}(Y |X \in \mathcal{C}_l)Vol(\mathcal{C}_l)+Vol(\mathcal{C}_l) \mathbb{E}(Y)\sum_{j \neq j' = 0}^q \rho_j\rho_{j'} \\
	&= \alpha \mathbb{E}(Y\mathds{I}(X \in \mathcal{C}_l))+Vol(\mathcal{C}_l) \mathbb{E}(Y)\sum_{j \neq j' = 0}^q \rho_j\rho_{j'}. 
	\end{aligned}
	\end{equation}
	Finally, defining $\hat{\mathbf{r}} := \left(\hat{r}_1, \cdots, \hat{r}_{g^{(k)}+1}\right)$, and noting that the dimension of  $\frac{1}{n}\mathbf{Z}^\top\bQ \mathbf{Z}$ does not grow with $n$, we have 
	$$
	\begin{aligned}
	\lim_{n \to \infty} \bm{\hat\beta} &= \lim_{n \to \infty} \left(\frac{1}{n}\mathbf{Z}^\top\bQ\mathbf{Z} \right)^{-1}  \lim_{n \to \infty}\frac{1}{n}\left(\mathbf{Z}^\top\bQ\mathbf{Y}\right)\\
	&=  \alpha ^{-1} \left[diag(\mathbf{c}^{-1}) -\frac{\left(\sum_{j \neq j' = 0}^q \rho_j\rho_{j'}\right)}{\alpha + \sum_{j \neq j' = 0}^q \rho_j\rho_{j'}}\mathbf{1}\mathbf{1}^\top \right] \hat{\mathbf{r}}\\
	&= \alpha ^{-1} \left[diag(\mathbf{c}^{-1})\hat{\mathbf{r}} -\frac{\left(\sum_{j \neq j' = 0}^q \rho_j\rho_{j'}\right)}{\alpha + \sum_{j \neq j' = 0}^q \rho_j\rho_{j'}}\mathbf{1}\mathbf{1}^\top\hat{\mathbf{r}} \right]\\
	&= \alpha ^{-1} \left[\mathbf{\hat{o}}  - \frac{\left(\sum_{j \neq j' = 0}^q \rho_j\rho_{j'}\right)}{\alpha+\left(\sum_{j \neq j' = 0}^q \rho_j\rho_{j'}\right)}\mathbf{1}\left(\alpha\mathbb{E}(Y) + \mathbb{E}(Y)\sum_{j \neq j' = 0}^q \rho_j\rho_{j'}  \right)\right],
	\end{aligned}
	$$
	where $\mathbf{\hat{o}} := \left(\hat{o}_1, \hat{o}_2, \cdots,\hat{o}_{g^{(k)}+1} \right); \hat{o}_l = \hat{r}_l/c_l = \alpha \mathbb{E}(Y |X \in \mathcal{C}_l)+\mathbb{E}(Y)\sum_{j \neq j' = 0}^q \rho_j\rho_{j'}$. Hence
	$$
	\begin{aligned}
	\lim_{n \to \infty} \bm{\hat\beta}_l &\overset{a.s.}{=} \alpha ^{-1} \hat{o}_l - \mathbb{E}(Y)\frac{(\sum_{j \neq j' = 0}^q \rho_j\rho_{j'})(\alpha+\sum_{j \neq j' = 0}^q \rho_j\rho_{j'})}{\alpha (\alpha + \sum_{j \neq j' = 0}^q \rho_j\rho_{j'})}\\
	&= \mathbb{E}(Y |X \in \mathcal{C}_l) + \mathbb{E}(Y) \left(\frac{\sum_{j \neq j' = 0}^q \rho_j\rho_{j'}}{\alpha } - \frac{\sum_{j \neq j' = 0}^q \rho_j\rho_{j'}}{\alpha} \right)\\
	&= \mathbb{E}(Y |X \in \mathcal{C}_l).
	\end{aligned}
	$$
	This completes the proof of Lemma \ref{lem:beta}.
\end{proof}

\begin{proof}[Proof of Theorem \ref{lemma:DART-theoretical}] 
	With $k = k+1$, let $\mathcal{C}_l = \mathcal{C}_l^{(k)}$ for $l < g^{(k)}$; $\mathcal{C}_{g^{(k)}} = \mathcal{C}_{g^{(k)}}^{(k+1)}$, $\mathcal{C}_{g^{(k)}+1} = \mathcal{C}_{g^{(k)}+1}^{(k+1)}$. Since 	$
	\left(\mathbf{Y} - \mathbf{Z}\bm{\hat{\beta}}(\mathbf Z) \right)^\top \bQ\left(\mathbf{Y} - \mathbf{Z}\bm{\hat{\beta}}(\mathbf Z) \right)
	= \mathbf{Y}^\top\bQ\mathbf{Y} -  \mathbf{Y}^\top\bQ\mathbf{Z} \bm{\hat{\beta}}(\mathbf Z)
	$, we have 
	$$
	v_{n,\bQ}(\mathfrak{C}^{(k)},l_1,(d,c)) = \frac{1}{n}\left(\mathbf{Y}^\top\bQ\mathbf{Z} \bm{\hat{\beta}}(\mathbf Z) - \mathbf{Y}^\top\bQ\mathbf{Z}^{(0)} \bm{\hat{\beta}}(\mathbf Z^{(0)}) \right)
	$$
	where 
	$
	\mathbf Z^{(0)}_{{i,{l}}}  =\mathds{I}\left(\mathbf{x}_i \in {\mathcal{C}}_{{l}}\right);\: \mbox{ for } {l} = 1,2, \cdots, g^{(k)} -1,  
	$ and $
	\mathbf Z^{(0)}_{{i,{g^{(k)}}}} = \mathds{I}\left(\mathbf{x}_i \in {\mathcal{C}}_{g^{(k)}} \cup {\mathcal{C}}_{g^{(k)}+1}\right). 
	$
	We use the notations 
	$\mathcal B = \mathcal{B}^L \cup \mathcal{B}^R,\mbox{ where } \mathcal B^L = \mathcal{C}_{g^{(k)}} \mbox{ and } \mathcal B^R = \mathcal{C}_{g^{(k)}+1}$.
	
	From Lemma \ref{lem:beta}, we have  
	$
	\lim_{n \to \infty} \frac{1}{n} \mathbf{Y}^\top\bQ\mathbf{Z} \bm{\hat{\beta}}(\mathbf Z) 
	\overset{a.s.}{=} \hat{\mathbf{r}}^\top\hat{\mathbf{b}}
	$,
	where $\hat{\mathbf{b}} := (\hat{b}_1, \hat{b}_2, \cdots, \hat{b}_{g^{(k)}+1}); \hat{b}_l = \mathbb{E}(Y | X \in \mathcal{C}_l)$ and $\hat\br$ is defined in (\ref{eq:r}). 
	Now 
	$$
	\begin{aligned}
	\hat\br^\top\hat\bb &= \sum_{l = 1}^{g^{(k)}+1} \mathbb{E}(Y |X \in \mathcal{C}_l)\left(\alpha \mathbb{E}(Y |X \in \mathcal{C}_l)Vol(\mathcal{C}_l)+Vol(\mathcal{C}_l) \mathbb{E}(Y)\sum_{j \neq j' = 0}^q \rho_j\rho_{j'}  \right)\\
	&= \sum_{l = 1}^{g^{(k)}+1} \left(\alpha \mathbb{E}(Y |X \in \mathcal{C}_l)^2Vol(\mathcal{C}_l)+Vol(\mathcal{C}_l) \mathbb{E}(Y)\mathbb{E}(Y |X \in \mathcal{C}_l)\sum_{j \neq j' = 0}^q \rho_j\rho_{j'}  \right)\\
	&= \sum_{l = 1}^{g^{(k)}+1} \alpha \mathbb{E}(Y |X \in \mathcal{C}_l)^2Vol(\mathcal{C}_l) + \mathbb{E}(Y)^2\sum_{j \neq j' = 0}^q \rho_j\rho_{j'}. 
	\end{aligned}
	$$
	Substituting this in the expression of asymptotic value of $ v_{n,\bQ}(\mathfrak{C}^{(k)},l_1,(d,c))$,  
	we have
	\begin{align*}
	&\lim_{n \to \infty} v_{n,\bQ}(\mathfrak{C}^{(k)},l_1,(d,c))\\ &\overset{a.s.}{=} \sum_{l = g^{(k)}}^{g^{(k)}+1} \alpha \mathbb{E}(Y |X \in \mathcal{C}_l)^2Vol(\mathcal{C}_l)  -\alpha \mathbb{E}(Y |X \in \mathcal{C}_{g^{(k)}} \cup\mathcal{C}_{g^{(k)}+1})^2Vol(\mathcal{C}_{g^{(k)}} \cup \mathcal{C}_{g^{(k)}+1})\\
	&= \alpha \left(\sum_{l = g^{(k)}}^{g^{(k)}+1} \mathbb{E}(Y |X \in \mathcal{C}_l)^2Vol(\mathcal{C}_l) - \mathbb{E}(Y |X \in \mathcal{B})^2Vol(\mathcal{B})\right)\\
	&= \alpha \Bigg( \left[\mathbb{E}(Y^2 |X \in \mathcal{B}) - \mathbb{E}(Y |X \in \mathcal{B})^2\right]Vol(\mathcal{B})\\ 
	&\quad\quad\quad\quad- \left[\mathbb{E}(Y^2 |X \in \mathcal{B})Vol(\mathcal B)  - \sum_{l = g^{(k)}}^{g^{(k)}+1} \mathbb{E}(Y |X \in \mathcal{C}_l)^2Vol(\mathcal{C}_l)\right] \Bigg)\\
	&= \alpha \left( \mathbb{V}(Y |X \in \mathcal{B})Vol(\mathcal{B}) - \left[\mathbb{E}(Y^2 \mathds{I}(X \in \mathcal{B}))  - \sum_{l = g^{(k)}}^{g^{(k)}+1} \mathbb{E}(Y |X \in \mathcal{C}_l)^2Vol(\mathcal{C}_l)\right]\right)\\
	&= \alpha \Bigg( \mathbb{V}(Y |X \in \mathcal{B})Vol(\mathcal{B})\\ 
	&\quad\quad\quad\quad-\left[\mathbb{E}(Y^2 \mathds{I}(X \in \mathcal{C}_{g^{(k)}} \cup \mathcal{C}_{g^{(k)}+1}))  - \sum_{l = g^{(k)}}^{g^{(k)}+1} \mathbb{E}(Y |X \in \mathcal{C}_l)^2Vol(\mathcal{C}_l)\right] \Bigg)\\
	&= \alpha \left( \mathbb{V}(Y |X \in \mathcal{B})Vol(\mathcal{B}) - \sum_{l = g^{(k)}}^{g^{(k)}+1} \mathbb{V}(Y |X \in \mathcal{C}_l)Vol(\mathcal{C}_l)\right)\\
	&=\alpha Vol(\mathcal B) \Big[\mathbb{V}(Y | \mathbf{X} \in \mathcal B) - \mathbb{P}(\mathbf{X} \in \mathcal B^R | \mathbf{X} \in \mathcal B)\mathbb{V}(Y | \mathbf{X} \in \mathcal B^R)\\ 
	& \quad \quad \quad \quad \quad \quad- \mathbb{P}(\mathbf{X} \in \mathcal B^L |  \mathbf{X} \in \mathcal B)\mathbb{V}(Y | \mathbf{X} \in \mathcal B^L) \Big].
	\end{align*}
	This completes the proof of Theorem \ref{lemma:DART-theoretical}.
\end{proof}

\begin{proof}[Proof of Proposition \ref{lemma:equicontinuity}] We first introduce some additional notation on splits. 
	
	\subsection*{Notation of splits}\label{sec:splitnot}
	We introduce some additional notations of splits. The split associated with a specific node (a subset of the feature space) indicates the direction and the cutoff associated with its partition. A split is denoted by $s = (d, c)$, where $d$ denotes the direction of the split (the feature along which the aforementioned split is performed) $\in \{1,2,\cdots,D \}$ and $c$ is the cutoff value of the split. 
	Let the complete set of nodes in level $k$ is $\mathfrak{C}^{(k)} = \{\mathcal{C}_1^{(k)},  \cdots, \mathcal{C}_{g^{(k)}}^{(k)}\}$. Let the split of the $l^{th}$ node of $k^{th}$ level i.e. $\mathcal{C}_{l}^{(k)}$ be denoted by ${s}_{l}^{(k)}$. We observe that $\mathfrak{C}^{(k+1)}$ is determined by $\mathfrak{C}^{(k)}$ and $\mathcal{S}^{(k)}$ where $\mathcal{S}^{(k)} = \{{s}_{1}^{(k)}, {s}_{2}^{(k)}, \cdots,  {s}_{g^{(k)}}^{(k)}\}$ is the set of splits on the partitions in level $k$ to create the partitions at level $k+1$. This in turn implies that $\mathfrak{C}^{(k+1)}$ is determined by $\tilde{\mathbf{s}}_{k} = \{ \mathcal{S}^{(1)}, \mathcal{S}^{(2)}, \cdots, \mathcal{S}^{(k)}\}$, as by definition, $\mathcal{C}^{(1)} = [0,1]^D$. We define the set of all possible such $\tilde{\mathbf{s}}_{k}$ to be $\tilde{\mathfrak{S}}_k$. 
	
	For any fixed $\mathbf{x} \in [0,1]^D$ and $k \geqslant 1$, $\mathfrak{S}_k(\mathbf{x})$ to is the set of all possible splits that built the node containing $\mathbf{x}$ in $k+1^{th}$ level. Members of $\mathfrak{S}_k(\mathbf{x})$ are denoted as $\mathbf s_k := \mathbf s_k(\mathbf x) = (\mathcal{S}^{(1)}, \mathcal{S}^{(2)}, \cdots, \mathcal{S}^{(k-1)}, s^{(k)}_{(\mathbf x)}) = (\tilde{\mathbf{s}}_{k-1}, s^{(k)}_{(\mathbf x)})$, where $\tilde{\mathbf{s}}_{k-1} \in \tilde{\mathfrak{S}}_{k-1}$ and $s^{(k)}_{(\mathbf x)}$ denotes the split associated with the node at level $k$, containing $\mathbf{x}$. 
	
	The node at level $k+1$, containing $\mathbf{x}$, built with $\mathbf{s}_k$ is $\mathcal B(\mathbf{x}, \mathbf{s}_k)$. The node containing $\bx$ in a RF-GLS tree built with random parameter $\Theta$ and data $\calD_n$ is denoted by $\mathcal{B}_n(\bx, \Theta)$. 
	The optimal splits obtained from empirical DART-split criterion, that build the node containing $\bx$ at level $k+1$ of RF-GLS tree (built with $n$ points and randomness $\Theta$) is denoted as $\hat{\mathbf{s}}_{k,n}(\bx, \Theta)$. The distance between $\mathbf{s}_{k}^{(1)}, \mathbf{s}_{k}^{(2)} \in  \mathfrak{S}_k(\mathbf{x})$ is the $\mathbb{L}_{\infty}$ norm of their difference; i.e. $\|\mathbf{s}_{k}^{(1)} - \mathbf{s}_{k}^{(2)} \|_{\infty}$. The distance between $\mathbf{s}_{k} \in \mathfrak{S}_k(\mathbf{x})$ and $\acute{\mathfrak{S}_k} \subseteq \mathfrak{S}_k(\mathbf{x})$ is 
	$
	dist(\mathbf{s}_{k},\acute{\mathfrak{S}_k}) = \inf_{\mathbf{s} \in \acute{\mathfrak{S}_k}} \|\mathbf{s}_{k} - \mathbf{s} \|_\infty
	$. 
	
	Next, for a fixed $\mathbf{x} \in [0,1]^D$, and any $k$ levels of split $\mathbf{s}_k$, we define $v_{n,k,\bQ}(\mathbf{x}, \mathbf{s}_k)$ to be the DART split criterion (\ref{DART_def}) to maximize in $s^{(k)}$ of $\mathbf{s}_k$, i.e. the final $k^{th}$ level split of $\mathcal{B}(\mathbf{x}, \mathbf{s}_{k-1})$. For all $\varepsilon> 0$, we define $\tilde{\mathfrak{S}}_{k-1}^\varepsilon\subset \tilde{\mathfrak{S}}_{k-1}$, the set of all splits of the $(k-1)$ level nodes, such that each node in $\mathfrak{C}^{(k)}$ contains a hypercube of edge length $\varepsilon$. Additionally, we define $\Bar{\mathfrak{S}}_{k}^\varepsilon(\mathbf{x}) = \{\mathbf{s}_k := \mathbf{s}_k(\bx): \tilde{\mathbf{s}}_{k-1} \in \tilde{\mathfrak{S}}_{k-1}^\varepsilon\}$.  
	
	Equipped with the notation. We state a more technical version of equicontinuity result in Proposition \ref{lemma:equicontinuity}. 
	The proof is deferred to Section \ref{sec:pf-approx}.
	
	\noindent \textbf{Technical statement of Proposition \ref{lemma:equicontinuity}:} 
	Under Assumptions \ref{as:chol}, \ref{as:diag}, and \ref{as:tail}(b) and \ref{as:tail}(c), for fixed $\mathbf x$, $k \in \mathbb{N}$ and $\varepsilon> 0$, $v_{n,k,\bQ}(\mathbf{x},\bs_k(\bx))$ is stochastically equicontinuous with respect to $\bs_k$ on $\Bar{\mathfrak{S}}_k^\varepsilon(\mathbf{x})$, i.e. $\forall \phi, \pi > 0$, $\exists\: \delta > 0$, i.e., 
	$$
	\lim_{n \to\infty} \mathbb{P} \left[\sup_{\underset{\mathbf{s}_{k}^{(1)},\mathbf{s}_{k}^{(2)} \in \Bar{\mathfrak{S}}_k^\varepsilon(\mathbf{x})}{\|\mathbf{s}_{k}^{(1)} - \mathbf{s}_{k}^{(2)} \|_{\infty} \leqslant \delta}} |v_{n,k, \bQ}(\mathbf{x},\mathbf{s}_{k}^{(1)}) - v_{n,k, \bQ}(\mathbf{x},\mathbf{s}_{k}^{(2)}) | > \phi\right] \leqslant \pi.
	$$
	
	We will show that $v_{n,k,\bQ}(\mathbf{x},\bs_k)$ (defined in Section \ref{sec:splitnot}) is stochastically equicontinuous with respect to $\bs_k$ for all $ \mathbf{x} \in [0,1]^D$, provided the volumes of leaf nodes in the previous level are not arbitrarily close to $0$. 
	
	By the Glivenko-Cantelli theorem \citep{loeve1977elementary}, $\exists n_3 \in \mathbb{N}$, such that for all $n > n_3$, and $\mathcal{C} \subseteq [0,1]^D$, we have with probability at least $1 - \pi/4$,
	\begin{equation}\label{eq:glc}
	Vol(\mathcal{C})- \delta^2 \leqslant \frac{1}{n}\sum_{i} \mathds{I}(X_i \in \mathcal{C}) \leqslant Vol(\mathcal{C})+ \delta^2.    
	\end{equation}
	As each split can be chosen from at-most $n^D$ candidates, the collection $\calP$ of all possible nodes created by splits up to level $k$ is of polynomial-in-$n$ cardinality. Since Assumptions \ref{as:tail}(b) holds for some $n_2$ and $\pi=\pi/8$, and \ref{as:tail}(c) holds, then Lemma \ref{lem:weaklaw2} holds for $\calP$ for some $n_3$ and $\pi=\pi/4$. For the rest of the proof, we consider $n > \max\{n_0, n_2, n_3\}$ and restrict ourselves to the set $\Omega_n$ where all of these assertions (Equation \ref{eq:glc}, Assumption \ref{as:tail}(b) and Lemma \ref{lem:weaklaw2}) hold, which occurs with probability at-least $1-3\pi/4$.
	
	From the definition of the distance in $\mathfrak{S}_{k}$, if  $\mathbf{s}_{k}^{(1)}, \mathbf{s}_{k}^{(2)} \in {\mathfrak{S}}_k$ satisfy $\|\mathbf{s}_{k}^{(1)} - \mathbf{s}_{k}^{(2)} \|_\infty < 1$, then the split directions are identical. So we can always consider $\delta <1$. Since we consider two sets of splits $\mathbf{s}_{k}^{(1)}$ and $\mathbf{s}_{k}^{(2)}$, for convenience of notation, we use $\ddot{\mathcal{C}}_{{l}}^{(h)}$ to denote the ${l}^{th}$ leaf node of the partition induced by $\mathbf{s}_{k}^{(h)}$, for $h=1,2$ and $l = 1, \ldots, g^{(k)}+1$, i.e, $\ddot{\mathcal{C}}_{{l}}^{(h)}$ is the node $\calC^{(k)}_l$ for the $h^{th}$ set of splits.  Also, we will be using the notation $\mathbf Z_1$ and $\mathbf{Z}_2$ instead of $\mathbf Z(\mathbf{s}_{k}^{(1)})$ and $\mathbf Z(\mathbf{s}_{k}^{(2)})$ respectively. Along the same line as proof of Theorem \ref{lemma:DART-theoretical}, we can write 
	\begin{equation}\label{eq:equi}
	\begin{aligned}
	|v_{n,k, \bQ}(\mathbf{x},\mathbf{s}_{k}^{(1)}) - v_{n,k, \bQ}(\mathbf{x},\mathbf{s}_{k}^{(2)}) | &= \frac{1}{n}|\mathbf{Y}^\top\bQ\mathbf{Z}_1 \bm{\hat{\beta}}(\mathbf{Z}_1) - \mathbf{Y}^\top\bQ\mathbf{Z}_2 \bm{\hat{\beta}}(\mathbf{Z}_2)| \\
	&+ \frac{1}{n}|\mathbf{Y}^\top\bQ\mathbf{Z}^{(0)}_1 \bm{\hat{\beta}}(\mathbf Z^{(0)}_1) - \mathbf{Y}^\top\bQ\mathbf{Z}^{(0)}_2 \bm{\hat{\beta}}(\mathbf Z^{(0)}_2)| 
	\end{aligned},
	\end{equation}
	where $
	\mathbf Z^{(0)}_{h_{i,{l}}}  =\mathds{I}\left(\mathbf{x}_i \in {\ddot{\mathcal{C}}}_{{l}}^{(h)}\right);\: {l} = 1,2, \cdots, g^{(k)} -1,
	$
	and $
	\mathbf Z^{(0)}_{h_{i,{g^{(k)}}}} = \mathds{I}\left(\mathbf{x}_i \in \ddot{\mathcal{C}}_{g^{(k)}}^{(h)} \cup \ddot{\mathcal{C}}_{g^{(k)}+1}^{(h)}\right) 
	$ for $h = 1,2$.
	
	We first focus on $\frac{1}{n}|\mathbf{Y}^\top\bQ\mathbf{Z}_1 \bm{\hat{\beta}}(\mathbf{Z}_1) - \mathbf{Y}^\top\bQ\mathbf{Z}_2 \bm{\hat{\beta}}(\mathbf{Z}_2)| = \frac{1}{{n}} |\mathbf{Y}^\top \bQ^{\frac{1}{2}} \left[ {\mathbf{\dot P}}_{\mathbf{Z}_1} - {\mathbf{\dot P}}_{\mathbf{Z}_2}\right]\bQ^{\frac \top 2} \mathbf{Y}|
	$
	where,
	$$
	{\bQ^{\frac{1}{2}} = \bm\Sigma^{-\frac{T}{2}};\: \mathbf{\dot P}}_{\mathbf{Z}} = \bQ^{\frac \top 2} \mathbf{Z} \left[\mathbf Z^\top \bQ\mathbf Z\right]^{-1}\mathbf Z^{\top} \bQ^{\frac{1}{2}};\:\: \mathbf{Z} \in \{\mathbf{Z}_1, \mathbf{Z}_2\},
	$$

	and consider the two possible scenarios:
	
	\begin{itemize}
		\item \textbf{R1}: $\max\left( \frac{\min \left(Vol(\ddot{\mathcal{C}}_{g^{(k)}}^{(1)}), Vol(\ddot{\mathcal{C}}_{g^{(k)}+1}^{(1)})\right)}{Vol(\ddot{\mathcal{C}}_{g^{(k)}}^{(1)}) + Vol(\ddot{\mathcal{C}}_{g^{(k)}+1}^{(1)})},\frac{\min \left(Vol(\ddot{\mathcal{C}}_{g^{(k)}}^{(2)}), Vol(\ddot{\mathcal{C}}_{g^{(k)}+1}^{(2)})\right)}{Vol(\ddot{\mathcal{C}}_{g^{(k)}}^{(2)}) + Vol(\ddot{\mathcal{C}}_{g^{(k)}+1}^{(2)})} \right) \geqslant \sqrt{\delta}$
		
		\item \textbf{R2}: $\max\left( \frac{\min \left(Vol(\ddot{\mathcal{C}}_{g^{(k)}}^{(1)}), Vol(\ddot{\mathcal{C}}_{g^{(k)}+1}^{(1)})\right)}{Vol(\ddot{\mathcal{C}}_{g^{(k)}}^{(1)}) + Vol(\ddot{\mathcal{C}}_{g^{(k)}+1}^{(1)})},\frac{\min \left(Vol(\ddot{\mathcal{C}}_{g^{(k)}}^{(2)}), Vol(\ddot{\mathcal{C}}_{g^{(k)}+1}^{(2)})\right)}{Vol(\ddot{\mathcal{C}}_{g^{(k)}}^{(2)}) + Vol(\ddot{\mathcal{C}}_{g^{(k)}+1}^{(2)})} \right) < \sqrt{\delta}$
	\end{itemize}
	
	Scenario \textbf{R1} happens when at least for one of the two sets of splits, both the new child nodes are ``significantly different" from their parent node; i.e. their volumes are bounded away from the volume of the parent node and zero. Here, we will show equicontinuity by exploiting perturbation bounds on orthogonal projections \citep{chen2016perturbation}. 
	The other possibility is Scenario \textbf{R2}  where for both the set of splits, the volume of the larger child node is arbitrary close to that of parent node. Here, we prove equicontinuity by showing that the DART-split criterion value asymptotically vanishes. 
	
	Without loss of generality, we consider $\delta > 0 $ small enough such that, $\sqrt{\delta} < 1 - \sqrt{\delta} $. Under \textbf{R1}, we have
	\begin{equation}\label{eq:singularvalue}
	\frac{1}{{n}} |\mathbf{Y}^\top \bQ^{\frac{1}{2}} \left[ {\mathbf{\dot P}}_{\mathbf{Z}_1} - {\mathbf{\dot P}}_{\mathbf{Z}_2}\right]\bQ^{\frac \top 2} \mathbf{Y}|
	\leqslant \frac{1}{{n}}\mathbf{Y}^\top\bQ \mathbf{Y} \|{\mathbf{\dot P}}_{\mathbf{Z}_1} - {\mathbf{\dot P}}_{\mathbf{Z}_2}\|_2.
	\end{equation}
	Defining $
	\mathbf{H}_{\mathbf Z} = \left[\mathbf Z^\top \bQ\mathbf Z\right]^{-1}\mathbf Z^{\top} \bQ^{\frac{1}{2}} \mbox{ for } \mathbf{Z} = \mathbf{Z}_1, \mathbf{Z}_2
	$ and using the perturbation bounds on projection operators from \cite{chen2016perturbation} Theorem 1.2 (1.8), we have
	\begin{equation}\label{eq:pert}
	\|{\mathbf{\dot P}}_{\mathbf{Z}_1} - {\mathbf{\dot P}}_{\mathbf{Z}_2}\|_2 \leqslant \min \{\|\mathbf{H}_{\mathbf{Z}_1} \|_2, \|\mathbf{H}_{\mathbf{Z}_2} \|_2\}
	\|\bQ^{\frac \top 2}\mathbf{Z}_1 - \bQ^{\frac \top 2}\mathbf{Z}_2 \|_2.
	\end{equation}
	By definition of matrix  $\mathbb{L}_2$ norm, $\|\bQ^{\frac \top 2}\mathbf{Z}_1 - \bQ^{\frac \top 2}\mathbf{Z}_2 \|_2 \leq \lambda^{\frac 12}_{\max}(\bQ) \|\bZ_1 - \bZ_2\|_2$. Since  Assumption \ref{as:chol} implies  $\lambda_{\max}(\bQ)$ is bounded, we focus on $\ddot{\mathbf{D}}=\bZ_1 - \bZ_2$. 
	By Gershgorin circle theorem, 
	$$
	\begin{aligned}
	\lambda_{\max}(\ddot{\bD}^2) &\leqslant \max_{1 \leqslant l_1 \leqslant g^{(k)}+ 1} \sum_{l_2=1}^{g^{(k)}+ 1} |(\ddot{\bD}^2)_{l_1,l_2}|\\
	&= \max_{1 \leqslant l_1 \leqslant g^{(k)}+ 1} \sum_{l_2=1}^{g^{(k)}+ 1} \sum_{{i}}|\ddot{\mathbf{D}}_{{i,l_1}}\ddot{\mathbf{D}}_{{i,l_2}}|\\
	&\leqslant \max_{1 \leqslant l_1 \leqslant g^{(k)}+ 1} \sum_i |\ddot{\mathbf{D}}_{{i,l_1}}| \sum_{l_2=1}^{g^{(k)}+ 1} |\ddot{\mathbf{D}}_{{i,l_2}}|\\
	& \leqslant 2 \max_{1 \leqslant l_1 \leqslant g^{(k)}+ 1} \sum_i |\ddot{\mathbf{D}}_{{i,l_1}}|. 
	\end{aligned}
	$$
	The last inequality follows from the fact that the $\bZ_1$ and $\bZ_2$ are binary matrices whose row sums are $1$. 
	Now 
	$$
	\sum_i|\ddot{\mathbf{D}}_{{i,l}}| = \sum_i \mathds{I}(\mathbf{x}_i \in  \ddot{\mathcal{C}}_{l}^{(1)} \bigtriangleup \ddot{\mathcal{C}}_{l}^{(2)}) \mbox{ where } \ddot{\mathcal{C}}_{l}^{(1)} \bigtriangleup \ddot{\mathcal{C}}_{l}^{(2)} = \left(\ddot{\mathcal{C}}_{l}^{(1)} \cup \ddot{\mathcal{C}}_{l}^{(2)}\right) \cap \left(\ddot{\mathcal{C}}_{l}^{(1)} \cap \ddot{\mathcal{C}}_{l}^{(2)}  \right)^c.
	$$
	From Algorithm \ref{algo:glstree}, $\ddot{\mathcal{C}}_{l}^{(1)}$ and $\ddot{\mathcal{C}}_{l}^{(2)}$ are both $D$-dimensional boxes since both of them are Cartesian product of $D$ intervals. Let 
	$$
	\ddot{\mathcal{C}}_{l}^{(h)} = [\breve{a}_1^{(h)}, \breve{b}_1^{(h)}] \times \cdots \times  [\breve{a}_D^{(h)}, \breve{b}_D^{(h)}] \subseteq [0,1]^D; \breve{a}_d^{(h)} < \breve{b}_d^{(h)}; \forall d \in \{1,2,\cdots, D \};  h=1,2. 
	$$
	Then 
	$$
	|\mathbf{s}_k^{(1)} - \mathbf{s}_k^{(2)}| \leqslant \delta \mbox{  implies } |\breve{a}_d^{(1)} - \breve{a}_d^{(2)}| \leqslant \delta; |\breve{b}_d^{(1)} - \breve{b}_d^{(2)}| \leqslant \delta; \forall d \in \{1,2,\cdots, D\}.
	$$
	Without loss of generality, we assume $Vol(\ddot{\mathcal{C}}_{l}^{(1)}) \leqslant Vol(\ddot{\mathcal{C}}_{l}^{(2)})$. One scenario where $Vol(\ddot{\mathcal{C}}_{l}^{(1)} \bigtriangleup \ddot{\mathcal{C}}_{l}^{(2)})$ is maximized is
	$$
	\breve{a}_d^{(2)} = \breve{a}_d^{(1)} - \delta;\:\: \breve{b}_d^{(2)} = \breve{b}_d^{(1)} + \delta;\:\: \forall d \in \{1,2,\cdots, D \}.
	$$
	Hence we have
	$$
	\begin{aligned}
	Vol(\ddot{\mathcal{C}}_{l}^{(1)} \bigtriangleup \ddot{\mathcal{C}}_{l}^{(2)}) &\leqslant \prod_{d=1}^D |\breve{b}_d^{(2)} - \breve{a}_d^{(2)} | - \prod_{d=1}^D |\breve{b}_d^{(2)} - \breve{a}_d^{(2)} - 2\delta|\\
	&\leqslant \sum_{d=1}^D \binom{D}{d}(2\delta)^d\\
	&\leqslant \delta 2^D\left(\sum_{d=1}^D \binom{D}{d}\right); [\text{ as $\delta < 1$}]\\
	&\leqslant 2^{2D}\delta
	\end{aligned}
	$$
	By (\ref{eq:glc})  on $\Omega_n$, $\frac{1}{n}\sum_i|\ddot{\mathbf{D}}_{{i,l}}| 
	\leqslant Vol(\ddot{\mathcal{C}}_{l}^{(1)} \bigtriangleup \ddot{\mathcal{C}}_{l}^{(2)})+ \delta^2 \leqslant 2^{2D}\delta + \delta^2 = O(\delta)$. Therefore $\frac 1n \lambda_{\max} (\ddot{\bD}^2) = \frac 1n \|\bZ_1 - \bZ_2\|^2_2 =O_p(\delta)$. Plugging this into (\ref{eq:pert}), we have 
	\begin{subequations}
		\begin{equation}
		\label{eq:L4_P1_1}
		\|\bQ^{\frac \top 2}\mathbf{Z}_1 - \bQ^{\frac \top 2}\mathbf{Z}_2 \|_2=O_p(\sqrt{n\delta}).
		\end{equation}
		For the other component in (\ref{eq:pert}), i.e. 
		$\min \{\|\mathbf{H}_{\mathbf{Z}_1} \|_2, \|\mathbf{H}_{\mathbf{Z}_2} \|_2\}$, for $\mathbf{Z} \in \{\mathbf{Z}_1,\mathbf{Z}_2\}$, we have  
		$$
		\begin{aligned}
		\|\mathbf{H}_{\mathbf{Z}} \|_2 &= \sqrt{\lambda_{max} ((\bQ^{\frac{T}{2}}\mathbf Z \left[\mathbf Z^\top \bQ\mathbf Z\right]^{-1}) (\bQ^{\frac{T}{2}}\mathbf Z \left[\mathbf Z^\top \bQ\mathbf Z\right]^{-1})^\top)}\\
		&=\sqrt{\lambda_{\max}(\left[\mathbf Z^\top \bQ\mathbf Z\right]^{-1})}  
		= \left(\lambda_{\min}(\mathbf Z^\top \bQ\mathbf Z)\right)^{-1/2}.
		\end{aligned}
		$$
		From Gershgorin circle theorem \citep{loeve1977elementary}, we have
		$$
		\lambda_{\min}(\left[\mathbf Z^\top \bQ\mathbf Z\right]) \geqslant \min_{1 \leqslant l_1 \leqslant g^{(k)}+1} \Big\{ (\mathbf Z^\top \bQ\mathbf Z)_{l_1,l_1} - \sum_{l_2 \neq l_1} |(\mathbf Z^\top \bQ\mathbf Z)_{l_1,l_2}|\Big\}.
		$$
		Using Lemma \ref{lem:diag} and the diagonal dominance from Assumptions \ref{as:chol} and \ref{as:diag},  
		\begin{equation}\label{eq:eigen}
		\lambda_{\min}(\left[\mathbf Z^\top \bQ\mathbf Z\right]) \geq \xi \sum_i \bZ_{i,l_1} \mbox{ where } \xi = \min_{1\leq i \leq q+1} (\bQ_{ii} - \sum_{j \neq i} |\bQ_{ij}|)
		\end{equation}
		is just a positive constant only dependent on $\brho$ and $\bL$. 
		Hence we have
		\begin{equation}\label{eq:eigenvol}
		\lambda_{\min}^{-1}(\left[\mathbf Z_h^\top \bQ\mathbf Z_h\right])  \leqslant \frac{1}{\xi}\frac{1}{\min_{1 \leqslant l \leqslant g^{(k)}+1}|\ddot{\mathcal{C}}_l^{(h)}|}, h = 1,2.
		\end{equation}
		By (\ref{eq:glc}) on $\Omega_n$, 
		$|{\ddot{\mathcal{C}}_l^{(h)}|/n} \geqslant  Vol(\ddot{\mathcal{C}}_{{l}}^{(h)}) - \delta^2$. As $\mathbf{s}_k^{(1)},\mathbf{s}_k^{(2)} \in \bar{\mathfrak{S}}_k^\varepsilon(\mathbf{x})$, by definition of $\bar{\mathfrak{S}}_k^\varepsilon(\mathbf{x})$,  $\tilde{\mathbf{s}}_{k-1}^{(1)}, \tilde{\mathbf{s}}_{k-1}^{(2)} \in \tilde{\mathfrak{S}}^\varepsilon_{k-1}$. As each node of $\mathfrak{C}^{(k)}$ corresponding to both these splits contains a hypercube of edge length $\varepsilon$, we have 
		$$
		Vol(\ddot{\mathcal{C}}_{{l}})^{(h)} \geqslant \varepsilon^D; \: \: \forall i = 1,2, \cdots, g^{(k)}-1; Vol(\ddot{\mathcal{C}}_{g^{(k)}}^{(h)}) + Vol(\ddot{\mathcal{C}}_{g^{(k)}+1}^{(h)}) \geqslant \varepsilon^D; h = 1,2.
		$$ 
		Since \textbf{R1} is satisfied, without loss of generality we can assume 
		$$
		Vol(\ddot{\mathcal{C}}_{g^{(k)}}^{(h)}) \geqslant \varepsilon^D\sqrt{\delta}; Vol(\ddot{\mathcal{C}}_{g^{(k)}+1}^{(h)}) \geqslant \varepsilon^D\sqrt{\delta}; \mbox{ for } h = 1.
		$$
		This implies for sufficiently large $n$,
		\begin{equation}
		\label{eq:L4_P1_2}
		\sqrt{n}\min \{\|\mathbf{H}_{\mathbf{Z}_1} \|_2, \|\mathbf{H}_{\mathbf{Z}_2} \|_2\}  \leqslant 
		\sqrt{\frac{\left(\min_{1 \leqslant l \leqslant g^{(k)}+1}\frac{|\ddot{\mathcal{C}}_l^{(h)}|}{n}\right)^{-1}}{\xi}} \leqslant \sqrt{\frac{\left(\varepsilon^D\sqrt{\delta} - \delta^2\right)^{-1}}{\xi}}. 
		\end{equation}
		Finally, using  Lemma \ref{lem:weaklaw2} and boundedness of $m$ and $\lambda_{\max}(\bQ)$, we have $\frac 1n \mathbf{Y}^\top\bQ\mathbf{Y}=O(1)$ on $\Omega_n$.  		
		Next, combining \cref{eq:L4_P1_1,eq:L4_P1_2}, we have under \textbf{R1}, on $\Omega_n$
		$$
		\begin{aligned}
		&\frac{1}{{n}} |\mathbf{Y}^\top \bQ^{\frac{1}{2}} \left[ {\mathbf{\dot P}}_{\mathbf{Z}_1} - {\mathbf{\dot P}}_{\mathbf{Z}_2}\right]\bQ^{\frac \top 2} \mathbf{Y}|\\ 
		&\leqslant \left(\frac{1}{{n}}\mathbf{Y}^\top\bQ \mathbf{Y}\right) \left(\sqrt{n}\min \{\|\mathbf{H}_{\mathbf{Z}_1} \|_2, \|\mathbf{H}_{\mathbf{Z}_2} \|_2\}\right)
		\left(\frac{1}{{\sqrt{n}}}\|\bQ^{\frac \top 2}\mathbf{Z}_1 - \bQ^{\frac \top 2}\mathbf{Z}_2 \|_2\right)\\
		&= O(\delta^{\frac 14}).
		\end{aligned}
		$$
		To bound the second term  $\frac{1}{n}|\mathbf{Y}^\top\bQ\mathbf{Z}_1^{(0)} \bm{\hat{\beta}}(\mathbf{Z}_1^{(0)}) - \mathbf{Y}^\top\bQ\mathbf{Z}_2^{(0)} \bm{\hat{\beta}}(\mathbf{Z}_2^{(0)})|$ of (\ref{eq:equi}), we observe that the leaf nodes corresponding to $\mathbf{Z}_1^{(0)}$ are given by $\ddot{\mathcal{C}}_1^{(h)}, \cdots, \ddot{\mathcal{C}}_{g^{(k)}-1}^{(h)}, \ddot{\mathcal{C}}_{g^{(k)}}^{(h)} \cup \ddot{\mathcal{C}}_{g^{(k)}+1}^{(h)}; h = 1,2$. From definition of $\Bar{\mathfrak{S}}_k^\varepsilon(\mathbf{x})$,
		$$
		Vol(\ddot{\mathcal{C}}_{{l}})^{(h)} \geqslant \varepsilon^D \: \: \forall i = 1,2, \cdots, g^{(k)}-1; Vol(\ddot{\mathcal{C}}_{g^{(k)}}^{(h)}) + Vol(\ddot{\mathcal{C}}_{g^{(k)}+1}^{(h)}) \geqslant \varepsilon^D; h = 1,2.
		$$ 
		Hence, using similar perturbation bounds as in (\ref{eq:singularvalue}), (\ref{eq:pert}), we can conclude that on $\Omega_n$ 
		$$
		\frac{1}{n}|\mathbf{Y}^\top\bQ\mathbf{Z}^{(0)}_1 \bm{\hat{\beta}}(\mathbf Z^{(0)}_1) - \mathbf{Y}^\top\bQ\mathbf{Z}^{(0)}_2 \bm{\hat{\beta}}(\mathbf Z^{(0)}_2)| \leqslant C_6 \sqrt{\frac{\left(2^{2D}\delta + \delta^2\right)}{\left(\varepsilon^D - \delta^2\right)}} = O(\delta^\frac 12).
		$$
		Next, combining under \textbf{R1}, 
		we have on $\Omega_n$, 
		$$
		|v_{n,k, \bQ}(\mathbf{x},\mathbf{s}_{k}^{(1)}) - v_{n,k, \bQ}(\mathbf{x},\mathbf{s}_{k}^{(2)}) | \leqslant O(\delta^{\frac 14}) + O(\delta^\frac 12).
		$$
		This quantity goes to $0$ uniformly in $\delta$ (as $\delta^\frac 14$ is a uniformly continuous function of $\delta$ on [0,1]). This completes the proof under \textbf{R1}. The proof under scenario \textbf{R2} is more technical and is available in the  the Additional results section (Section \ref{sec:add}). 
	\end{subequations}
\end{proof}

\begin{proof}[Proof of Proposition \ref{lem:approx}] For $\zeta_n > M_0$, 
	\begin{align*}
	&\lim_{n \to \infty} \mathbb E \left[\inf_{f \in T_{\zeta_n}\mathcal{F}_n}  \mathbb{E}_{X} |f(X) -m(X) |^2 \right]\\ 
	&\leqslant \lim_{n \to \infty} \mathbb E \left[\inf_{f \in T_{\zeta_n}(\{\sum_{l=1}^{t_n} m(\bx_l)I(\bx \in \calB_l) | \bx_l \in \calB_l, l=1,\ldots,t_n \})} \mathbb{E}_{X} |f(X) -m(X) |^2 \right]\\
	&\leqslant \lim_{n \to \infty} \mathbb E \left[\mathbb{E}_{X} \Delta^2(m,\calB_n(X,\Theta))\right]
	\end{align*}
	where the variation of $m$ in node $\mathcal{B}$ is given by $\Delta(m, \mathcal B) = \sup_{\mathbf x_1, \mathbf x_2 \in \mathcal B}|m(\mathbf{x}_1)-m(\mathbf{x}_2)|$. Hence, it is enough to show that the variation of $m$ in leaf nodes $\calB_n$ of GLS-style tree vanishes asymptotically. 
	
	Theorem \ref{lemma:DART-theoretical} shows the existence of a theoretical DART-split criterion. Let $\mathcal{B}_k^{*}(X, \Theta)$ be the leaf node of a theoretical RF-GLS (i.e. in $(k+1)^{th}$ level) built with randomness $\Theta$ and containing $X$ and let the set of optimal theoretical splits to build $\mathcal{B}_k^{*}(X, \Theta)$ be denoted by $\mathfrak{S}_k^{*}(X, \Theta)$. Since Theorem \ref{lemma:DART-theoretical} also shows that the theoretical DART-split criterion is simply a constant multiplier of the theoretical CART-split criterion, under Assumption \ref{as:add}, we immediately have from Lemma 1 of \cite{scornet2015consistency} that 
	\begin{equation}\label{eq:thvar}
	\Delta(m, \mathcal{B}_k^{*}(X, \Theta)) \overset{a.s.}{\to} 0\text{ as $k \to \infty$}.
	\end{equation}
	Next, using (\ref{eq:thvar}) and Proposition \ref{lemma:equicontinuity} on stochastic equicontinuity of the empirical DART split criterion, we are now ready to prove that optimal theoretical and empirical splits become identical asymptotically as follows. 		 Let $\hat{\mathbf{s}}_{k,n}(X, \Theta) \in \mathfrak{S}_k$ be the set of empirical optimal splits used to build the node containing $X$ in level $k+1$. For fixed $\varepsilon, \tau > 0$; $k \in \mathbb{N}$; $\exists \: n_4 \in \mathbb{N}$, such that
	\begin{equation}\label{eq:themp}
	\mathbb{P}\left[dist\left(\hat{\mathbf{s}}_{k,n}(X, \Theta),  \mathfrak{S}_k^{*}(X, \Theta)\right)  \leqslant \varepsilon\right] \geqslant 1 - \tau,\,\, \forall n > n_4.
	\end{equation}
	Equation (\ref{eq:themp}) is established identical to the proof of the analogous result  (Lemma 3) in \cite{scornet2015consistency}. Only change is that for a single split at level $k$, we now condition on all the splits in levels $1,2,\cdots,k-1$ i.e. $\tilde{\mathbf{s}}_{k -1} \in \tilde{\mathfrak{S}}_{k-1}$ with $\tilde{\mathfrak{S}}_{k-1}^{\varepsilon}$ playing the role of conditioning set. 
	
	The result in (\ref{eq:thvar}) lets us control the variation of regression function $m$ in theoretical GLS-style tree leaf nodes, and (\ref{eq:themp}) establishes that the difference between the theoretical optimal splits and the empirical optimal splits vanishes asymptotically. Combining these two results, we obtain,
	identical to the proof of Proposition 2 in \cite{scornet2015consistency} that 
	$\forall \varepsilon, \gamma > 0$, $\exists n_5 \in \mathbb{N}$  such that $\forall n > n_5$,
	$$
	\mathbb{P}_{X,\Theta} [\Delta (m, \mathcal B_n(X, \Theta)) \leqslant \varepsilon] \geqslant 1- \gamma.
	$$
	Since $\Delta$ is bounded, this ensures $\mathbb E \Delta^2 \leq \varepsilon + M_0^2 \mathbb{P}_{X,\Theta} [\Delta (m, \mathcal B_n(X, \Theta)) \geqslant \varepsilon]$ and letting $n \to \infty$ and then $\eps \to 0$ yields the result. 
\end{proof}

\subsection{Estimation Error}\label{sec:pf-estimn}

\noindent \textbf{(C.3.iid)} (Analog of ULLN \textbf{C.3} for i.i.d. processes):
There exists a function class $\calF_n \ni m_n(\cdot)$ such that for all arbitrary $L > 0 $
\begin{enumerate}[(a)]
	\item Let $\{\eps^*_i\}$ denote an i.i.d. process, independent of $\calD_n$ and $\dot \calD_n$, such $\eps^*_i$ is identically distributed as $\eps_i$. 
	Let $Y_i^* = m(X_i) + \eps_i^*; Y_{i,L}^* = T_L  Y_i^*
	$. Then,    
	$$
	\lim_{n \to \infty} \mathbb{E} \left[ \sup_{f \in T_{\zeta_n}\mathcal{F}_n}  |\frac{1}{{n}}\sum_i|f(X_i) - Y^*_{i,L}|^2  - \mathbb{E}_{\dot \calD_n}|f(\dot X_1) - \dot Y_{1,L}|^2|\right]=0 
	$$
	\item For all $1 \leqslant j \leqslant q$, let $\{(\tilde X_{i}, \ddot X_{i-j})\}_{i \geq j+1}$ and $\{(\tilde \eps_{i}, \ddot \eps_{i-j})\}_{i \geq j+1}$ be bivariate i.i.d. processes, independent of the data and of $\dot G$, such that 
	$(\tilde X_{i}, \ddot X_{i-j})$ is identically distributed as $(X_{i}, X_{i-j})$; 
	$(\tilde \eps_{i}, \ddot \eps_{i-j})$ is identically distributed as $(\eps_{i}, \eps_{i-j})$  for all $i$.  
	Define $\tilde Y_i = m(\tilde X_i)+\tilde \eps_i$, $\tilde Y_{i,L} =  T_L\tilde{Y}_i$, $\ddot Y_i = m(\ddot X_i)+\ddot \eps_i$, and  $\ddot Y_{i,L} =  T_L\ddot{Y}_i \forall i$.
	Then the following holds:
	$$
	{
		\begin{aligned}
		\lim_{n \to \infty} \mathbb{E}\Bigg[& \sup_{f \in T_{\zeta_n}\mathcal{F}_n}  \Big|\frac{1}{n}
		\sum_{i}
		(f(\tilde X_{i}) - \tilde Y_{i,L}) (f(\ddot X_{i-j}) - \ddot Y_{i-j,L})\\
		&- \mathbb{E}_{\dot \calD_n}(f(\dot X_{1+j}) - \dot Y_{1+j,L})(f(\dot X_{1}) - \dot Y_{{1},L})\Big|\Bigg]=0
		\end{aligned}
	}
	$$ 
\end{enumerate}

To prove Condition \textbf{C.3.iid}(b) for RF-GLS trees, we will first prove a general result on the concentration rates of cross-product function classes. We  recall the definition of 
random $\mathbb L_p$ norm entropy numbers. For a sequence of i.i.d random variable $\{R_i\}_1^{n}  = \left(R_1, R_2, \cdots, R_{n} \right); R_i \in \mathbb{R}^D$, $\varepsilon> 0$, $1 \leqslant p < \infty$, let $\mathcal{W}_n$ be a  set of functions $\mathbb{R}^D \mapsto \mathbb{R}$, and for a function $w: \mathbb{R}^d \to \mathbb{R}$, $\|w\|^p_{\{R_i\}_1^{n}} = \Big\{\frac{1}{n}\sum_{i=1}^n |w(R_i)|^p \Big\}^{1/p}$.  Then, $\mathcal{N}_p(\varepsilon, \mathcal{W}_n, \{R_i\}_1^{n})$, the $\varepsilon$-covering number of $\mathcal{W}_n$ w.r.t the random $\mathbb L_p$-norm $\|. \|^p_{\{R_i\}_1^{n}}$  is the minimal $C \in \mathbb{N}$, such that there exists functions $w_1, w_2, \cdots, w_{C} : \mathbb{R}^D \to \mathbb{R}$ with the property that for every $w \in \mathcal{W}_n$, there is a $j = j(w) \in \{1,2,\cdots, C \}$ such that
$
\{\frac{1}{n}\sum_{i =1}^n |w(R_i) - w_j(R_i)|^p\}^{\frac{1}{p}} < \varepsilon.
$

\begin{proposition}
	\label{lemma:cross_prod_indep}
	\textbf{(ULLN for cross-product terms for i.i.d. data)} Let $D_X$ denote a 
	distribution on $\mathbb R^D$ 
	and $D_\bY=(D_{Y_1},D_{Y_2})$ denote a bivariate distribution.  
	Let $\bX_i$ and $\bY_i$ denote bivariate i.i.d. processes such that $\{\bX_i=(X_{1i},X_{2i})\}_{i \geq 1} \iid D_X \times D_X$ (product measure) and $\{\bY_i=(Y_{1i},Y_{2i})\}_{i \geq 1} \iid D_\bY$. Let $Y_{hi,L}=T_L Y_{hi}; h = 1,2$ for any $L > 0 $. Let $(\dot X_1, \dot X_2) \sim D_X \times D_X$ and $(\dot Y_1, \dot Y_2) \sim D_\bY$, independent of $\{\bX_i\}$ and $\{\bY_i\}$.  
	Let $\calF_n$ denote some class of real-valued functions on $\mathbb R^D$ class and $\zeta_n \to \infty$. Then for all $\varepsilon > 0$, we have 
	\begin{align*}
	&\mathbb{P}\Bigg[\sup_{f \in T_{\zeta_n} \calF_n}  \Big|\frac{1}{n}
	\sum_{i}
	(f(X_{1i}) - Y_{1i,L}) (f(X_{2i}) -  Y_{2i,L})\\
	&\quad\quad\quad\quad\quad\quad\quad- \mathbb{E}(f(\dot X_{1}) - \dot Y_{1,L})(f(\dot X_{2}) - \dot Y_{{2},L})\Big| > \varepsilon \Bigg] \\
	& \qquad \leqslant 8\mathbb E \calN_1\left(\frac {\varepsilon}{32\zeta_n},T_{\zeta_n} \calF_n,\{X^*_i\}_{i=1}^{2n}\right) \exp\left(-\frac {n\varepsilon^2}{2048\zeta_n^4}\right).
	\end{align*}
	where $X^*_{2i-1}=X_{1i}$ and $X^*_{2i}=X_{2i}$ for $i=1,\ldots,n$.
\end{proposition}

Note that the class $\calF_n$ is dependent on the data $ \{\bX_i | 1 \leq i \leq n\}$ and hence generally we cannot say $\mathbb E f(X_{1i}) = \mathbb E f(\dot X_1)$ for the new sample $\dot X_1$. So, the cross-term is not a sample covariance and cannot be bounded by direct application of Cauchy-Schwartz inequality on the corresponding ULLN for the squared terms. 

\begin{proof}[Proof of Proposition \ref{lemma:cross_prod_indep}]  
	For convenience, we denote $H_i = (X_{1i},X_{2i},\eps_{1i},\eps_{2i})$ and $\dot{{H}} =(\dot X_{1},\dot X_{2},\dot \eps_{1},\dot \eps_{2})$, and $\calW_n=\{w \given w(x_1,x_2,y_1,y_2) = (f(x_1) - y_1)(f(x_2) - y_2), f \in T_{\zeta_n}\calF_n \}$. For large enough $n$, $\zeta_n > L$, and $|w| \leq 4\zeta^2_n$ for all $w \in \calW_n$. By Theorem 9.1 of \cite{gyorfi2006distribution}, 
	$$\mathbb P \Big[\sup_{w \in \calW_n} \Big| \frac 1n \sum_i w(H_i) - \mathbb E w(\dot H) \Big| > \varepsilon \Big] \leq 8 \calN_1\left(\frac {\varepsilon}{8},\calH_n,\{W_i\}_{i=1}^n\right) \exp\left(-\frac {n\varepsilon^2}{128 (4\zeta_n^2)^2}\right).$$

	Let $w_j({H_i})= (f_j(X_{1i}) - Y_{1i,L}) (f_j(X_{2i}) - Y_{2i,L})$, for some $f_j \in T_{\zeta_n}{\mathcal{F}}_n$. Then for large enough $n$, $\zeta_n > L$ and we have,
	\begin{align*}
	&\frac{1}n\sum_{i} |w_j(H_i) - w_{j'}(H_i)|\\
	&=\frac{1}n\sum_{i} |(f_j(X_{1i}) - Y_{1i,L}) (f_j(X_{2i}) - Y_{2i,L})\\ 
	&\quad - (f_{j'}(X_{1i}) - Y_{1i,L}) (f_{j'}(X_{2i}) - Y_{2i,L})|\\
	& = \frac{1}n \sum_{i} \Bigg|\left[f_j(X_{1i})f_j(X_{2i}) - f_{j'}(X_{1i})f_{j'}(X_{2i}) \right]\Bigg|\\
	&+ \frac{1}n \sum_{i}\Bigg|\left[Y_{1i,L} (f_{j'}(X_{2i}) - f_j(X_{2i})) +  Y_{2i,L} (f_{j'}(X_{1i})-f_j(X_{1i})) \right]\Bigg| \\
	& = \frac{1}n \sum_{i} \Bigg|\left[f_j(X_{1i})(f_j(X_{2i}) - f_{j'}(X_{2i})) + f_{j'}(X_{2i})(f_j(X_{1i}) - f_{j'}(X_{1i}) \right]\Bigg|\\
	&+ \frac{1}n \sum_{i}\Bigg|\left[Y_{1i,L} (f_{j'}(X_{2i}) - f_j(X_{2i})) +  Y_{2i,L} (f_{j'}(X_{1i})-f_j(X_{1i})) \right]\Bigg| \\
	& \leqslant \frac{1}n \sum_{i} \Bigg|(|f_j(X_{1i})|+|Y_{1i,L}|)\Big|f_j(X_{2i}) - f_{j'}(X_{2i})\Big|\Bigg|\\ 
	& + \frac{1}n \sum_{i}  \Bigg| (|f_{j'}(X_{2i})|+|Y_{2i,L}|)\Big|f_j(X_{1i}) - f_{j'}(X_{1i})\Big| \Bigg|\\
	& \leqslant 2\zeta_n \frac{1}{n}\sum_{i} |f_j(X_{2i}) - f_{j'}(X_{2i})|+ 2\zeta_n \frac{1}{n}\sum_{i} |f_j(X_{1i}) - f_{j'}(X_{1i})|\\
	& \leqslant 4\zeta_n \frac{1}{2n}\sum_{i=1}^{2n} |f_j(X^*_i) - f_{j'}(X^*_i)|.
	\end{align*}
	Hence
	$$
	\mathcal{N}_1\left(\frac{\varepsilon}{8}, \mathcal{W}_n, \{H_i\}_{i=1}^n \right) \leqslant
	\mathcal{N}_1\left(\frac{\varepsilon}{32\zeta_n}, T_{\zeta_n} \mathcal{F}_n, \{X^*_i\}_{i=1}^{2n} \right).
	$$
\end{proof}

\begin{proof}[Proof of Proposition \ref{lemma:dependent_data}] We consider the following two cases:\\
	\textbf{Case 1:} Functions in $\mathcal{G}_n$ are uniformly bounded by $C_6 \in \mathbb{R}, C_6<\infty$. 
	
	In this scenario, due to the fact that the constant function $C_6$ works as a bounded envelope for $\mathcal{G}_n, \forall n$, the result of \cite{nobel1993note} ensures almost sure convergence for $\{U_i\}$ even under the assumption of $\mathbb{L}_1$ convergence of the independent counterpart $\{U_i\}$ \footnote{The statement of Theorem 1 of \cite{nobel1993note} assumes almost strong (almost sure) ULLN for the i.i.d. process, but the proof only requires weak (convergence in probability) ULLN. As we assume an $\mathbb L_1$ ULLN for the i.i.d. process, by Markov inequality weak ULLN holds here as well, and the same result is obtained.} 
	Since $\calG_n$ is bounded, almost sure convergence is enough to ensure convergence in $\mathbb L_1$. This completes the proof for this case. \\
	\textbf{Case 2:} No uniformly bounded envelope. For every $g \in \mathcal{G}_n$, we define $g_1 = g(\mathds{I}(G_n \leqslant C_6))$ and $g_2 = g(\mathds{I}(G_n > C_6))$. So we have
	$$
	\begin{aligned}
	\sup_{g \in \mathcal{G}_n}\Big|\frac{1}{{n}}\sum_i g(U_i^{(C)})\Big| &=  \sup_{g_1,\: g \in \mathcal{G}_n}\Big|\frac{1}{{n}}\sum_i (g_1(U_i)- \mathbb{E}g_1(U_i))\Big|\\ &+ \sup_{g_2,\: g \in \mathcal{G}_n}\Big|\frac{1}{{n}}\sum_i (g_2(U_i) - \mathbb{E}g_2(U_i))\Big|.
	\end{aligned}
	$$
	The first term converges to $0$ in $\mathbb{L}_1$, using Case 1. 
	For the second term, we have
	$$
	\begin{aligned}
	\mathbb{E}\sup_{g_2,\: g \in \mathcal{G}_n}\Big|\frac{1}{{n}}\sum_i (g_2(U_i) - \mathbb{E}g_2(U_i))\Big| &\leqslant \frac{2}{{n}} \sum_i \mathbb{E}|G_n(U_i)|\mathds{I}(|G_n(U_i)|> C_6 ).\\
	\end{aligned}
	$$
	From (\ref{eq:ui}), this goes to zero as $C_6 \to \infty$. 
\end{proof}

\begin{proposition}	\textbf{(ULLN for squared and cross-product classes for dependent data)} \label{prop:dep-square-crossloss}
	If $\calF_n$ denote a class of functions 
	satisfying
	\textbf{(C3.iid)} for i.i.d. processes, and Condition \textbf{C.4}, then $\calF_n$ satisfies Condition \textbf{(C3)} for a $\beta$-mixing process $\{\eps_i\}$ under Assumption \ref{as:chol}.
\end{proposition}

Condition \textbf{(C.4)} on uniformly bounded $(2+\delta)^{th}$ moment for the class $\calF_n$ will be sufficient to ensure the moment-bound of (\ref{eq:ui}) for squared and cross-product function classes allowing use of the ULLN in Proposition \ref{lemma:dependent_data} to prove Proposition \ref{prop:dep-square-crossloss}. 

\begin{proof}[Proof of Proposition \ref{prop:dep-square-crossloss}]
	To establish the ULLN in condition \textbf{(C.3)} for the dependent error process, it is enough to establish the following two conditions:
	\begin{enumerate}[(a)]
		\item     
		$$
		\lim_{n \to \infty} \mathbb{E} \left[ \sup_{f \in T_{\zeta_n}\mathcal{F}_n}  |\frac{1}{{n}}\sum_i|f(X_i) - Y_{i,L}|^2  - \mathbb{E}|f(\dot X_1) - \dot Y_{1,L}|^2|\right]=0.
		$$
		\item For all $ 1 \leq j \leq q$, 
		$$
		{
			\begin{aligned}
			\lim_{n \to \infty} \mathbb{E}\Bigg[& \sup_{f \in T_{\zeta_n}\mathcal{F}_n}  \Big|\frac{1}{n}
			\sum_{i}
			(f( X_{i}) - Y_{i,L}) (f( X_{i-j}) - Y_{i-j,L})\\
			&- \mathbb{E}(f(\dot X_{1+j}) - \dot Y_{1+j,L})(f(\dot X_{1}) - \dot Y_{{1},L})\Big|\Bigg]=0.
			\end{aligned}
		}
		$$ 
	\end{enumerate}
	Since $\{X_i\}$ is an i.i.d. process independent of the $\beta$-mixing process $\{\eps_i\}$, using the property (\ref{eq:betamix}) of $\beta$-mixing coefficients presented in Lemma \ref{lem:strlaw}, we have $U_i=(X_i,\eps_i)$ is also a $\beta$-mixing process. Let $\calG_n$ be the class of functions $g: \mathbb R^{D+1} \to \mathbb R$ such that $g(U_i) = (Y_{i,L} - f(X_i))^2$ for some $f \in T_{\zeta_n}\calF_n$. Clearly, under condition \textbf{(C.4)}, $G_n = 2L^2 + F_n^2$ is an envelope for $\calG_n$. Choosing $C>2L^2+1$ we have
	\begin{align*}
	\frac 1n \sum_i \mathbb E|G_n (Y_i,X_i)|\mathds{I}(|G_n(Y_i,X_i)|>C)
	&= \frac 1n \sum_i \mathbb E 2(L^2 + F_n^2(X_i))\mathds{I} (|F_n(X_i)|> \sqrt{C-2L^2})\\
	&\leq \frac {2(L^2+1)}{{(C-2L^2)}^{\frac \delta 2}} \frac 1n \sum_i \mathbb E |F_nX_i)|^{2+\delta} 
	\end{align*} 
	Here, the inequality uses the fact that for any random variable $X$, $\mathbb E|X|\mathds{I} (|X|>C) \leq \mathbb E |X|^{1+\delta}/C^\delta$. Clearly, the quantity above goes to $0$ by Condition \textbf{(C.4)} by first taking $n \to \infty$ and then $C \to \infty$. Hence (\ref{eq:ui}) is satisfied by the class $\calG_n$, and ULLN \textbf{(C.3)}(a) is established using Proposition \ref{lemma:dependent_data}. 
	
	Next, define   $H_i=(X_i,X_{i-j},\eps_i,\eps_{i-j})$. Since  $X_i$'s are i.i.d., $(X_i,X_{i-j})$ is a bivariate $m$-dependent process with lag at most $j$. Also, since  $\{\eps_i\}$ is a $\beta$-mixing process, so is $(\eps_i,\eps_{i-j})$ with mixing coefficient at lag $a$ given by  $\beta_{(\eps_i,\eps_{i-j})}(a) \leq \beta_{\eps_i}(a-j)$. Once again, since  $\{(X_i,X_{i-j})\} \independent \{(\eps_i,\eps_{i-j})\}$ and both are $\beta$-mixing, using (\ref{eq:betamix}) 
	established in the proof of Lemma \ref{lem:strlaw}, $\{H_i\}$ is a $\beta$-mixing process. Define $\calG_n^{(j)}$ to be the class of functions $g^{(j)}$ of the form $g^{(j)}(H_i)=(Y_{i,L}-f(X_i))(Y_{i-j,L}-f(X_{i-j}))$ for $f \in T_{\zeta_n}\calF_n$. Like $\calG_n$, $\calG_n^{(j)}$ admits an envelope $G_n^{(j)} \leq 2L^2 + F_n^2$ which satisfies the mean uniform integrability condition  (\ref{eq:ui}). 
	Hence, using condition \textbf{(C.3.iid)}(b) and \textbf{(C.4)}, and applying Proposition \ref{lemma:dependent_data}, we have \textbf{(C.3)}(b). 
\end{proof}

\begin{proof}[Proof of Proposition \ref{sec:prop-rf-ulln}]
	For a GLS-style regression tree $m_n(.,\Theta)$ in RF-GLS, built with data $\mathcal{D}_n$ and randomness $\Theta$, let the partition obtained from $\calD_n$ and $\bTheta$ be denoted by $\mathcal{P}_{GLS,n}(\Theta)=\mathcal{P}_n(\Theta)$. Then $\calF_n=\calF_n(\Theta)$ is the set of all functions $f : [0,1]^D \to \mathbb{R}$ piece-wise constant on each cell of the partition $\mathcal{P}_n(\Theta)$. 
	
	We will show that $T_{\zeta_n}\calF_n$ satisfies \textbf{(C.3.iid)}(a) and (b) which would imply that the smaller class $T_{\zeta_n}\tilde \calF_n$ also satisfy these. As the class $\tilde F_n$ satisfies \textbf{C.4}, the result is then proved by Proposition \ref{prop:dep-square-crossloss}. 
	
	Condition \textbf{(C.3.iid)}(a) is proved in \cite{scornet2015consistency} (page 1731) (and more generally in Theorem 13.1 of \cite{gyorfi2006distribution}) for the class $T_{\zeta_n} \dot \calF_n$, the set of all functions $f : [0,1]^D \to \mathbb{R}$ piece-wise constant on each cell of the partition $\mathcal{P}_{OLS,n}(\Theta)$ generated by a RF-tree. The result only relies on the number of RF trees $t_n$, the number of samples $n$ and the number of features $D$, and hence the proof holds also for $T_{\zeta_n} F_n$ for RF-GLS using same number of trees. 
	
	To prove the cross-product loss estimation error in Condition \textbf{(C.3.iid)}(b), we let $X_{1i}=\tilde X_i$, $X_{2i}=\ddot X_{i-j}$, $(Y_{1i},Y_{2i})=(m(X_{1i}),m(X_{2i}))+(\tilde \eps_{i},\ddot \eps_{i-j})$ and use Proposition \ref{lemma:cross_prod_indep}. The  bounding tail probability for the cross-product term in that proposition is almost identical to that for the squared error term used in \cite{scornet2015consistency}. Only difference is that that the entropy number $\mathcal{N}_1\left(\frac{\varepsilon}{32\zeta_n}, T_{\zeta_n} \calF_n, \{X^*_i\}_{i=1}^{2n} \right)$ is based on the empirical measure on the samples $\{X^*_i\}$ of size $2n$ as opposed to $n$ samples for the square term.  
	However, by Theorem 9.4 of \cite{gyorfi2006distribution}, the bound on the entropy number $\mathcal{N}_1\left(\frac{\varepsilon}{32\zeta_n}, T_{\zeta_n} \mathcal{F}_n, \nu \right)$ is free of the choice of the measure $\nu$ and only depends on the choice of the class $\calF_n$.
	Hence, like the squared error ULLN, Condition \textbf{(C.3.iid)}(b) holds as long as the scaling of $t_n$ and tail moments in Assumption \ref{as:tail}(a) are satisfied with $\zeta_n$.  
\end{proof}

\subsection{Consistency of data-driven-partitioning-based GLS estimates under dependent errors}\label{sec:pf-l2}
\begin{proof}[Proof of Theorem \ref{th:gyorfi}] We first show that to prove $\mathbb L_2$ consistency of $m_n$, it is enough to show that $\mathbb E \Big[ \mathbb E_{\dot G} [\brho^\top(m_n(\dot X^{(q+1)}) - m(\dot X^{(q+1)}))]^2 \Big] \to 0$, where $\dot X^{(q+1)} = (\dot X_{q+1}, \dot X_{q}, \cdots, \dot X_{1})$, $\dot \eps^{(q+1)}= (\dot \eps_{q+1}, \dot \eps_{q}, \cdots, \dot \eps_{1})$, 
	$\dot G = (\dot X^{(q+1)}, \dot \eps^{(q+1)})$ 
	and for any function $f: \mathbb{R}^d \to \mathbb{R}$, $f(\dot X^{(q+1)})=(f(\dot X_{q+1}),\ldots,f(\dot X_1))^\top$. 
	Note that from Assumption \ref{as:chol}, for any $q < i \leqslant n-q$, $\bQ_{ii} = \alpha$, $\bQ_{ij} = 0$ for $|j-i|>q$ and $\bQ_{ij}=\sum_{j'=|j-i|}^q \rho_{j'}\rho_{j'-|i-j|}$ for $|i-j| \leq q$. Hence, by Assumption \ref{as:diag}, 
	\begin{equation}\label{eq:diag2}
	\alpha - 2 \sum_{j=1}^q |\sum_{j'=j}^q \rho_{j'}\rho_{j'-j}| > 0.
	\end{equation}
	Since $\dot X_i$'s are i.i.d., using (\ref{eq:diag2}) and Jensen's inequality,
	\begin{equation}\label{eq:l2alt}
	\begin{aligned}
	\mathbb{E}_{X} \left(\brho^\top f(\dot X^{(q+1)})\right)^2 =& \alpha \mathbb E f(\dot X_1)^2 + 2 (\mathbb E f(\dot X_1))^2 \sum_{j=1}^q \sum_{j'=j}^q \rho_{j'}\rho_{j'-j}\\
	\geqslant & \mathbb E f(\dot X_1)^2 (\alpha - 2 \sum_{j=1}^q | \sum_{j'=j}^q \rho_{j'}\rho_{j'-j}|)
	\end{aligned} 
	\end{equation}
	Choosing $f=m_n - m$ proves the result showing that it is enough to work with  
	$\mathbb E \Big[ \mathbb E [\brho^\top(m_n(\dot X^{(q+1)}) - m(\dot X^{(q+1)}))]^2 \Big]$.

	The next part rest of the proof showing consistency of the truncated estimator $T_{\zeta_n}m_n$ emulates the technique from \cite{gyorfi2006distribution} (Theorem 10.2). Throughout, careful adjustments need to be made to account for use of the quadratic form $\bQ$. The result is summarized in Lemma \ref{lem:gyorfiqf}. 
	
	\begin{lemma}\label{lem:gyorfiqf}
		Under the conditions of Theorem \ref{th:gyorfi}, we have
		\begin{align*}
		\lim_{n \to \infty}\mathbb{E}\left[\mathbb{E}_{X}  \left[\brho^\top\left(T_{\zeta_n}m_n(\dot X^{(q+1)}, \Theta) - m(\dot X^{(q+1)})\right)\right]^2\right]=0.
		\end{align*}
	\end{lemma}
	
	\begin{proof}[Proof of Lemma \ref{lem:gyorfiqf}] Let $\tilde m_n = T_{\zeta_n}m_n$. 
		$$
		\begin{aligned}
		&\mathbb{E}_{\dot G} \left[\brho^\top \left(\tilde m_n(\dot X^{(q+1)}) - \dot Y^{(q+1)} \right) \right]^2\\ 
		&= \mathbb{E}_{\dot G} \left[\brho^\top \left(\tilde m_n(\dot X^{(q+1)}) -  m(\dot X^{(q+1)}) \right) + \brho^\top \left( m(\dot X^{(q+1)}) - \dot Y^{(q+1)} \right)  \right]^2\\
		& = \mathbb{E}_{\dot G} \left[\brho^\top \left(\tilde m_n(\dot X^{(q+1)}) -  m(\dot X^{(q+1)}) \right)\right]^2 + \mathbb{E}_{\dot G} \left[\brho^\top \left( m(\dot X^{(q+1)}) - \dot Y^{(q+1)} \right) \right]^2\\ 
		&+ 2\mathbb{E}_{\dot G} \left[\brho^\top \left(\tilde m_n(\dot X^{(q+1)}) -  m(\dot X^{(q+1)}) \right)\brho^\top \left( m(\dot X^{(q+1)}) - \dot Y^{(q+1)} \right)  \right]
		\end{aligned}
		$$
		Now,
		
		$$
		\begin{aligned}
		&\mathbb{E}_{\dot G}  \left[\brho^\top \left(\tilde m_n(\dot X^{(q+1)}) -  m(\dot X^{(q+1)}) \right)\brho^\top \left( m(\dot X^{(q+1)}) - \dot Y^{(q+1)} \right)  \right]\\ 
		&= \mathbb{E}_{\dot X^{(q+1)}} \left[  \brho^\top \left(\tilde m_n(\dot X^{(q+1)}) -  m(\dot X^{(q+1)}) \right) \brho^\top \mathbb{E}_{\dot Y^{(q+1)}} \left[\left( m(\dot X^{(q+1)}) - \dot Y^{(q+1)} \right)\right]  \right]\\
		& = 0
		\end{aligned}
		$$
		Hence, we have,
		
		$$
		\begin{aligned}
		&\mathbb{E}_{\dot G} \left[\brho^\top \left({\tilde m}_n(\dot X^{(q+1)}) - m(\dot X^{(q+1)}) \right) \right]^2 \\
		&= \mathbb{E}_{\dot G} \left[\brho^\top \left(\tilde m_n(\dot X^{(q+1)}) - \dot Y^{(q+1)} \right) \right]^2 - \mathbb{E}_{\dot G} \left[\brho^\top \left(m(\dot X^{(q+1)}) -\dot Y^{(q+1)} \right) \right]^2\\
		& =  A \left(A + 2\left( \mathbb{E}_{\dot G} \left[\brho^\top \left(m(\dot X^{(q+1)} -\dot Y^{(q+1)} \right) \right]^2\right)^{\frac{1}{2}}\right)
		\end{aligned}
		$$
		where 
		$$
		A := \left(\mathbb{E} \left[\brho^\top \left(\tilde m_n(\dot X^{(q+1)}) - \dot Y^{(q+1)} \right) \right]^2 \right)^{\frac{1}{2}} - \left(\mathbb{E} \left[\brho^\top \left(m(\dot X^{(q+1)}) -\dot Y^{(q+1)} \right) \right]^2\right)^{\frac{1}{2}}  
		$$
		As the term $\mathbb{E}_{\dot G} \left[\brho^\top \left(m(\dot X^{(q+1)} -\dot Y^{(q+1)} \right) \right]^2$ is non-random and $O(1)$ ($\brho$ being of fixed-dimension), using Cauchy-Schwartz inequality, it is enough to show $\mathbb E A^2 \to 0$. Applying $(a+b)^2 \leq 2(a^2 + b^2)$, we have
		$$
		\begin{aligned}
		\mathbb EA^2 &\leqslant  2\mathbb{E}\Bigg[ \left(\mathbb{E}_{\dot G} \left[\brho^\top \left(\tilde m_n(\dot X^{(q+1)}) - \dot Y^{(q+1)} \right) \right]^2\right)^{\frac{1}{2}}\\ 
		&\quad\quad\quad\quad-  \inf_{f \in T_{\zeta_n}\mathcal{F}_n} \left(\mathbb{E}_{\dot G} \left[\brho^\top\left(f(\dot X^{(q+1)}) - \dot Y^{(q+1)} \right) \right]^2\right)^{\frac{1}{2}}\Bigg]^2 \\
		&+2\mathbb{E}\Bigg[ \inf_{f \in T_{\zeta_n}\mathcal{F}_n} \left(\mathbb{E}_{\dot G} \left[\brho^\top\left(f(\dot X^{(q+1)}) - \dot Y^{(q+1)} \right) \right]^2\right)^{\frac{1}{2}}\\ 
		&\quad\quad\quad\quad- \left(\mathbb{E}_{\dot G} \left[\brho^\top \left(m(\dot X^{(q+1)}) - \dot Y^{(q+1)} \right) \right]^2\right)^{\frac{1}{2}}\Bigg]^2
		\end{aligned}
		$$
		
		Using triangular inequality 
		for the second quantity, we have 
		$$
		\begin{aligned}
		&\mathbb{E}\Bigg[ \inf_{f \in T_{\zeta_n}\mathcal{F}_n} \left(\mathbb{E}_{\dot G} \left[\brho^\top \left(f(\dot X^{(q+1)}) - \dot Y^{(q+1)} \right) \right]^2\right)^{\frac{1}{2}}\\ 
		&\quad\quad\quad- \left(\mathbb{E}_{\dot G} \left[\brho^\top \left(m(\dot X^{(q+1)}) - \dot Y^{(q+1)} \right) \right]^2\right)^{\frac{1}{2}}\Bigg]^2\\
		& \leqslant \mathbb{E} \inf_{f \in T_{\zeta_n}\mathcal{F}_n} \left(\mathbb{E}_{\dot G} \left[\brho^\top \left(f(\dot X^{(q+1)}) - m(\dot X^{(q+1)}) \right) \right]^2\right) \\
		& \leqslant   \left(\alpha  +\sum_{j \neq j'}|\rho_{j}\rho_{j'}| \right) \mathbb{E} \left[\inf_{f \in T_{\zeta_n}\mathcal{F}_n}  \mathbb{E}_{\dot X_1} |f(\dot X_1) -m(\dot X_1) |^2 \right]
		\end{aligned}
		$$
		This vanishes asymptotically, due to the approximation error condition \textbf{(C.2)}. Hence, we focus on the other term $\mathbb EA_1^2$ in the expression of $\mathbb E A^2$, where 
		$$
		\begin{aligned}
		A_1 &:= \left(\mathbb{E}_{\dot G} \left[\brho^\top \left(\tilde m_n(\dot X^{(q+1)}) - \dot Y^{(q+1)} \right) \right]^2\right)^{\frac{1}{2}}\\
		&\quad\quad\quad -  \inf_{f \in T_{\zeta_n}\mathcal{F}_n} \left(\mathbb{E}_{\dot G} \left[\brho^\top \left(f(\dot X^{(q+1)}) - \dot Y^{(q+1)} \right) \right]^2\right)^{\frac{1}{2}}.
		\end{aligned}
		$$
		$A1$ can be decomposed and bounded as follows by sum of ten terms:
		\begin{align*}
		A_1 &\leqslant \sup_{f \in T_{\zeta_n}\mathcal{F}_n} \Bigg\{ \left(\mathbb{E}_{\dot G} \left[\brho^\top \left(\tilde m_n(\dot X^{(q+1)}) - \dot Y^{(q+1)} \right) \right]^2\right)^{\frac{1}{2}}\\
		&\quad\quad\quad\quad\quad\quad\quad\quad\quad\quad -\left(\mathbb{E}_{\dot G} \left[\brho^\top \left(\tilde m_n(\dot X^{(q+1)}) - \dot Y_L^{(q+1)} \right) \right]^2\right)^{\frac{1}{2}}\\
		&+ \left(\mathbb{E}_{\dot G} \left[\brho^\top \left(\tilde m_n(\dot X^{(q+1)}) - \dot Y_L^{(q+1)} \right) \right]^2\right)^{\frac{1}{2}} - \left(\frac{1}{{n}} \sum_i \left[\brho^\top (\tilde m_n(\mathbf X^{(i)})- \mathbf{Y}_L^{(i)})\right]^2 \right)^{\frac{1}{2}}\\
		&+\left(\frac{1}{{n}} \sum_i \left[\brho^\top (\tilde m_n(\mathbf X^{(i)})- \mathbf{Y}_L^{(i)})\right]^2 \right)^{\frac{1}{2}} - \left(\frac{1}{{n}} \sum_i \left[\brho^\top ( m_n(\mathbf X^{(i)})- \mathbf{Y}_L^{(i)})\right]^2 \right)^{\frac{1}{2}}\\
		&+\left(\frac{1}{{n}} \sum_i \left[\brho^\top ( m_n(\mathbf X^{(i)})- \mathbf{Y}_L^{(i)})\right]^2 \right)^{\frac{1}{2}} - \left(\frac{1}{{n}} \sum_i \left[\brho^\top ( m_n(\mathbf X^{(i)})- \mathbf{Y}^{(i)})\right]^2 \right)^{\frac{1}{2}}\\
		&+\left(\frac{1}{{n}} \sum_i \left[\brho^\top ( m_n(\mathbf X^{(i)})- \mathbf{Y}^{(i)})\right]^2 \right)^{\frac{1}{2}} - \left( \frac{1}{{n}}\left( m_n(\mathbf X)- \mathbf{Y}\right)^\top\bQ\left( m_n(\mathbf X)- \mathbf{Y}\right)\right)^{\frac{1}{2}}\\
		&+\left( \frac{1}{{n}}\left( m_n(\mathbf X)- \mathbf{Y}\right)^\top\bQ\left( m_n(\mathbf X)- \mathbf{Y}\right)\right)^{\frac{1}{2}} - \left( \frac{1}{{n}}\left(f(\mathbf X)- \mathbf{Y}\right)^\top\bQ\left(f(\mathbf X)- \mathbf{Y}\right)\right)^{\frac{1}{2}}\\
		&+ \left( \frac{1}{{n}}\left(f(\mathbf X)- \mathbf{Y}\right)^\top\bQ\left(f(\mathbf X)- \mathbf{Y}\right)\right)^{\frac{1}{2}} -\left(\frac{1}{{n}} \sum_i \left[\brho^\top ( f(\mathbf X^{(i)})- \mathbf{Y}^{(i)})\right]^2 \right)^{\frac{1}{2}}\\
		&+\left(\frac{1}{{n}} \sum_i \left[\brho^\top ( f(\mathbf X^{(i)})- \mathbf{Y}^{(i)})\right]^2 \right)^{\frac{1}{2}} - \left(\frac{1}{{n}} \sum_i \left[\brho^\top ( f(\mathbf X^{(i)})- \mathbf{Y}_L^{(i)})\right]^2 \right)^{\frac{1}{2}}\\
		&+\left(\frac{1}{{n}} \sum_i \left[\brho^\top ( f(\mathbf X^{(i)})- \mathbf{Y}_L^{(i)})\right]^2 \right)^{\frac{1}{2}}- \left(\mathbb{E}_{\dot G} \left[\brho^\top \left(f(\dot X^{(q+1)}) - \dot Y_L^{(q+1)} \right) \right]^2\right)^{\frac{1}{2}}\\
		&+\left(\mathbb{E}_{\dot G} \left[\brho^\top \left(f(\dot X^{(q+1)}) - \dot Y_L^{(q+1)} \right) \right]^2\right)^{\frac{1}{2}} - \left(\mathbb{E}_{\dot G} \left[\brho^\top \left(f(\dot X^{(q+1)}) - \dot Y^{(q+1)} \right) \right]^2\right)^{\frac{1}{2}}\Bigg\}
		\end{align*}
		
		Here $\bX^{(i)} = (X_i, X_{i-1}, \ldots, X_{i-q})^\top $. Let the $10$ terms in the above inequality be denoted by $b_1, \ldots, b_{10}$. The $6^{th}$ term is negative by definition (Equation \ref{eqn:m_n}).  Hence, $A_1 \leq \sum_{t \in \{1,\ldots,10\} \setminus \{6\}} b_t$. 
		
		On the other hand, 
		\begin{align*}
		A_1 &\geqslant \left(\mathbb{E}_{\dot G} \left[\brho^\top \left(m(\dot X^{(q+1)}) - \dot Y^{(q+1)} \right) \right]^2\right)^{\frac{1}{2}}\\ 
		&\quad\quad\quad\quad\quad\quad-  \inf_{f \in T_{\zeta_n}\mathcal{F}_n} \left(\mathbb{E}_{\dot G} \left[\brho^\top \left(f(\dot X^{(q+1)}) - \dot Y^{(q+1)} \right) \right]^2\right)^{\frac{1}{2}}\\
		&\geqslant  - \inf_{f \in T_{\zeta_n}\mathcal{F}_n} \left(\mathbb{E}_{\dot G} \left[\brho^\top \left(f(\dot X^{(q+1)}) - m(\dot X^{(q+1)}) \right) \right]^2\right)^{\frac{1}{2}}
		\end{align*}
		
		Denoting the right-hand side of the above equation by $a$, and using \\$(\sum_{t \in \{1,\ldots,10\} \setminus \{6\}} b_t)^2 \leqslant 9\sum_{t \in \{1,\ldots,10\} \setminus \{6\}} b_t^2$, we have $A_1^2 \leq  a^2 + 9\sum_{t \in \{1,\ldots,10\} \setminus \{6\}} b_t^2$. \\Hence, to show $\mathbb{E}(A_1^2)$ vanishes it is enough to show the terms $\mathbb E(a^2)$ and $\mathbb E(b_t^2)$ vanishes. 
		
		The term $\mathbb E(a^2)$ directly goes to $0$ using the approximation error condition \textbf{(C.2)}. 
		
		Using triangular inequality, the $1^{st}$ and $10^{th}$ $\mathbb E (b^2_t)$ terms are bounded above by 
		
		$$(q+1)\alpha\left(\mathbb{E}_{\dot G}\left[ \dot Y-\dot Y_L\right]^2\right)$$
		
		Similarly, the $4^{th}$ and $8^{th}$ term is bounded above by the following:
		
		$$(q+1)\alpha\mathbb{E}\left(\frac 1n \sum_i \left[Y_i- Y_{i,L}\right]^2\right).$$
		
		The $3^{rd}$ term is bounded above by the following which vanishes by Condition \textbf{(C.1)}.
		$$
		(q+1)\alpha \mathbb{E}\left( \frac{1}{{n}}\sum_{{i}} (m_n(X_i) - \tilde m_n(X_i))^2.\right) \leq (q+1)\alpha \mathbb{E}\max_i  \left(m_n(X_i) - \tilde m_n(X_i)\right)^2.
		$$

		The $5^{th}$ and $7^{th}$ $\mathbb E(b_t^2)$ terms only consists of the $q^2$ residual terms arising from the first $q$ rows of the Cholesky factorization in Assumption \ref{as:chol}. Hence, they are bounded by:
		\begin{align*}
		&\left(\sum_{1 \leq i,j \leq q} |(\bL^\top\bL)_{ij}| \right)\frac 1n \mathbb E \max_{1 \leqslant i \leq q} \sup_{f \in \{m_n\} \cup T_{\zeta_n}\calF_n} (f(X_i) - Y_i)^2\\
		\leq 4&\left(\sum_{1 \leq i,j \leq q} |(\bL^\top\bL)_{ij}| \right) \frac 1n \mathbb E \left(\zeta_n^2 + \|m\|_\infty^2 + \max_{1 \leqslant i \leq q} \left[m_n(X_i)^2\mathds{I}(|m_n(X_i) \geqslant \zeta_n) + \eps_i^2\right]\right)
		\end{align*} 
		
		Using $\zeta_n^2/n \to 0$, boundedness of $m$, and 
		Condition \textbf{(C.1)}, this goes to zero. 
		
		The $2^{nd}$ and $9^{th}$ term are bounded  by the following:
		\begin{align*}
		&\mathbb{E} \Bigg[ \sup_{f \in T_{\zeta_n}\mathcal{F}_n}  \Bigg|\frac{1}{{n}}\sum_i \left[\brho^\top \left(f(X^{(i)})-Y^{(i)}\right)\right]^2  - \mathbb{E}_{\dot G}\left[\brho^\top \left(f(\dot X^{(q+1)}) - \dot Y_L^{(q+1)} \right) \right]^2 \Bigg|\\
		&\leqslant  \mathbb{E} \Bigg[ \sup_{f \in T_{\zeta_n}\mathcal{F}_n}  \Bigg| \alpha \left(\frac{1}{{n}}\sum_i(f(X_i) - Y_{i,L})^2  - \mathbb{E}_{\dot G}(f(\dot X_1) - \dot Y_{1,L})^2 \right)  \\
		&+ 2 \sum_{j=1}^q \sum_{j'\neq j}^q \rho_{j'}\rho_{j'-j} \Bigg( \frac{1}{n}
		\sum_{i}
		(f(X_{i}) - Y_{i,L}) (f(X_{i-j}) - Y_{i-j,L})\\
		&- \mathbb{E}_{\dot G}(f(\dot X_{i}) - \dot Y_{i,L})(f(\dot X_{i-j}) - \dot Y_{{i-j},L})\Bigg) \Bigg|\Bigg]
		\end{align*}
		Direct application of Assumption \ref{as:chol} (Equation \ref{eq:qf}) and the ULLN \textbf{C.3} sends this to zero. 
		
		Combining all of this, as their are $9$ $b_t$'s, we have
		\begin{align*}
		\lim_{n \to \infty} \mathbb EA_1^2 \leq 18(q+1)\alpha \left(\mathbb{E}_{\dot G}\left[ \dot Y-\dot Y_L\right]^2\right) + 18(q+1)\alpha \lim_{n \to \infty} \mathbb{E}\left(\frac 1n \sum_i \left[Y_i- Y_{i,L}\right]^2\right).
		\end{align*}
		The first term above goes to $0$ as $L \to \infty$. As $\eps_i$ is $\beta$-mixing, it can be proved similar to Lemma \ref{lem:strlaw} part 2, that $\frac 1n \sum_i \mathbb E(\eps_i^2 I(|\eps_i| > L)) \to \mathbb E(\eps_1^2 I(|\eps_1| > L))$ a.s. Hence, for $L>M_0 +1$, that the limit in the second term is bounded by $(M_0^2+1) \mathbb E \eps_1^2 \mathds{I}(|\eps_1| > L - M_0)$ which also goes to 0 as $L \to \infty$ due to the finite $(2+\delta)^{th}$ moment from Assumption \ref{as:mix}. 
	\end{proof}
	
	Returning to the proof of Theorem \ref{th:gyorfi}, Lemma \ref{lem:gyorfiqf} and (\ref{eq:l2alt}) implies that 
	$
	\lim_{n \to \infty}\mathbb{E}[\mathbb{E}_{\dot X_1} (T_{\zeta_n} m_n(\dot X_1) - m(\dot X_1))^2)=0
	$.
	The RF-GLS estimate $m_n(\bx)$ for any $\bx$ can only take one of the possible leaf node values. Hence $|m_n(x)| \leq \max_i |m_n(X_i)|$. 
	\begin{align*}
	\lim_{n \to \infty}\mathbb{E}\left[\mathbb{E}_{\dot X_1} (m_n(\dot X_1) - m(\dot X_1))^2)\right] &\leq \lim_{n \to \infty}\mathbb{E}\left[\mathbb{E}_{\dot X_1} (T_{\zeta_n} m_n(\dot X_1) - m(\dot X_1))^2)\right] \\
	& \quad + \lim_{n \to \infty} \mathbb E \max_i m^2_n(X_i)\mathds{I}(|m_n(X_i)| \geqslant \zeta_n).
	\end{align*}
	The last term is zero by Assumption \textbf{(C.1)},  completing the $\mathbb L_2$ consistency result for the tree estimates $m_n(\cdot, \Theta)$.  To get the result for the average estimate $\bar m_n=\mathbb E_\Theta m_n$, 
	by Jensen's inequality and Fubini's theorem, we have 
	\begin{align*}
	\mathbb E_{\calD_n} \left[\mathbb{E}_{\dot X_1} ( \mathbb{E}_{\Theta} (m_n(\dot X_1, \Theta)) - m( \dot X_1)^2)\right] &\leqslant \mathbb{E}_{\calD_n}\left[\mathbb{E}_{\dot X_1} \mathbb E_\Theta ( m_n(\dot X_1,\Theta) - m(\dot X_1))^2)\right]\\
	&= \mathbb{E}_{\calD_n,\Theta}\left[\mathbb{E}_{\dot X_1} ( m_n( X_1,\Theta) - m(\dot X_1))^2)\right] \to 0
	\end{align*}
\end{proof}

\subsection{Choice of Truncation Threshold}\label{sec:pf-trunc}
\begin{proposition}\label{lem:trunc}
	Under Assumptions \ref{as:chol}, \ref{as:diag} and \ref{as:tail}(a), the GLS tree estimator $m_n$ satisfies the truncation threshold condition \textbf{(C.1)}.
\end{proposition}

\begin{proof}[Proof of Proposition \ref{lem:trunc}]
	For any $n$, let the values corresponding to the $t_n$ leaf nodes be denoted by $\mathbf r = \left( \mathbf{Z}^\top\bQ\mathbf{Z}\right)^{-1}\mathbf{Z}^\top\bQ\mathbf{Y} $, where $i^{th}$ row of $\mathbf{Z}_{n \times t_n}$ denotes the membership of $i^{th}$ data in any of the $t_n$ leaf nodes. 
	Let us define
	$$\mathbf{B} := \left( \mathbf{Z}^\top\bQ\mathbf{Z}\right); \: \: \mathbf{u}:= \mathbf{Z}^\top\bQ\mathbf{Y};   \mbox{ and } l_n:= \argmaxA_{l \in \{1,\cdots, t_n \}} |r_l|; \text{ i.e., } \| \mathbf r\|_{\infty} = |\mathbf r_{l_n}|.$$
	Then 
	\begin{align*}
	\mathbf B_{l_n,l_n} \mathbf r_{l_n} =  \mathbf u_{l_n} - \sum_{l \neq l_n} \mathbf B_{l_n,l}\mathbf r_{l} &\implies \mathbf B_{l_n,l_n} |\mathbf r_{l_n}| \leqslant | \mathbf u_{l_n}| + \sum_{l \neq l_n} |\mathbf B_{l_n,l}||\mathbf r_{l}|\\
	&\implies \| \mathbf r\|_{\infty} = |\mathbf r_{l_n}| \leqslant \frac{|u_{l_n}|}{\mathbf B_{l_n,l_n} - \sum_{l \neq l_n} |\mathbf B_{l_n,l}|},
	\end{align*}
	since Lemma \ref{lem:diag} and the  Assumptions \ref{as:chol} and \ref{as:diag} implies
	$$
	\mathbf B_{l_n,l_n} - \sum_{l \neq l_n} |\mathbf B_{l_n,l}| \geqslant \xi |\mathcal{C}_{l_n}| \mbox{ where } \xi = \min_{i=1,\ldots,q+1} (\bQ_{ii} - \sum_{j \neq i} |\bQ_{ij}|).
	$$
	Using (\ref{eq:qf}),
	$$
	\begin{aligned}
	|u_{l_n}| &= |\mathbf{Z}_{,l_n}^\top \bQ \mathbf{Y}|\\ 
	&=  \alpha\sum_i\mathbf{Z}_{i,l_n}\mathbf{Y}_{i} + \sum_{j \neq j' = 0}^q \rho_j\rho_{j'} \sum_i\mathbf{Z}_{i-j,l_n}\mathbf{Y}_{i-j'} +\sum_{i \in \tilde{\mathcal{A}}_1}\sum_{i' \in \tilde{\mathcal{A}}_2} \tilde\gamma_{i,i'}\mathbf{Z}_{i,l_n}\mathbf{Y}_{i'}\\
	&\leqslant \max_{i} |y_i|\left[\alpha\sum_i\mathbf{Z}_{i,l_n} + \left[ \sum_{j \neq j' = 0}^q |\rho_j\rho_{j'}| \sum_i\mathbf{Z}_{i-j,l_n} +\sum_{i \in \tilde{\mathcal{A}}_1}\sum_{i' \in \tilde{\mathcal{A}}_2} |\tilde\gamma_{i,i'}|\mathbf{Z}_{i,l_n}\right] \right]\\
	&\leqslant \max_{i} |y_i| \left(\alpha + \sum_{j \neq j' = 0}^q |\rho_j\rho_{j'}| +\sum_{i \in \tilde{\mathcal{A}}_1}\sum_{i' \in \tilde{\mathcal{A}}_2} |\tilde\gamma_{i,i'}| \right)|\mathcal{C}_{l_n}|.\:\:\: 
	\end{aligned}
	$$
	Hence
	\begin{equation}\label{eq:bound}
	\begin{aligned}
	\| \mathbf r\|_{\infty} &\leqslant \frac{\left(\alpha + \sum_{j \neq j' = 0}^q |\rho_j\rho_{j'}| +\sum_{i \in \tilde{\mathcal{A}}_1}\sum_{i' \in \tilde{\mathcal{A}}_2} |\tilde\gamma_{i,i'}| \right)}{\xi}\max_{i} |y_i|\\
	&\leqslant 
	C \left[\| m\|_\infty + \max_{i} |\eps_i| \right], 
	\end{aligned}
	\end{equation}
	and 
	$$
	\begin{aligned}
	\max_i \left[m_n(X_i) - T_{\zeta_n}m_n(X_i) \right]^2 & \leqslant \left[\| \mathbf r\|_{\infty}^2 \mathds{I}\left(\| \mathbf r\|_{\infty} \geqslant \zeta_n \right)  \right]\\
	& \leqslant \left[C^2\left[\| m\|_\infty^2 + \max_{i} |\eps_i|^2 \right] \mathds{I}\left(  C \left[\| m\|_\infty + \max_{i} |\eps_i| \right]  \geqslant \zeta_n \right)  \right]\\
	& \leqslant \left[C^2\left[\| m\|_\infty^2 + \max_{i} |\eps_i|^2 \right] \mathds{I}\left(  \max_{i} |\eps_i| \geqslant \tilde{C}\zeta_n \right)  \right], \\
	\end{aligned}
	$$
	where as $\zeta_n \to \infty$ we can choose a constant $\tilde C$ such that $\tilde C \zeta_n \leq \zeta_n/C - \|m\|_\infty$ for large $n$. Choosing $\zeta_n$ to be $\zeta_n/\bar C$ from Assumption \ref{as:tail}(a), Condition \textbf{(C.1)} is satisfied.
\end{proof}

\subsection{Proof of corollaries}\label{sec:pf-cor}

\begin{proof}[Proof of Corollary \ref{cor:rfgls}] 
	To apply Theorem \ref{th:main} to prove this corollary, we only need to show that moment condition $\lim_{n \to \infty} \frac 1n \sum_i \mathbb E |m_n(X_i)|^{2+\delta} < \infty$ is satisfied.
	
	For bounded errors, direct application of (\ref{eq:bound}) implies $|m_n(x)|$ is uniformly bounded and hence will satisfy the $(2+\delta)^{th}$ moment condition. Hence, part (a) is proved. 
	
	For part (b), let $D=\min_i \bQ_{ii}$ and $O=\max_{i,j} \sum_i \sum_{j \neq i} |\bQ_{ij}|$. By the condition for part 2,  $D > \sqrt 2 O$. As we defined earlier, let
	
	$$\mathbf{B} = \left( \mathbf{Z}^\top\bQ\mathbf{Z}\right); \: \: \mathbf{Z}^\top\bQ\mathbf{Y} = \mathbf{u}; \implies  \mathbf B_{l_1,l_1} \mathbf r_{l_1} =  \mathbf u_{l_1} - \sum_{l_2 \neq l_1} \mathbf B_{l_1,l_2}\mathbf r_{l_1}$$
	With these notations,
	
	$$
	\sum_{{i}} |m_n(\mathbf{X}_i)|^{2+\dot{\delta}}  = \sum_{l_1 = 1}^{t_n} |\mathcal{C}_{l_1}| \mathbf |r_{l_1}|^{2+\dot\delta}
	$$
	
	Hence, $\mathbf B_{l_1,l_1} \mathbf r_{l_1} =  \mathbf u_{l_1} - \sum_{l_2 \neq l_1} \mathbf B_{l_1,l_2}\mathbf r_{l_1}$ implies the following:
	\begin{align*}
	\br_{l_1} \left(\sum_i \bZ_{il_1}\bQ_{ii} + \sum_i \sum_{j \neq i} \bZ_{il_1}\bZ_{jl_1} \bQ_{ij} \right) &= \bu_{l_1} - \sum_{l \neq l_1} \br_l \left(\sum_i \sum_{j \neq i} \bZ_{il_1}\bZ_{jl} \bQ_{ij} \right) \\
	\implies \br_{l_1} \sum_i \bZ_{il_1}\bQ_{ii} &= \bu_{l_1} - \sum_{l} \br_l \left(\sum_i \sum_{j \neq i} \bZ_{il_1}\bZ_{jl} \bQ_{ij} \right) \\
	\implies |\br_{l_1}| \sum_i \bZ_{il_1}\bQ_{ii} &\leqslant |\bu_{l_1}| + \sum_{l} w_l^{(l_1)} |\br_l| \\
	\implies D |\calC_{l_1}| |\br_{l_1}| &\leqslant |\bu_{l_1}| + \sum_{l} w_l^{(l_1)} |\br_l| 
	\end{align*}
	where 
	$$ w_l^{(l_1)} = \sum_i \sum_{j \neq i} \bZ_{il_1}\bZ_{jl} |\bQ_{ij}|
	$$
	satisfies 
	$ \sum_l w_l^{(l_1)} = \sum_i \sum_{j \neq i} \bZ_{il_1} |\bQ_{ij}| \leq O |\calC_{l_1}|
	$ and similarly, $ \sum_{l_1} w_l^{(l_1)} \leq O |\calC_{l}|$.
	Using Jensen's inequality twice we have 
	\begin{align*}
	|\br_{l_1}|^{2+\delta} (D |\calC_{l_1}|)^{2+\delta} &\leqslant 2^{1+\delta}\left[|\bu_{l_1}|^{2+\delta} + \left(\sum_{l} w_l^{(l_1)}\right)^ {1+\delta}\sum_{l} w_l^{(l_1)} |\br_l|^{2+\delta} \right]\\
	\implies \sum_{l_1=1}^{t_n} |\calC_{l_1}| |\br_{l_1}|^{2+\delta} &\leq \frac{2^{1+\delta}}{D^{2+\delta}}\sum_{l_1=1}^{t_n}\left[\frac{|\bu_{l_1}|^{2+\delta}}{|\calC_{l_1}|^{1+\delta}} +O^{1+\delta}\sum_{l} w_l^{(l_1)} |\br_l|^{2+\delta} \right]\\
	\implies \sum_{l_1=1}^{t_n} |\calC_{l_1}| |\br_{l_1}|^{2+\delta} &\leq \frac{2^{1+\delta}}{D^{2+\delta}}\sum_{l_1=1}^{t_n} \frac{|\bu_{l_1}|^{2+\delta}}{|\calC_{l_1}|^{1+\delta}} + \frac{(2O)^{1+\delta}}{D^{2+\delta}}\sum_{l}|\br_l|^{2+\delta}  \sum_{l_1=1}^{t_n}  w_l^{(l_1)} \\
	\implies \sum_{l_1=1}^{t_n} |\calC_{l_1}| |\br_{l_1}|^{2+\delta} &\leq \frac{2^{1+\delta}}{D^{2+\delta}}\sum_{l_1=1}^{t_n} \frac{|\bu_{l_1}|^{2+\delta}}{|\calC_{l_1}|^{1+\delta}} + \frac{2^{1+\delta}O^{2+\delta}}{D^{2+\delta}}\sum_{l}|\br_l|^{2+\delta} |\calC_l|
	\end{align*}
	Bring over the second term from the right hand side to the left, we have
	\begin{align*}
	\left(1 -\frac{2^{1+\delta}O^{2+\delta}}{D^{2+\delta}} \right) \frac 1n \sum_{l_1=1}^{t_n} |\calC_{l_1}| |\br_{l_1}|^{2+\delta} &\leq \frac 1n \frac{2^{1+\delta}}{D^{2+\delta}}\sum_{l_1=1}^{t_n} \frac{|\bu_{l_1}|^{2+\delta}}{|\calC_{l_1}|^{1+\delta}} 
	\end{align*}
	As $D > \sqrt 2 O$, the term $1 -\frac{2^{1+\delta}O^{2+\delta}}{D^{2+\delta}}$ is positive and bounded away from 0 for small enough $\delta$. Hence, we only need to show the right hand side has finite expectation.  
	\begin{align*}
	\frac 1n \sum_{l=1}^{t_n} |\calC_l| \left(\frac{|\bu_{l}|}{|\calC_{l}|} \right)^{2+\delta} \leq& \frac 1n \sum_{l=1}^{t_n} |\calC_l| \left(\frac 1{|\calC_l|}\sum_i \sum_{j \neq i}\bZ_{il} |\bQ_{ij}| |Y_j|\right)^{2+\delta}\\
	=& \frac 1n \sum_{l=1}^{t_n} |\calC_l| \left(\frac 1{|\calC_l|}\sum_{i \in \calC_l} \sum_{j=-q}^q |\bQ_{i,i+j}|Y_{i+j}|\right)^{2+\delta} \mbox{  [Assumption \ref{as:chol}]}\\
	\leq& \frac 1n \sum_{l=1}^{t_n} |\calC_l| \frac 1{|\calC_l|} \sum_{i \in \calC_l} \left(\sum_{j=-q}^q |\bQ_{i,i+j}|||Y_{i+j}|\right)^{2+\delta} \mbox{  [Jensen's Inequality]} \\
	=& \frac 1n \sum_{i=1}^n \left(\sum_{j=-q}^q |\bQ_{i,i+j}|||Y_{i+j}|\right)^{2+\delta} \\
	\leq& \frac 1n K_1 \sum_{i=1}^n \sum_{j=-q}^q |Y_{i+j}|^{2+\delta}. \mbox{  [Jensen's Inequality]}
	\end{align*}
	The last inequality also uses the fact that $\bQ$ has only $O(q^2)$ unique entries whose maximum is hence bounded by some $K_1$. As $\eps_i$'s have finite $(2+\delta)^{th}$ moment, the expectation of this term is finite. 
\end{proof}

\begin{proof}[Proof of Corollary \ref{cor:rf}] RF is RF-GLS with $\bQ=\bI$. Hence, the condition $\min_i \bQ_{ii} > \sqrt 2 \max_i \sum_{j \neq I} |\bQ_{ij}|$ is trivially satisfied. So, the proof follows from Corollary \ref{cor:rfgls}.
\end{proof}

\subsection{Examples}\label{sec:pf-examples}
\begin{proof}[Proof of Proposition \ref{th:ar}]
	We will directly apply Corollary \ref{cor:rfgls} to prove the result and hence only need to prove that all assumptions are satisfied.
	
	Assumption \ref{as:mix} is satisfied as AR processes have been shown to be $\beta$-mixing \citep{mokkadem1988mixing}, and sub-Gaussianity of the errors ensures all moments are finite \citep[Lemma 5.5,][]{vershynin2010introduction}. 
	Next we verify Assumptions \ref{as:chol} and \ref{as:diag} on the working covariance matrix $\bSigma$. We can write $\bSigma = Cov(\tilde \beps)$ where $\tilde \beps = (\tilde \eps_1, \ldots, \tilde \eps_n)^\top$ generated from an $AR(q)$ process with coefficients $\tilde a_i$'s. We note from (\ref{eq:ar}) that we can write $\bA\tilde \beps = \tilde \boeta$ where $\tilde \boeta = (\tilde \eta_{q+1}, \ldots, \tilde \eta_n)^\top$ is the white noise process used to generate the $\tilde \eps_i$'s. 
	\begin{equation*}
	\bA_{(n-q) \times n} = \left(\begin{array}{cccccccccc}
	-\tilde a_q & -\tilde a_{q-1}& \ldots & -
	\tilde a_1 & 1 & 0 & 0 & \ldots & 0 & 0\\
	0 & -\tilde a_q & -\tilde a_{q-1}& \ldots & -
	\tilde a_1 & 1 & 0 & \ldots & 0 & 0\\
	\cdots & \cdots & \cdots & \cdots & \cdots & \cdots & \cdots & \cdots & \cdots  & \cdots\\
	0 & 0 & \ldots & \ldots & 
	\ldots & 0 & -\tilde a_q & \ldots & -\tilde a_1 & 1\\
	\end{array} \right), 
	\end{equation*}
	Let $\tilde \beps_{1:q}=(\tilde \eps_1, \ldots, \tilde \eps_q)^\top$ and let $\bL_{q \times q}$ be the lower triangular Cholesky factor of the inverse of $\bM=Cov(\tilde \beps_{1:q})$  i.e. $\bL^\top\bL = \bM^{-1}$. Then $\mathbb V(\bL\tilde \beps_{1:q})=\bI_{q \times q}$ and $Cov(\bL\tilde \beps_{1:q}, \bA\tilde \beps)= Cov(\bL\tilde \beps_{1:q}, \tilde \boeta) = \bO$ as $\tilde \boeta$ is independent of $\beps_{1:q}$. Defining $\sigs=\mathbb V(\tilde{\eta}_i)$, 
	\begin{equation}\label{eq:chol}
	\bB=\left(\begin{array}{cc}
	\bL_{q \times q} \quad & \bO_{q \times n-q} \\
	\multicolumn{2}{c}{\frac 1\sigma \bA_{n-q \times n}}
	\end{array}\right),
	\end{equation}
	we have $Cov(\bB\tilde \beps)=\bI$. Hence, $\bB$  is the Cholesky factor $\bSigma^{-1/2}$ making it clear that $\bSigma^{-1/2}$ satisfies (\ref{as:chol}) with $\brho =\frac 1\sigma(-\tilde a_q, \ldots, -
	\tilde a_1, 1)^\top$. 
	
	To check Assumption \ref{as:diag}, we first consider $q=1$. Then $\bQ$ is simply the autoregressive covariance matrix with parameter $\rho$, i.e., $\bQ_{11}=\bQ_{nn}=1$, $\bQ_{ii}=1+\rho^2$ for $2 \leq i \leq n-1$, $\bQ_{i,i+1}=\bQ_{i,i-1}=-\rho$, $\bQ_{ij}=0$ for $|i-j| \geq 2$. 
	Hence Assumption \ref{as:diag} is always satisfied as $1+\rho^2 > 2|\rho|$, i.e.,  for any $|\rho| < 1$. For unbounded errors, additionally $\bQ$ needs to satisfy 
	\begin{equation}\label{eq:hyperdiag}
	\min_i \bQ_{ii} > \sqrt 2 \max_i \sum_{j \neq i} |\bQ_{ij}|.
	\end{equation} 
	This reduces to $1 > 2\sqrt 2 |\rho|$, i.e., $|\rho| < 1/2\sqrt 2$. 
	For $q \geq 2$, by the statement of Proposition \ref{th:ar} part 2, Assumption \ref{as:diag} or condition (\ref{eq:hyperdiag}) (for unbounded errors) is directly satisfied. 
	
	Finally, we need to verify the tail bounds of Assumption \ref{as:tail}. Proof of part (a) is same as that in \cite{scornet2015consistency} (p. 1733) with $\zeta_n=O(\log n)^2$ as maximum of $n$ sub-Gaussian and correlated random variables satisfy the same tail bound \citep[Theorem 1.14,][]{mitsubg}. 
	
	The same bound can be used to prove part (b). As $\eps_i$'s are identically distributed being a stationary process, once again using Theorem 1.14 of \cite{mitsubg} , we have
	\begin{align*}
	\mathbb P(\max_{i \in \calI_n} |\eps_i| > C_\pi \sqrt{\log |\calI_n|}) &\leqslant
	|\calI_n|^{\left(1 - C^2_\pi/(2\sigs_\eps) \right)}
	\end{align*}
	where $\sigs_\eps$ denote the sub-Gaussian parameter of $\eps_i$'s. For any choice of $C_\pi$ such that $C^2_\pi > 2\sigs_\eps$, this goes to zero as $n \to \infty$, proving part (b).
	
	For part (c), we make the observation that if $\bSigma_0$ denote the true autoregressive covariance matrix of the errors $\beps$, then following the argument above, we can write $\beps = \bSigma_0^{1/2}\boeta$ where $\boeta=(\eta_1,\ldots,\eta_n)$ is the collection of i.i.d. sub-Gaussian random variables. If $\ba$ denotes the $n \times 1$ binary vector corresponding to the selection $\calI_n$, we have 
	\begin{align*}
	\mathbb P(\frac 1{|\calI_n|} |\sum_{i \in \calI_n} \eps_i| > \delta) =&  \mathbb P(\frac 1{|\calI_n|} |\ba^\top\bSig_0^{1/2}\boeta| > \delta)\\
	\leq& C\exp(-c|\calI_n|^2/\ba^\top\bSig_0\ba)\\
	\leq&  C\exp(-c|\calI_n|/\lambda_{\max}(\bSig_0)).
	\end{align*}
	The first inequality follows from \cite{vershynin2010introduction} Proposition 5.10. Hence it is enough to show that $\lambda_{\max}(\bSig_0)$ is bounded.  As $\eps_i$'s are generated from a stable $AR(q_0)$ process, the roots of the characteristic polynomial lie outside zero and the spectral density is bounded from above \citep[Eqn. 2.6,][]{basu2015regularized} which in turn bounds the spectral norm of $\bSigma_0=Cov(\beps)$ \citep[Prop. 2.3,][]{basu2015regularized}, proving the first part of Assumption \ref{as:tail}(c).

	For the second part, as $\eta_i$ are sub-Gaussian, $\eta^2_i$ are sub-exponential and $\|\beps\|_2^2 \leq \lambda_{\max}(\bSigma_0)\|\boeta\|_2^2$. We have 
	$$\begin{aligned}
	\mathbb P(\frac 1{n} \|\beps\|_2^2 > 1.1\lambda_{\max}(\bSigma_0)\mathbb E\eta_1^2)  &\leq \mathbb P(\frac 1{n} \|\boeta \|_2^2 > 1.1\mathbb E\eta_1^2 ) \\ &= \mathbb P(\frac 1{n} \sum_i \eta_i^2 - \mathbb E\eta_1^2 > .1\mathbb E\eta_1^2 )\\ &\leq  C\exp(-cn).
	\end{aligned}$$ The last inequality is from  \cite{vershynin2010introduction} Corollary 5.17.  
\end{proof}

\begin{proof}[Proof of Proposition \ref{th:mat}] 
	We first prove the result without the nugget process $\eps^*(\ell)$. Once we prove that the Mat\'ern GP is $\beta$-mixing, the extension to include the nugget is trivial as sum of a $\beta$-mixing and an i.i.d process by (as proved from \ref{eq:betamix}). Let $\eps(\ell)$ denote a Mat\'ern process with $\nu=q_0-1/2$ on the real line, $q_0$ being a positive integer. From \cite{hartikainen2010kalman}, $\eps(\ell)$ admits a state-space representation:
	\begin{align*}
	a_q \frac {\partial^{q_0}\eps(\ell)}{\partial \ell^{q_0}}  + a_{q_0-1}\frac {\partial^{q_0-1}\eps(\ell)}{\partial \ell^{q_0-1}} + \ldots  a_{1}\frac {\partial\eps(\ell)}{\partial \ell} + \ldots  a_{0}\eps(\ell) = b_0 z(\ell)
	\end{align*}
	where $z(\ell)$ is a white-noise process with covariance $\mathds{I}(\ell-\ell'=0)$. Processes satisfying such $q_0$-order stochastic differential equations are continuous-domain AR$(q_0)$ processes \citep[Eqn. B.2,][]{rasmussen2003gaussian} which when sampled on a discrete integer lattice become an ARMA$(q_0,q_0')$ process $(q_0' < q_0)$ \citep[Theorem 2.7.1,][]{ihara1993information}. As ARMA processes are $\beta$-mixing \citep{mokkadem1988mixing}, the process $\eps(\ell_i)$ sampled on the integer lattice is proved to be $\beta$-mixing. Since they are also Gaussian, they have a bounded $(2+\delta)^{th}$ moment. Hence, Assumption \ref{as:mix} is satisfied. 
	
	Assumption (\ref{as:chol}) is directly satisfied by a NNGP working covariance matrix $\bSigma$ as explained in Section \ref{sec:gp}. Since we are considering Gaussian (unbounded) errors, to apply Theorem \ref{th:main}, the working precision matrix $\bQ$ needs to satisfy (\ref{eq:hyperdiag}) which will also ensure Assumption \ref{as:diag} holds. Let $d_i=\bQ_{ii}$ and $o_i=\sum_{j \neq i} |\bQ_{ij}|$. For an NNGP using $q$ nearest neighbors, there are only $q+1$ unique combinations of $(o_i,d_i)$. For each combination, as $o_i \to 1$ and $d_i \to 0$ as $\phi \to \infty$, there exists $\phi_i$ such that $o_i > max_{i=1}^{q+1} \; d_i$, choosing $\phi > K = \max_{i=1}^{q+1} \phi_i$, (\ref{eq:hyperdiag}) holds. 
	
	For the tail-bounds in Assumption \ref{as:tail}, we note that the proofs for showing Assumptions \ref{as:tail}(a) and (b) hold remain identical to those proofs in Proposition \ref{th:ar} as they only require result on maximal inequalities of sub-Gaussian variables. For part(c), the proof will once again emulate that from Proposition \ref{th:ar} and we only need to show that the spectral norm of $\bSigma_0=Cov(\beps,\beps)$ is bounded. Let  
	$f_\mathbb Z(\omega)$ the spectral density of a Mat\'ern process when sampled on the integer lattice $\mathbb Z$. Then,
	$f_\mathbb Z(\omega) = \sum_{k=\infty}^\infty C(k)\exp(ik\omega)$. Hence, $\sup_\omega |f_\mathbb Z(\omega)| \leq \sum_{k=\infty}^\infty C(k)$. Using \cite{abramowitz1948handbook} (Eqn. 9.7.2) for large $k$, $C(k)$ is equivalent to $k^{q_0 - 1}\exp(-k)$. So the above series is summable and $\sup_\omega |f_\mathbb Z(\omega)| < \infty$. This is sufficient to uniformly bound the spectral norm $\|\bSigma_0\|_2$ \citep[Prop. 2.3,][]{basu2015regularized}. 
\end{proof}

\section{Additional Proofs}\label{sec:add}
\begin{proof}[Proof of Proposition \ref{lemma:equicontinuity} under Scenario \textbf{R2}]
	
	In Section \ref{sec:proofs}, we proved the lemma under \textbf{R1}. The equicontinuity of the term  $\frac{1}{n}|\mathbf{Y}^\top\bQ\mathbf{Z}^{(0)}_1 \bm{\hat{\beta}}(\mathbf Z^{(0)}_1) - \mathbf{Y}^\top\bQ\mathbf{Z}^{(0)}_2 \bm{\hat{\beta}}(\mathbf Z^{(0)}_2)|$ established in that proof also holds under \textbf{R2} since the term does not involve the children nodes. 
	So, it is enough to prove equicontinuity of 
	$\frac{1}{n}|\mathbf{Y}^\top\bQ\mathbf{Z}_1 \bm{\hat{\beta}}(\mathbf Z_1) - \mathbf{Y}^\top\bQ\mathbf{Z}_2 \bm{\hat{\beta}}(\mathbf Z_2)|$ under  conditions of the proposition and \textbf{R2}. 
	
	For $h = 1,2$, $\forall i \in  \{1, 2, \cdots, n\}$, define matrices $\tilde\bZ_h^{(0)}$ such that
	$\mathbf{\tilde{Z}}^{(0)}_{h_{i,{l}}}  =\mathds{I}\left(\mathbf{x}_i \in \ddot{\mathcal{C}}_{{l}}^{(h)}\right);\: {l} = 1,2, \cdots, g^{(k)} -1
	$, and 
	
	$$
	\mathbf{\tilde{Z}}^{(0)}_{h_{i,{l}}} = \mathds{I}\left(\frac{Vol(\ddot{\mathcal{C}}_{{l}}^{(h)})}{Vol(\ddot{\mathcal{C}}_{g^{(k)}}^{(h)}) + Vol(\ddot{\mathcal{C}}_{g^{(k)}+1}^{(h)})} \geqslant \sqrt{\delta}\right) \mathds{I}\left(\mathbf{x}_i \in \{\ddot{\mathcal{C}}_{g^{(k)}}^{(h)} \cup \ddot{\mathcal{C}}_{g^{(k)}+1}^{(h)}\}\right);\: {l} = g^{(k)}, g^{(k)} +1. 
	$$
	Basically, $\mathbf{\tilde{Z}}^{(0)}_{h}$ adds a column of zeros to $\mathbf{{Z}}^{(0)}_{h}$ to match the dimensions of $\bZ_{h}$ and rearranges the columns so that the column of zeros aligns with the column in  $\bZ_{h}$ corresponding to the child node with few members (which is posited under \textbf{R2}). 
	
	\begin{align*} 
	&\frac{1}{n}|\mathbf{Y}^\top\bQ\mathbf{Z}_1 \bm{\hat{\beta}}(\mathbf Z_1) - \mathbf{Y}^\top\bQ\mathbf{Z}_2 \bm{\hat{\beta}}(\mathbf Z_2)|\\ &= |\frac{1}{{n}}\mathbf{Y}^\top \bQ\mathbf Z_1\left[\mathbf Z_1^\top \bQ\mathbf Z_1\right]^{-1}\mathbf Z_1^{\top} \bQ\mathbf y -\frac{1}{{n}}\mathbf{Y}^\top \bQ\mathbf Z_2\left[\mathbf Z_2^\top \bQ\mathbf Z_2\right]^{-1}\mathbf Z_2^{\top} \bQ\mathbf y |\\
	& \leqslant J_1 + J_2 + J_3 + J_4 + J_5
	\end{align*} 
	
	where,
	
	$$
	J_1 =\frac{1}{{n}}|\mathbf{Y}^\top \bQ\mathbf Z_1\left(\left[\mathbf Z_1^\top \bQ\mathbf Z_1\right]^{-1} - \left[\mathbf{\tilde{Z}}_1^{(0)^\top} \bQ\mathbf{\tilde{Z}}_1^{(0)}\right]^{+}\right)\mathbf Z_1^{\top} \bQ\mathbf y|,
	$$
	$$
	J_2  =\frac{1}{{n}}|\mathbf{Y}^\top \bQ\mathbf Z_1 \left[\mathbf{\tilde{Z}}_1^{(0)^\top} \bQ\mathbf{\tilde{Z}}_1^{(0)}\right]^{+}\mathbf Z_1^{\top} \bQ\mathbf y - \mathbf{Y}^\top \bQ\mathbf{\tilde{Z}}_1^{(0)} \left[\mathbf{\tilde{Z}}_1^{(0)^\top} \bQ\mathbf{\tilde{Z}}_1^{(0)}\right]^{+}\mathbf{\tilde{Z}}_1^{(0)^\top} \bQ\mathbf y |,
	$$
	$$
	J_3 =\frac{1}{{n}}|\mathbf{Y}^\top \bQ\mathbf Z_2\left(\left[\mathbf Z_2^\top \bQ\mathbf Z_2\right]^{-1} - \left[\mathbf{\tilde{Z}}_2^{(0)^\top} \bQ\mathbf{\tilde{Z}}_2^{(0)}\right]^{+}\right)\mathbf Z_2^{\top} \bQ\mathbf y|,
	$$
	$$
	J_4  = \frac{1}{{n}}|\mathbf{Y}^\top \bQ\mathbf Z_2 \left[\mathbf{\tilde{Z}}_2^{(0)^\top} \bQ\mathbf{\tilde{Z}}_2^{(0)}\right]^{+}\mathbf Z_2^{\top} \bQ\mathbf y - \mathbf{Y}^\top \bQ\mathbf{\tilde{Z}}_2^{(0)} \left[\mathbf{\tilde{Z}}_2^{(0)^\top} \bQ\mathbf{\tilde{Z}}_2^{(0)}\right]^{+}\mathbf{\tilde{Z}}_2^{(0)^\top} \bQ\mathbf y |,
	$$
	$$
	J_5 = \frac{1}{n}| \mathbf{Y}^\top \bQ\mathbf{\tilde{Z}}_1^{(0)} \left[\mathbf{\tilde{Z}}_1^{(0)^\top} \bQ\mathbf{\tilde{Z}}_1^{(0)}\right]^{+}\mathbf{\tilde{Z}}_1^{(0)^\top} \bQ\mathbf y - \mathbf{Y}^\top \bQ\mathbf{\tilde{Z}}_2^{(0)} \left[\mathbf{\tilde{Z}}_2^{(0)^\top} \bQ\mathbf{\tilde{Z}}_2^{(0)}\right]^{+}\mathbf{\tilde{Z}}_2^{(0)^\top} \bQ\mathbf y |.
	$$

	First we focus on $J_5$. The terms are of the form $$\mathbf{Y}^\top \bQ\mathbf{\tilde{Z}}_h^{(0)} \left[\mathbf{\tilde{Z}}_h^{(0)^\top} \bQ\mathbf{\tilde{Z}}_h^{(0)}\right]^{+}\mathbf{\tilde{Z}}_h^{(0)^\top} \bQ\mathbf y = \mathbf{Y}^\top \bQ^{\frac 12}\dot\bP_{\bQ^{\frac \top 2}\mathbf{\tilde{Z}}_h^{(0)}} \bQ^{\frac \top 2}\mathbf y$$ where $\dot\bP_\bX$ denotes the projection operator for a matrix $\bX$. By construction of $\mathbf{\tilde{Z}}_h^{(0)}$, column space of $\bQ^{\frac \top 2}\mathbf{\tilde{Z}}_h^{(0)}$ is same as that of $\bQ^{\frac \top 2}\mathbf{{Z}}_h^{(0)}$. Hence, $\bP_{\bQ^{\frac \top 2}\mathbf{\tilde{Z}}_h^{(0)}}=\bP_{\bQ^{\frac \top 2}\mathbf{{Z}}_h^{(0)}}$, and, 
	$$\mathbf{Y}^\top \bQ\mathbf{\tilde{Z}}_h^{(0)} \left[\mathbf{\tilde{Z}}_h^{(0)^\top} \bQ\mathbf{\tilde{Z}}_h^{(0)}\right]^{+}\mathbf{\tilde{Z}}_h^{(0)^\top} \bQ\mathbf y = \mathbf{Y}^\top \bQ\mathbf{{Z}}_h^{(0)} \left[\mathbf{{Z}}_h^{(0)^\top} \bQ\mathbf{{Z}}_h^{(0)}\right]^{+}\mathbf{{Z}}_h^{(0)^\top} \bQ\mathbf y.$$
	Thus, $J_5$ simply becomes $\frac{1}{n}|\mathbf{Y}^\top\bQ\mathbf{Z}^{(0)}_1 \bm{\hat{\beta}}(\mathbf Z^{(0)}_1) - \mathbf{Y}^\top\bQ\mathbf{Z}^{(0)}_2 \bm{\hat{\beta}}(\mathbf Z^{(0)}_2)|$ which has been shown to be small earlier in the proof of this Lemma in Section \ref{sec:proofs}, irrespective of \textbf{R1} or \textbf{R2}. 
	
	To simplify $J_2$, let for $h=1,2$, $\mathbf{\tilde{Z}}^L_{h}$ and $\mathbf{\tilde{Z}}^R_{h}$ respectively denote the columns of $\mathbf{{Z}_h}$ corresponding to the bigger and smaller child nodes. As the zero column of $\mathbf{\tilde{Z}}_h^{(0)}$ aligns with the column $\mathbf{\tilde{Z}}^R_{h}$ of $\mathbf{{Z}_h}$, we have 
	
	\begin{equation}\label{eq:align}
	\mathbf{{Z}}_h \left[ \tilde{\mathbf{Z}}_h^{(0)^\top} \bQ \tilde{\mathbf{Z}}_h^{(0)}\right]^{+} \mathbf{{Z}}_h^{\top} = \mathbf{\tilde{Z}}_h \left[ \mathbf{Z}_h^{(0)^\top} \bQ \mathbf{Z}_h^{(0)}\right]^{-1} \mathbf{\tilde{Z}}_h^{\top} \mbox{ where } \mathbf{\tilde{Z}}_h = \left({\mathbf{{Z}}_h}_{\cdot,1:g^{(k)}-1}\;\; \mathbf{\tilde{Z}}^L_{h}\right).
	\end{equation}
	
	Writing, 
	$
	\mathbf{Z}^{(0)}_1 = \mathbf{\tilde{Z}}_1 + \begin{bmatrix} \mathbf{0}_{{n} \times g^{(k)}-1} & \mathbf{\tilde{Z}}^R_{1} \end{bmatrix} $, we have
	$$
	\begin{aligned}
	J_2 &=\frac 1n \Big|\mathbf{Y}^\top \bQ \mathbf{Z}_1^{(0)} \left[ \mathbf{Z}_1^{(0)^\top} \bQ \mathbf{Z}_1^{(0)}\right]^{-1} \mathbf{Z}_1^{(0)^\top} \bQ\mathbf y - \mathbf{Y}^\top \bQ \mathbf{\tilde{Z}}_1 \left[ \mathbf{Z}_1^{(0)^\top} \bQ \mathbf{Z}_1^{(0)}\right]^{-1} \mathbf{\tilde{Z}}_1^{\top} \bQ\mathbf y\Big|\\ 
	&\leq \frac 1n 
	\Big|\begin{bmatrix} \mathbf{0}_{{1} \times g^{(k)}-1}& \mathbf{Y}^\top \bQ\mathbf{\tilde{Z}}^R_{1} \end{bmatrix} \left[ \mathbf{Z}_1^{(0)^\top} \bQ \mathbf{Z}_1^{(0)}\right]^{-1}\begin{bmatrix} \mathbf{0}_{{1} \times g^{(k)}-1}& \mathbf{Y}^\top \bQ\mathbf{\tilde{Z}}^R_{1} \end{bmatrix}^\top\Big|\\
	&\qquad + \frac 2n \Big|\begin{bmatrix} \mathbf{0}_{{1} \times g^{(k)}-1}& \mathbf{Y}^\top \bQ\mathbf{\tilde{Z}}^R_{1} \end{bmatrix} \left[ \mathbf{Z}_1^{(0)^\top} \bQ \mathbf{Z}_1^{(0)}\right]^{-1}(\mathbf{Y}^\top \bQ\mathbf{\tilde{Z}}_1)^\top\Big| \\
	\end{aligned}
	$$
	Denote these two terms on the right hand side respectively as $J_{21}$ and $J_{22}$. For $J_{21}$, we have on $\Omega_n$ (set with probability at least $1-3\pi/4$),
	$$
	\begin{aligned}
	J_{21} &=\frac{1}{{n}}| \begin{bmatrix} \mathbf{0}_{{1} \times g^{(k)}-1}& \mathbf{Y}^\top \bQ\mathbf{\tilde{Z}}^R_{1} \end{bmatrix} \left[ \mathbf{Z}_1^{(0)^\top} \bQ \mathbf{Z}_1^{(0)}\right]^{-1}\begin{bmatrix} \mathbf{0}_{{1} \times g^{(k)}-1}& \mathbf{Y}^\top \bQ\mathbf{\tilde{Z}}^R_{1} \end{bmatrix}^\top|\\
	& \leqslant \frac 1n \lambda_{\max} \left[ \mathbf{Z}_1^{(0)^\top} \bQ \mathbf{Z}_1^{(0)}\right]^{-1} \| \left(\mathbf{0}_{{1} \times g^{(k)}-1}, \mathbf{Y}^\top \bQ\mathbf{\tilde{Z}}^R_{1}\right)  \|^2\\
	& = \left(\frac{\mathbf{Y}^\top \bQ\mathbf{\tilde{Z}}^R_{1} }{{n}}\right)^2\frac{{n}}{|\lambda_{min} (\mathbf{Z}_1^{(0)^\top} \bQ \mathbf{Z}_1^{(0)})|} \\
	& {\leqslant \left(\frac{\mathbf{Y}^\top \bQ\mathbf{\tilde{Z}}^R_{1} }{{n}}\right)^2 \frac{1}{\xi} \left[\left(\displaystyle{\min\limits_{l\in {1,2,\cdots,g^{(k)} - 1}} }|\ddot{\mathcal{C}}_l^{(1)}|, |\ddot{\mathcal{C}}_{g^{(k)}}^{(1)}|+|\ddot{\mathcal{C}}_{g^{(k)}+1}^{(1)}|\right) \right]^{-1}n}\\
	& \leqslant \left(\frac{\mathbf{Y}^\top \bQ\mathbf{\tilde{Z}}^R_{1} }{{n}}\right)^2 \frac{1}{\xi}{\frac{1}{\varepsilon^D - \delta^2}}; [\text{ for sufficiently large $n$}]
	\end{aligned}
	$$
	
	{Here the last two inequalities are obtained similar to the derivation of (\ref{eq:L4_P1_2})}, exploiting the lower bound of the eigenvalue in (\ref{eq:eigen}) and the fact that $Vol(\ddot{\mathcal{C}}_l^{(1)}) \geqslant \varepsilon^D, \forall l = 1, 2, \cdots , g^{(k)} - 1$, and $Vol(\ddot{\mathcal{C}}_{g^{(k)}}^{(1)} \cup \ddot{\mathcal{C}}_{g^{(k)}+1}^{(1)}) \geqslant \varepsilon^D$, in conjunction with the Glivenko-Cantelli result of (\ref{eq:glc}) . 
	We now focus on $\mathbf{Y}^\top \bQ\mathbf{\tilde{Z}}^R_{h}$. Let $l_0 \in \{g^{(k)},g^{(k)}+1\}$ denote the column of $\bZ_1$ corresponding to the smaller child node. Using (\ref{eq:qf}), we have
	
	$$
	\begin{aligned}
	\frac{1}{n}\mathbf{Y}^\top \bQ\mathbf{\tilde{Z}}^R_{1} &= \alpha \frac{\sum_i\mathbf{Z}_{1_{i,l_0}}m(X_i)}{n} + \sum_{j \neq j' = 0}^q \rho_j\rho_{j'} \frac{\sum_i\mathbf{Z}_{1_{i-j,l_0}} m(X_{i-j'} )}{n}\\
	&+ \alpha \frac{\sum_i\mathbf{Z}_{_1{i,l_0}}\eps_i}{n} + \sum_{j \neq j' = 0}^q \rho_j\rho_{j'} \frac{\sum_i\mathbf{Z}_{1_{i-j,l_0}} \eps_{i-j'} }{n}\\
	&+\frac{1}{n}\sum_{i \in \tilde{\mathcal{A}}_1}\sum_{i' \in \tilde{\mathcal{A}}_2} \tilde\gamma_{i,i'}\mathbf{Z}_{1_{h_1,l_0}}m(X_{h_2}) + \frac{1}{n}\sum_{i \in \tilde{\mathcal{A}}_1}\sum_{i' \in \tilde{\mathcal{A}}_2} \tilde\gamma_{i,i'}\mathbf{Z}_{1_{h_1,l_0}}\eps_{h_2}
	\end{aligned}
	$$

	On $\Omega_n$, using the tail bounds from Assumptions \ref{as:tail}(b) and (c) and Lemma \ref{lem:weaklaw2},   
	and the fact that $|m(\mathbf{x})| \leqslant M_0; \mathbf{x} \in [0,1]^D$, we have for large enough $n$, 
	$$
	\begin{aligned}
	\frac{1}{n}|\mathbf{Y}^\top \bQ\mathbf{\tilde{Z}}^R_{1}| &\leqslant \frac{1}{n}M_0 \left[\alpha \sum_i\mathbf{Z}_{i,l_0} +\left[\sum_{j \neq j' = 0}^q |\rho_j\rho_{j'}| \sum_i\mathbf{Z}_{i-j,l_0} + \sum_{i \in \tilde{\mathcal{A}}_1}\sum_{i' \in \tilde{\mathcal{A}}_2} |\tilde\gamma_{i,i'}|\mathbf{Z}_{h_1,l_0}\right]\right]\\
	& +\frac{\alpha}{n}\Big| \sum_i\mathbf{Z}_{i,l_0}\eps_i\Big| + \sum_{j \neq j' = 0}^q \frac{|\rho_j\rho_{j'}|}{n} |\sum_i\mathbf{Z}_{i-j,l_0} \eps_{i-j'}| + \frac{C_{\pi} \sqrt{\ln n}}{n}\sum_{i \in \tilde{\mathcal{A}}_1}\sum_{i' \in \tilde{\mathcal{A}}_2} |\tilde\gamma_{i,i'}|\mathbf{Z}_{h_1,l_0}\\
	& \leqslant (\alpha +\sum_{j \neq j' = 0}^q |\rho_j\rho_{j'}|) \left( M_0 \frac{\sum_i\mathbf{Z}_{i,l_0}}{n} + \phi \right)  +(C_{\pi} \sqrt{\log n} + M_0) \frac{1}{n}\sum_{i \in \tilde{\mathcal{A}}_1}\sum_{i' \in \tilde{\mathcal{A}}_2} |\tilde\gamma_{i,i'}|.
	\end{aligned}
	$$
	
	The terms involving $\tilde \gamma_{i,i'}$ are $o(1)$ as there are only $O(q^2)$ of them and the factor $\log n/n \to 0$. The terms involving $\eps_i$'s were bounded using Lemma \ref{lem:weaklaw2}.
	
	Under \textbf{R2} on $\Omega_n$, using (\ref{eq:glc}), for large enough $n$ 
	\begin{equation}\label{eq:crosssmall}
	\frac{1}{n}|\mathbf{Y}^\top \bQ\mathbf{\tilde{Z}}^R_{1}| \leqslant (\alpha +\sum_{j \neq j' = 0}^q |\rho_j\rho_{j'}|) M_0 (\sqrt{\delta} + \delta^2) + C_7\phi. 
	\end{equation}
	
	As $\phi$ can be chosen arbitrarily small, we have $J_{21} \to 0$ uniformly as $\delta \to 0$. Similarly, 
	on $\Omega_n$ for large enough $n$, we have 
	$$
	\begin{aligned}
	J_{23} &:= \frac{1}{{n}}| \mathbf{Y}^\top \bQ\mathbf{\tilde{Z}}_1 \left[ \mathbf{Z}_1^{(0)^\top} \bQ \mathbf{Z}_1^{(0)}\right]^{-1}(\mathbf{Y}^\top \bQ\mathbf{\tilde{Z}}_1)^\top|\\ &\leqslant \frac{1}{n}\lambda_{\max}\left(\left[ \mathbf{Z}_1^{(0)^\top} \bQ \mathbf{Z}_1^{(0)}\right]^{-1}\right) \| \mathbf{Y}^\top \bQ\mathbf{\tilde{Z}}_1\|_2^2\\
	&\leqslant \frac{\| \mathbf{Y}^\top \bQ\mathbf{\tilde{Z}}_1\|_2^2/n^2}{\lambda_{\min}\left(\left[ \mathbf{Z}_1^{(0)^\top} \bQ \mathbf{Z}_1^{(0)}\right]\right)/n}\\  
	&\leqslant {\frac{C_8}{\varepsilon^D - \delta^2} }
	\end{aligned}
	$$
	
	The last inequality follows from (\ref{eq:eigenvol}) as the minimum volume of the parent nodes are $\eps^D$, and the following bound for $\frac{1}{n^2}\| \mathbf{Y}^\top \bQ\mathbf{\tilde{Z}}_1\|_2^2$  on $\Omega_n$. Using sub-multiplicative property of $\mathbb L_2$ norm, we write $\frac 1n \|\mathbf{Y}^\top \bQ\mathbf{\tilde{Z}}_1\|_2 \leq \frac 1n \|\bQ\|_2 \|\widetilde \bZ_1\|_2 \|\bY\|$. Following assumption \ref{as:tail} (c) and its corollary Lemma \ref{lem:weaklaw2}, $y^\top y/n$ is $O(1)$. From Assumption \ref{as:chol}, $\|\bQ\|_2$ is bounded. Finally, as $\bZ_1^\top \bZ_1/n$ is diagonal with entries $\leqslant 1$, we have $\frac{1}{n^2}\| \mathbf{Y}^\top \bQ\mathbf{\tilde{Z}}_1\|_2^2$ to be $O(1)$. 
	
	By Cauchy-Schwartz inequality, we have 
	$J_{22} \leq \sqrt{J_{21}J_{23}}$ thereby making the $J_2$ term $O(1)$. By symmetry of \textbf{R2}, $J_4$ is also $O(1)$.
	
	Next we focus on $J_1$.  
	Noting that $\mathbf{\tilde{Z}}^{(0)}_1$ has a column of zeros aligned with the column $\mathbf{\tilde{Z}}_1^R$ of $\mathbf Z_1$, without loss of generality we can write
	$$
	J_1 = \frac{1}{{n}}\Bigg|\bm{\psi}^\top\left(\begin{bmatrix} \mathbf{\tilde{Z}}_1^\top \bQ\mathbf{\tilde{Z}}_1 & \mathbf{\tilde{Z}}_1^\top \bQ\mathbf{\tilde{Z}}_1^R;\\ \mathbf{\tilde{Z}}_1^{R^\top}\bQ\mathbf{\tilde{Z}}_1 & \mathbf{\tilde{Z}}_1^{R^\top}\bQ\mathbf{\tilde{Z}}_1^{R} \end{bmatrix}^{-1} - \begin{bmatrix} \left(\mathbf{Z}_1^{(0)^\top} \bQ\mathbf{Z}_1^{(0)}\right)^{-1} & 0\\ 0 & 0 \end{bmatrix}\right)\bm{\psi}\Bigg|
	$$
	where,
	$
	\bm{\psi}^\top = \bm{\psi}_1^\top+\bm{\psi}_2^\top;\: \: \bm{\psi}_1^\top = \begin{bmatrix} \mathbf{Y}^\top \bQ\mathbf{\tilde{Z}}_1 & 0 \end{bmatrix};\: \: \bm{\psi}_2^\top = \begin{bmatrix} \mathbf{0}_{1 \times g^{(k)}} &  \mathbf{Y}^\top \bQ\mathbf{\tilde{Z}}_1^R \end{bmatrix}.
	$
	Let $$
	\mathbf{U}_{11} = \mathbf{\tilde{Z}}_1^\top \bQ\mathbf{\tilde{Z}}_1;\:\:\: \mathbf{U}_{12} = \mathbf{\tilde{Z}}_1^\top\bQ\mathbf{\tilde{Z}}_1^R;\:\:\: \mathbf{U}_{22} = \mathbf{\tilde{Z}}_1^{R^\top}\bQ\mathbf{\tilde{Z}}_1^{R};\:\:\: \mathbf{U}_{11}^{(0)} = \mathbf{Z}_1^{(0)^\top} \bQ\mathbf{Z}_1^{(0)};
	$$
	
	Hence, we have
	$$
	\begin{aligned}
	J_1 \leqslant J_{11} +  J_{12} +  J_{13} 
	\end{aligned}
	$$
	
	where,
	$$
	J_{11} = \frac{1}{{n}}\Bigg|\bm{\psi}_1^\top\left(\begin{bmatrix} \mathbf U_{11} & \mathbf  U_{12};\\ \mathbf U_{12}^\top & \mathbf U_{22} \end{bmatrix}^{-1} - \begin{bmatrix} \mathbf U_{11}^{(0)^{-1}} & 0\\ 0 & 0 \end{bmatrix}\right)\bm{\psi}_1\Bigg|
	$$
	
	$$
	J_{12} = \frac{2}{{n}}\Bigg|\bm{\psi}_2^\top\left(\begin{bmatrix} \mathbf U_{11} & \mathbf  U_{12};\\ \mathbf U_{12}^\top & \mathbf U_{22} \end{bmatrix}^{-1} - \begin{bmatrix} \mathbf U_{11}^{(0)^{-1}} & 0\\ 0 & 0 \end{bmatrix}\right)\bm{\psi}_1\Bigg|
	$$

	$$
	J_{13} = \frac{1}{{n}}\Bigg|\bm{\psi}_2^\top\left(\begin{bmatrix} \mathbf U_{11} & \mathbf  U_{12};\\ \mathbf U_{12}^\top & \mathbf U_{22} \end{bmatrix}^{-1} - \begin{bmatrix} \mathbf U_{11}^{(0)^{-1}} & 0\\ 0 & 0 \end{bmatrix}\right)\bm{\psi}_2\Bigg|
	$$
	
	$$
	\begin{aligned}
	J_{11} =&\frac{1}{{n}}\Bigg|\bm{\psi}_1^\top\left(\begin{bmatrix} \mathbf U_{11} & \mathbf  U_{12};\\ \mathbf U_{12}^\top & \mathbf U_{22} \end{bmatrix}^{-1} - \begin{bmatrix} \mathbf U_{11}^{(0)^{-1}} & 0\\ 0 & 0 \end{bmatrix}\right)\bm{\psi}_1\Bigg|\\ 
	&=\frac{1}{{n}}|\mathbf{Y}^\top \bQ\mathbf{\tilde{Z}}_1\left(\left(\mathbf{U}_{11} - \mathbf{U}_{12}\mathbf{U}_{22}^{-1}\mathbf{U}_{12}^\top \right)^{-1} - \mathbf{U}_{11}^{(0)^{-1}}\right) (\mathbf{Y}^\top \bQ\mathbf{\tilde{Z}}_1)^\top|\\
	&\leqslant \frac{\mathbf{Y}^\top \bQ\mathbf{\tilde{Z}}_1(\mathbf{Y}^\top \bQ\mathbf{\tilde{Z}}_1)^\top}{{n}}\|\left(\left(\mathbf{U}_{11} - \mathbf{U}_{12}\mathbf{U}_{22}^{-1}\mathbf{U}_{12}^\top \right)^{-1} - \mathbf{U}_{11}^{(0)^{-1}}\right)\|_2\\
	&\leqslant \frac{\mathbf{Y}^\top \bQ\mathbf{\tilde{Z}}_1(\mathbf{Y}^\top \bQ\mathbf{\tilde{Z}}_1)^\top}{{n}} \frac{\|\mathbf{U}_{11} - \mathbf{U}_{12}\mathbf{U}_{22}^{-1}\mathbf{U}_{12}^\top - \mathbf{U}_{11}^{(0)}\|_2}{\lambda_{min}\left(\mathbf{U}_{11} - \mathbf{U}_{12}\mathbf{U}_{22}^{-1}\mathbf{U}_{12}^\top \right)\lambda_{min}\left(\mathbf{U}_{11}^{(0)} \right)}
	\end{aligned}
	$$
	where the 
	second inequality follows from taking $\mathbf{A} = \mathbf{U}_{11} - \mathbf{U}_{12}\mathbf{U}_{22}^{-1}\mathbf{U}_{12}^\top$ and $\mathbf{B} = \mathbf{U}_{11}^{(0)}$ in the following identity. 
	$$
	\begin{aligned}
	\| \mathbf{A}^{-1} - \mathbf{B}^{-1}\|_2
	&= \|\mathbf{A}^{-1}(\mathbf{B} - \mathbf{A})\mathbf{B}^{-1}\|_2\\
	&\leqslant \| \mathbf{A}^{-1}\|_2 \|\mathbf{A} - \mathbf{B} \|_2 \| \mathbf{B}^{-1}\|_2 ; [\|.\|_2 \text{ is submultiplicative}]\\
	&= \frac{\|\mathbf{A} - \mathbf{B} \|_2}{\lambda_{\min}  (\mathbf{A}) \lambda_{\min}  (\mathbf{B})} .
	\end{aligned}
	$$
	We have already established the following bound on $|\lambda^{-1}_{min}\left(\mathbf{U}_{11}^{(0)} \right)| = |\lambda^{-1}_{\min} (\mathbf{Z}_1^{(0)^\top} \bQ\mathbf{Z}_1^{(0)})|$ in  (\ref{eq:eigenvol}), 
	
	\begin{subequations}
		\begin{equation}
		\label{eq:J_1_P1}
		|\lambda^{-1}_{min}\left(\mathbf{U}_{11}^{(0)} \right)| \leqslant \frac{1}{\xi} \left[\min\left(\displaystyle{\min\limits_{l\in {1,2,\cdots,g^{(k)} - 1}} }|\ddot{\mathcal{C}}_l^{(1)}|, |\ddot{\mathcal{C}}_{g^{(k)}}^{(1)}|+|\ddot{\mathcal{C}}_{g^{(k)}+1}^{(1)}|\right) \right]^{-1}
		\end{equation}
		
		Next, we use Weyl's inequality \citep{horn1994topics}. Let $\mathbf{A}, \mathbf{B}, \mathbf{C}$ be any $j \times j$ Hermitian matrices, with eigenvalues $\lambda_{\max}(\mathbf{A}) = a_{(1)} \geqslant a_{(2)} \geqslant \cdots \geqslant a_{(j)} = \lambda_{\min}(\mathbf{A})$;  $\lambda_{\max}(\mathbf{B}) = b_{(1)} \geqslant b_{(2)} \geqslant \cdots \geqslant b_{(j)} = \lambda_{\min}(\mathbf{B})$; and $\lambda_{\max}(\mathbf{C}) = c_{(1)} \geqslant c_{(2)} \geqslant \cdots \geqslant c_{(j)} = \lambda_{\min}(\mathbf{C})$ respectively, with $\mathbf{A} = \mathbf{B} + \mathbf{C}$. Then, we have:
		
		$$
		b_{(j_0)} + c_{(j)} \leqslant a_{(j_0)} \leqslant b_{(j_0)} + c_{(1)} ; \forall j_0 = 1,2,\cdots,j.
		$$
		
		Specifically, with $j_0 = j$, we have,
		\begin{equation}\label{eq:weylmin}
		\lambda_{\min}(\mathbf{A} - \mathbf{B}) \geqslant \lambda_{\min}(\mathbf{A}) + \lambda_{\min}(-\mathbf{B}) = \lambda_{\min}(\mathbf{A}) - \lambda_{\max}(\mathbf{B})
		\end{equation}
		
		Applying the aforementioned inequality with $\mathbf{A} = \mathbf{U}_{11}$ and $\mathbf{B} = \mathbf{U}_{12}\mathbf{U}_{22}^{-1}\mathbf{U}_{12}^\top$, we have
		
		$$
		\lambda_{min}\left(\mathbf{U}_{11} - \mathbf{U}_{12}\mathbf{U}_{22}^{-1}\mathbf{U}_{12}^\top \right) \geqslant \lambda_{min}(\mathbf{U}_{11})- \lambda_{max} (\mathbf{U}_{12}\mathbf{U}_{22}^{-1}\mathbf{U}_{12}^\top)
		$$
		
		Using (\ref{eq:eigenvol}) as before, we can derive	
		\begin{equation}
		\label{eq:J_1_P21}
		|\lambda_{min}^{-1}(\mathbf{U}_{11})| \leqslant \frac{1}{\xi} \min\left(\displaystyle{\min_{l\in {1,2,\cdots,g^{(k)} - 1}} }|\ddot{\mathcal{C}}_l^{(1)}|, \mathbf{1}^\top \widetilde{\bZ}_1^L \right)^{-1} = \frac{1}{\xi} \left(\displaystyle{\min_{l\in {1,2,\cdots,g^{(k)} }} }|\ddot{\mathcal{C}}_l^{(1)}|\right)^{-1}
		\end{equation}
		where the last equality follows from letting, without loss of generality, $\widetilde \bZ_1^L$ and $\widetilde \bZ_1^R$ to be respectively the $(g^{(k)})^{th}$ and $(g^{(k)}+1)^{th}$ column of $\bZ_1$.

		Using $\mathbf{U}_{22}$ is $1 \times 1$ and $\mathbf{U}_{12}$ is $g^{(k)} \times 1$, 
		
		$$
		\begin{aligned}
		\lambda_{max} (\mathbf{U}_{12}\mathbf{U}_{22}^{-1}\mathbf{U}_{12}^\top) 
		= \frac{\left(\mathbf{\tilde{Z}}_1^{R^\top}\bQ\mathbf{\tilde{Z}}_1\right)\left(\mathbf{\tilde{Z}}_1^{R^\top}\bQ\mathbf{\tilde{Z}}_1\right)^\top }{\mathbf{\tilde{Z}}_1^{R^\top}\bQ\mathbf{\tilde{Z}}_1^{R}}
		\end{aligned}
		$$
		Using Assumption \ref{as:chol}, we have
		\begin{align*}
		&\left(\mathbf{\tilde{Z}}_1^{R^\top}\bQ\mathbf{\tilde{Z}}_1\right)\left(\mathbf{\tilde{Z}}_1^{R^\top}\bQ\mathbf{\tilde{Z}}_1\right)^\top \\
		&= \sum_{l = 1}^{g^{(k)}}\left(\mathbf{\tilde{Z}}_1^{R^\top}\bQ\mathbf{\tilde{Z}}_{1_{.,l}}\right)^2\\
		&= \sum_{l = 1}^{g^{(k)}}\Bigg(\alpha\sum_i\mathbf{Z_1}_{i,g^{(k)}+1}\mathbf{Z_1}_{i,l} + \sum_{j \neq j' = 0}^q \rho_j\rho_{j'} \sum_i\mathbf{Z_1}_{i-j,g^{(k)}+1} \mathbf{Z_1}_{i-j',l} \\
		&\quad\quad\quad +\sum_{i \in \tilde{\mathcal{A}}_1}\sum_{i' \in \tilde{\mathcal{A}}_2} \tilde\gamma_{i,i'}\mathbf{Z_1}_{i,g^{(k)}+1}\mathbf{Z_1}_{i',l} \Bigg)^2\\
		&=  \sum_{l = 1}^{g^{(k)}}\left(\sum_{j \neq j' = 0}^q \rho_j\rho_{j'} \sum_i\mathbf{Z_1}_{i-j,g^{(k)}+1} \mathbf{Z_1}_{i-j',l} +\sum_{i \in \tilde{\mathcal{A}}_1}\sum_{i' \in \tilde{\mathcal{A}}_2} \tilde\gamma_{i,i'}\mathbf{Z_1}_{i,g^{(k)}+1}\mathbf{Z_1}_{i',l} \right)^2; [\text{ $l < g^{(k)}+1$}]\\
		&\leqslant \sum_{l = 1}^{g^{(k)}}\left(\sum_{j \neq j' = 0}^q |\rho_j\rho_{j'}| \sum_i\mathbf{Z_1}_{i-j,g^{(k)}+1} \mathbf{Z_1}_{i-j',l} +\sum_{i \in \tilde{\mathcal{A}}_1}\sum_{i' \in \tilde{\mathcal{A}}_2} |\tilde\gamma_{i,i'}|\mathbf{Z_1}_{i,g^{(k)}+1}\mathbf{Z_1}_{i',l} \right)^2\\
		&\leqslant \left(\sum_{j \neq j' = 0}^q |\rho_j\rho_{j'}| \sum_i\mathbf{Z_1}_{i-j,g^{(k)}+1} \sum_{l = 1}^{g^{(k)}}\mathbf{Z_1}_{i-j',l} +\sum_{i \in \tilde{\mathcal{A}}_1}\sum_{i' \in \tilde{\mathcal{A}}_2} |\tilde\gamma_{i,i'}|\mathbf{Z_1}_{i,g^{(k)}+1}\sum_{l = 1}^{g^{(k)}}\mathbf{Z_1}_{i',l} \right)^2\\
		&\leqslant \left(\sum_{j \neq j' = 0}^q |\rho_j\rho_{j'}| \sum_i\mathbf{Z_1}_{i-j,g^{(k)}+1}  +\sum_{i \in \tilde{\mathcal{A}}_1}\sum_{i' \in \tilde{\mathcal{A}}_2} |\tilde\gamma_{i,i'}|\mathbf{Z_1}_{i,g^{(k)}+1} \right)^2\\
		\end{align*}
		The last expression above can be written as $(\sum_i c_i \bZ_{i,g^{(k)}+1})^2$ where the constants $c_i \geq 0$ can only take values from a set of positive values free of $n$ and of $O(q^2)$ cardinality. Hence, replacing $c_i$'s with their maximum, we have $\left(\mathbf{\tilde{Z}}_1^{R^\top}\bQ\mathbf{\tilde{Z}}_1\right)\left(\mathbf{\tilde{Z}}_1^{R^\top}\bQ\mathbf{\tilde{Z}}_1\right)^\top \leq c (|\ddot{\mathcal C}^{(1)}_{g^{(k)}+1}|^2)$ for some constant $c$.
		
		Similarly, we can prove, 
		
		$$
		\mathbf{\tilde{Z}}_1^{R^\top}\bQ\mathbf{\tilde{Z}}_1^{R} \geqslant \xi \left(|\ddot{\mathcal{C}}_{g^{(k)}+1}^{(1)}|\right)
		$$

		Combining, we have
		
		\begin{equation}
		\label{eq:J_1_P22}
		\lambda_{max} (\mathbf{U}_{12}\mathbf{U}_{22}^{-1}\mathbf{U}_{12}^\top) \leq \frac c\xi (|\ddot{\mathcal{C}}_{g^{(k)}+1}^{(1)}|)
		\end{equation}
		
		From \cref{eq:J_1_P21,eq:J_1_P22}, we have 
		\begin{equation}
		\label{eq:J_1_P2}
		\lambda_{min}\left(\mathbf{U}_{11} - \mathbf{U}_{12}\mathbf{U}_{22}^{-1}\mathbf{U}_{12}^\top \right) \geqslant \xi \min_{l\in {1,2,\cdots,g^{(k)}}} |\ddot{\mathcal{C}}_l^{(1)}| - \frac{c}{\xi}|\ddot{\mathcal{C}}_{g^{(k)}+1}^{(1)}|.
		\end{equation}

		Next, to control $\|\mathbf{U}_{11} - \mathbf{U}_{12}\mathbf{U}_{22}^{-1}\mathbf{U}_{12}^\top - \mathbf{U}_{11}^{(0)}\|_2$, let $\mathbf{A} = \mathbf{U}_{11} - \mathbf{U}_{11}^{(0)}$ and $\mathbf{B} = \mathbf{U}_{12}\mathbf{U}_{22}^{-1}\mathbf{U}_{12}^\top$. We note that for any symmetric matrix $\bA-\bB$, $\|\bA-\bB\|_2 =\max \{\lambda_{\max}(\bA-\bB), -\lambda_{\min}(\bA-\bB)\}$. By Weyl's inequality \citep{horn1994topics}, with $j_0 = 1$, we have 
		
		$$
		\lambda_{\max}(\mathbf{A} - \mathbf{B}) \leqslant \lambda_{\max}(\mathbf{A}) + \lambda_{\max}(-\mathbf{B}) = \lambda_{\max}(\mathbf{A}) - \lambda_{\min}(\mathbf{B})
		$$
		
		Since  $\mathbf{U}_{12}\mathbf{U}_{22}^{-1}\mathbf{U}_{12}^\top$ is rank-deficient, $\lambda_{\min}(\bB)=0$, and we have
		$$
		\lambda_{max}\left(\mathbf{U}_{11} - \mathbf{U}_{12}\mathbf{U}_{22}^{-1}\mathbf{U}_{12}^\top - \mathbf{U}_{11}^{(0)}\right) \leqslant \lambda_{max}\left(\mathbf{U}_{11} - \mathbf{U}_{11}^{(0)}\right).
		$$

		$$
		\begin{aligned}
		& \mathbf{U}_{11} - \mathbf{U}_{11}^{(0)} \\
		& =\mathbf{\tilde{Z}}_1^\top \bQ\mathbf{\tilde{Z}}_1 - \mathbf{Z}_1^{(0)^\top} \bQ\mathbf{Z}_1^{(0)}\\ 
		&= \mathbf{\tilde{Z}}_1^\top \bQ\mathbf{\tilde{Z}}_1  - \left(\mathbf{\tilde{Z}}_1 + \begin{bmatrix} \mathbf{0}_{{n} \times g^{(k)}-1} & \mathbf{\tilde{Z}}^R_{1} \end{bmatrix} \right)^\top \bQ \left(\mathbf{\tilde{Z}}_1 + \begin{bmatrix} \mathbf{0}_{{n} \times g^{(k)}-1} & \mathbf{\tilde{Z}}^R_{1} \end{bmatrix} \right)\\
		&= - \mathbf{\tilde{Z}}_1^\top \bQ\begin{bmatrix} \mathbf{0}_{{n} \times g^{(k)}-1} & \mathbf{\tilde{Z}}^R_{1} \end{bmatrix} - \begin{bmatrix} \mathbf{0}_{{n} \times g^{(k)}-1} & \mathbf{\tilde{Z}}^R_{1} \end{bmatrix}^\top\bQ\mathbf{\tilde{Z}}_1\\
		&\quad\quad\quad\quad\quad\quad- \begin{bmatrix} \mathbf{0}_{{n} \times g^{(k)}-1} & \mathbf{\tilde{Z}}^R_{1} \end{bmatrix}^\top\bQ\begin{bmatrix} \mathbf{0}_{{n} \times g^{(k)}-1} & \mathbf{\tilde{Z}}^R_{1} \end{bmatrix}\\
		&= - \begin{bmatrix} \mathbf{0}_{{g^{(k)}} \times g^{(k)}-1} & \mathbf{\tilde{Z}}_1^\top \bQ\mathbf{\tilde{Z}}^R_{1} \end{bmatrix} - \begin{bmatrix} \mathbf{0}_{{g^{(k)}} \times g^{(k)}-1} & \mathbf{\tilde{Z}}_1^\top \bQ\mathbf{\tilde{Z}}^R_{1} \end{bmatrix}^\top - \begin{bmatrix} \mathbf{0}_{{g^{(k)}} - 1 \times g^{(k)}-1} & \mathbf{0}_{{g^{(k)}} - 1 \times 1} \\ \mathbf{0}_{{1} \times g^{(k)}-1} & \mathbf{\tilde{Z}}_1^{R^\top}\bQ\mathbf{\tilde{Z}}_1^{R} \end{bmatrix}
		\end{aligned}
		$$
		
		Since $\left(\mathbf{U}_{11} - \mathbf{U}_{11}^{(0)}\right)$ is symmetric and only has a non-zero last row and column, by Gershgorin circle theorem, all its eigen-values lie in $[-u,u]$ where $u$ is the sum of absolute values of its last row, i.e., $u \leqslant \|\mathbf{\tilde{Z}}_1^{R^\top}\bQ\mathbf{\tilde{Z}}_1\|_1 + \mathbf{\tilde{Z}}_1^{R^\top}\bQ\mathbf{\tilde{Z}}_1^{R}$.
		Similar to the derivation of (\ref{eq:J_1_P22}), we can establish that each of the terms in $u$ are of the order $|\ddot{\mathcal{C}}_{g^{(k)}+1}^{(1)}|$. Hence, 
		
		\begin{equation}
		\label{eq:J_1_P3}
		\lambda_{max}\left(\mathbf{U}_{11} - \mathbf{U}_{12}\mathbf{U}_{22}^{-1}\mathbf{U}_{12}^\top - \mathbf{U}_{11}^{(0)}\right) \leqslant C_9\left(|\ddot{\mathcal{C}}_{g^{(k)}+1}^{(1)}|\right)
		\end{equation}
		
		Similarly, by Weyl's inequality and applying (\ref{eq:J_1_P2}) we have
		\begin{equation}
		\label{eq:J_1_P3p2}
		\begin{array}{cl}
		\lambda_{min}\left(\mathbf{U}_{11} - \mathbf{U}_{12}\mathbf{U}_{22}^{-1}\mathbf{U}_{12}^\top - \mathbf{U}_{11}^{(0)}\right) & \geqslant -\lambda_{\max}(\mathbf{U}_{12}\mathbf{U}_{22}^{-1}\mathbf{U}_{12}^\top) + \lambda_{\min}\left(\mathbf{U}_{11} - \mathbf{U}_{11}^{(0)}\right)\\
		& \geqslant - C_{10} \left(|\ddot{\mathcal{C}}_{g^{(k)}+1}^{(1)}|\right)
		\end{array}
		\end{equation}
		
		As proved in the case of $J_2$, for large $n$, we have on $\Omega_n$, 
		\begin{equation}
		\label{eq:J_1_P4}
		\frac{\mathbf{Y}^\top \bQ\mathbf{\tilde{Z}}_1(\mathbf{Y}^\top \bQ\mathbf{\tilde{Z}}_1)^\top}{{n}^2} \leqslant C_{8}
		\end{equation}
		Combining \cref{eq:J_1_P1,eq:J_1_P2,eq:J_1_P3,eq:J_1_P3p2,eq:J_1_P4} for large $n$, we have 
		
		$$
		\begin{aligned}
		J_{11}&\leqslant \frac{\mathbf{Y}^\top \bQ\mathbf{\tilde{Z}}_1(\mathbf{Y}^\top \bQ\mathbf{\tilde{Z}}_1)^\top}{{n^2}} \frac{\|\left(\mathbf{U}_{11} - \mathbf{U}_{12}\mathbf{U}_{22}^{-1}\mathbf{U}_{12}^\top - \mathbf{U}_{11}^{(0)}\right)\|_2/n}{(|\lambda_{min}\left(\mathbf{U}_{11} - \mathbf{U}_{12}\mathbf{U}_{22}^{-1}\mathbf{U}_{12}^\top \right)|/n)|\lambda_{min}\left(\mathbf{U}_{11}^{(0)} \right)/n|} \\
		&=C_{11}\frac{|\ddot{\mathcal{C}}_{g^{(k)}+1}^{(1)}|/n}{(\xi \min_{l\in {1,2,\cdots,g^{(k)}}} |\ddot{\mathcal{C}}_l^{(1)}|/n - \frac{c}{\xi}|\ddot{\mathcal{C}}_{g^{(k)}+1}^{(1)}|/n)}\times\\
		&\quad\quad\quad \frac{1}{\left[\min\left(\displaystyle{\min\limits_{l\in {1,2,\cdots,g^{(k)} - 1}} }|\ddot{\mathcal{C}}_l^{(1)}|, |\ddot{\mathcal{C}}_{g^{(k)}}^{(1)}|+|\ddot{\mathcal{C}}_{g^{(k)}+1}^{(1)}|\right) \right]/n}\\
		&\leqslant C_{11}\frac{(\sqrt{\delta}+\delta^2)}{(\xi(\varepsilon^D(1-\sqrt{\delta})-\delta^2)  - \frac{c}{\xi}(\sqrt{\delta}+\delta^2)}\frac{1}{\varepsilon^D - \delta^2}
		\end{aligned}
		$$
		The aforementioned inequality used the following:
		\begin{enumerate}
			\item  $Vol(|\ddot{\mathcal{C}}_l^{(1)}|) \geqslant \varepsilon^D; \forall l = 1,\cdots, g^{(k)}-1;$
			\item $Vol(|\ddot{\mathcal{C}}_{g^{(k)}}^{(1)}|) \geqslant \varepsilon^D(1-\sqrt{\delta});$
			\item $Vol(|\ddot{\mathcal{C}}_{g^{(k)}+1}^{(1)}|) \leqslant \sqrt{\delta}$
			\item $Vol(|\ddot{\mathcal{C}}_{g^{(k)}+1}^{(1)}|) + Vol(|\ddot{\mathcal{C}}_{g^{(k)}}^{(1)}|) \geqslant \varepsilon^D$
		\end{enumerate}
		Hence $J_{11}$ converges to 0 in probability uniformly as $\delta \downarrow 0$. 
		Next,
		
		\begin{align*}
		J_{13} &= \frac{1}{{n}}\Bigg|\begin{bmatrix} \mathbf{0}_{1 \times g^{(k)}} &  \mathbf{Y}^\top \bQ\mathbf{\tilde{Z}}_1^R \end{bmatrix}\left(\begin{bmatrix} \mathbf U_{11} & \mathbf  U_{12};\\ \mathbf U_{12}^\top & \mathbf U_{22} \end{bmatrix}^{-1} - \begin{bmatrix} \mathbf U_{11}^{(0)^{-1}} & 0\\ 0 & 0 \end{bmatrix}\right)\begin{bmatrix} \mathbf{0}_{1 \times g^{(k)}} &  \mathbf{Y}^\top \bQ\mathbf{\tilde{Z}}_1^R \end{bmatrix}^\top\Bigg| \\
		&= \frac{1}{{n}}\Bigg|\begin{bmatrix} \mathbf{0}_{1 \times g^{(k)}} &  \mathbf{Y}^\top \bQ\mathbf{\tilde{Z}}_1^R \end{bmatrix}\begin{bmatrix} \mathbf U_{11} & \mathbf  U_{12};\\ \mathbf U_{12}^\top & \mathbf U_{22} \end{bmatrix}^{-1} \begin{bmatrix} \mathbf{0}_{1 \times g^{(k)}} &  \mathbf{Y}^\top \bQ\mathbf{\tilde{Z}}_1^R \end{bmatrix}^\top\Bigg| \\
		&= \frac{1}{{n}}\Bigg|\begin{bmatrix} \mathbf{0}_{1 \times g^{(k)}} &  \mathbf{Y}^\top \bQ\mathbf{\tilde{Z}}_1^R \end{bmatrix}\left[\begin{bmatrix} \mathbf{\tilde{Z}}_1 & \mathbf{\tilde{Z}}_1^{R} \end{bmatrix}^{\top}\bQ\begin{bmatrix} \mathbf{\tilde{Z}}_1 & \mathbf{\tilde{Z}}_1^{R} \end{bmatrix}\right]^{-1} \begin{bmatrix} \mathbf{0}_{1 \times g^{(k)}} &  \mathbf{Y}^\top \bQ\mathbf{\tilde{Z}}_1^R \end{bmatrix}^\top\Bigg|\\
		&\leqslant \frac{1}{n}\lambda_{\max} \left[\begin{bmatrix} \mathbf{\tilde{Z}}_1 & \mathbf{\tilde{Z}}_1^{R} \end{bmatrix}^{\top}\bQ\begin{bmatrix} \mathbf{\tilde{Z}}_1 & \mathbf{\tilde{Z}}_1^{R} \end{bmatrix}\right]^{-1} \|  (\mathbf{0}_{1 \times g^{(k)}}, \mathbf{Y}^\top \bQ\mathbf{\tilde{Z}}_1^R )\|_2^2 \\
		& \leqslant \frac{\mathbf{Y}^\top \bQ\mathbf{\tilde{Z}}_1^R (\mathbf{Y}^\top  \bQ\mathbf{\tilde{Z}}_1^R)^\top}{{n^2}}\frac{n}{\xi}\left[\min\limits_{l\in {1,2,\cdots,g^{(k)} + 1}}|\ddot{\mathcal{C}}_l^{(1)}|\right]^{-1}
		\end{align*}

		Here the last inequality follows from (\ref{eq:eigenvol}). 
		Additionally, in (\ref{eq:crosssmall}) we have established for large enough $n$, we have 
		
		$$
		\begin{aligned}
		\frac{1}{n}|\mathbf{Y}^\top \bQ\mathbf{\tilde{Z}}^R_{1}| &\leqslant (\alpha + \sum_{j \neq j' = 0}^q |\rho_j\rho_{j'}|) M_0 (\sqrt{\delta} + \delta^2) + C_7\phi. 
		\end{aligned}
		$$
		
		Hence, 
		we have the following:
		$$
		J_{13} \leqslant \left((\alpha + \sum_{j \neq j' = 0}^q |\rho_j\rho_{j'}|) M_0 (\sqrt{\delta} + \delta^2) + C_7\phi\right)^2 \frac{1}{\xi}\frac{1}{\sqrt{\delta} -\delta^2}
		$$
		
		As $\phi$ can be chosen arbitrarily small, this converges to 0 in probability uniformly as $\delta \downarrow 0$. 
		By the Cauchy-Schwarz inequality, $J_{12} \leqslant 2 \sqrt{|J_{11}||J_{13}|}$.  
		Hence $J_1$ (and similarly $J_3$) converges to 0 in probability uniformly as $\delta \downarrow 0$, completing the proof of the lemma under \textbf{R2}.
		
	\end{subequations}
\end{proof}

\appendix

\section{Appendix: Technical Results}\label{sec:tech}

\begin{lemma}\label{lem:strlaw}
	
	Under Assumption \ref{as:mix}, we have
	\begin{enumerate}
		\item {$\frac{1}{n}\sum_{i} B_i\eps_i \overset{a.s.}{\to} 0; B_i \overset{i.i.d}{\sim} Bernoulli(p); \forall p> 0; B_i \independent \eps_i$; 
			\item $\exists n_1$ such that $\forall n > n_1$, we have $\frac{1}n\sum_{i } \eps_i^2 \leqslant C_0$} a.s.
	\end{enumerate}
\end{lemma}

\begin{proof}
	
	For any $\beta$-mixing process $\{\chi_i\}$, the $\beta$-mixing coefficients can be written as 
	$$ \beta_\chi(a)  := \sup_t \| \mathbb{P}_{-\infty}^t \otimes \mathbb{P}_{t+a}^\infty - \mathbb{P}_{t,a} \|_{TV}$$ where, $ \mathbb{P}_{a}^b$ is the joint-distribution of $\{\chi_i\}_{\{a < i \leq b\}}$, 
	$\mathbb{P}_{t,a}$ is the joint-distribution of \\$\{\chi_i\}_{\{i \leq t\} \cup \{i > t+a\}}$, and $\|\cdot\|_{TV}$ is the total variation norm for measures \citep[Definition 2.1,][]{mcdonald2011estimating}. Hence, we can write 
	\begin{align*}
	\beta_{(\eps,B)}(a)  &:= \sup_t \| \mathbb{P}_{-\infty}^t (\eps,B) \otimes \mathbb{P}_{t+a}^\infty (\eps,B) - \mathbb{P}_{t,a} (\eps,B)\|_{TV} \\
	&= \sup_t \| \mathbb{P}_{-\infty}^t (\eps) \otimes \mathbb{P}_{-\infty}^t (B) \otimes \mathbb{P}_{t+a}^\infty (\eps) \otimes \mathbb{P}_{t+a}^\infty (B) - \mathbb{P}_{t,a} (\eps) \otimes \mathbb{P}_{t,a} (B)\|_{TV}\\
	&\leq \sup_t \| \mathbb{P}_{-\infty}^t (\eps) \otimes \mathbb{P}_{t+a}^\infty (\eps) - \mathbb{P}_{t,a} (\eps)\|_{TV}\\ 
	&\quad\quad\quad\quad\quad\quad\quad\quad+ \sup_t \| \mathbb{P}_{-\infty}^t (B) \otimes \mathbb{P}_{t+a}^\infty (B) - \mathbb{P}_{t,a} (B)\|_{TV}
	\end{align*}
	Here the first equality follows from $\{B_i\}$ being independent of $\{\eps_i \}$, and the inequality 
	is from Lemma 1 of \cite{eberlein1984weak}. Hence, \begin{equation}\label{eq:betamix}
	\beta_{(\eps,B)}(a) \leq \beta_{(\eps)}(a) + \beta_{(B)}(a).
	\end{equation} This implies that as both  $\{B_i\}$ (i.i.d. process) and $\{\eps_i\}$  $\beta$-mixing, so is $\{\eps_i,B_i\}$. 
	Now, using Theorem 1 of \cite{nobel1993note} with the singleton class comprising of the function $(B_i , \eps_i) \mapsto B_i\eps_i$ ( as $B_i\eps_i \leqslant |\eps_i|$ which is uniformly integrable by Assumption \ref{as:mix}), we have the strong-law of part 1.
	
	For part 2, as $\{\eps_i \}_{i = 1}^n$ is stationary and absolutely-regular mixing, and $\mathbb E \eps_1^2 < \infty$ 
	(by Assumption \ref{as:mix}), once again using Theorem 1 of \cite{nobel1993note} on $\{\eps_i \}_{i = 1}^n$ now with the singleton class $\{f(x)=x^2\}$, we have the result.
\end{proof}

\begin{lemma}\label{lem:weaklaw2} Let $X_i \overset{i.i.d}{\sim} X, \{X_i\} \independent \{\eps_i\}$. For $\calC \in \sigma(X)$, let $B^{(\calC)}_i=\mathds{I}(X_i \in \mathcal{C}); i = 1,\cdots,n$. Let $\calP$ denote a collection of polynomial-in-$n$ number of such sets $\calC$. Then under Assumptions \ref{as:tail}(b) and \ref{as:tail}(c), for any $\pi, \phi > 0, \exists$ $n_2 \in \mathbb{N}^{*}$ such that with probability $1 - \pi$, $\forall n > n_2$,
	$$
	\cap_{\calC \in \calP} \Bigg\{\Big|\frac{1}{n}\sum_{i} \eps_iB^{(\calC)}_i\Big| \leqslant \phi \Bigg\} \cap
	\Big\{\Big|\frac{1}{n}\sum_{i} \eps^2_i\Big| \leqslant \sigs_0\Big\}.$$
\end{lemma}

\begin{proof} Let $U_\calC=\{\Big|\frac{1}{n}\sum_{i} \eps_iB^{(\calC)}_i\Big| \geqslant \phi \}$, $V_\calC = \{\sum_{i}B^{(\calC)}_i \geqslant \sqrt{n} \}$, $W=\{\max_i |\eps_i| \leqslant C_\pi \sqrt{\log n}\}$.
	Then, on $V_\calC^c \cap W$ for all $\calC \in \sigma(X)$, we have, for large enough $n$,
	$$
	\Big|\frac{1}{n}\sum_{i} \eps_iB^{(\calC)}_i\Big| \leqslant C_\pi \frac{1}{n}\sqrt{\log n} \sum_{i}B^{(\calC)}_i \leqslant C_\pi \frac{1}{\sqrt{n}}\sqrt{\log n} < \phi. 
	$$
	Hence, for large enough $n$, we have $\cup_\calC \{U_\calC \cap V_\calC^c \cap W\}=\{\}$. Using $\mathbb P(W) \geq 1 - \pi/4$ for large enough $n$ from Assumption \ref{as:tail}(b), we write 
	\begin{align*}
	\mathbb P\left(\cup_{\calC \in \calP}\, \Bigg\{\Big|\frac{1}{n}\sum_{i} \eps_iB^{(\calC)}_i\Big| \geqslant \phi \Bigg\}\right)  \leqslant &  \pi/4 +  \mathbb P\left(\cup_{\calC \in \calP}\, U_\calC \cap W \right)\\
	= & \pi/4 + \mathbb P\left(\cup_{\calC \in \calP}\, U_\calC \cap V_\calC \cap W \right) \\
	\leqslant& \pi/4 + \mathbb P\left(\cup_{\calC \in \calP}\, U_\calC \cap V_\calC \right).
	\end{align*}
	
	We can write $\Big|\frac1{n}\sum_{i} \eps_iB^{(\calC)}_i\Big| \leq \Big|\frac 1{\sum_{i} B^{(\calC)}_i}\sum_{i} \eps_iB^{(\calC)}_i\Big|$. 
	Let $b_i, i=1,\ldots,n$ denote a realization of $B_i^{(\calC)}$ for some $\calC$. Denoting $\calI_n= \{i \in \{1,\ldots,n\} \given b_i > 0\}$ and using independence of $\{X_i\}$'s and $\{\eps_i\}$'s, we can write 
	$$ \mathbb P(\Big|\frac{1}{\sum_{i}B^{(\calC)}_i}\sum_i \eps_iB^{(\calC)}_i\Big| > \phi) = \sum_{\{b_i\}} \mathbb P(\frac 1{|\calI_n|} \sum_{i \in \calI_n} \eps_i > \phi) \mathbb P(\cap_{i \in \calI_n} \{B_i=b_i\}) $$
	
	The sub-Gaussian tail of Assumption \ref{as:tail}(c) implies that $P(\frac 1{|\calI_n|} |\sum_{i \in \calI_n} \eps_i |> \phi) \leq C\exp(-c|\calI_n|)$ for all choices of the sub-sequence $\calI_n$. Hence, we have $$P(|\frac{1}{\sum_{i}B^{(\calC)}_i}\sum_{i} \eps_iB^{(\calC)}_i| > \phi) \leq C\exp(-c\sqrt{n}) \mbox{ on } \{\sum_{i}B^{(\calC)}_i \geqslant \sqrt{n}\}.$$ 
	This proves $\mathbb P (U_\calC \cap V_\calC) \leqslant C\exp(-c\sqrt{n})$. 
	As there are a polynomial in $n$ number of $\calC$'s, taking union bounds yields the result simultaneously for all $\calC$. The adjustment in probability due to the $\eps^2_i$ term follows directly from the second tail-bound assumption in \ref{as:tail}(c).
\end{proof}

\begin{lemma}\label{lem:diag} 
	For any matrix $\bQ$ and a binary matrix $\bZ$ with columns $\bZ_{\cdot l}$ and row-sums 1,  
	$\bZ_{\cdot l}^\top\bQ\bZ_{\cdot l} - \sum_{l' \neq l} |\bZ_{\cdot l}^\top\bQ\bZ_{\cdot l'}| \geqslant \sum_i \bZ_{il}(\bQ_{ii} - \sum_{j \neq i} |\bQ_{ij}|)$.
\end{lemma}

\begin{proof}Using the statement of the Lemma, $\bZ_{il}\bZ_{il'}=\bZ_{il}\mathds{I}(l=l')$. Hence, 
	\begin{equation*}
	\begin{aligned}
	\bZ_{\cdot l}^\top\bQ\bZ_{\cdot l} - \sum_{l' \neq l} |\bZ_{\cdot l}^\top\bQ\bZ_{\cdot l'}| &= \sum_i \bQ_{ii} \bZ_{il} + \sum_{j \neq i} \bQ_{ij} \bZ_{il}\bZ_{jl} - \sum_{l \neq l'} |\sum_{j \neq i} \bQ_{ij} \bZ_{il}\bZ_{jl'}|\\
	&\geqslant \sum_i \bQ_{ii} \bZ_{il} + \sum_{j \neq i} \bQ_{ij} \bZ_{il}\bZ_{jl} - \sum_{l \neq l'} \sum_{j \neq i} |\bQ_{ij}| \bZ_{il}\bZ_{jl'}\\
	&\geqslant \sum_i \bQ_{ii} \bZ_{il} + \sum_{j \neq i} \bQ_{ij} \bZ_{il}\bZ_{jl} - \sum_{j \neq i} |\bQ_{ij}| \bZ_{il}(1-\bZ_{jl})\\
	&\geqslant \sum_i \bQ_{ii} \bZ_{il} - \sum_{j \neq i} |\bQ_{ij}| \bZ_{il}\bZ_{jl} - \sum_{j \neq i} |\bQ_{ij}| \bZ_{il}(1-\bZ_{jl})\\
	&= \sum_i \bQ_{ii} \bZ_{il} - \sum_{j \neq i} |\bQ_{ij}| \bZ_{il}
	\end{aligned}
	\end{equation*}
	
\end{proof}

\end{document}